\documentclass[review]{elsarticle}
\pdfoutput=1
\usepackage{lineno,hyperref}
\usepackage{algorithm}
\usepackage{adjustbox}
\usepackage{graphicx}
\usepackage{amssymb}
\usepackage{amsmath}
\usepackage{amsthm}
\usepackage{amsfonts}
\usepackage{url}
\usepackage{hyperref}
\usepackage{multirow}
\usepackage{float}
\usepackage{subfig}
\usepackage{scalefnt}
\usepackage{lscape}
\usepackage{bm}
\usepackage{color}
\usepackage{rotating}
\usepackage {cleveref}
\usepackage{longtable}
\usepackage{algpseudocode}
\usepackage{tikz}
\usepackage{changepage}
\usepackage{geometry}
\usepackage{lipsum}
\usepackage{tabularx,ragged2e,booktabs,caption}
\usepackage{array}
\usepackage{paralist}
\usepackage{stackengine}
\usepackage{siunitx}
\usepackage[normalem]{ulem}
\journal{Applied Soft Computing}

\newtheorem{proposition}{Proposition}
\newcommand{\argmin}{\operatornamewithlimits{arg\;min}}

\newcolumntype{C}{>{\Centering\arraybackslash}X}
\def\BibTeX{{\rm B\kern-.05em{\sc i\kern-.025em b}\kern-.08em
		T\kern-.1667em\lower.7ex\hbox{E}\kern-.125emX}}

\newcommand\tikzmark[1]{%
  \tikz[remember picture,overlay]\node[inner sep=2pt] (#1) {};}

\algnewcommand\algorithmicinput{\textbf{Input:}}
\algnewcommand\INPUT{\item[\algorithmicinput]}

\algnewcommand\algorithmicoutpu{\textbf{Output:}}
\algnewcommand\OUTPUT{\item[\algorithmicoutpu]}

\begin{document}

\begin{frontmatter}

\title{Fuzzy clustering algorithms with distance metric learning and entropy regularization}

\author[CIn]{Sara I.R. Rodr\'iguez}
\ead{sirr@cin.ufpe.br}

\author[CIn]{Francisco de A.T. de Carvalho\corref{cor1}}
\ead{fatc@cin.ufpe.br}

\cortext[cor1]{Corresponding Author. tel.:+55-81-21268430; fax:+55-81-21268438}

\address[CIn]{Centro de Inform\'atica, Universidade Federal de
Pernambuco, Av. Jornalista Anibal Fernandes, s/n - Cidade Universit\'aria, CEP 50740-560, Recife (PE), Brazil}

\begin{abstract}
The clustering methods have been used in a variety of fields such as image processing, data mining, pattern recognition, and statistical analysis. Generally, the clustering algorithms consider all variables equally relevant or not correlated for the clustering task. Nevertheless, in real situations, some variables can be correlated or may be more or less relevant or even irrelevant for this task. This paper proposes partitioning fuzzy clustering algorithms based on Euclidean, City-block and Mahalanobis distances and entropy regularization. These methods are an iterative three steps algorithms which provide a fuzzy partition, a representative for each fuzzy cluster, and the relevance weight of the variables or their correlation by minimizing a suitable objective function. Several experiments on synthetic and real datasets, including its application to noisy image texture segmentation, demonstrate the usefulness of these adaptive clustering methods.
\end{abstract}


\end{frontmatter}

%\linenumbers

\section{Introduction} \label{sect:introduction}
Clustering refers to a procedure that groups similar objects  %together 
while separating dissimilar ones apart \cite{bezdek2013pattern,havens2012fuzzy,wu2012generalization,pimentel2013multivariate}. Clustering algorithms are an efficient tool for image processing, data mining, pattern recognition, and statistical analysis \cite{vapnik1998statistical,gan2007data,namburu2017soft,zhao2015multiobjective}. The most popular clustering algorithms provide hierarchical and partitioning structures. Hierarchical methods deliver an output represented by a hierarchical structure of groups known as a dendrogram, i.e., a nested sequence of partitions of the input data, whereas partitioning methods create a partition of the input data into a fixed number of clusters based on either distance or density criteria of a dataset, typically by optimizing an objective function. An advantage of partitioning methods is its ability to manipulate large datasets, since the construction of dendrogram by the hierarchical approach may be computationally impractical in some applications. 

Partitioning clustering methods were performed mainly in two different ways: hard and fuzzy. In hard clustering, the clusters are disjoint and non-overlapping. In this case, any pattern may belong to one and only one group. On the other hand, in fuzzy clustering, an object may belong to all clusters with a specific membership degree. The membership degree is essential for discovering intricate relations which may arise between a given data object and all clusters \cite{kaufman2009finding}. The fuzzy clustering methods are peculiarly effective when the boundaries between groups of data are ambiguous.

In the past few decades, various partitioning clustering algorithms have been proposed as K-Means \cite{macqueen1967some}, Fuzzy C-Means (FCM) \cite{bezdek2013pattern} and the FCM, which takes into account entropy regularization (FCM-ER) \cite{sadaaki1997}. One drawback of these clustering algorithms is that they treat each feature of data point as equally relevant and independent from others. This assumption is not always be satisfied in real applications, especially in high dimensional data clustering where some variables may not be tightly related to the topic of interest. In contrast, a distance metric with good quality should identify essential features and discriminate relevant and irrelevant features \cite{modha2003feature,Chan2004,DiGov77,Huang_05,Tsai2008,zhu2014evolving}. Recent studies have shown that learning the distance function from the data can improve the performance effectively. The advantage of this idea is that the clustering method can recognize groups of different shapes and sizes.

The Weighted Fuzzy C-Means (WFCM) \cite{wang2004improving} method proposed by Wang et al. applies a weighted Euclidean distance in FCM formulation to improve the performance of the FCM clustering algorithm. Later, Deng et al. \cite{deng2011eew} developed an Enhanced Entropy-Weighting Subspace Clustering algorithm (EEW-SC) for high dimensional gene expression data clustering analysis by integrating the within-cluster and between-class information simultaneously. 
%Later, Ye et al. present the Adaptive Metric Learning (AML) \cite{ye2007adaptive} method where unsupervised distance metric learning and clustering are applied simultaneously to obtain maximum separability among different clusters. 
Ref. \cite{hanmandlu2013color} shows a fuzzy co-clustering approach using a multi-dimensional distance function as the dissimilarity measure and entropy as the regularization term for image segmentation. Later, Rodr\'iguez and Carvalho \cite{rodriguez2017fuzzy} presented a fuzzy clustering algorithm based on Adaptive Euclidean distance and entropy regularization named AFCM-ER. Aiming to simplify the presentation and the discussion of the experimental results of Section \ref{sect:experimentRes}, hereafter we adopt the notation AFCM-ER-LP-L2.

Traditionally, the squared Euclidean %$\ell_2$-norm
distance is used to compare the objects and the prototypes in the Fuzzy C-Means algorithms, but theoretical studies indicate that methods based on City-Block distances are more robust concerning the presence of outliers in the dataset than those based on squared Euclidean %$\ell_2$-norm 
distance. For this purpose, Ref. \cite{sadaaki1997} proposed an entropy regularized FCM clustering algorithm where the squared Euclidean %$\ell_2$-norm 
and %the 
City-Block %$\ell_1$-norm 
are used as a dissimilarity measure. Despite the usefulness of these algorithms, %(hereafter named, respectively, FCM-ER-L2 and FCM-ER-L1), 
the variables have the same importance for the clustering task. To solve the problem, Rodr\'iguez and Carvalho \cite{rodriguez2018fuzzy} introduced the fuzzy clustering algorithm based on Adaptive City-block distance and entropy regularization.% (hereafter referred to as AFCM-ER-LP-L1). 

This paper extends Refs. \cite{sadaaki1997,hanmandlu2013color} proposing new partitioning fuzzy clustering methods based on suitable adaptive distances and entropy regularization. 

The main contributions are as follows:
\begin{itemize}
	\item Adaptive distances are proposed, taking into account the relevance of the variables in the clustering process, which allows recognizing clusters of different shapes and sizes. Besides, these adaptive distances change in each iteration of the algorithm and can be different from one cluster to another. It was used as dissimilarity measures, the Euclidean distance that is one of the most used in the literature, but also the Mahalanobis distance to consider the covariance of the data points. This last distance is defined by a positive definite symmetric matrix that can be different from one cluster to another. Additionally, methods based on City-Block distance were employed because they show better robustness in noisy data environment than those based on Euclidean and Mahalanobis distances.	
	\item Additionally, we use both kinds of adaptive distances, local (the set of relevant variables is different for each cluster) and global (the set of relevant variables is the same to all clusters), because in some situations local adaptive distances may not be appropriate, for example, when the internal dispersion of the clusters are almost the same. As in Ferreira and de Carvalho \cite{ferreira2014kernel}, the derivation of the expressions of the relevance weights of the variables was done considering two cases. In the first case, the sum of the weights of the variables (global constraint) or the sum of the weights of the variables on each cluster (local constraint) must be equal to one \cite{huang2005automated}. In the second case, the product of the weights of the variables (global constraint) or the product of the weights of the variables on each cluster (local constraint) must be equal to one \cite{de2006partitional,diday1977classification,gustafson1979fuzzy}. An advantage of the product constraint over the sum is that it requires the tuning of fewer parameters.
\end{itemize}

The paper is organized as follows. Section \ref{sect:fcaer} reviews two works closely related to the proposed approaches. Section \ref{sect:proposedmethods} presents the proposed fuzzy clustering algorithms based on Adaptive %quadratic 
Euclidean, Mahalanobis and City-block distances and entropy regularization. Experiments and results are reported in Section \ref{sect:experimentRes}. Finally, conclusions are drawn in Section \ref{sect:conclusion}.

\section{%Fuzzy clustering algorithms with entropy regularization}
Related work
}\label{sect:fcaer}
Several maximum entropy clustering algorithms %and their variants
are available in the literature which aims to search for global regularity and obtain the smoothest reconstructions from the available data. This section briefly describes two algorithms closely related to our approaches.

Let $E = \{e_1, \ldots, e_N\}$ be a set of $N$ objects. Each object $e_i \, (1 \leq i \leq N)$ is described by the vector $\mathbf{x}_i = (x_{i1},\ldots,x_{iP})$, with $x_{ij} \in \rm I\!R \, (1 \leq j \leq P)$. Let $\mathcal{D}= \{ \mathbf{x}_1,\ldots,\mathbf{x}_N \}$ be the dataset.

It is assumed that the fuzzy clustering algorithms considered in this paper provide:
\begin{itemize}
	\item A fuzzy partition represented by the matrix $\mathbf{U} = (\mathbf{u}_1,\ldots,\mathbf{u}_N) = (u_{ik})_{\substack{1 \leq i \leq N\\1 \leq k \leq C}}$, where $u_{ik}$ is the membership degree of object $e_i$ into the fuzzy cluster $k$ and $\mathbf{u}_i = (u_{i1},\ldots,u_{iC})$;
	\item A matrix $\mathbf{G} = (\mathbf{g}_1, \ldots, \mathbf{g}_C)= (g_{kj})_{\substack{1 \leq k \leq C\\1 \leq j \leq P}}$ where the component $\mathbf{g}_k = (g_{k1}, \ldots,  g_{kP})$ is the representative (prototype)  of fuzzy cluster $k$, where $g_{kj} \in \rm I\!R$.
\end{itemize}
\subsection{%The FCM-ER-L2 and FCM-ER-L1 algorithms
Fuzzy clustering algorithms based on Euclidean and City-Block distances and entropy regularization
}
The first approach \cite{sadaaki1997}, hereafter named FCM-ER, is a %the 
variant of the FCM algorithm, which takes into account entropy regularization. It involves the minimization of the following objective function:

\begin{equation} \label{eq:fcmer}
J_{FCM-ER}=\sum_{k=1}^{C} \sum_{i=1}^{N}(u_{ik}) d(\textbf{x}_{i}, \textbf{g}_{k}) + T_u \sum_{k=1}^{C} \sum_{i=1}^{N} (u_{ik}) \ln(u_{ik})
\end{equation}

Subject to: $u_{ik} \in [0,1]$ and $\sum_{k=1}^{C}(u_{ik})=1$

In Equation \ref{eq:fcmer}, $d$ is a dissimilarity function %that
which compares the object $e_i$ and the cluster prototype $\mathbf{g}_k$.
FCM-ER is named FCM-ER-L2 when $d$ is the squared Euclidean distance such that $d(\mathbf{x}_{i}, \mathbf{g}_{k}) = \sum_{j=1}^{P}(x_{ij} - g_{kj})^2$.
Additionally, FCM-ER is named FCM-ER-L1 when $d$ is the City-Block distance
such that $d(\mathbf{x}_{i}, \mathbf{g}_{k}) = \sum_{j=1}^{P}|x_{ij} - g_{kj}|$. 

The first term in the Equation (\ref{eq:fcmer}) denotes the total heterogeneity of the fuzzy partition as the sum of the heterogeneity of the fuzzy clusters; the second term is related to the entropy which serves as a regulating factor during minimization process. The parameter $T_u$ is the weight factor in the entropy term.

\subsection{%The AFCM-ER-LP-L2 and AFCM-ER-LP-L1 algorithms
Fuzzy clustering algorithms based on Adaptive Euclidean and City-Block distances and entropy regularization
}
Rodr\'iguez and de Carvalho \cite{rodriguez2017fuzzy,rodriguez2018fuzzy} introduced AFCM-ER-LP, a variant of FCM-ER with adaptive distances, where in addition to the matrix $\mathbf{U}$ of membership degrees and the matrix $\mathbf{G}$ of prototypes, it is also provided:
%The AFCM-ER-LP-L2 and AFCM-ER-LP-L1 algorithms are another work of interest. Despite the matrix $\mathbf{U}$ of membership degrees and the vector $\mathbf{G}$ of prototypes, it provides:
\begin{itemize}
	\item A matrix of relevance weights $\mathbf{V} = (\mathbf{v}_1, \ldots, \mathbf{v}_C) = (v_{kj})_{\substack{1 \leq k \leq C\\1 \leq j \leq P}}$ where $v_{kj}$ is the relevance weight of the $j$-th variable of the $k-th$ fuzzy cluster and $\mathbf{v}_k = (v_{k1}, \ldots, v_{kP})$.
\end{itemize}

The corresponding algorithm of this method is based the minimization of the following objective function:

\begin{eqnarray} \label{eq:oberpl}
J_{AFCM-ER-LP} &=& \sum_{k=1}^{C} \sum_{i=1}^{N}(u_{ik}) \, d_{\mathbf{v}_k}(\mathbf{x}_{i}, \mathbf{g}_{k}) + T_u \sum_{k=1}^{C} \sum_{i=1}^{N} (u_{ik}) \ln(u_{ik})  \\
&=& \sum_{k=1}^{C} \sum_{i=1}^{N}(u_{ik}) \,  \sum_{j=1}^{P} (v_{kj}) d(x_{ij},g_{kj}) + T_u \sum_{k=1}^{C} \sum_{i=1}^{N} (u_{ik}) \nonumber
\end{eqnarray}

Subject to: $u_{ik} \in [0,1]$, $v_{kj} > 0$, $\sum_{k=1}^{C}(u_{ik})=1$ and $\prod_{j=1}^{P}(v_{kj})=1$. 

AFCM-ER-LP is named AFCM-ER-LP-L2 when $d_{\mathbf{v}_k} \, (1 \leq k \leq C)$ is the squared local adaptive Euclidean distance such that $d_{\mathbf{v}_k}(\mathbf{x}_{i}, \mathbf{g}_{k}) = \sum_{j=1}^{P} v_{kj} d(x_{ij},g_{kj})$, with $d(x_{ij},g_{kj}) = (x_{ij} - g_{kj})^2$.
Additionally, AFCM-ER-LP is named AFCM-ER-LP-L1 when $d_{\mathbf{v}_k}$ is the local adaptive City-Block distance
such that $d_{\mathbf{v}_k}(\mathbf{x}_{i}, \mathbf{g}_{k}) = \sum_{j=1}^{P} v_{kj} d(x_{ij},g_{kj})$, with $d(x_{ij},g_{kj}) = |x_{ij} - g_{kj}|$. 
%For AFCM-ER-LP-L2 algorithm \cite{rodriguez2017fuzzy}, squared $L_2$-norm is used as dissimilarity such that $d(\mathbf{x}_{i},\mathbf{g}_{k}) = \sum_{j=1}^{P}(x_{ij} - g_{kj})^2$. For AFCM-ER-LP-L1 algorithm \cite{rodriguez2018fuzzy}, $L_1$-norm is used as dissimilarity for clustering and $d(\mathbf{x}_{i}, \mathbf{g}_{k}) = \sum_{j=1}^{P}|x_{ij} - g_{kj}|$.

The first %distance-based 
term defines the shape and size of the clusters and encourages agglomeration, while the second term is the negative entropy and is used to control the membership degree. $T_u$ is a weighting parameter that specifies the fuzziness degree; increasing $T_u$ increases the fuzziness of the clusters.

\section{The proposed clustering algorithms with automatic variable selection and entropy regularization}\label{sect:proposedmethods}

This section presents new partitioning fuzzy clustering algorithms based on feature-weight learning that measure the heterogeneity of the fuzzy partition as the sum of the heterogeneity in each fuzzy cluster, where the distance-based term defines the shape and size of the groups and encourages agglomeration. Additionally, it is employed an entropy term which serves as a regulating factor during the minimization process.

One of the proposed adaptive distance is defined by a local co-variance matrix introduced by Gustafson and Kessel \cite{gustafson1979fuzzy} that changes in each iteration of the algorithm and is different from one cluster to another. In this case, the algorithm is named AFCM-ER-Mk and it is based on the minimization of the following objective function:

\begin{eqnarray}\label{eq:functionGMk}
J_{AFCM-ER-Mk} &=& \sum_{k=1}^{C}\sum_{i=1}^{N} (u_{ik}) \, d_{\mathbf{M}_k}(\mathbf{x}_{i}, \mathbf{g}_{k}) + T_u\sum_{k=1}^{C}\sum_{i=1}^{N} (u_{ik}) \ln(u_{ik}) \\
&=& \sum_{k=1}^{C}\sum_{i=1}^{N} (u_{ik}) (\mathbf{x}_i-\mathbf{g}_k)^T \mathbf{M}_k(\mathbf{x}_i-\mathbf{g}_k) + T_u\sum_{k=1}^{C}\sum_{i=1}^{N} (u_{ik}) \ln(u_{ik}) \nonumber
\end{eqnarray}

Subject to: $u_{ik} \in [0,1]$, $\sum_{k=1}^{C}(u_{ik})=1$ and $det(\mathbf{M}_k)=1$.

If the adaptive distance is defined by a global co-variance matrix that changes in each iteration of the algorithm and is the same for all clusters, the algorithm is named AFCM-ER-M and it is based on the minimization of the following objective function:

\begin{eqnarray}\label{eq:functionGM}
J_{AFCM-ER-M} &=& \sum_{k=1}^{C}\sum_{i=1}^{N} (u_{ik}) \, d_{\mathbf{M}}(\mathbf{x}_{i}, \mathbf{g}_{k}) + T_u\sum_{k=1}^{C}\sum_{i=1}^{N} (u_{ik}) \ln(u_{ik}) \\
&=& \sum_{k=1}^{C}\sum_{i=1}^{N} (u_{ik}) (\mathbf{x}_i-\mathbf{g}_k)^T \mathbf{M}(\mathbf{x}_i-\mathbf{g}_k) + T_u\sum_{k=1}^{C}\sum_{i=1}^{N} (u_{ik}) \ln(u_{ik})
\end{eqnarray}

Subject to: $u_{ik} \in [0,1]$, $\sum_{k=1}^{C}(u_{ik})=1$ and $det(\mathbf{M})=1$.

For both cases, besides the matrix of membership degrees and the vector of prototypes, %it is also returned 
the global co-variance matrix $\mathbf{M}$ and the local co-variance matrix $\mathbf{M}_k$ %estimated locally for each cluster 
are also returned for AFCM-ER-M and AFCM-ER-Mk, respectively.

Another proposed adaptive distance takes into account the relevance of the variables for the clustering task. As an extension of Ref. \cite{hanmandlu2013color}, this set of relevant variables is the same to all clusters and the sum of the variables weights is equal to one ($v_{j}\geq 0$ and $\sum_{j=1}^{P}v_{j}=1$). This variant is named AFCM-ER-GS and the corresponding algorithms involve the minimization of:

\begin{eqnarray}\label{eq:functionsg}
J_{AFCM-ER-GS} &=& \sum_{k=1}^{C}\sum_{i=1}^{N} (u_{ik}) \, d_{\mathbf{v}}(\mathbf{x}_{i}, \mathbf{g}_{k}) + T_u \sum_{k=1}^{C} \sum_{i=1}^{N} (u_{ik}) \ln(u_{ik}) + T_v \sum_{j=1}^{P} (v_{j}) \ln(v_{j}) \\
&=& \sum_{k=1}^{C}\sum_{i=1}^{N} (u_{ik}) \sum_{j=1}^{P} (v_{j}) d(x_{ij},g_{kj}) + T_u \sum_{k=1}^{C} \sum_{i=1}^{N} (u_{ik}) \ln(u_{ik}) + T_v \sum_{j=1}^{P} (v_{j}) \ln(v_{j}) \nonumber
\end{eqnarray}

With $\mathbf{v}=(v_1,\ldots,v_P)$ and subject to: $u_{ik} \in [0,1]$, $v_{j} \in [0,1]$, $\sum_{k=1}^{C}(u_{ik})=1$ and $\sum_{j=1}^{P}(v_{j})=1$. $T_u$ and $T_v$ are weighting parameters that specify the fuzziness degree. Increasing the values of $T_u$ and $T_v$ increases the fuzziness of the clusters.
%$T_u$ is a weighting parameter that specifies the fuzziness degree.

AFCM-ER-GS is named AFCM-ER-GS-L2 when $d_{\mathbf{v}}$ is the squared global adaptive Euclidean distance such that $d_{\mathbf{v}}(\mathbf{x}_{i}, \mathbf{g}_{k}) = \sum_{j=1}^{P} v_{j} d(x_{ij},g_{kj})$, with $d(x_{ij},g_{kj}) = (x_{ij} - g_{kj})^2$. In this case, the objective function becomes:

\begin{align}\label{eq:functionsgl2}
J_{AFCM-ER-GS-L2}=\sum_{k=1}^{C}\sum_{i=1}^{N} (u_{ik}) \sum_{j=1}^{P} (v_{j}) (x_{ij}-g_{kj})^2 + T_u \sum_{k=1}^{C} \sum_{i=1}^{N} (u_{ik}) \ln(u_{ik}) + T_v \sum_{j=1}^{P} (v_{j}) \ln(v_{j})
\end{align}

Furthermore, AFCM-ER-GS is named AFCM-ER-GS-L1 when $d_{\mathbf{v}}$ is the global adaptive City-Block distance such that $d_{\mathbf{v}}(\mathbf{x}_{i}, \mathbf{g}_{k}) = \sum_{j=1}^{P} v_{j} d(x_{ij},g_{kj})$, with $d(x_{ij},g_{kj}) = |x_{ij} - g_{kj}|$. In this case, the objective function becomes:

\begin{align}\label{eq:functionsgl1}
J_{AFCM-ER-GS-L1}=\sum_{k=1}^{C}\sum_{i=1}^{N} (u_{ik}) \sum_{j=1}^{P} (v_{j}) |x_{ij}-g_{kj}| + T_u \sum_{k=1}^{C} \sum_{i=1}^{N} (u_{ik}) \ln(u_{ik}) + T_v \sum_{j=1}^{P} (v_{j}) \ln(v_{j})
\end{align}

Another alternative is that the product of the weights of the variables is equal to one. This dissimilarity function is parameterized by the vector of relevance weights $\mathbf{v}=(v_{1},...,v_{P})$, in which $v_{j}>0$ and $\prod_{j=1}^{P}v_{j}=1$. The advantage compared to the sum is that it requires the setting  %tuning 
of one less parameter. This aproach is named AFCM-ER-GP and its objective function is defined as:

\begin{eqnarray}\label{eq:functionpg}
J_{AFCM-ER-GP} &=& \sum_{k=1}^{C}\sum_{i=1}^{N} (u_{ik}) \, d_{\mathbf{v}}(\mathbf{x}_{i}, \mathbf{g}_{k}) + T_u \sum_{k=1}^{C} \sum_{i=1}^{N} (u_{ik}) \ln(u_{ik}) \\
&=& \sum_{k=1}^{C}\sum_{i=1}^{N} (u_{ik}) \sum_{j=1}^{P} (v_{j}) d(x_{ij},g_{kj}) + T_u \sum_{k=1}^{C} \sum_{i=1}^{N} (u_{ik}) \ln(u_{ik}) \nonumber
\end{eqnarray}

Subject to: $u_{ik} \in [0,1]$, $v_{j} > 0$, $\sum_{k=1}^{C}(u_{ik})=1$ and $\prod_{j=1}^{P}(v_{j})=1$. $T_u$ is a weighting parameter that specifies the fuzziness degree.

AFCM-ER-GP is named AFCM-ER-GP-L2 when $d_{\mathbf{v}}$ is the squared global adaptive Euclidean distance. In this case, the objective function becomes:

\begin{align}\label{eq:functionpgl2}
J_{AFCM-ER-GP-L2}=\sum_{k=1}^{C}\sum_{i=1}^{N} (u_{ik}) \sum_{j=1}^{P} (v_{j}) (x_{ij}-g_{kj})^2
+T_u\sum_{k=1}^{C}\sum_{i=1}^{N}(u_{ik}) \ln(u_{ik})
\end{align}

Furthermore, AFCM-ER-GP is named AFCM-ER-GP-L1 when $d_{\mathbf{v}}$ is the global adaptive City-Block distance. In this case, the objective function becomes:

\begin{align}\label{eq:functionpgl1}
J_{AFCM-ER-GP-L1}=\sum_{k=1}^{C}\sum_{i=1}^{N} (u_{ik}) \sum_{j=1}^{P} (v_{j}) |x_{ij}-g_{kj}|
+T_u\sum_{k=1}^{C}\sum_{i=1}^{N}u_{ik}\ln(u_{ik})
\end{align}

The algorithms AFCM-ER-GS-L2, AFCM-ER-GS-L1, AFCM-ER-GP-L2 and AFCM-ER-GP-L1 return the matrix of membership degrees, the vector of prototype for each fuzzy cluster and the vector of relevance weights $\mathbf{v}=(v_{1},...,v_{P})$, where $v_j$ is the relevance weight of the $j$-th variable estimated globally.

Additionally, a %suitable 
variable-wise dissimilarity with relevance weight for the variables estimated locally is also considered, where the sum of the wights is equal to one \cite{huang2005automated} and the City-Block distance compares the objects and the prototypes. The dissimilarity function is parameterized by the vector of relevance weights $\mathbf{v}_k=(v_{k1},...,v_{kP})$, in which $v_{kj}\geq 0$ and $\sum_{j=1}^{P}v_{kj}=1$, and it is associated with the $k$-th fuzzy cluster $(k=1,...,C)$. This approach is named AFCM-ER-LS-L1 and its objective function is defined as:

\begin{align}\label{eq:functionsl}
J_{AFCM-ER-LS-L1}=\sum_{k=1}^{C}\sum_{i=1}^{N} (u_{ik}) \sum_{j=1}^{P} (v_{kj}) |x_{ij}-g_{kj}| + T_u \sum_{k=1}^{C} \sum_{i=1}^{N} (u_{ik}) \ln(u_{ik}) + T_v \sum_{j=1}^{P} (v_{kj}) \ln(v_{kj})
\end{align}

Subject to: $u_{ik} \in [0,1]$, $v_{kj} \in [0,1]$, $\sum_{k=1}^{C}(u_{ik})=1$ and $\sum_{j=1}^{P}(v_{j})=1$. $T_u$ and $T_v$ are weighting parameters that control the degree of fuzziness of the clusters. %$T_u$ is a weighting parameter that specifies the fuzziness degree.

The AFCM-ER-LS-L1 algorithm, besides the matrix of membership degrees and the vector of prototypes, returns the matrix of relevance weights $\mathbf{V} = (\mathbf{v}_1, \ldots, \mathbf{v}_C) = (v_{kj})_{\substack{1 \leq k \leq C\\1 \leq j \leq P}}$ where $v_{kj}$ is the relevance weight of the $j$-th variable in the fuzzy cluster $k$ and $\mathbf{v}_k = (v_{k1}, \ldots, v_{kP})$.

\subsection{Optimization steps% of the algorithms
}
This section provides the optimization steps of the algorithms aiming to compute the prototypes, the fuzzy partition, and the covariance matrix, or the relevance weights of the variables. %For the algorithms, 
The minimization of the objective functions is performed iteratively in three steps (representation, weighting, and assignment).

\subsubsection{Representation step}\label{sect:prototype}
This step %section 
provides the solution for the optimal computation of the prototype vector associated with each fuzzy cluster. During the representation step, the matrix of membership degree $\mathbf{U}$, the global matrix $\mathbf{M}$ for AFCM-ER-M or the local %the 
matrices $\mathbf{M}_k$ %for each cluster 
for AFCM-ER-Mk and the relevance weights of the variables for the other approaches are kept fixed. Then, the adequacy criterion for the algorithms is minimized concerning to the %the matrix $\mathbf{G}$ of 
prototypes. 

It is observed that %the choice of 
the dissimilarity function plays an essential role in the computation of the %components of the 
prototypes. This paper provides an exact solution for each of the three possible choices of the dissimilarity functions.%: the Mahalanobis, Euclidean, and City-Block distances.

\textit{Case 1}: %Dissimilarity functions based on the Mahalanobis distance: 
If the dissimilarity function between the objects and the prototypes is the Mahalanobis distance $(\mathbf{x}_i-\mathbf{g}_k)^T \mathbf{M}(\mathbf{x}_i-\mathbf{g}_k)$, then taking the partial derivative of $J_{AFCM-ER-M}$ (see Eq. (\ref{eq:functionGM})) concerning $\mathbf{g}_{k}$ we have:

\begin{equation}\label{eq:dercentroidm}
  \frac{\partial J_{AFCM-ER-M}}{\partial \mathbf{g}_{k}}=-2\sum_{i=1}^{N}\mathbf{M}(\mathbf{x}_i-\mathbf{g}_k)=0
\end{equation}

\begin{flalign}\label{eq:protM}
\text{Solving Equation \ref{eq:dercentroidm}, $\mathbf{g}_k$ is defined as:}\quad\quad
    \mathbf{g}_{k}=\frac{\sum_{i=1}^{N}u_{ik} \mathbf{x}_{i}}{\sum_{i=1}^{N}u_{ik}}&&
  \end{flalign}  
  
Similarly, if the dissimilarity function is $(\mathbf{x}_i-\mathbf{g}_k)^T \mathbf{M}_k(\mathbf{x}_i-\mathbf{g}_k)$ (Eq. \ref{eq:functionGMk}), $\mathbf{g}_k$ is computed as in Eq. (\ref{eq:protM}).

\textit{Case 2}: If the dissimilarity function between the objects and the prototypes is the squared global adaptive Euclidean distance, then taking the partial derivative of $J_{AFCM-ER-GS-L2}$ (see Eq. (\ref{eq:functionsgl2})) concerning $g_{kj}$ we have:
%Dissimilarity functions based on the squared Euclidean distance: The AFCM-ER-GS-L2 algorithm uses the squared Euclidean distance ($d(x_{ij},g_{kj})=(x_{ij}-g_{kj})^2$) to compare objects and prototypes. Taking the partial derivative of $J_{AFCM-ER-GS-L2}$ concerning $g_{kj}$ we have:

\begin{equation}\label{eq:dercentroidE}
  \frac{\partial J_{AFCM-ER-GS-L2}}{\partial g_{kj}}= -2 \sum_{i=1}^{N} (u_{ik}) (x_{ij}-g_{kj}) =0
\end{equation}

  \begin{flalign}\label{eq:prot}
    \text{Solving Equation \ref{eq:dercentroidE} yields:}\quad g_{kj}=\frac{\sum_{i=1}^{N}u_{ik} x_{ij}}{\sum_{i=1}^{N}u_{ik}}&&
  \end{flalign}
  
Following a similar reasoning, the prototype $g_{kj}$ of the $k$-th cluster which minimizes the clustering criterion $J_{AFCM-ER-GP-L2}$ (see Eq. (\ref{eq:functionpgl2})) is computed as in Eq. (\ref{eq:prot}).
  
\textit{Case 3}: %Dissimilarity functions based on the City-Block distance: 
If the dissimilarity function %between the objects and the prototypes 
is the global adaptive City-Block distance, then the minimization problem of Equations (\ref{eq:functionsgl1}), (\ref{eq:functionpgl1}) and (\ref{eq:functionsl}) with respect to $g_{kj}$ leads to the minimization of $\sum_{i=1}^{N}|y_i-az_i|$, where $y_i= (u_{ik}) x_{ij}$, $z_i=(u_{ik})$ and $a=g_{kj}$. There is no algebraic solution for this problem, but an algorithmic solution is known and to solve it the following algorithm \cite{jajuga1991l1} can be used:

\begin{enumerate}
\item Rank $(y_i,z_i)$ such that $\frac{y_{i1}}{z_{i1}} \leq ... \leq \frac{y_{iN}}{z_{iN}}$;
\item To $-\sum_{l=1}^{N}|z_{il}|$ add successive values of $2|z_{il}|$ and find r such that $-\sum_{l=1}^{N}|z_{il}|+2\sum_{s=1}^{r}|z_{is}|<0$ and $-\sum_{l=1}^{N}|z_{il}|+2\sum_{s=1}^{r+1}|z_{is}|>0$;
\item Then $a=\frac{y_{ir}}{z_{ir}}$;
\item If $-\sum_{l=1}^{N}|z_{il}|+2\sum_{s=1}^{r}|z_{is}|=0$ and $-\sum_{l=1}^{N}|z_{il}|+2\sum_{s=1}^{r+1}|z_{is}|=0$,then $a=\frac{\frac{y_{ir}}{z_{ir}}+\frac{y_{k(r+1)}}{z_{i(r+1)}}}{2}$.
\end{enumerate}

Alternatively, the %problem 
minimization of Equations (\ref{eq:functionsgl1}), (\ref{eq:functionpgl1}) and (\ref{eq:functionsl}) with respect to $g_{kj}$ also can be solved expressing \cite{jajuga1991l1}:

\begin{equation}\label{eq:met2gkj}
g_{kj}=\frac{\sum_{i=1}^{N} w_{ik} \, x_{ij}}{\sum_{i=1}^{N}w_{ik}}
   \quad\text{where}\quad 
w_{ik}=\frac{u^{(t-1)}_{ik}}{|x_{ij}-g_{kj}^{(t-1)}|}
\end{equation}

\subsubsection{Weighting step}\label{prop:calcv}
This step %section 
provides an optimal solution to compute %for the computation of 
the covariance matrix for the AFCM-ER-M and the AFCM-ER-Mk algorithms, or the relevance weight of the variables for the other approaches, globally for all clusters or locally for each cluster. 
During the weighting step, the prototypes vector $\mathbf{G}$, and the matrix of membership degrees $\mathbf{U}$ are kept fixed.

\begin{proposition}\label{prop:calv}
The covariance matrix or the weights of the variables, which minimize the proposed objective functions are calculated according to the adaptive distance function used:

(a) If the distance function is the local adaptive Mahalanobis distance $d_{\mathbf{M}_k}(\mathbf{x}_i,\mathbf{g}_k) = (\mathbf{x}_i-\mathbf{g}_k)^T \mathbf{M}_k(\mathbf{x}_i-\mathbf{g}_k)$, the positive definite symmetric matrices $\mathbf{M}_k$ which minimizes the criterion $J_{AFCM-ER-Mk}$ (Eq. (\ref{eq:functionGMk})) under $\det(M_k)=1$ is updated according to the following expression:

\begin{equation}\label{eq:covMk}
M_k=[\det(C_k)]^\frac{1}{P} C_k^{-1}\text{ with}\quad C_k=\sum^{N}_{i=1} (u_{ik}) (\mathbf{x}_i-\mathbf{g}_k)(\mathbf{x}_i-\mathbf{g}_k)^T
\end{equation}

(b) If the distance function is the global adaptive Mahalanobis distance $d_{\mathbf{M}}(\mathbf{x}_i,\mathbf{g}_k) = (\mathbf{x}_i-\mathbf{g}_k)^T \mathbf{M}(\mathbf{x}_i-\mathbf{g}_k)$, the positive definite symmetric matrix $\mathbf{M}$ which minimizes the criterion $J_{AFCM-ER-M}$ (Eq. (\ref{eq:functionGM})) under $\det(M)=1$ is updated according to:% the following expression:

\begin{equation}\label{eq:covM}
M=[\det(Q)]^\frac{1}{P} Q^{-1}\text{,}\quad Q=\sum^{C}_{k=1}C_k\text{ and}\quad C_k=\sum^{N}_{i=1} (u_{ik}) (\mathbf{x}_i-\mathbf{g}_k)(\mathbf{x}_i-\mathbf{g}_k)^T
\end{equation}

(c) If the adaptive distance function is given by $d_{\mathbf{v}}(\mathbf{x}_i,\mathbf{g}_k) = \sum_{j=1}^{P}v_{j}d(x_{ij},g_{kj})$, the vector of weights $\boldsymbol{v} = (v_{1}, \ldots, v_{P})$, which minimizes the criterion $J_{AFCM-ER-GS}$ (Eq. (\ref{eq:functionsg})) under $v_j$ $\epsilon$ $[0,1]$ $\forall$ $j$, and $\sum_{j=1}^{P}v_j=1$, has its components $v_j (j=1,...,P)$ computed according to the following expression:

\begin{equation}\label{eq:weightsg}
  v_j=\frac{exp\{-\frac{\sum_{k=1}^{C}\sum_{i=1}^{N} (u_{ik}) d(x_{ij},g_{kj})}{T_v}\}}{\sum_{w=1}^{P}exp\{-\frac{\sum_{k=1}^{C}\sum_{i=1}^{N} (u_{ik}) d(x_{iw},g_{kw})}{T_v}\}}
\end{equation}

If $d$ is the squared global adaptive Euclidean distance then:

\begin{equation}\label{eq:weightsgl2}
v_j=\frac{exp\{-\frac{\sum_{k=1}^{C}\sum_{i=1}^{N} (u_{ik}) (x_{ij}-g_{kj})^2}{T_v}\}}{\sum_{w=1}^{P}exp\{-\frac{\sum_{k=1}^{C}\sum_{i=1}^{N} (u_{ik}) (x_{iw}-g_{kw})^2}{T_v}\}}
\end{equation}

Finally, if $d$ is the City-Block distance then:

\begin{equation}\label{eq:weightsgl1}
v_j=\frac{exp\{-\frac{\sum_{k=1}^{C}\sum_{i=1}^{N} (u_{ik}) |x_{ij}-g_{kj}|}{T_v}\}}{\sum_{w=1}^{P}exp\{-\frac{\sum_{k=1}^{C}\sum_{i=1}^{N} (u_{ik}) |x_{iw}-g_{kw}|}{T_v}\}}
\end{equation}

(d) If the adaptive distance function is given by $d_{\mathbf{v}}(\mathbf{x}_i,\mathbf{g}_k) = \sum_{j=1}^{P}v_{j}d(x_{ij},g_{kj})$, the vector of weights $\boldsymbol{v} = (v_{1}, \ldots, v_{P})$ which minimizes the criterion $J_{AFCM-ER-GP}$ (Eq. (\ref{eq:functionpg})) under $v_j>0$ $\forall$ $j$ and $\prod_{j=1}^{P}v_{j}=1$, has its components $v_j(j=1,...,P)$ computed according to the expression:

\begin{equation}\label{eq:weightpg}
v_{j}=\frac{\left\{\prod_{w=1}^{P}\sum_{k=1}^{C}\sum_{i=1}^{N} (u_{ik}) d(x_{iw},g_{kw})\right\}^\frac{1}{P}}{\sum_{k=1}^{C}\sum_{i=1}^{N} (u_{ik}) d(x_{ij},g_{kj})}
\end{equation}

If $d$ is the squared global adaptive Euclidean distance then:

\begin{equation}\label{eq:weightpgl2}
v_{j}=\frac{\left\{\prod_{w=1}^{P}\sum_{k=1}^{C}\sum_{i=1}^{N}u_{ik}(x_{iw}-g_{kw})^2\right\}^\frac{1}{P}}{\sum_{k=1}^{C}\sum_{i=1}^{N}u_{ik}(x_{ij}-g_{kj})^2}
\end{equation}

Finally, if $d$ is the global adaptive City-Block distance then:

\begin{equation}\label{eq:weightpgl1}
v_{j}=\frac{\left\{\prod_{w=1}^{P}\sum_{k=1}^{C}\sum_{i=1}^{N} (u_{ik}) |x_{iw}-g_{kw}|\right\}^\frac{1}{P}}{\sum_{k=1}^{C}\sum_{i=1}^{N} (u_{ik}) |x_{ij}-g_{kj}|}
\end{equation}

(e) If the adaptive distance function is given by $\sum_{j=1}^{P}v_{kj}|x_{ij}-g_{kj}|$ the vector of weights $\boldsymbol{v}_k = (v_{k1}, \ldots, v_{kP})$ which minimizes the criterion $J_{AFCM-ER-LS-L1}$ (Eq. (\ref{eq:functionsl})) under $v_{kj} \in [0,1]$ $\forall$ $k,j$ and $\sum_{j=1}^{P}v_{kj}=1$ $\forall$ $k$, has its components $v_{kj}(k=1,...C, j=1,...,P)$ computed as follows:

\begin{equation}\label{eq:weightsl}
  v_{kj}=\frac{exp\{-\frac{\sum_{i=1}^{N} (u_{ik}) |x_{ij}-g_{kj}|}{T_v}\}}{\sum_{w=1}^{P}exp\{-\frac{\sum_{i=1}^{N} (u_{ik}) |x_{iw}-g_{kw}|}{T_v}\}}
\end{equation}

\begin{proof}
 The proof is given in \ref{proof-prop-calcv}.
\end{proof}
\end{proposition}

\subsubsection{Assignment step}
This step provides the solution to compute the matrix $\mathbf{U}$ of membership degree. In the assignment step, %In this step, 
the cluster centroids $\mathbf{G}$ and the matrix $\mathbf{M}$ for AFCM-ER-M, $\mathbf{M}_k$ for AFCM-ER-Mk or the relevance weights of the variables for the other approaches are kept fixed.

\begin{proposition}\label{prop:calu}
The fuzzy partition represented by $\mathbf{U}=(\mathbf{u}_1,\dots,\mathbf{u}_N)$, where $\mathbf{u}_i=(u_{ik},\dots,u_{iC})$ which minimizes the clustering criterion is such that the membership degree $u_{ik}(i=1,\dots,N;k=1,\dots,C)$ of object $e_i$ in the $k$-th fuzzy cluster, under $u_{ik}\in [0,1]$ and $\sum_{k=1}^{C}u_{ik}=1$, is update according to the following assignment rule:

\begin{equation}\label{eq:membership1}
u_{ik}=\frac{\exp\left\{-\frac{\Delta(\mathbf{x}_i,\mathbf{g}_k)}{T_u}\right\}}{\sum_{w=1}^{C}\exp\left\{-\frac{\Delta(\mathbf{x}_i,\mathbf{g}_w)}{T_u}\right\}}
\end{equation}

\noindent where $\Delta$ is the %is the used adaptive 
distance function which compares the $i$-th object and the prototype of the fuzzy cluster $k$. Table \ref{tab:uikrules} specifies %presents 
the assignment rules to obtain the fuzzy partition according to the different adaptive distance functions.

\begin{table}[!htb]
 \caption{Assignment rules for the fuzzy partition according to the distance functions}\label{tab:uikrules}
    \centering
    \begin{tabular} {|c|c|c|}
    \hline
    Distance function $\Delta$ &Algorithms&Rules for $u_{ik}$\\
    \hline
    \multirow{2}{*}{$(\mathbf{x}_i-\mathbf{g}_k)^T \mathbf{M}(\mathbf{x}_i-\mathbf{g}_k)$} &\multirow{2}{*}{ AFCM-ER-M }& \multirow{2}{*}{$\frac{\exp\left\{-\frac{(\mathbf{x}_i-\mathbf{g}_k)^T \mathbf{M}(\mathbf{x}_i-\mathbf{g}_k)}{T_u}\right\}}{\sum_{w=1}^{C}\exp\left\{-\frac{(\mathbf{x}_i-\mathbf{g}_w)^T \mathbf{M}(\mathbf{x}_i-\mathbf{g}_w)}{T_u}\right\}}$}\\
    &&\\
        \hline
    \multirow{2}{*}{$(\mathbf{x}_i-\mathbf{g}_k)^T \mathbf{M}_k(\mathbf{x}_i-\mathbf{g}_k)$} & \multirow{2}{*}{AFCM-ER-Mk} & \multirow{2}{*}{$\frac{\exp\left\{-\frac{(\mathbf{x}_i-\mathbf{g}_k)^T \mathbf{M}_k(\mathbf{x}_i-\mathbf{g}_k)}{T_u}\right\}}{\sum_{w=1}^{C}\exp\left\{-\frac{(\mathbf{x}_i-\mathbf{g}_w)^T \mathbf{M}_w(\mathbf{x}_i-\mathbf{g}_w)}{T_u}\right\}}$}\\
    &&\\
    \hline
    \multirow{2}{*}{$\sum_{j=1}^P v_j (x_{ij}-g_{kj})^2$}& AFCM-ER-GS-L2 & \multirow{2}{*}{$\frac{\exp\left\{-\frac{\sum_{j=1}^{P}v_j(x_{ij}-g_{kj})^2}{T_u}\right\}}{\sum_{w=1}^{C}\exp\left\{-\frac{\sum_{j=1}^{P}v_j(x_{ij}-g_{wj})^2}{T_u}\right\}}$}\\
    & AFCM-ER-GP-L2 & \\
    \hline
     \multirow{2}{*}{$\sum_{j=1}^P v_{j} |x_{ij}-g_{kj}|$} & AFCM-ER-GS-L1 & \multirow{2}{*}{$\frac{\exp\left\{-\frac{\sum_{j=1}^{P}v_{j}|x_{ij}-g_{kj}|}{T_u}\right\}}{\sum_{w=1}^{C}\exp\left\{-\frac{\sum_{j=1}^{P}v_{j}|x_{ij}-g_{wj}|}{T_u}\right\}}$}\\
     & AFCM-ER-GP-L1 &\\
     \hline
     \multirow{2}{*}{$\sum_{j=1}^P v_{kj} |x_{ij}-g_{kj}|$} & \multirow{2}{*}{AFCM-ER-LS-L1} & \multirow{2}{*}{$\frac{\exp\left\{-\frac{\sum_{j=1}^{P}v_{kj}|x_{ij}-g_{kj}|}{T_u}\right\}}{\sum_{w=1}^{C}\exp\left\{-\frac{\sum_{j=1}^{P}v_{wj}|x_{ij}-g_{wj}|}{T_u}\right\}}$}\\
     &  &\\
    \hline
    \end{tabular}
\end{table}

\end{proposition}

\begin{proof}
  The proof is given in \ref{proof-prop-calcu}.
\end{proof}

\subsection{The proposed clustering algorithms}	
The clustering algorithms are summarized in Algorithm \ref{alg:alg1}
\begin{algorithm}[!htbp]
\caption{Proposed algorithms}
\label{alg:alg1}
\begin{algorithmic}[1]
\vspace{-4mm}
\Statex
\INPUT\tikzmark{c}
The dataset $\mathcal{D}$; The number $C$ of clusters; The parameter $T_u > 0$ and $T_v > 0$; The maximum number of iterations %parameter 
$T$; %(maximum number of iterations); 
The threshold $\varepsilon>0$ and $\varepsilon<<1$.
\vspace{-6mm}
\Statex
\OUTPUT\tikzmark{c}
The vector of prototypes $\mathbf{G}$; The matrix of membership degrees $\mathbf{U}$; The matrix $\mathbf{M}$, the matrix $\mathbf{M}_k$ or the relevance weights globally for all clusters or locally for each group.
\State Initialization
\begin{quote}
   Set $t = 0$; Randomly initialize the matrix of membership degrees $\mathbf{U}=(u_{ik})_{\substack{1 \leq i \leq N\\1 \leq k \leq C}}$  such that $u_{ik} \geq 0$ and  $\sum_{k=1}^{C} u_{ik}^{(t)}=1$;
\end{quote}
\vspace{-4mm}
\Repeat
\begin{quote}
Set $t = t + 1$
\end{quote}
\vspace{-4mm}
\State {\bf{Step 1: representation}}.
\begin{quote}
For $k=1,\ldots,C;j=1,\ldots,P$, compute the component $g_{kj}$ of the prototype $\mathbf{g}_{k}=(g_{k1},...,g_{kP})$ according to the dissimilarity function considered: i) for the adaptive Mahalanobis distances $g_{kj}$ is computed according to Eq. (\ref{eq:protM}); ii) for the adaptive Euclidean distances,  $g_{kj}$ is computed according to Eq. (\ref{eq:prot});  iii) for the adaptive City-Block distances, $g_{kj}$ is computed according the algorithm described in Section \ref{sect:prototype} or using Eq. (\ref{eq:met2gkj}).
%if the dissimilarity function is based on the City-block distance. And according to Equations \ref{eq:protM} and \ref{eq:prot}, if Euclidean or Mahalanobis distances.
\end{quote}
\State {\bf{Step 2: weighting}}.
\begin{quote}
To obtain the matrices for all or each cluster
\begin{itemize}
\item Compute $\mathbf{M}$ and $\mathbf{M_k}$ according to Equations (\ref{eq:covM}) and (\ref{eq:covMk}), respectively.
\end{itemize}
To obtain the matrix of relevance weights $(k=1,...,C; j=1,...,P)$
\begin{itemize}
	\item Compute the component $v_{kj}$ of the vector of relevance	weights $\mathbf{v}_{k}=(v_{k1},...,v_{kP})$ according to Eq. (\ref{eq:weightsl}) if the objective function is given by Eq. (\ref{eq:functionsl}).
	\item Compute the component $v_{j}$ of the vector of relevance weights $\mathbf{v}=(v_1,...,v_P)$ according to Eq. \ref{eq:weightsg} or Eq. \ref{eq:weightpg} if the objective function is given by Eq. (\ref{eq:functionsg}) or Eq. (\ref{eq:functionpg}) respectively.
\end{itemize}
\end{quote}
\State {\bf{Step 3: assignment}}.
\begin{quote}
Compute the elements $u_{ij}$ of the matrix of membership degrees $\mathbf{U}=(u_{ij})_{\substack{1 \leq i \leq N\\1 \leq j \leq C}}$ according to Equation (\ref{eq:membership1}).
\end{quote}
\Until $max(|u_{ij}^{(t)}-u_{ij}^{(t-1)}|) < \varepsilon$ or $t \geq T$
\end{algorithmic}
\end{algorithm}

\subsubsection{Convergence of proposed algorithms}

The AFCM-ER-M and AFCM-ER-Mk algorithms provide a global co-variance matrix $\mathbf{m}^{*}$ such that $det(\mathbf{m}^{*})=1$, and a local co-variance matrix $\mathbf{M}^{*}$ estimated locally such that $det(\mathbf{M}^{*})=1$ for each cluster respectively, a fuzzy partition $\mathbf{U}^*=(\mathbf{u}_1^*,\dots,\mathbf{u}_N^*)$ and a vector of prototypes $\mathbf{G}^*=(\mathbf{g}_1^*,\dots,\mathbf{g}_C^*)$ such that:
\begin{itemize}
    \item $J_{AFCM-ER-M}(\mathbf{G}^*,\mathbf{m}^*,\mathbf{U}^*)=\min \{J_{AFCM-ER-M}(\mathbf{G},\mathbf{m},\mathbf{U}),\mathbf{G} \in \mathbb{L}^C, \mathbf{m} \in \mathbb{M}, \mathbf{U}\in\mathbb{U}^N\}$
     \item $J_{AFCM-ER-Mk}(\mathbf{G}^*,\mathbf{M}^*,\mathbf{U}^*)=\min \{J_{AFCM-ER-Mk}(\mathbf{G},\mathbf{M},\mathbf{U}),\mathbf{G} \in \mathbb{L}^C, \mathbf{M} \in \mathbb{M}^C, \mathbf{U}\in\mathbb{U}^N\}$
\end{itemize}

where
\begin{itemize}
    \item [$-$] $\mathbb{L}$ is the representation space of the prototypes such that $\mathbf{g}_k \in \mathbb{L} \, (k=1,\dots, C)$ and $\mathbf{G} \in \mathbb{L}^C=\mathbb{L}\times\dots\times\mathbb{L}$. In this paper $\mathbb{L}=\mathbb{R}^P$.
    \item [$-$] $\mathbb{U}$ is the space of fuzzy partition membership degrees such that $\mathbf{u}_i \in \mathbb{U} \, (i=1,\dots,N)$. In this paper $\mathbb{U}=\{\mathbf{u}=(\mathbf{u}_1,\dots,\mathbf{u}_C)\in [0,1]\times\dots\times [0,1]=[0,1]^C:\sum_{k=1}^{C}u_{ik}=1 \mbox{ and } u_{ik} \geq 0\}$ and $\mathbf{U}\in\mathbb{U}^N=\mathbb{U}\times\dots\times\mathbb{U}$.
    \item [$-$] $\mathbb{M}$ is the space of positive definite symmetric matrix with determinant equal to 1, such that $\mathbf{m} \in \mathbb{M}$ and $\mathbf{M} \in \mathbb{M}^C=\mathbb{M}\times\dots\times\mathbb{M}$.
\end{itemize}

Moreover, the AFCM-ER-GS (AFCM-ER-GS-L2 and AFCM-ER-GS-L1) and AFCM-ER-GP (AFCM-ER-GP-L1 and AFCM-ER-GP-L2) algorithms provide a fuzzy partition $\mathbf{U}^*=(\mathbf{u}_1^*,\dots,\mathbf{u}_N^*)$, a vector of prototypes $\mathbf{G}^*=(\mathbf{g}_1^*,\dots,\mathbf{g}_C^*)$ and a relevance weight vector $\mathbf{v}^{*}$ such that:

\begin{itemize}
    \item $J_{AFCM-ER-GS}(\mathbf{G}^*,\mathbf{v}^*,\mathbf{U}^*)=\min\{J_{AFCM-ER-GS}(\mathbf{G},\mathbf{v},\mathbf{U}),\mathbf{G} \in \mathbb{L}^C, \mathbf{v} \in \Xi, \mathbf{U}\in\mathbb{U}^N\}$
    \item $J_{AFCM-ER-GP}(\mathbf{G}^*,\mathbf{v}^*,\mathbf{U}^*)=\min\{J_{AFCM-ER-GP}(\mathbf{G},\mathbf{v},\mathbf{U}),\mathbf{G} \in \mathbb{L}^C, \mathbf{v} \in \Xi, \mathbf{U}\in\mathbb{U}^N\}$
\end{itemize}

where
\begin{itemize}
    \item [$-$] $\Xi$ is the space of vectors of weights such that $\mathbf{v} \in \Xi$. In this paper $\Xi =\{\mathbf{v}=(\mathbf{v}_1,\dots,\mathbf{v}_P) \in \mathbb{R}^P:\mathbf{v}_j>0$ and $\prod_{j=1}^{P}v_j=1\}$ or $\Xi =\{ \mathbf{v}=(v_1,\dots,v_P) \in \mathbb{R}^P: v_j\in [0,1]$ and $\sum_{j=1}^{P}v_j=1\}$.
\end{itemize}

Besides, AFCM-ER-LS-L1 algorithms provide a fuzzy partition $\mathbf{U}^*=(\mathbf{u}_1^*,\dots,\mathbf{u}_N^*)$, a vector of prototypes $\mathbf{G}^*=(\mathbf{g}_1^*,\dots,\mathbf{g}_C^*)$ and a vector of relevance weight vectors $\mathbf{V}^{*}=(\mathbf{v}_1^{*},\dots,\mathbf{v}_C^*)$ such that:
\begin{itemize}
    \item $J_{AFCM-ER-LS-L1}(\mathbf{G}^*,\mathbf{V}^*,\mathbf{U}^*)=\min\{J_{AFCM-ER-LS-L1}(\mathbf{G},\mathbf{V},\mathbf{U}),\mathbf{G} \in \mathbb{L}^C, \mathbf{V} \in \Xi^C, \mathbf{U}\in\mathbb{U}^N\}$
\end{itemize}
where
\begin{itemize}
    \item [$-$] $\Xi$ is the space of vectors of weights such that $\mathbf{v}_k \in \Xi, (k=1,\dots,C)$. In this paper $\Xi =\{ \mathbf{v}=(v_1,\dots,v_P) \in \mathbb{R}^P: v_j\in [0,1]$ and $\sum_{j=1}^{P}v_j=1\}$, and $\mathbf{V}\in \Xi^C=\Xi\times\dots\times\Xi$.
\end{itemize}

Similarly to Ref \cite{diday1976clustering}, the convergence properties of the proposed algorithms can be studied from the series:
\begin{itemize}
    \item $v^{(t)}_{AFCM-ER-M}(\mathbf{G}^{(t)},\mathbf{m}^{(t)},\mathbf{U}^{(t)}) \in \mathbb{L}^C\times\mathbb{M}\times\mathbb{U}^N$ and $u^{(t)}_{AFCM-ER-M}=\\J_{AFCM-ER-M}(v^{(t)}_{AFCM-ER-M})=J_{AFCM-ER-M}(\mathbf{G}^{(t)},\mathbf{m}^{(t)},\mathbf{U}^{(t)})$ where $t=0,1,\dots$ is the iteration number;
    \item $v^{(t)}_{AFCM-ER-Mk}(\mathbf{G}^{(t)},\mathbf{M}^{(t)},\mathbf{U}^{(t)}) \in \mathbb{L}^C\times\mathbb{M}^C\times\mathbb{U}^N$ and $u^{(t)}_{AFCM-ER-Mk}=\\J_{AFCM-ER-Mk}(v^{(t)}_{AFCM-ER-Mk})=J_{AFCM-ER-Mk}(\mathbf{G}^{(t)},\mathbf{M}^{(t)},\mathbf{U}^{(t)})$ where $t=0,1,\dots$ is the iteration number;
    \item $v^{(t)}_{AFCM-ER-GS}(\mathbf{G}^{(t)},\mathbf{v}^{(t)},\mathbf{U}^{(t)}) \in \mathbb{L}^C\times\Xi\times\mathbb{U}^N$ and $u^{(t)}_{AFCM-ER-GS}=\\J_{AFCM-ER-GS}(v^{(t)}_{AFCM-ER-GS})=J_{AFCM-ER-GS}(\mathbf{G}^{(t)},\mathbf{v}^{(t)},\mathbf{U}^{(t)})$ where $t=0,1,\dots$ is the iteration number;
     \item $v^{(t)}_{AFCM-ER-GP}(\mathbf{G}^{(t)},\mathbf{v}^{(t)},\mathbf{U}^{(t)}) \in \mathbb{L}^C\times\Xi\times\mathbb{U}^N$ and $u^{(t)}_{AFCM-ER-GP}=\\J_{AFCM-ER-GP}(v^{(t)}_{AFCM-ER-GP})=J_{AFCM-ER-GP}(\mathbf{G}^{(t)},\mathbf{v}^{(t)},\mathbf{U}^{(t)})$ where $t=0,1,\dots$ is the iteration number;
      \item $v^{(t)}_{AFCM-ER-LS-L1}(\mathbf{G}^{(t)},\mathbf{V}^{(t)},\mathbf{U}^{(t)}) \in \mathbb{L}^C\times\Xi^C\times\mathbb{U}^N$ and $u^{(t)}_{AFCM-ER-LS-L1}=\\J_{AFCM-ER-LS-L1}(v^{(t)}_{AFCM-ER-LS-L1})=J_{AFCM-ER-LS-L1}(\mathbf{G}^{(t)},\mathbf{V}^{(t)},\mathbf{U}^{(t)})$ where $t=0,1,\dots$ is the iteration number;
\end{itemize}
From the initial terms: \\
\noindent $v^{(0)}_{AFCM-ER-M}(\mathbf{G}^{(0)}, \mathbf{m}^{(0)},\mathbf{U}^{(0)})$, $v^{(0)}_{AFCM-ER-Mk}(\mathbf{G}^{(0)},\mathbf{M}^{(0)},\mathbf{U}^{(0)})$, $v^{(0)}_{AFCM-ER-GS}(\mathbf{G}^{(0)},\mathbf{v}^{(0)},\mathbf{U}^{(0)})$, \\
\noindent $v^{(0)}_{AFCM-ER-GP}(\mathbf{G}^{(0)},\mathbf{v}^{(0)},\mathbf{U}^{(0)})$ and $v^{(0)}_{AFCM-ER-LS-L1}(\mathbf{G}^{(0)},\mathbf{V}^{(0)},\mathbf{U}^{(0)})$, 
the algorithms AFCM-ER-M, AFCM-ER-Mk, AFCM-ER-GS, AFCM-ER-GP and AFCM-ER-LS-L1 compute the different terms of the series, $v^{(t)}_{AFCM-ER-M}$, $v^{(t)}_{AFCM-ER-Mk}$, $v^{(t)}_{AFCM-ER-GS}$, $v^{(t)}_{AFCM-ER-GP}$, and $v^{(t)}_{AFCM-ER-LS-L1}$, until the respective convergence (to be demonstrate) when the objective functions $J_{AFCM-ER-M}$, $J_{AFCM-ER-Mk}$, $J_{AFCM-ER-GS}$, $J_{AFCM-ER-GP}$ and $J_{AFCM-ER-LS-L1}$ reach stationary values.

\begin{proposition}\label{prop:convergency1}
\begin{enumerate}[i)]
    \item The series $u^{(t)}_{AFCM-ER-M}=J_{AFCM-ER-M}(v_{AFCM-ER-M}^{(t)})=\\J_{AFCM-ER-M}(\mathbf{G}^{(t)}, \mathbf{m}^{(t)}, \mathbf{U}^{(t)}), t=0,1,\dots$,  decreases at each iteration and converge;
    \item The series $u^{(t)}_{AFCM-ER-Mk}=J_{AFCM-ER-Mk}(v_{AFCM-ER-Mk}^{(t)})=\\J_{AFCM-ER-Mk}(\mathbf{G}^{(t)}, \mathbf{M}^{(t)}, \mathbf{U}^{(t)}), t=0,1,\dots$,  decreases at each iteration and converge;
    \item The series $u^{(t)}_{AFCM-ER-GS}=J_{AFCM-ER-GS}(v_{AFCM-ER-GS}^{(t)})=\\J_{AFCM-ER-GS}(\mathbf{G}^{(t)}, \mathbf{v}^{(t)}, \mathbf{U}^{(t)}), t=0,1,\dots$,  decreases at each iteration and converge;
    \item The series $u^{(t)}_{AFCM-ER-GP}=J_{AFCM-ER-GP}(v_{AFCM-ER-GP}^{(t)})=\\J_{AFCM-ER-GP}(\mathbf{G}^{(t)}, \mathbf{v}^{(t)}, \mathbf{U}^{(t)}), t=0,1,\dots$,  decreases at each iteration and converge;
    \item The series $u^{(t)}_{AFCM-ER-LS-L1}=J_{AFCM-ER-LS-L1}(v_{AFCM-ER-LS-L1}^{(t)})=\\J_{AFCM-ER-LS-L1}(\mathbf{G}^{(t)}, \mathbf{V}^{(t)}, \mathbf{U}^{(t)}), t=0,1,\dots$,  decreases at each iteration and converge;
\end{enumerate}
\end{proposition}

\begin{proof}
The proof is given in \ref{proof-conv-1}.
\end{proof}

\begin{proposition}\label{prop:convergency2}
\begin{enumerate}[i)]
    \item The series $v_{AFCM-ER-M}^{(t)}=(\mathbf{G}^{(t)},\mathbf{m}^{(t)},\mathbf{U}^{(t)}), t=0,1,\dots,$ converges;
    \item The series $v_{AFCM-ER-Mk}^{(t)}=(\mathbf{G}^{(t)},\mathbf{M}^{(t)},\mathbf{U}^{(t)}), t=0,1,\dots,$ converges;
    \item The series $v_{AFCM-ER-GS}^{(t)}=(\mathbf{G}^{(t)},\mathbf{v}^{(t)},\mathbf{U}^{(t)}), t=0,1,\dots,$ converges;
    \item The series $v_{AFCM-ER-GP}^{(t)}=(\mathbf{G}^{(t)},\mathbf{v}^{(t)},\mathbf{U}^{(t)}), t=0,1,\dots,$ converges;
    \item The series $v_{AFCM-ER-LS-L1}^{(t)}=(\mathbf{G}^{(t)},\mathbf{V}^{(t)},\mathbf{U}^{(t)}), t=0,1,\dots,$ converges.
\end{enumerate}
\end{proposition}

\begin{proof}
The proof is given in \ref{proof-conv-2}.
\end{proof}

\subsubsection{Time complexity of proposed algorithms}
The computational complexity for computing the prototypes for methods based on Euclidean and Mahalanobis distances is $O(N\times C\times P)$, and $O(N\times C\times P \times \log(N))$ for approaches based on City-Block distance (in which $N$, $C$ and $P$ represent the number of objects, clusters and variables respectively). In the weighting step, the complexity to obtain $\mathbf{M}$ and $\mathbf{M}_k$ for AFCM-ER-M and AFCM-ER-Mk respectively, depends on the matrix inversion method used in the implementation of the clustering algorithm. In this paper, the complexity for obtaining $\mathbf{M}$ for AFCM-ER-M is $O(\max\{N\times C \times P^2,P^3\})$, and $O(C\times\max\{N \times P^2,P^3\})$ to compute $\mathbf{M}_k$ for the AFCM-ER-Mk algorithm. For the other methods, the complexity time to compute the relevance weights is $O(N\times C\times P)$. Finally, for computing the matrix of membership degree for AFCM-ER-M and AFCM-ER-Mk, the complexity time is $O(N\times C\times P^2)$, however, for other approaches, it is $O(N\times C\times P)$. 
Therefore, globally, assuming that the iterative function needs $T$ iterations to converge, we would have:
\begin{itemize}
    \item For AFCM-ER-M, a complexity time of $O(T\times\max\{N\times C \times P^2,P^3\})$.
    \item For AFCM-ER-Mk, a complexity time of $O(T\times C\times\max\{N \times P^2,P^3\})$.
    \item For AFCM-ER-GS-L2 and AFCM-ER-GP-L2, a complexity time of $O(T\times N\times C\times P)$.
    \item Finally, for AFCM-ER-GS-L1, AFCM-ER-GP-L1 and AFCM-ER-LS-L1, a complexity time of $O(T\times C \times N\times P\times \log(N))$.
\end{itemize}

\section{Experimental results}\label{sect:experimentRes}
This section aims to evaluate the performance and illustrates the usefulness of the proposed algorithms by applying it to suitable synthetic and real datasets. All experiments were conducted on the same machine (OS: Windows 10 Professional 64-bits, Memory: 8 GiB, Processor: Intel Core i7-5500U CPU @ 2.40 GHz).

\subsection{Experimental setting}\label{sect:expSetting}
The proposed algorithms were %will be 
compared with five previous most related fuzzy clustering models: FCM-ER-L2 and FCM-ER-L1 \cite{sadaaki1997}, the Fuzzy Co-Clustering algorithm for Images (hereafter we will adopt the notation AFCM-ER-LS-L2) \cite{hanmandlu2013color}, AFCM-ER-LP-L2 \cite{rodriguez2017fuzzy} and AFCM-ER-LP-L1 \cite{rodriguez2018fuzzy} algorithms. 

%The choice of t
The parameter $T_u$ for AFCM-ER-M, AFCM-ER-Mk, FCM-ER-L2, FCM-ER-L1, AFCM-ER-LP-L2, AFCM-ER-LP-L1, AFCM-ER-GP-L2, and AFCM-ER-GP-L1 was obtained without supervision as follows. For each dataset, the value of $T_u$ was varied between $0.01$ to $100$ (with step $0.01$), and the threshold for $T_u$ corresponds to the value of the fuzzifier at which the minimum centroid distance falls under $0.1$ for the first time, similar to Ref. \cite{schwammle2010simple}. For this purpose, before running the algorithms, each dataset is pre-processed so that every feature is standardized to have an average of zero and the standard deviation is equal to one \cite{schwammle2010simple}.

The choice of the parameter $T_u$ for the AFCM-ER-GS-L2, AFCM-ER-GS-L1, AFCM-ER-LS-L2, and AFCM-ER-LS-L1 algorithms followed the same procedure used above, with the value of $T_v$ obtained after several trials and error. The selected parameters correspond to a pair ($T_u$, $T_v$) with the maximum distance; this means parameter selection is made in an unsupervised way. The algorithm described in Section \ref{sect:prototype} was used to compute the prototype %vector 
associated with each fuzzy cluster for algorithms based on City-Block distance. The maximum number of iterations $T$ was 100, $\varepsilon$ was set to $10^{-5}$, and for the datasets, the number of clusters was set equal to the number of a priori classes.

From the fuzzy partition $\mathbf{U} = (\mathbf{u}_1,\ldots,\mathbf{u}_C)$ is obtained a hard partition  $Q=(Q_1,...,Q_C)$, where the cluster $Q_k (k=1,...,C)$ is defined as: $Q_k=\{i \in \{1,...,N\}: u_{ik} \geq u_{im}, \forall m \in \{1,...,C\}\}$. 
The clustering results obtained by the algorithms were compared using two measures: The Hullermeier index ($HUL$) \cite{hullermeier2009fuzzy} and the Adjusted Rand index ($ARI$) \cite{hubert1985comparing}. The $HUL$ index compares the a priori partition of the %synthetic 
datasets with the fuzzy partition provided by the algorithms and the $ARI$ with the crisp partition.

\subsection{Experiment 1}
In the first experiment, we investigate with synthetic datasets performance aspects of the proposed algorithms with Mahalanobis (AFCM-ER-M, AFCM-ER-Mk), Euclidean (AFCM-ER-GS-L2, AFCM-ER-GP-L2) and City-block (AFCM-ER-GS-L1, AFCM-ER-GP-L1, AFCM-ER-LS-L1) distances.

In this experiment was created four synthetic datasets described by two-dimensional vectors generated randomly from a normal distribution. The synthetic datasets were created having classes of different sizes and shapes as in Ref. \cite{de2006partitional}. Each synthetic datasets have 450 points, % each, 
divided into four classes of unequal sizes: two classes of size 150 each, one with 50 and other with 100. Each class in these data were drawn according to a bi-variate normal distribution with vector $\mathbf{\mu}$ and covariance matrix $\mathbf{\Sigma}$ represented by:

\[\mathbf{\mu}=\begin{bmatrix}
    \mu_1 \\
   \mu_2
\end{bmatrix} \quad \text{and} \quad \mathbf{\Sigma}=\begin{bmatrix}
    \sigma_1^2 & \sigma_1\sigma_2\rho \\
    \sigma_1\sigma_2\rho &  \sigma_2^2 
\end{bmatrix}\]

We consider four different data configurations: (1) the class covariance matrices are diagonal, and almost the same; (2) the class covariance matrices are diagonal but unequal; (3) the class covariance matrices are not diagonal but almost the same; (4) the class covariance matrices are not diagonal and are also unequal.

Patterns of each class in data configuration 1 (Fig. \ref{img:syndata} (a)) were drawn according to the following parameters:

\begin{inparaenum}[(1)]
\item Class 1: $\mu_1=45$, $\mu_2=30$, $\sigma_1^2=100$, $\sigma_2^2=9$, $\rho=0.0$;
\item Class 2: $\mu_1=70$, $\mu_2=38$, $\sigma_1^2=81$, $\sigma_2^2=16$, $\rho=0.0$;
\item Class 3: $\mu_1=45$, $\mu_2=42$, $\sigma_1^2=100$, $\sigma_2^2=16$, $\rho=0.0$;
\item Class 4: $\mu_1=42$, $\mu_2=20$, $\sigma_1^2=81$, $\sigma_2^2=9$, $\rho=0.0$
\end{inparaenum}

Patterns of each cluster in data configuration 2 (Fig. \ref{img:syndata}, (b)) were drawn according to the following parameters:

\begin{inparaenum}[(1)]
\item Class 1: $\mu_1=45$, $\mu_2=22$, $\sigma_1^2=144$, $\sigma_2^2=9$, $\rho=0.0$;
\item Class 2: $\mu_1=70$, $\mu_2=38$, $\sigma_1^2=81$, $\sigma_2^2=36$, $\rho=0.0$;
\item Class 3: $\mu_1=50$, $\mu_2=42$, $\sigma_1^2=36$, $\sigma_2^2=81$, $\rho=0.0$;
\item Class 4: $\mu_1=42$, $\mu_2=2$, $\sigma_1^2=9$, $\sigma_2^2=144$, $\rho=0.0$
\end{inparaenum}

Patterns of each class in data configuration 3 (Fig. \ref{img:syndata}, (c)) were drawn according to the following parameters:

\begin{inparaenum}[(1)]
\item Class 1: $\mu_1=45$, $\mu_2=30$, $\sigma_1^2=100$, $\sigma_2^2=9$, $\rho=0.7$;
\item Class 2: $\mu_1=70$, $\mu_2=38$, $\sigma_1^2=81$, $\sigma_2^2=16$, $\rho=0.8$;
\item Class 3: $\mu_1=45$, $\mu_2=42$, $\sigma_1^2=100$, $\sigma_2^2=16$, $\rho=0.7$;
\item Class 4: $\mu_1=42$, $\mu_2=20$, $\sigma_1^2=81$, $\sigma_2^2=9$, $\rho=0.8$
\end{inparaenum}

Finally, the patterns of each class in data configuration 4 (Fig. \ref{img:syndata}, (d)) were drawn according to the following parameters:

\begin{inparaenum}[(1)]
\item Class 1: $\mu_1=45$, $\mu_2=22$, $\sigma_1^2=144$, $\sigma_2^2=9$, $\rho=0.7$;
\item Class 2: $\mu_1=70$, $\mu_2=38$, $\sigma_1^2=81$, $\sigma_2^2=36$, $\rho=0.8$;
\item Class 3: $\mu_1=50$, $\mu_2=42$, $\sigma_1^2=36$, $\sigma_2^2=81$, $\rho=0.7$;
\item Class 4: $\mu_1=42$, $\mu_2=2$, $\sigma_1^2=9$, $\sigma_2^2=144$, $\rho=0.8$
\end{inparaenum}

\begin{figure}[!htb]
\centering
  \subfloat[Dataset 1]{\includegraphics[width=0.25\textwidth]{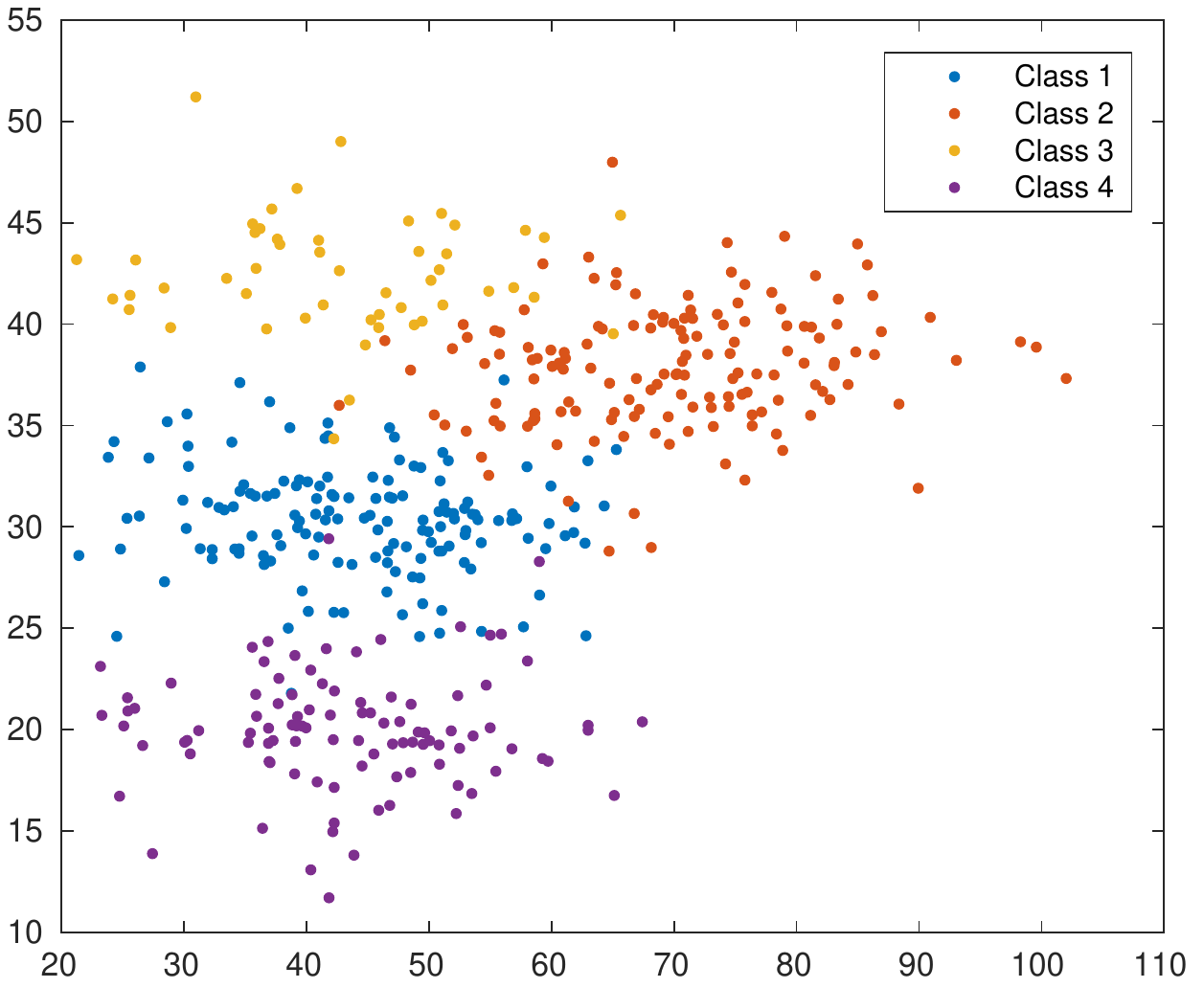}}
  \subfloat[Dataset 2]{\includegraphics[width=0.25 \textwidth]{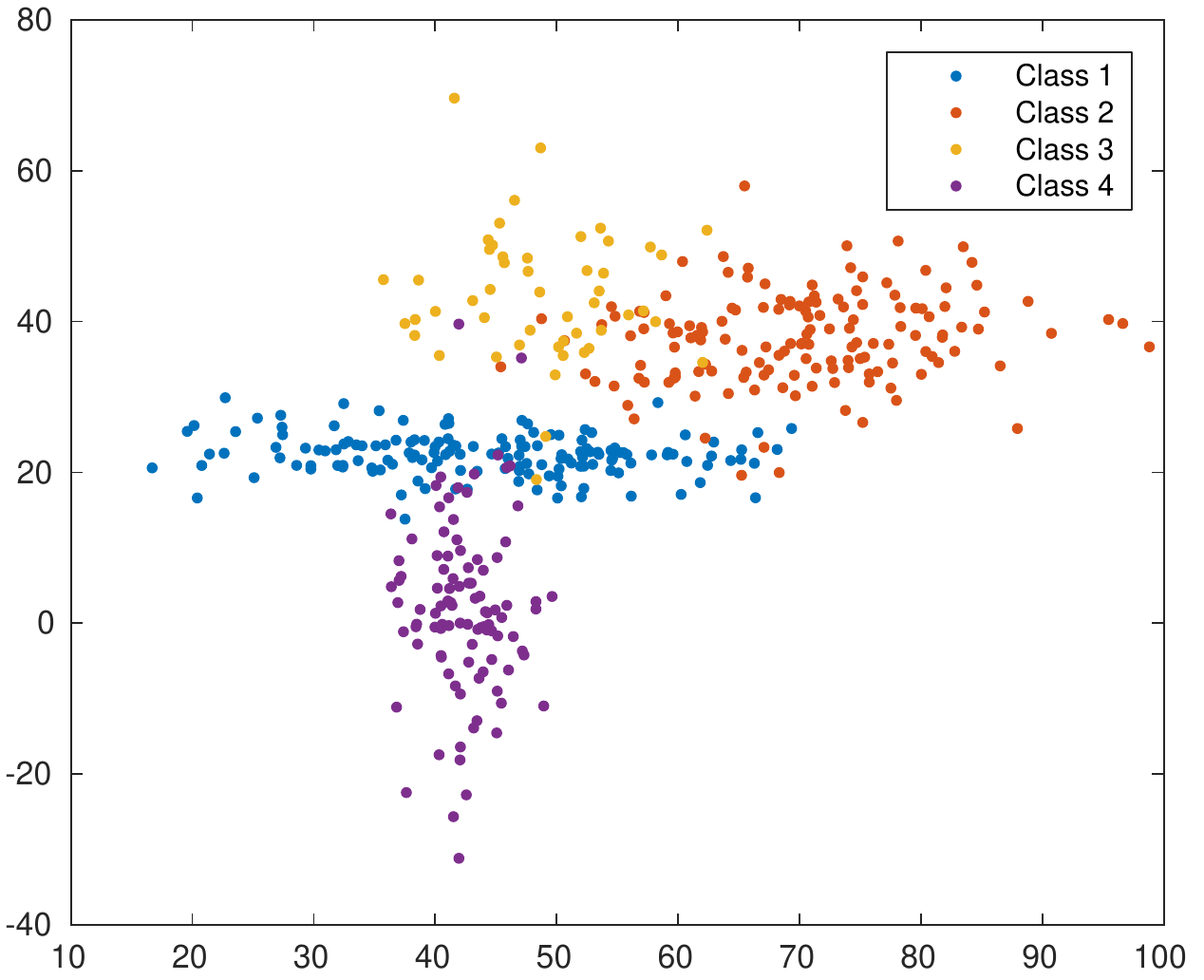}}
  \subfloat[Dataset 3]{\includegraphics[width=0.25 \textwidth]{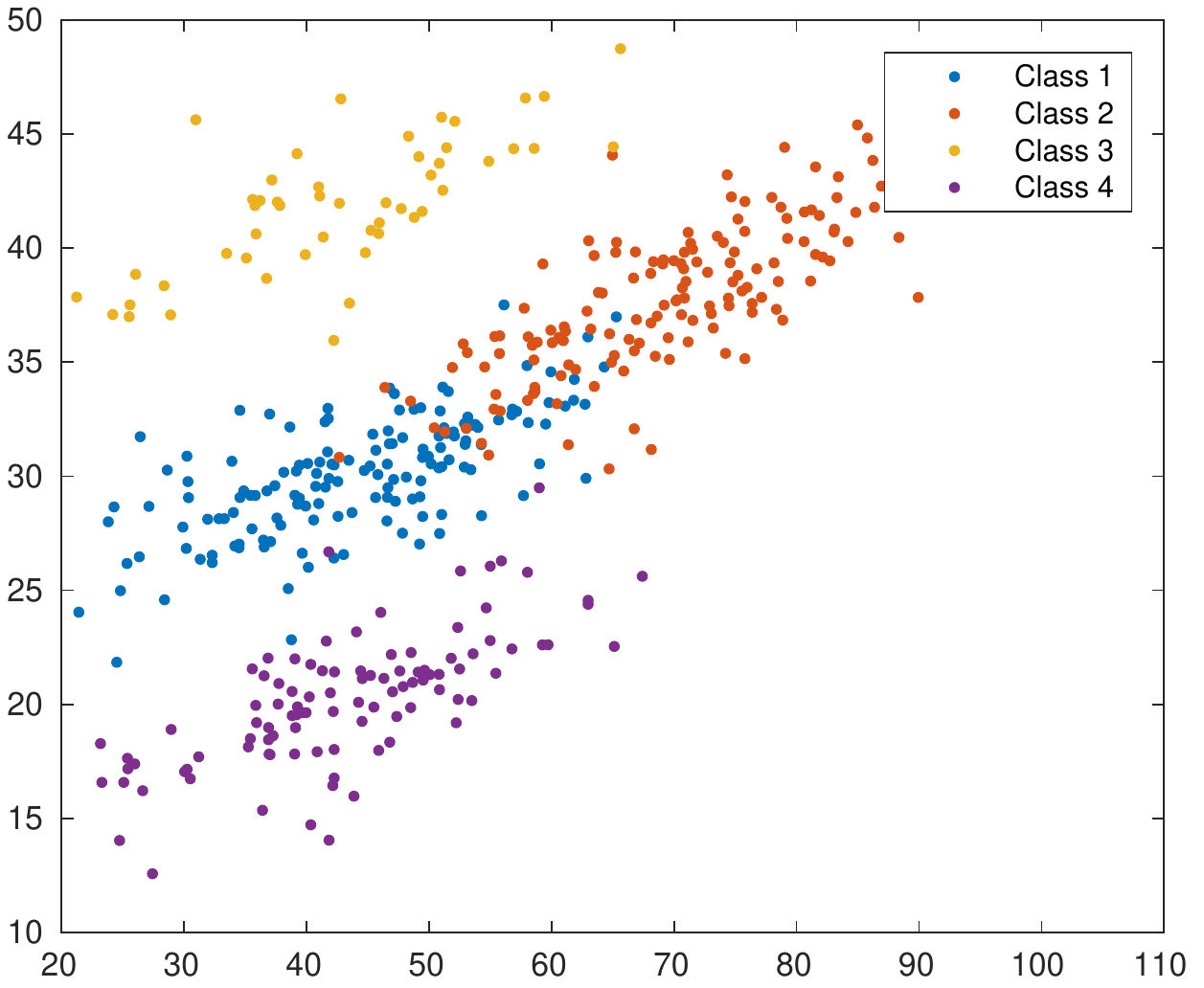}}
  \subfloat[Dataset 4]{\includegraphics[width=0.25 \textwidth]{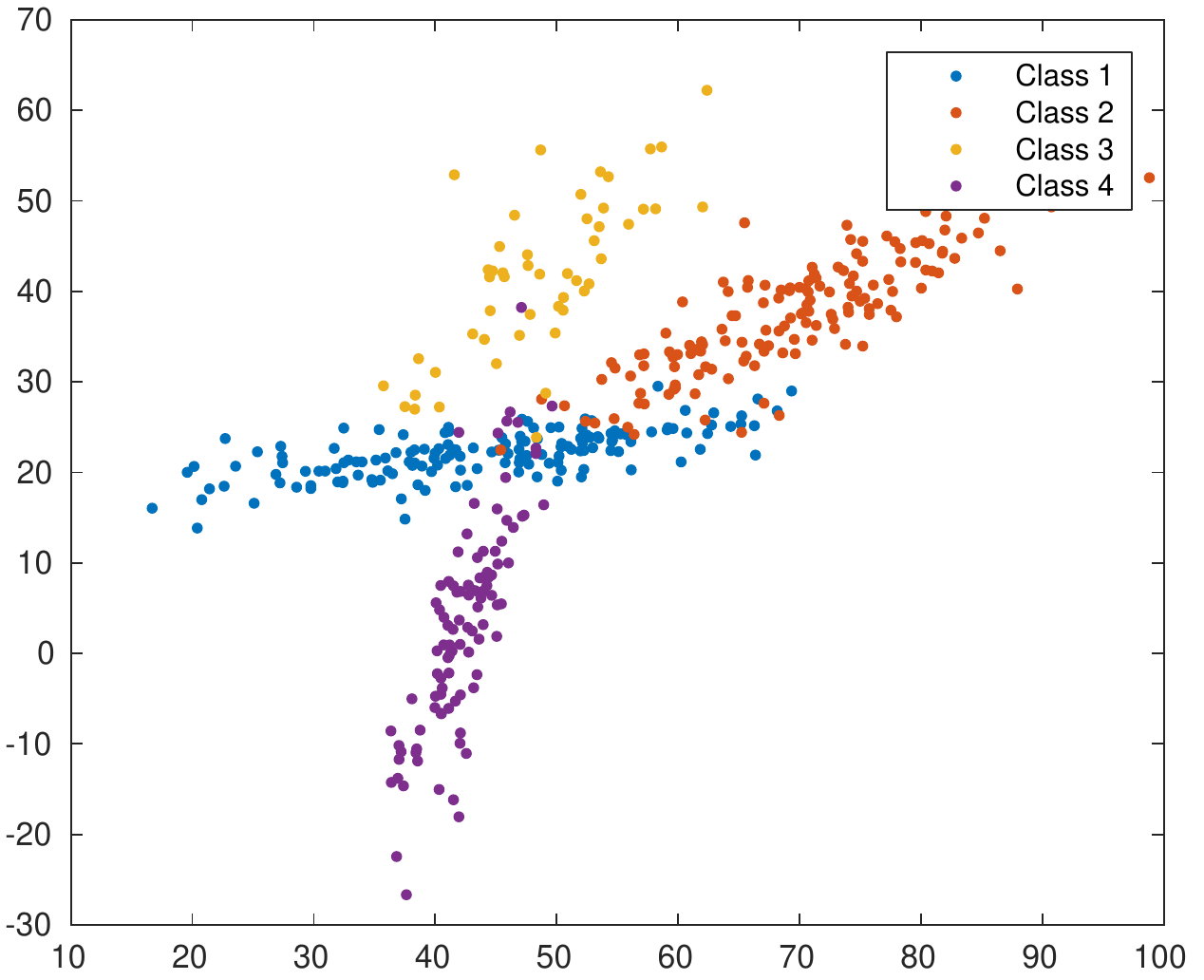}}
  \caption{Clusters drawn from data configurations 1, 2, 3 and 4.}
  \label{img:syndata}
\end{figure}

Fifty replications of each synthetic dataset were carried out in a framework of a Monte Carlo experiment. For each dataset, 50 random initializations of the %clustering 
algorithm are performed. The best result from these 50 repetitions is selected according to their respective objective function. The average and standard deviation of the indexes were calculated based on the 50 Monte Carlo iterations.

\subsubsection{Results}
Table \ref{tab:resSD} shows the values of the mean and the standard deviation for $HUL$ and $ARI$ %indices 
for different methods and data configurations. %Also, to explore the statistical significance of the results, the Friedman test \cite{friedman1937use} can be used to compare the algorithms over the synthetic datasets by ranking each algorithm on each dataset separately. The algorithm obtained the best performance gets the rank of 1, the second best ranks 2, and so on. In case of ties, the average ranks are assigned. Then the average ranks of all algorithms on all datasets are calculated and compared. If the null hypothesis, which is all algorithms are performing equivalently, is rejected under the Friedman test statistic, post-hoc tests such as the Nemenyi test \cite{nemenyi1963} can be used to determine which algorithms perform statistically different. The Nemenyi test compares algorithms in a pairwise manner. According to this test, the performances of the two algorithms are significantly different if the distance of the average ranks exceeds the critical distance. The objective is to determine whether or not there is at least one method that is significantly better than at least one other method at the $\alpha=0.05$ level. If this is the case, we will conduct a post hoc (Nemenyi test) test to determine which pairs of methods differ significantly. Figure \ref{img:hulricFriedN1} shows the comparison of the twelve algorithms against each other with the Nemenyi test. The first row shows the $HUL$ results and the second one the $ARI$ values.

\begin{table}[!ht]
	\caption{Mean and standard deviation (in parenthesis) for the data configurations.}\label{tab:resSD}
	\resizebox{\textwidth}{!}{
		\begin{tabular} {|p{2.82cm}|p{1.1cm}|p{1.1cm}|p{1.1cm}|p{1.1cm}|p{1.1cm}|p{1.1cm}|p{1.1cm}|p{1.1cm}|}
\hline
\multicolumn{1}{|c|}{}&\multicolumn{2}{c|}{Dataset 1}&\multicolumn{2}{c|}{Dataset 2}&\multicolumn{2}{c|}{Dataset 3}&\multicolumn{2}{c|}{Dataset 4}\\
\hline
Algorithms&$HUL$&$ARI$&$HUL$&$ARI$&$HUL$&$ARI$&$HUL$&$ARI$\\
\hline
FCM-ER-L2       &0.7006&0.5991&0.6735&0.5815&0.6314&0.4162&0.6368&0.4613\\
&(0.0231)&(0.0388)&(0.0186)&(0.0343)&(0.0206)&(0.0338)&(0.0220)&(0.0402)\\    
FCM-ER-L1       &\textbf{0.7946}&\textbf{0.6383}&0.7926&0.5693&0.7140&0.4393&0.7549&0.4529\\
&(0.0788)&(0.0815)&(0.0344)&(0.0555)&(0.0659)&(0.0465)&(0.0205)&(0.0344)\\               
AFCM-ER-M       &0.7091&0.5899&0.6190&0.5186&0.7178&\textbf{0.6631}&0.6197&0.4946\\
&(0.0225)&(0.0465)&(0.0365)&(0.0344)&(0.0485)&(0.1366)&(0.0372)&(0.0294)\\               
AFCM-ER-Mk      &0.6685&0.4127&0.7427&0.5990&0.7295&0.4936&\textbf{0.7933}&\textbf{0.6499}\\
&(0.0887)&(0.1111)&(0.0821)&(0.0681)&(0.0901)&(0.1581)&(0.0660)&(0.1128)\\               
AFCM-ER-GP-L2   &0.7063&0.5905&0.6648&0.5621&0.6357&0.4131&0.6347&0.4393\\
&(0.0219)&(0.0399)&(0.0252)&(0.0353)&(0.0203)&(0.0339)&(0.0197)&(0.0378)\\               
AFCM-ER-GP-L1   &0.7739&0.6226&\textbf{0.7995}&0.5819&\textbf{0.7322}&0.4506&0.7545&0.4506\\
&(0.0759)&(0.0801)&(0.0353)&(0.0641)&(0.0600)&(0.0595)&(0.0310)&(0.0464)\\               
AFCM-ER-LP-L2   &0.7055&0.4393&0.7954&\textbf{0.6466}&0.6424&0.3181&0.6942&0.5651\\
&(0.0169)&(0.0529)&(0.0342)&(0.0409)&(0.0130)&(0.0241)&(0.0443)&(0.0390)\\               
AFCM-ER-LP-L1   &0.6904&0.5328&0.7552&0.6452&0.6673&0.3994&0.7176&0.5168\\
&(0.0593)&(0.0769)&(0.0530)&(0.0600)&(0.0365)&(0.0438)&(0.0537)&(0.0533)\\               
AFCM-ER-GS-L2   &0.6836&0.4892&0.6890&0.5193&0.6327&0.3814&0.6515&0.4175\\
&(0.0368)&(0.1425)&(0.0422)&(0.0984)&(0.0229)&(0.0705)&(0.0336)&(0.0889)\\               
AFCM-ER-GS-L1   &0.6663&0.5142&0.7301&0.4444&0.6267&0.3987&0.7034&0.4175\\
&(0.0537)&(0.1054)&(0.0509)&(0.1298)&(0.0523)&(0.0775)&(0.0344)&(0.0876)\\               
AFCM-ER-LS-L2   &0.6936&0.4861&0.6988&0.5556&0.6394&0.3579&0.6483&0.4895\\
&(0.0344)&(0.1475)&(0.0390)&(0.0599)&(0.0235)&(0.0770)&(0.0276)&(0.0467)\\               
AFCM-ER-LS-L1   &0.6528&0.4243&0.7443&0.5432&0.6245&0.3611&0.6959&0.4864\\
&(0.0556)&(0.1392)&(0.0470)&(0.0464)&(0.0517)&(0.0875)&(0.0589)&(0.0708)\\
			\hline	
		\end{tabular}}
	\end{table}

In data configuration 1 (the cluster covariance matrices are diagonal and almost the same), the best results (Table \ref{tab:resSD}) according to the $HUL$ index were obtained by FCM-ER-L1, AFCM-ER-GP-L1 and AFCM-ER-M algorithms with values of 0.7946, 0.7739 and 0.7091 respectively. For the $ARI$ index, the best performance was presented by the FCM-ER-L1 algorithm; moreover, AFCM-ER-GP-L1 and FCM-ER-L2 achieved, respectively, the second and third best values. 
The AFCM-ER-LS-L1, AFCM-ER-GS-L1, and AFCM-ER-Mk algorithms produced the worst clustering results for $HUL$, and the algorithms AFCM-ER-Mk, AFCM-ER-LS-L1, and AFCM-ER-LP-L2 for $ARI$. 
As expected, in this data configuration, almost all the methods with global adaptive distance outperformed their respective variants based on local adaptive distance, i.e., the methods AFCM-ER-M, AFCM-ER-GP-L1, AFCM-ER-GP-L2, AFCM-ER-GS-L1 outperformed, respectively, the methods AFCM-ER-Mk, AFCM-ER-LP-L1, AFCM-ER-LP-L2, AFCM-ER-LS-L1. AFCM-ER-GS-L2 outperforms AFCM-ER-LS-L2 concerning $ARI$ index, but AFCM-ER-LS-L2 surpasses AFCM-ER-GS-L2 regarding $HUL$ index. 

Data configuration 2 presents cluster covariance matrices that are diagonal but unequal. In this case, the best result was provided by the algorithms AFCM-ER-GP-L1, AFCM-ER-LP-L2, and FCM-ER-L1 for $HUL$ and by AFCM-ER-LP-L2, AFCM-ER-LP-L1, and AFCM-ER-Mk %algorithms 
for $ARI$. The algorithms AFCM-ER-M, AFCM-ER-GP-L2, and FCM-ER-L2 obtained the worst performance for the $HUL$ index and AFCM-ER-GS-L1, AFCM-ER-M and AFCM-ER-GS-L2 for the $ARI$ index. For this configuration, almost all the methods with local adaptive distance presented better results compared with their respective variants based on global adaptive distance: the methods AFCM-ER-Mk, AFCM-ER-LP-L2, AFCM-ER-LS-L1, AFCM-ER-LS-L2 surpassed, respectively, the methods AFCM-ER-M, AFCM-ER-GP-L2, AFCM-ER-GS-L1, AFCM-ER-GS-L2. AFCM-ER-LP-L1 surpasses AFCM-ER-GP-L1 concerning $ARI$ index, but AFCM-ER-GP-L1 outperforms AFCM-ER-LP-L1 regarding $HUL$ index. 

In data configuration 3, (the cluster covariance matrices are not diagonal but almost the same) the first, second and third best performance for $HUL$ were presented by AFCM-ER-GP-L1, AFCM-ER-Mk, and AFCM-ER-M respectively. For the $ARI$ index, the best results were achieved, respectively, by AFCM-ER-M, AFCM-ER-Mk, and AFCM-ER-GP-L1. As expected, for this data configuration, AFCM-ER-M and AFCM-ER-Mk were among the best. The worst result was presented, respectively by the algorithms AFCM-ER-LS-L1 and AFCM-ER-GS-L1 for $HUL$ and by, %respectively, 
AFCM-ER-LP-L2 and AFCM-ER-LS-L2 for $ARI$. Finally, for this data configuration and concerning the $ARI$ index, the methods with global adaptive distance outperformed their respective variants based on local adaptive distance, i.e., the methods AFCM-ER-M, AFCM-ER-GP-L1, AFCM-ER-GP-L2, AFCM-ER-GS-L1 and AFCM-ER-GS-L2 outperformed, respectively, the methods AFCM-ER-Mk, AFCM-ER-LP-L1, AFCM-ER-LP-L2, AFCM-ER-LS-L1, and AFCM-ER-LS-L2. 

For data configuration 4, where the cluster covariance matrices are not diagonal and unequal, the algorithm AFCM-ER-Mk outperforms the other approaches for both indices. AFCM-ER-M presented the worst performance at $HUL$, and AFCM-ER-GS-L1 and AFCM-ER-GS-L2 had the worst performance for $ARI$. Finally, for this data configuration and concerning the $ARI$ index, the methods with local adaptive distance outperformed their respective variants based on global adaptive distance, i.e., the methods AFCM-ER-Mk, AFCM-ER-LP-L1, AFCM-ER-LP-L2, AFCM-ER-LS-L1 and AFCM-ER-LS-L2 outperformed, respectively, the methods AFCM-ER-M, AFCM-ER-GP-L1, AFCM-ER-GP-L2, AFCM-ER-GS-L1, and AFCM-ER-GS-L2.

\subsection{Experiment 2}	

We developed another experiment to %check
verify the behavior of the proposed methods in the presence of outliers. For this purpose, a synthetic dataset with $80$ objects described by two-dimensional vectors was generated randomly from a normal distribution according to the following parameters (Fig. \ref{img:syndata2} (a)):

\begin{itemize}
		\item Class 1: $\mu_1=0$, $\mu_2=0$, $\sigma_1^2=0.05$, $\sigma_2^2=0.05$;
		\item Class 2: $\mu_1=0.8$, $\mu_2=0.8$, $\sigma_1^2=0.05$, $\sigma_2^2=0.05$;
\end{itemize}

To evaluate the robustness in the presence of outliers, three different percentages of outliers ($10\%$, $20\%$, $30\%$) have been added to the dataset (Figs. \ref{img:syndata2} (b), (c) and (d)) with $\mu_1=0.8, \mu_2=1, \sigma_1^2=5$ and $\sigma_2^2=5$. Fifty replications of the synthetic dataset were carried out in a framework of a Monte Carlo experiment. % as before. 
For each dataset, 50 random initializations of the clustering algorithm are performed. The average and standard deviation of the indexes were calculated based on the 50 Monte Carlo iterations.

\begin{figure}[!htb]
	\centering
	\subfloat[Data with $0\%$ of Outliers]{\includegraphics[width=0.23\textwidth]{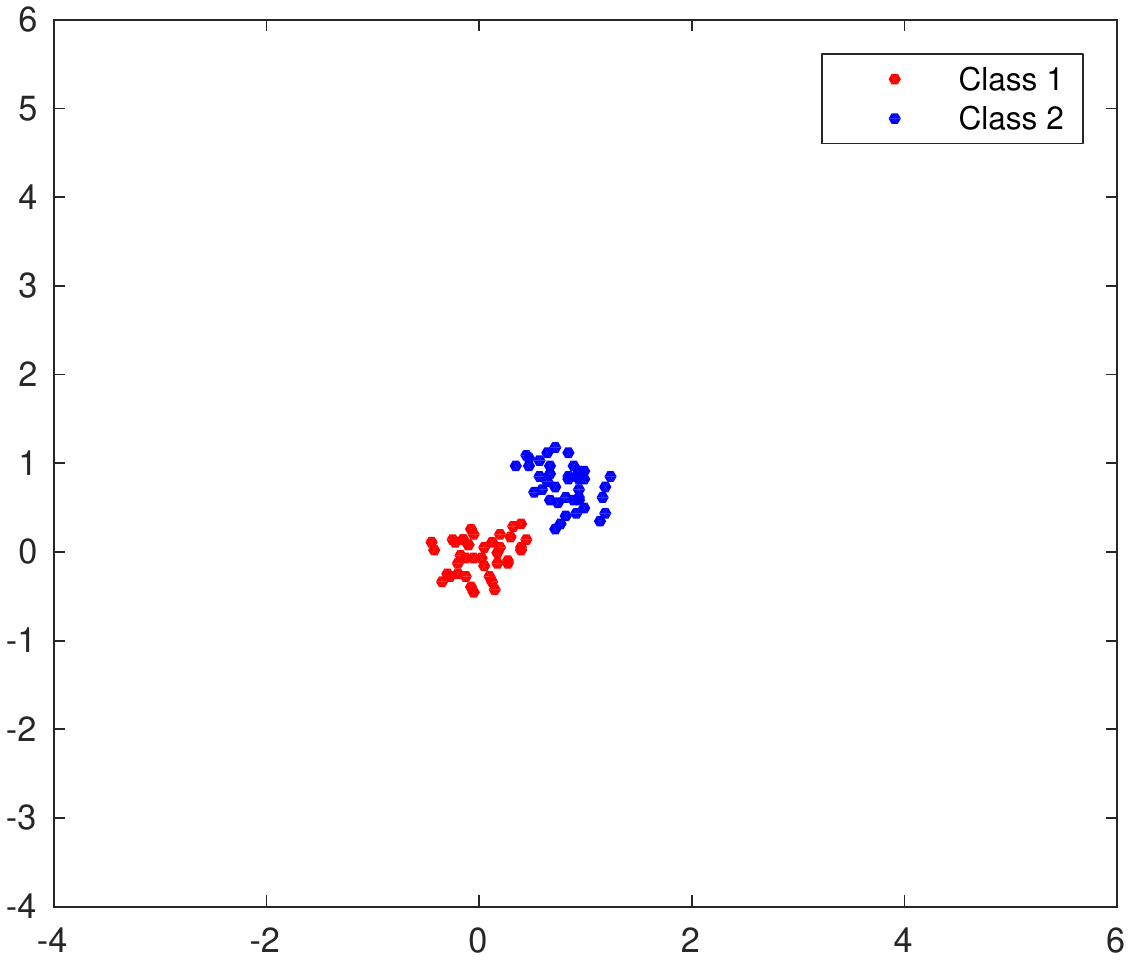}}
	\quad
	\subfloat[Data with $10\%$ of Outliers]{\includegraphics[width=0.23 \textwidth]{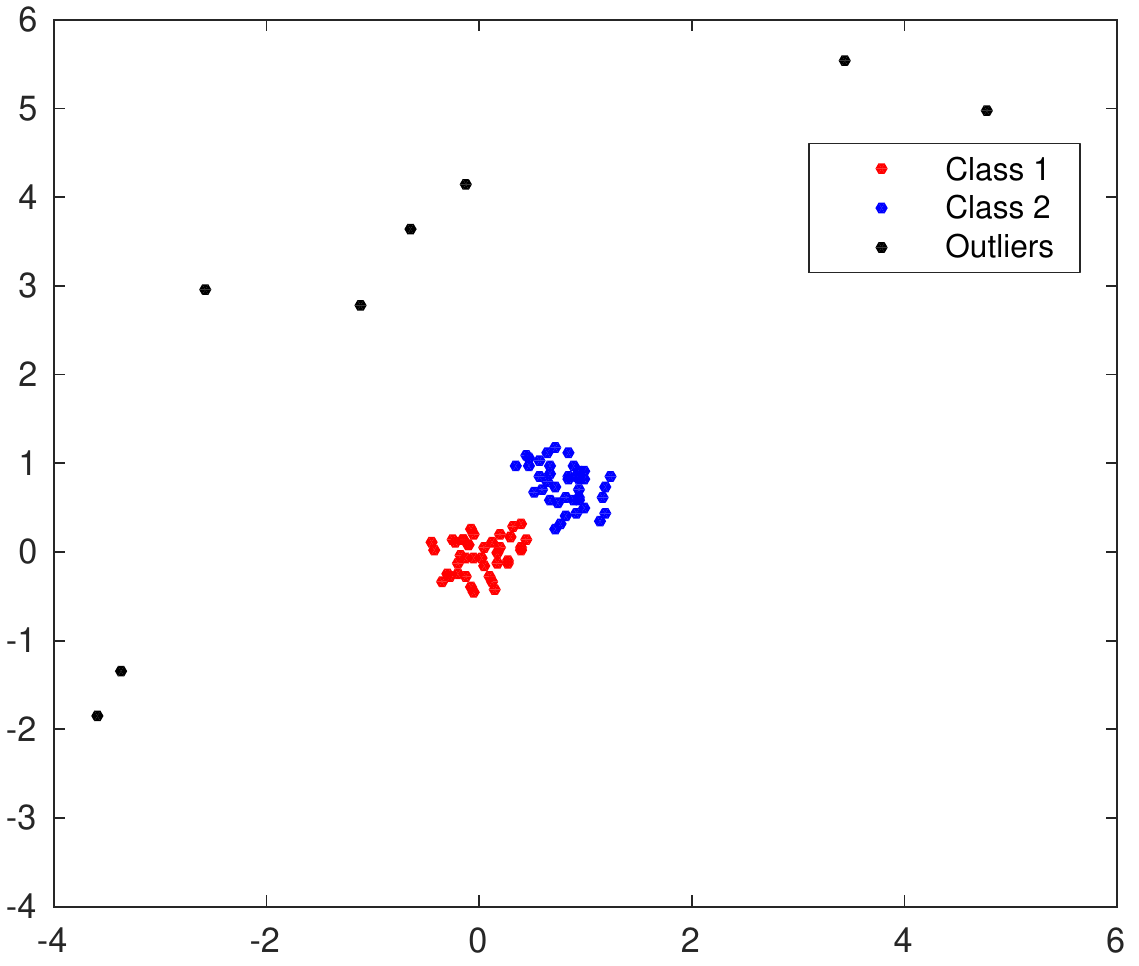}}
	\quad
	\subfloat[Data with $20\%$ of Outliers]{\includegraphics[width=0.23 \textwidth]{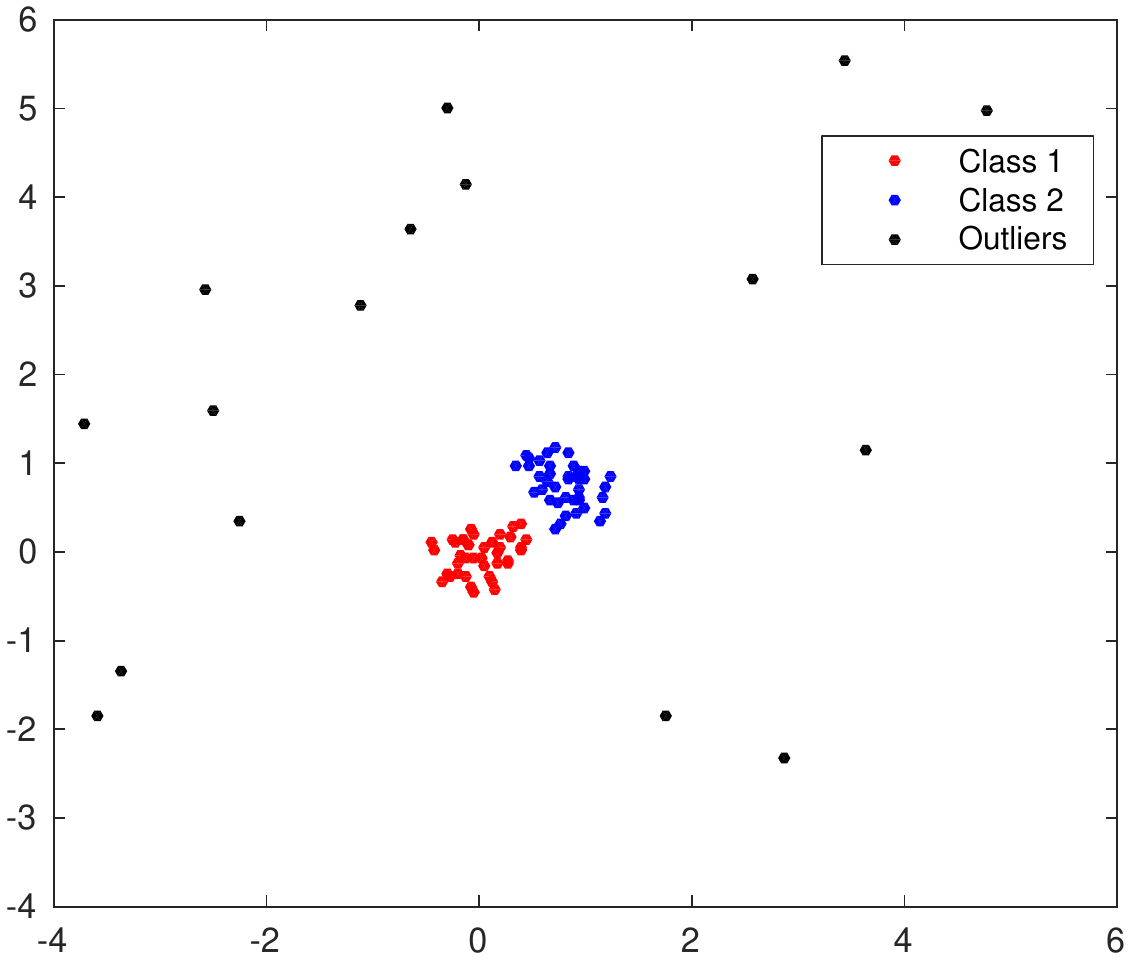}}
	\quad
	\subfloat[Data with $30\%$ of Outliers]{\includegraphics[width=0.23 \textwidth]{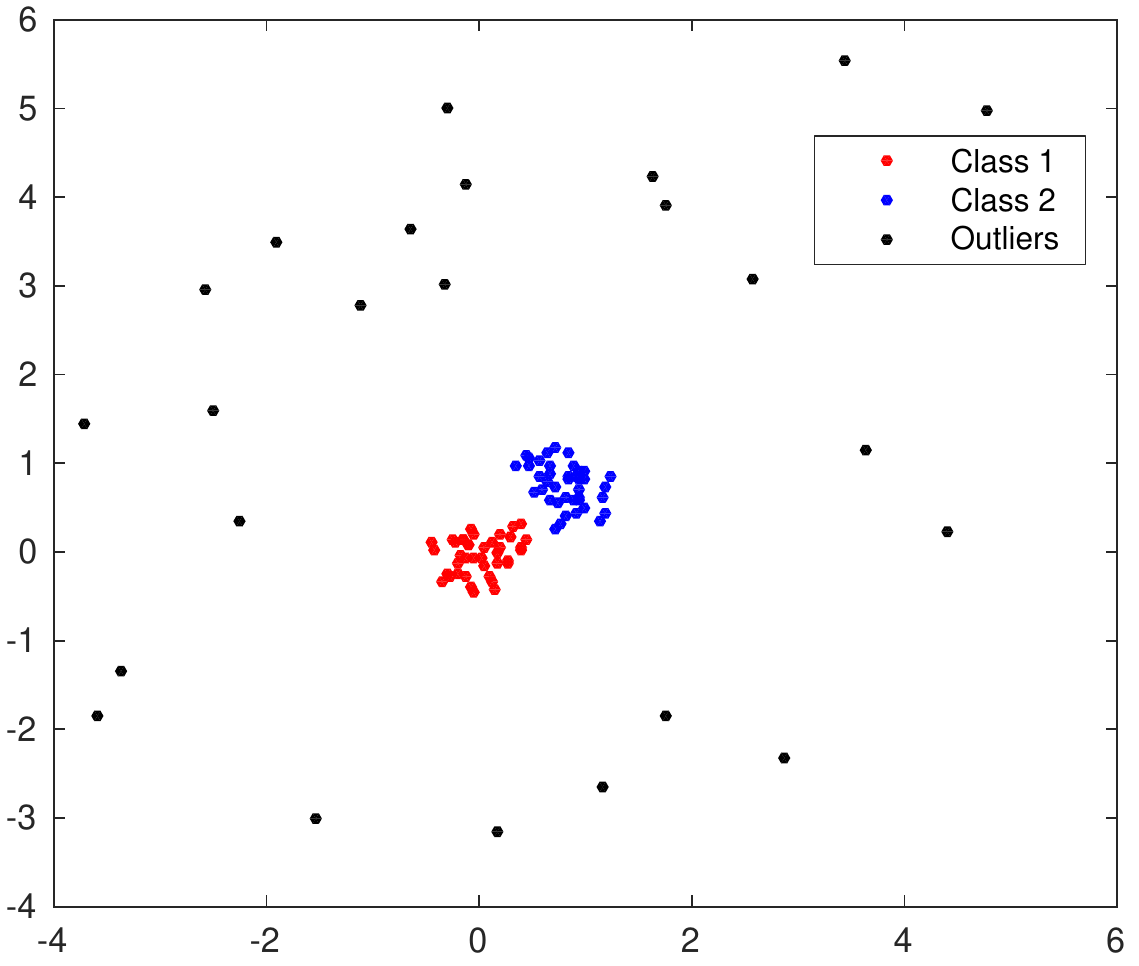}}
	\caption{Synthetic dataset with different percentage of outliers. (a) data with $0\%$ of outliers, (b) data with $10\%$ of outliers, (c) data with $20\%$ of outliers and (d) data with $30\%$ of outliers.}
	\label{img:syndata2}
\end{figure}

The robustness of the algorithms have been evaluated according to the misclassification ($HUL$ and $ARI$ indices) and according to the representative detection both in the presence of outliers. To evaluate this last index, D’Urso et al. \cite{d2015trimmed} introduced an index of robustness detection ($rd$) aiming to assess the robustness in the presence of outliers. This index compares the representatives provided by the algorithms %in the presence of outliers 
with the ideal representative of the a priori classes. Here we adopt this index as follows:

\begin{equation}
rd=\frac{\Delta(\mathbf{g}_{1}^{id},\mathbf{g}_1)+\Delta(\mathbf{g}_{1}^{id},\mathbf{g}_2)+\Delta(\mathbf{g}_{2}^{id},\mathbf{g}_1)+\Delta(\mathbf{g}_{2}^{id},\mathbf{g}_1)}{2\Delta(\mathbf{g}_{1}^{id},\mathbf{g}_{2}^{id})}
\end{equation}

\noindent where $\mathbf{g}_1$ and $\mathbf{g}_2$ denote, respectively, the representatives of clusters 1 and 2 provided by the algorithms, $\mathbf{g}_{1}^{id}$ and $\mathbf{g}_{2}^{id}$ denotes, respectively, the "ideal" representatives of the a priori classes 1 and 2, and $\Delta$ is a suitable Euclidean distance between vectors.

\subsubsection{Results}
%For each algorithm, 
Table \ref{tab:resSDExp2} presents the results for the $HUL$, $ARI$, and $rd$ indices with different percentages of outliers according to the mean and the standard deviation (in parentheses).

\begin{table}[!htb]
	\caption{Mean and standard deviation (in parentheses) of $HUL$, $ARI$, and $rd$ on the synthetic interval-valued dataset for different percentages of outliers.}\label{tab:resSDExp2}
	\resizebox{\textwidth}{!}
	{
		\begin{tabular} {|p{2.8191cm}|p{1.1cm}|p{1.1cm}|p{1.1cm}|p{1.1cm}|p{1.1cm}|p{1.1cm}|p{1.1cm}|p{1.1cm}|p{1.1cm}|p{1.1cm}|p{1.1cm}|p{1.1cm}|}
			\hline
			\multicolumn{1}{|c|}{}&\multicolumn{3}{c|}{$0\%$ Outliers}&\multicolumn{3}{c|}{$10\%$ Outliers}&\multicolumn{3}{c|}{$20\%$ Outliers}&\multicolumn{3}{c|}{$30\%$ Outliers}\\
			\hline
			Algorithms&$HUL$&$ARI$&$rd$&$HUL$&$ARI$&$rd$&$HUL$&$ARI$&$rd$&$HUL$&$ARI$&$rd$\\
			\hline
FCM-ER-L2       &0.5726&0.9761&1.0010&0.5120&0.8103&1.0272&0.5077&0.6524&1.0489&0.5039&0.4776&1.0765\\    
&(0.0226)&(0.0350)&(0.0014)&(0.0087)&(0.3117)&(0.0328)&(0.0097)&(0.3847)&(0.0471)&(0.0085)&(0.3990)&(0.0633)\\
FCM-ER-L1       &0.7708&0.9683&1.0176&0.7143&0.9683&1.0212&0.7383&0.9663&1.0366&0.7438&0.9567&1.0397\\    
&(0.1351)&(0.0431)&(0.0181)&(0.1308)&(0.0431)&(0.0280)&(0.1495)&(0.0416)&(0.0391)&(0.1346)&(0.0528)&(0.0405)\\     
AFCM-ER-M       &0.6319&0.8726&1.0028&0.5214&0.8285&1.0282&0.5108&0.7343&1.0478&0.5082&0.6231&1.0772\\
&(0.1209)&(0.2517)&(0.0064)&(0.0179)&(0.2573)&(0.0331)&(0.0067)&(0.2429)&(0.0466)&(0.0070)&(0.3188)&(0.0665)\\      
AFCM-ER-Mk      &0.6572&0.4467&1.0137&0.5308&0.1963&1.0537&0.5335&0.2199&1.1731&0.5361&0.2424&1.4651\\
&(0.2210)&(0.4029)&(0.0211)&(0.0964)&(0.2552)&(0.0645)&(0.0691)&(0.2561)&(0.5259)&(0.0831)&(0.3015)&(0.8398)\\      
AFCM-ER-GP-L2   &0.5392&0.9761&1.0011&0.5117&0.8145&1.0271&0.5051&0.6548&1.0486&0.5020&0.4765&1.0677\\
&(0.0646)&(0.0350)&(0.0021)&(0.0094)&(0.3042)&(0.0327)&(0.0070)&(0.3847)&(0.0498)&(0.0066)&(0.3981)&(0.0586)\\       
AFCM-ER-GP-L1   &0.7716&0.9692&1.0181&0.7407&0.9732&1.0245&0.7385&0.9633&1.0347&0.7550&0.9585&1.0443\\
&(0.1335)&(0.0397)&(0.0195)&(0.1449)&(0.0389)&(0.0290)&(0.1478)&(0.0432)&(0.0345)&(0.1366)&(0.0488)&(0.0407)\\       
AFCM-ER-LP-L2   &0.5657&0.9761&1.0010&0.5033&0.3116&1.0328&0.5051&0.4161&1.1113&0.5064&0.3339&1.2806\\
&(0.0338)&(0.0350)&(0.0014)&(0.0300)&(0.3210)&(0.0301)&(0.0309)&(0.3074)&(0.4038)&(0.0544)&(0.3145)&(0.6476)\\       
AFCM-ER-LP-L1   &0.7703&0.9653&1.0182&0.6903&0.9673&1.0178&0.6759&0.9586&1.0194&0.6838&0.9665&1.0214\\
&(0.1351)&(0.0436)&(0.0192)&(0.0996)&(0.0440)&(0.0217)&(0.1029)&(0.0504)&(0.0198)&(0.0777)&(0.0478)&(0.0234)\\       
AFCM-ER-GS-L2   &0.5062&0.9761&1.0009&0.5105&0.8128&1.0271&0.5063&0.6533&1.0506&0.5025&0.4818&1.0677\\
&(0.0175)&(0.0350)&(0.0013)&(0.0075)&(0.3044)&(0.0328)&(0.0101)&(0.3838)&(0.0474)&(0.0073)&(0.3974)&(0.0579)\\      
AFCM-ER-GS-L1   &0.7956&0.9702&1.0207&0.7248&0.9732&1.0237&0.7287&0.9672&1.0345&0.7243&0.9596&1.0395\\
&(0.1454)&(0.0399)&(0.0210)&(0.1368)&(0.0389)&(0.0301)&(0.1589)&(0.0394)&(0.0382)&(0.1365)&(0.0501)&(0.0436)\\       
AFCM-ER-LS-L2   &0.5444&0.8795&1.0010&0.4997&0.1466&1.0332&0.4982&0.1971&1.0542&0.4994&0.2690&1.1254\\
&(0.0271)&(0.2741)&(0.0014)&(0.0233)&(0.2276)&(0.0294)&(0.0103)&(0.2515)&(0.0434)&(0.0118)&(0.2950)&(0.3543)\\       
AFCM-ER-LS-L1   &0.7495&0.9491&1.0194&0.6383&0.8515&1.0266&0.5898&0.7302&1.0212&0.5503&0.5868&1.0181\\
&(0.1804)&(0.0642)&(0.0209)&(0.1724)&(0.2542)&(0.0344)&(0.1506)&(0.3209)&(0.0250)&(0.0868)&(0.3845)&(0.0204)\\
\hline	
\end{tabular}}
\end{table}

Table \ref{tab:resSDExp2} shows that for the values of $ARI$ and $rd$ %values 
for $0\%$ of outliers, methods with Euclidean distance presented the best results compared with those based on Mahalanobis and City-Block distances. However, concerning misclassification measured in $HUL$ and $ARI$ with different percents of outliers, algorithms based on City-Block distance outperform the other approaches, being able to identify the presence of clusters even in a noisy environment, and the performance clustering degrades very slowly as the percentage of outliers increases. Also, those methods are more robust concerning the ability to produce cluster prototypes which are not %too
very dissimilar from the ”ideal” centers ($rd$ index). In conclusion, as expected, it is observed that whatever the index considered, algorithms based on Euclidean and Mahalanobis distances are less robust to the presence of outliers than those based on City-Block distance.

\subsection{Real datasets}

%The algorithms proposed in this paper, as well as the previous proposed algorithms, 
The previous and proposed algorithms were also applied on 15 real datasets available at the $UCI$ machine learning repository \cite{asuncion1994uci}: Automobile, Balance Scale, Haberman's Survival, Statlog (Heart), Image Segmentation, Ionosphere, Iris plants, Mnist, Thyroid gland, User knowledge modeling (UKM), Vehicle, Vertebral column, Wisconsin diagnostic breast cancer (WDBC), Wall-Following Robot Navigation (WFRN) and Wine. Table \ref{tab:realDataset} briefly describes %shortly 
the datasets %considered 
in which $N$ represents the number of patterns, $P$ represents the number of variables and $C$ the number of a priori classes. We can see that several different sample sizes, number of attributes, and the number of classes were taken into account.

\begin{table}[htb!]
	\caption{Summary of the real datasets.\label{tab:realDataset}}
	\centering
	\resizebox{\textwidth}{!}{
	\begin{tabular}{ |rccc|rccc|rccc|}
		\hline
		Dataset & $N$ & $P$ & $C$   &Dataset & $N$ & $P$ & $C$  &Dataset & $N$ & $P$ & $C$ \\
		\hline
		Automobile&205&25&6  &Balance Scale&625&4&3 &Haberman & 306&3&2 \\
		Heart&270  &13&2    &Image Segmentation&2310&19&7&Ionosphere&351&33&2 \\
        Iris plants & 150 & 4 & 3    &Mnist&14 780&784&2    &Thyroid gland&215&5&3\\
		UKM&403&5&4    &Vehicle&846&18&4     &Vertebral column &310&6&3\\
        WDBC&569&30&2   &WFRN&5456&4&4  &Wine &178&13&3\\
		\hline
	\end{tabular}}
\end{table}

\subsubsection{Experimental setting}
For real datasets, the choice of the parameter value for the algorithms followed the same procedure used in Section \ref{sect:expSetting}, and it was varied between 0.01 to 300 (with step 0.01). Each algorithm was executed on each dataset 100 times, and the cluster centers were initialized randomly at each time. The best result for each algorithm was selected according to its respective objective function. The parameter $\varepsilon$ was set to $10^{-5}$, the maximum number of iterations $T$ was 100, and for each dataset, the number of clusters was set equal to the number of a priori classes. From the fuzzy partition provided by each algorithm, a hard partition was obtained as described in Section \ref{sect:expSetting}. The $HUL$ and $ARI$ indices were considered to %assess
evaluate the misclassification.

\subsubsection{Results}
Table \ref{tab:perfRD} gives the results provided by the algorithms on real datasets and the performance rank of each algorithm (in parenthesis) according to the indices and datasets. % considered. 
Besides, Table \ref{tab:oarRD} presents the average performance ranking of the clustering algorithms according to both indices computed from Table \ref{tab:perfRD}. It is also shown the performance ranking of the clustering algorithms (in parentheses) according to the average performance ranking.

\begin{table}[htbp]
	\centering
	\caption{Algorithms performance for real datasets.}\label{tab:perfRD}
	\resizebox{\textwidth}{!}{
\begin{tabular}	
{|p{2.9cm}|p{2.5cm}|p{2.5cm}|p{2.5cm}|p{2.5cm}|p{2.5cm}|p{2.5cm}|}
\hline
\multicolumn{1}{|c|}{}&\multicolumn{2}{c|}{Automobile}&\multicolumn{2}{c|}{Balance Scale}&\multicolumn{2}{c|}{Haberman}\\
\hline
Algorithms & $HUL$ & $ARI$& $HUL$ & $ARI$& $HUL$ & $ARI$\\
\hline	
FCM-ER-L2&0.5692 (8.0)&0.0941 (10.0)&0.4568 (6.0)&0.1420 (2.0)&0.5892 (8.0)&-0.0026 (10.0)\\ 
FCM-ER-L1&0.7180 (1.0)&0.1474 (4.0)&0.4299 (11.0)&0.0000 (10.0)&0.4990 (11.0)&-0.0011 (8.0)\\ 
AFCM-ER-M&0.4208 (12.0)&0.1073 (6.0)&0.4535 (7.0)&0.0652 (7.0)&0.5844 (9.0)&0.0035 (6.0)\\ 
AFCM-ER-Mk&0.4806 (11.0)&0.0253 (12.0)&0.5741 (1.0)&0.1131 (3.5)&0.6086 (1.0)&0.1596 (3.0)\\ 
AFCM-ER-GP-L2&0.7073 (3.0)&0.1415 (5.0)&0.5121 (2.0)&0.1131 (3.5)&0.6025 (5.0)&-0.0027 (11.0)\\ 
AFCM-ER-GP-L1&0.7084 (2.0)&0.2160 (1.0)&0.4299 (11.0)&0.0000 (10.0)&0.4987 (12.0)&-0.0011 (8.0)\\ 
AFCM-ER-LP-L2&0.5585 (9.0)&0.1060 (7.0)&0.4640 (5.0)&0.0729 (6.0)&0.6084 (2.0)&0.1456 (4.0)\\ 
AFCM-ER-LP-L1&0.6440 (6.0)&0.0869 (11.0)&0.4977 (4.0)&-0.0110 (12.0)&0.6047 (4.0)&0.1725 (2.0)\\ 
AFCM-ER-GS-L2&0.5700 (7.0)&0.0971 (9.0)&0.4474 (8.0)&0.2934 (1.0)&0.5097 (10.0)&-0.0040 (12.0)\\ 
AFCM-ER-GS-L1&0.6790 (4.0)&0.1604 (3.0)&0.4299 (11.0)&0.0000 (10.0)&0.6018 (6.0)&-0.0011 (8.0)\\ 
AFCM-ER-LS-L2&0.5336 (10.0)&0.0975 (8.0)&0.4992 (3.0)&0.1024 (5.0)&0.6072 (3.0)&0.1001 (5.0)\\ 
AFCM-ER-LS-L1&0.6683 (5.0)&0.2020 (2.0)&0.4299 (9.0)&0.0011 (8.0)&0.6005 (7.0)&0.1789 (1.0)\\ 
\hline
\multicolumn{1}{|c|}{}&\multicolumn{2}{c|}{Heart}&\multicolumn{2}{c|}{Image Segmentation}&\multicolumn{2}{c|}{Ionosphere}\\
\hline
FCM-ER-L2&0.5451 (5.0)&0.3487 (8.0)&0.8027 (6.0)&0.4911 (6.0)&0.5358 (12.0)&0.1588 (4.0)\\ 
FCM-ER-L1&0.6910 (1.0)&0.4227 (4.0)&0.8627 (2.0)&0.5257 (1.0)&0.5377 (8.0)&0.0936 (9.0)\\ 
AFCM-ER-M&0.5156 (7.0)&0.1338 (10.0)&0.1425 (12.0)&0.0041 (12.0)&0.5407 (5.0)&0.0085 (12.0)\\ 
AFCM-ER-Mk&0.5027 (12.0)&-0.0036 (11.5)&0.2101 (11.0)&0.0151 (11.0)&0.5713 (1.0)&0.0321 (11.0)\\ 
AFCM-ER-GP-L2&0.5050 (9.0)&0.3576 (7.0)&0.8365 (3.0)&0.4309 (8.0)&0.5384 (7.0)&0.1588 (4.0)\\ 
AFCM-ER-GP-L1&0.5908 (4.0)&0.1807 (9.0)&0.8659 (1.0)&0.5221 (2.0)&0.5418 (4.0)&0.1045 (8.0)\\ 
AFCM-ER-LP-L2&0.5312 (6.0)&0.4131 (5.0)&0.7281 (9.0)&0.3242 (9.0)&0.5360 (11.0)&0.1406 (6.0)\\ 
AFCM-ER-LP-L1&0.5046 (10.0)&-0.0036 (11.5)&0.6862 (10.0)&0.2979 (10.0)&0.5362 (10.0)&0.0870 (10.0)\\ 
AFCM-ER-GS-L2&0.5045 (11.0)&0.4325 (3.0)&0.8008 (7.0)&0.5021 (4.0)&0.5384 (6.0)&0.1588 (4.0)\\ 
AFCM-ER-GS-L1&0.6792 (3.0)&0.4423 (1.5)&0.8300 (4.0)&0.5137 (3.0)&0.5558 (2.0)&0.1243 (7.0)\\ 
AFCM-ER-LS-L2&0.5111 (8.0)&0.3757 (6.0)&0.7937 (8.0)&0.4864 (7.0)&0.5365 (9.0)&0.1634 (2.0)\\ 
AFCM-ER-LS-L1&0.6794 (2.0)&0.4423 (1.5)&0.8291 (5.0)&0.4928 (5.0)&0.5508 (3.0)&0.2092 (1.0)\\ 
\hline
\multicolumn{1}{|c|}{}&\multicolumn{2}{c|}{Iris plants}&\multicolumn{2}{c|}{Mnist}&\multicolumn{2}{c|}{Thyroid}\\
\hline
FCM-ER-L2&0.7524 (11.0)&0.6199 (11.0)&0.5072 (6.0)&0.9564 (3.0)&0.5997 (12.0)&0.3623 (8.0)\\ 
FCM-ER-L1&0.8538 (6.0)&0.6656 (7.5)&0.5660 (1.0)&0.9503 (8.0)&0.8702 (1.0)&0.7324 (1.0)\\ 
AFCM-ER-M&0.9020 (2.0)&0.9037 (1.0)&0.5538 (3.0)&0.9569 (1.5)&0.6106 (11.0)&0.1136 (11.0)\\ 
AFCM-ER-Mk&0.5870 (12.0)&0.2824 (12.0)&0.5538 (2.0)&0.9569 (1.5)&0.7184 (4.0)&0.0931 (12.0)\\ 
AFCM-ER-GP-L2&0.8878 (4.0)&0.8510 (3.0)&0.5022 (11.0)&0.9561 (4.5)&0.6547 (8.0)&0.5038 (4.0)\\ 
AFCM-ER-GP-L1&0.9481 (1.0)&0.8857 (2.0)&0.5082 (5.0)&0.9532 (6.0)&0.8586 (2.0)&0.7167 (2.0)\\ 
AFCM-ER-LP-L2&0.7811 (8.0)&0.6882 (5.0)&0.5168 (4.0)&0.0513 (9.0)&0.6892 (5.0)&0.6931 (3.0)\\ 
AFCM-ER-LP-L1&0.8931 (3.0)&0.8019 (4.0)&0.5033 (7.0)&0.0169 (10.0)&0.7287 (3.0)&0.4148 (7.0)\\ 
AFCM-ER-GS-L2&0.7607 (9.0)&0.6303 (9.5)&0.5022 (10.0)&0.9561 (4.5)&0.6611 (7.0)&0.4731 (5.0)\\ 
AFCM-ER-GS-L1&0.8535 (7.0)&0.6656 (7.5)&0.5028 (8.0)&0.9511 (7.0)&0.6830 (6.0)&0.3337 (9.0)\\ 
AFCM-ER-LS-L2&0.7549 (10.0)&0.6303 (9.5)&0.5021 (12.0)&0.0041 (12.0)&0.6215 (9.0)&0.4392 (6.0)\\ 
AFCM-ER-LS-L1&0.8584 (5.0)&0.6757 (6.0)&0.5023 (9.0)&0.0096 (11.0)&0.6208 (10.0)&0.1349 (10.0)\\ 
\hline
\end{tabular}}
\end{table}

\begin{table}[!htb]
\ContinuedFloat
\centering
\caption{Algorithms performance for real datasets (continued).}\label{tab:perfRD2}
\resizebox{\textwidth}{!}{
\begin{tabular}
	{|p{2.9cm}|p{2.4cm}|p{2.4cm}|p{2.4cm}|p{2.4cm}|p{2.4cm}|p{2.4cm}|}
	\hline
	\multicolumn{1}{|c|}{}&\multicolumn{2}{c|}{UKM}&\multicolumn{2}{c|}{Vehicle}&\multicolumn{2}{c|}{Vertebral column}\\
\hline
Algorithms & $HUL$ & $ARI$& $HUL$ & $ARI$& $HUL$ & $ARI$\\
\hline	
FCM-ER-L2&0.4900 (10.0)&0.1647 (11.0)&0.5660 (8.0)&0.0736 (11.0)&0.5749 (6.0)&0.2993 (9.0)\\ 
FCM-ER-L1&0.6095 (6.0)&0.1935 (9.0)&0.6503 (3.0)&0.1124 (7.0)&0.6016 (5.0)&0.3330 (8.0)\\ 
AFCM-ER-M&0.7151 (2.0)&0.3530 (1.0)&0.6653 (2.0)&0.1399 (2.0)&0.3811 (12.0)&0.2479 (10.0)\\ 
AFCM-ER-Mk&0.4269 (12.0)&0.0176 (12.0)&0.4362 (12.0)&0.0951 (8.0)&0.4658 (10.0)&0.0045 (12.0)\\ 
AFCM-ER-GP-L2&0.6937 (3.0)&0.3024 (4.0)&0.5535 (9.0)&0.0755 (10.0)&0.5640 (7.0)&0.3344 (7.0)\\ 
AFCM-ER-GP-L1&0.7373 (1.0)&0.3499 (2.0)&0.6435 (4.0)&0.1269 (5.0)&0.6131 (4.0)&0.3416 (4.0)\\ 
AFCM-ER-LP-L2&0.5459 (7.0)&0.2750 (6.0)&0.5730 (7.0)&0.1299 (4.0)&0.5595 (9.0)&0.3526 (3.0)\\ 
AFCM-ER-LP-L1&0.6633 (4.0)&0.3264 (3.0)&0.6695 (1.0)&0.1425 (1.0)&0.6204 (3.0)&0.3398 (5.0)\\ 
AFCM-ER-GS-L2&0.4728 (11.0)&0.1679 (10.0)&0.5416 (10.0)&0.0698 (12.0)&0.5628 (8.0)&0.3368 (6.0)\\ 
AFCM-ER-GS-L1&0.5236 (8.0)&0.2327 (7.0)&0.5043 (11.0)&0.1135 (6.0)&0.6984 (1.0)&0.4365 (1.0)\\ 
AFCM-ER-LS-L2&0.6180 (5.0)&0.2858 (5.0)&0.5865 (6.0)&0.0792 (9.0)&0.4628 (11.0)&0.0269 (11.0)\\ 
AFCM-ER-LS-L1&0.5107 (9.0)&0.2319 (8.0)&0.6047 (5.0)&0.1353 (3.0)&0.6417 (2.0)&0.3558 (2.0)\\ 
\hline
\multicolumn{1}{|c|}{}&\multicolumn{2}{c|}{WDBC}&\multicolumn{2}{c|}{WFRN}&\multicolumn{2}{c|}{Wine}\\
\hline
FCM-ER-L2&0.5677 (5.0)&0.6895 (6.0)&0.4972 (10.0)&0.1581 (10.0)&0.5550 (9.0)&0.5953 (8.0)\\ 
FCM-ER-L1&0.8032 (2.0)&0.7551 (4.0)&0.5527 (5.0)&0.1591 (9.0)&0.7601 (4.0)&0.8804 (3.0)\\ 
AFCM-ER-M&0.5314 (12.0)&0.0100 (12.0)&0.5893 (1.0)&0.1451 (11.0)&0.9765 (1.0)&0.9651 (1.0)\\ 
AFCM-ER-Mk&0.5451 (7.0)&0.1578 (10.0)&0.5381 (6.0)&0.2165 (4.0)&0.5014 (12.0)&0.2619 (12.0)\\ 
AFCM-ER-GP-L2&0.5318 (11.0)&0.6954 (5.0)&0.5297 (8.0)&0.1608 (8.0)&0.6149 (7.0)&0.6316 (7.0)\\ 
AFCM-ER-GP-L1&0.8207 (1.0)&0.7736 (3.0)&0.5252 (9.0)&0.2777 (3.0)&0.7730 (3.0)&0.8804 (3.0)\\ 
AFCM-ER-LP-L2&0.5342 (9.0)&0.4513 (8.0)&0.5548 (3.0)&0.3419 (2.0)&0.6505 (6.0)&0.5114 (10.0)\\ 
AFCM-ER-LP-L1&0.7963 (3.0)&0.7925 (1.0)&0.5530 (4.0)&0.3543 (1.0)&0.7905 (2.0)&0.8804 (3.0)\\ 
AFCM-ER-GS-L2&0.7121 (4.0)&0.7802 (2.0)&0.4847 (12.0)&0.1887 (5.0)&0.5147 (11.0)&0.5083 (11.0)\\ 
AFCM-ER-GS-L1&0.5620 (6.0)&0.5057 (7.0)&0.5613 (2.0)&0.1814 (7.0)&0.5869 (8.0)&0.7185 (6.0)\\ 
AFCM-ER-LS-L2&0.5332 (10.0)&0.1390 (11.0)&0.4850 (11.0)&0.1864 (6.0)&0.5213 (10.0)&0.5188 (9.0)\\ 
AFCM-ER-LS-L1&0.5388 (8.0)&0.4292 (9.0)&0.5369 (7.0)&0.0717 (12.0)&0.6601 (5.0)&0.8185 (5.0)\\
\hline
\end{tabular}}
\end{table}

\begin{table}[!htb]
\caption{Average performance ranking in real datasets.\label{tab:oarRD}}
\resizebox{\textwidth}{!}{
\begin{tabular}{|l|c|c|c|c|c|c|}
\hline
Index &FCM-ER-L2&FCM-ER-L1&AFCM-ER-M&AFCM-ER-Mk&AFCM-ER-GP-L2&AFCM-ER-GP-L1\\
  \hline	   
HUL&8.1 (10.0)&4.5 (2.0)&6.5 (7.0)&7.6 (9.0)&6.5 (6.0)&4.3 (1.0)\\
ARI&7.8 (11.0)&6.2 (7.0)&6.9 (9.0)&9.0 (12.0)&6.1 (5.0)&4.5 (1.0)\\
\hline

Index &AFCM-ER-LP-L2&AFCM-ER-LP-L1&AFCM-ER-GS-L2&AFCM-ER-GS-L1&AFCM-ER-LS-L2&AFCM-ER-LS-L1\\
  \hline	   
HUL&6.7 (8.0)&4.9 (3.0)&8.7 (12.0)&5.8 (4.0)&8.3 (11.0)&6.1 (5.0)\\
ARI&5.8 (3.0)&6.1 (6.0)&6.5 (8.0)&6.0 (4.0)&(10.0)&(2.0)\\
   \hline
\end{tabular}}
\end{table}

Table \ref{tab:oarRD} shows that the algorithm AFCM-ER-GP-L1 presents the best average performance ranking for real datasets whatever the considered index. Moreover, the algorithms FCM-ER-L1 and AFCM-ER-LP-L1 achieved the second and third best average ranking, respectively, for the index $HUL$ and, AFCM-ER-LS-L1 and AFCM-ER-LP-L2 for $ARI$. Finally, the algorithm AFCM-ER-GS-L2 obtained the worst results for $HUL$ and AFCM-ER-Mk for $ARI$.

\subsection{Brodatz texture images for segmentation}

Image segmentation is a challenging but essential task in many image analysis or computer vision applications. This section presents different experiments with texture images to evaluate the performance and robustness of the proposed algorithms in texture image segmentation with and without noise. The used images were taken from the Brodatz texture dataset \cite{brodtimage} and they are employed to demonstrate the performance of the algorithms for datasets with high dimensionality. The texture images without and with Gaussian noise (mean 0, variance 0.3) are shown in Figure \ref{img:texturesImages}. Moreover, in Fig. \ref{img:texturesImages} (c), (f) and (i), we can see the corresponding ideal segmentation results used as a reference to determine the segmentation performance quantitatively. The images are synthesized with different kinds of texture images: two-textural image (D4 and D49), five-textural image (D21, D22, D49, D53, and D55), and seven-textural image (D3, D6, D21, D49, D53, D56, and D93).

\begin{figure}[!ht]
	\centering
	\subfloat[2-textural image]	{\includegraphics[width=0.13 \textwidth]{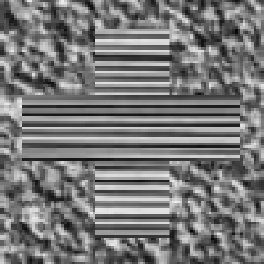}}
	\quad
	\subfloat[2-textural image with Gaussian noise]	{\includegraphics[width=0.13\textwidth]{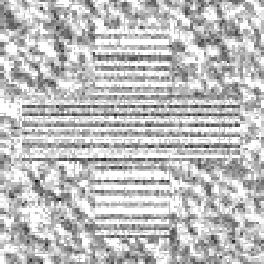}}
	\quad
	\subfloat[Ideal segmentation of 2-textural image]	{\includegraphics[width=0.16\textwidth]{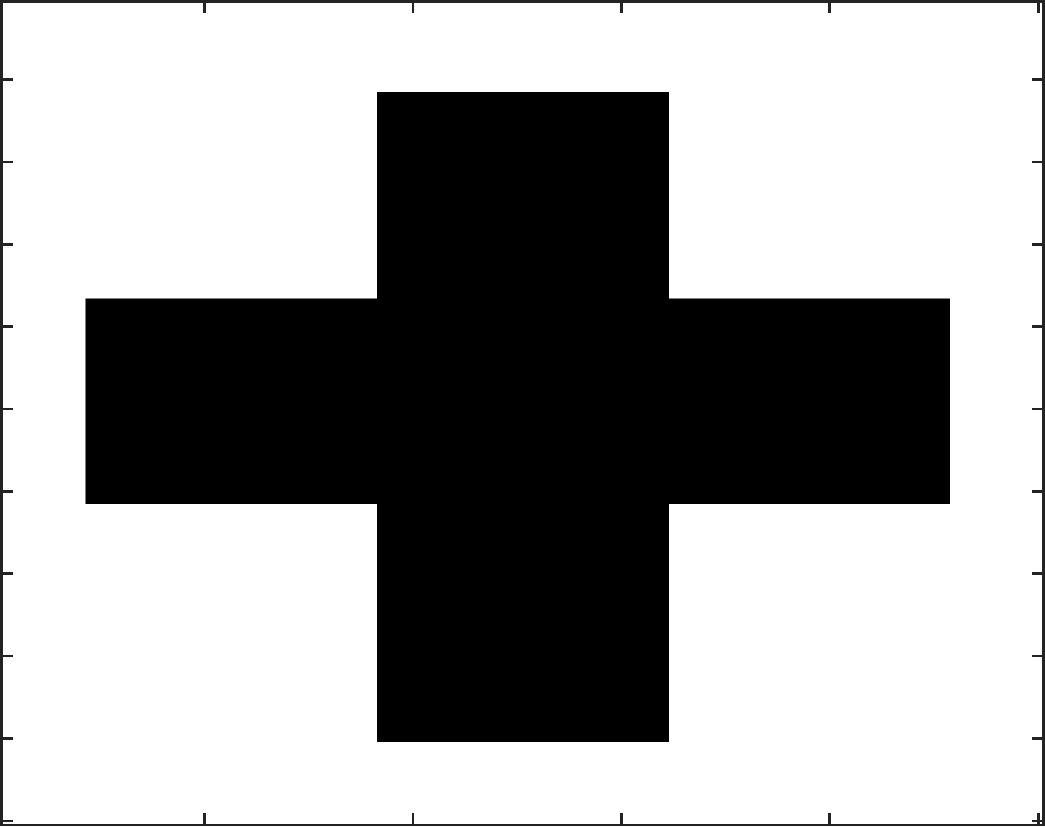}}
	\quad
	\subfloat[5-textural image] {\includegraphics[width=0.13\textwidth]{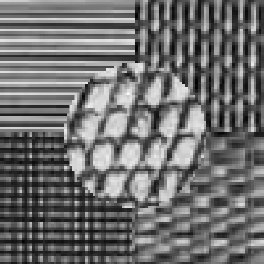}}{}
	\quad
	\subfloat[5-textural image with Gaussian noise] {\includegraphics[width=0.13\textwidth]{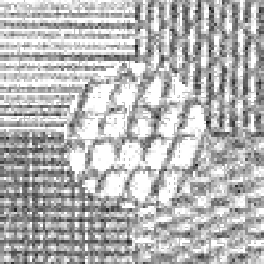}}
	\quad
	\subfloat[Ideal segmentation of 5-textural image] {\includegraphics[width=0.16\textwidth]{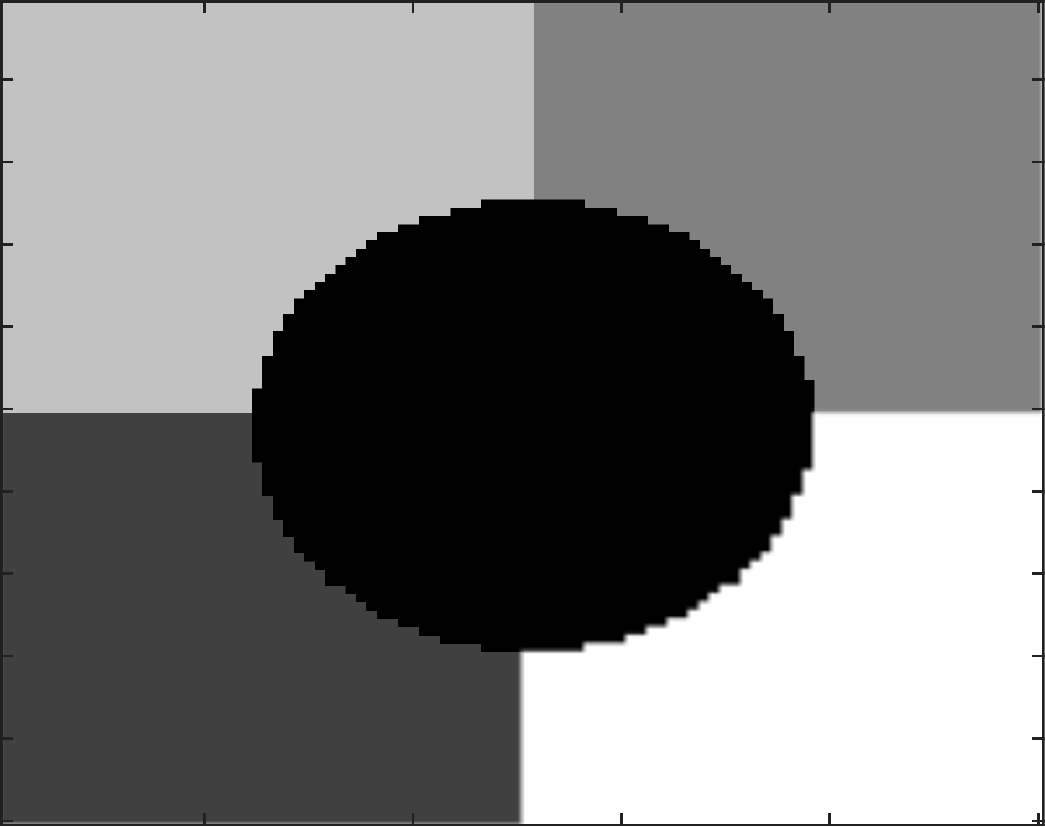}}
	\quad
	\subfloat[7-textural image] {\includegraphics[width=0.13\textwidth]{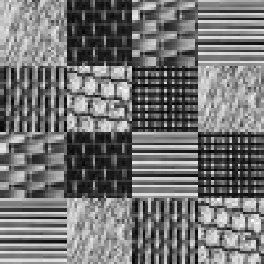}}{}
	\quad
	\subfloat[7-textural image with Gaussian noise] {\includegraphics[width=0.13\textwidth]{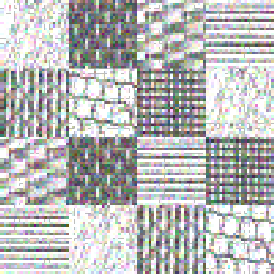}}{}
	\quad	
	\subfloat[Ideal segmentation of 7-textural image] {\includegraphics[width=0.16\textwidth]{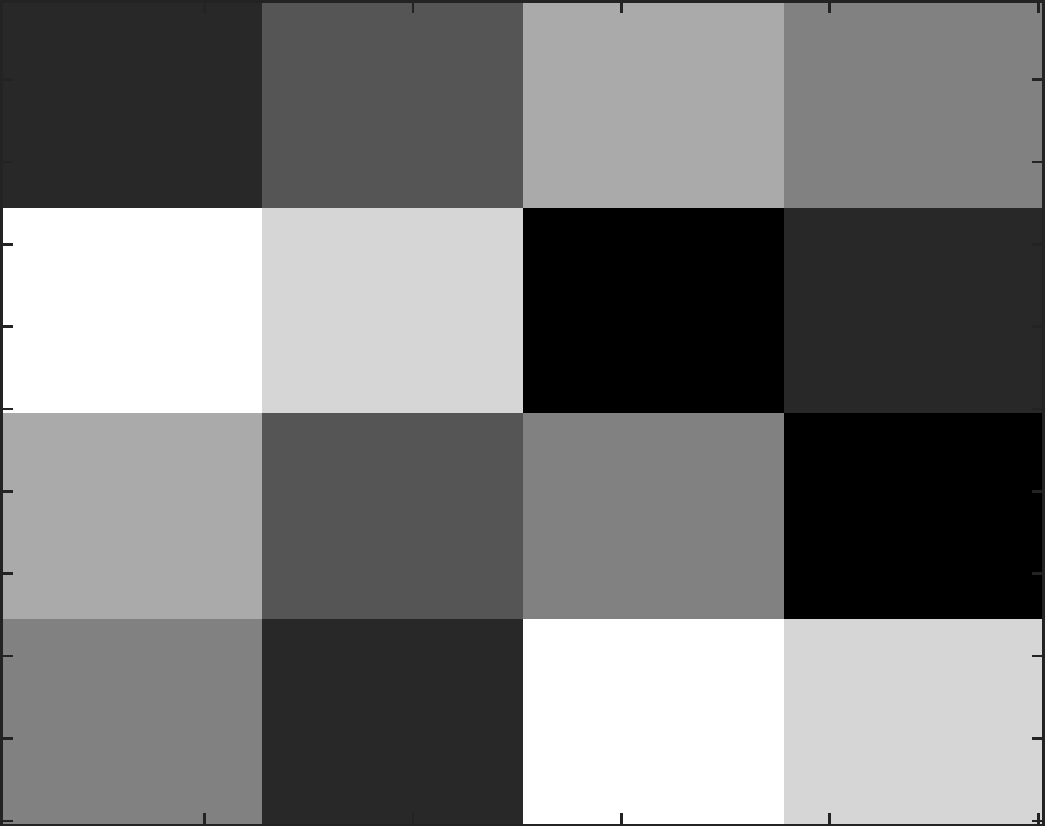}}{}	
	\caption{Two, five and seven-textural images. (a), (d) and (g) show the original 2, 5 and 7 textural images respectively. (b), (e) and (h) present the image with Gaussian noises $N(0,0.3)$ and (c), (f) and (i) are the ideal segmentation result of the original image.}
	\label{img:texturesImages}
\end{figure}

\subsubsection{Experimental setting}

To extract the features of the texture images was used the Gabor filter as in Ref. \cite{kyrki2004simple}. A filter bank with six orientations (every $\ang{30}$ ) and five frequencies starting from 0.4 was created by extracting 30-dimensional features for every pixel of the $100\times 100$ texture images filtered by the filter bank. After extracting the texture features, we used the algorithms to segment each texture image. The choice of the parameter values was obtained as in previous sections, and the $T_u$ value varied between 0.1 to 100 (with step 0.1). The algorithms were executed on each dataset 10 times, and we selected the best result according to their respective objective function.

\subsubsection{Results}

Tables \ref{tab:imageResults} and \ref{tab:imageResultsN} present the clustering results obtained for the algorithms on images without and with noise respectively according to $HUL$ and $ARI$. Also, Figures \ref{img:ResultsTextImagesCross}, \ref{img:ResultsTextImagesCircle} and \ref{img:ResultsTextImagesImage7} show the unsupervised segmentation results using each algorithms on 2, 4 and 7 textural images respectively without and with Gaussian noise.

\begin{table}[!htb]
	\caption{$HUL$ and $ARI$ index for 2, 5 and 7 textural image without noise.\label{tab:imageResults}}
	\centering
	\resizebox{\textwidth}{!}{
		\begin{tabular}%{|l|c|c|c|c|c|c|}
		{|p{2.9cm}|p{1.9cm}|p{1.9cm}|p{1.9cm}|p{1.9cm}|p{1.9cm}|p{1.9cm}|}
		\hline
		& \multicolumn{2}{c|}{2-Textural Image}&\multicolumn{2}{c|}{5-Textural Image}&\multicolumn{2}{c|}{7-Textural Image}\\
		 \hline
		 Algorithms&$HUL$&$ARI$&$HUL$&$ARI$&$HUL$&$ARI$\\
		 \hline
FCM-ER-L2		&0.5547 (6)&0.7438 (4)&0.6484(5)&0.7572 (3)&0.7577 (3)&0.4716 (5)\\
FCM-ER-L1		&0.7853 (1)&0.7222 (7)&0.5780 (6)&0.7502 (4)&0.7225 (6)&0.5191 (2)\\
AFCM-ER-M		&0.5398 (10)&0.0079 (10)&0.2990 (12)&0.1906 (10)&0.8765 (1)&0.5231 (1)\\
AFCM-ER-Mk		&0.5365 (11)&-0.0236 (11)&0.5558 (8)&0.1792 (11)&0.6561 (10)&0.2157 (10)\\
AFCM-ER-GP-L2	&0.5710 (5)&0.7438 (4)&0.6873 (3)&0.7740 (2)&0.8063 (2)&0.4981 (4)\\
AFCM-ER-GP-L1	&0.6980 (3)&0.7571 (2)&0.6487 (4)&0.7826 (1)&0.6939 (7)&0.5051 (3)\\
AFCM-ER-LP-L2	&0.5506 (7)&0.7266 (6)&0.6964 (2)&0.4811 (8)&0.6581 (9)&0.2632 (9)\\
AFCM-ER-LP-L1	&0.7005 (2)&0.7740 (1)&0.5580 (7)&0.6054 (6)&0.6862 (8)&0.4319 (8)\\
AFCM-ER-GS-L2	&0.5399 (9)&0.7438 (4)&0.5160 (10)&0.4860 (7)&0.7403 (5)&0.4666 (6)\\
AFCM-ER-GS-L1	&0.5925 (4)&0.7114 (8)&0.5053 (11)&0.1230 (12)&0.6003 (12)&0.1081 (11)\\
AFCM-ER-LS-L2	&0.5359 (12)&-0.0430 (12)&0.7629 (1)&0.4379 (9)&0.7404 (4)&0.4353 (7)\\
AFCM-ER-LS-L1	&0.5420 (8)&0.1129 (9)&0.5368 (9)&0.6319 (5)&0.6267 (11)&0.0743 (12)\\
			\hline
		\end{tabular}}
	\end{table}

	\begin{table}[!htb]
		\caption{$HUL$ and $ARI$ index for 2, 5 and 7 textural image with Gaussian noise.\label{tab:imageResultsN}}
		\centering
		\resizebox{\textwidth}{!}{
			\begin{tabular}%{|l|c|c|c|c|c|c|}
			{|p{2.9cm}|p{1.9cm}|p{1.9cm}|p{1.9cm}|p{1.9cm}|p{1.9cm}|p{1.9cm}|}
				\hline
				& \multicolumn{2}{c|}{2-Textural Image}&\multicolumn{2}{c|}{5-Textural Image}&\multicolumn{2}{c|}{7-Textural Image}\\
				\hline
				Algorithms&$HUL$&$ARI$&$HUL$&$ARI$&$HUL$&$ARI$\\
				\hline
FCM-ER-L2		&0.5694 (4)&0.5863 (7)&0.4795 (9)&0.4523 (6)&0.6437 (4)&0.2894 (6)\\
FCM-ER-L1		&0.630 (2)&0.6328 (2)&0.6164 (5)&0.5925 (4)&0.5830 (8)&0.3274 (2)\\
AFCM-ER-M		&0.5392 (8)&0.0434 (8)&0.2021 (12)&0.0678 (12)&0.1484 (12)&0.0390 (12)\\
AFCM-ER-Mk		&0.5195 (12)&0.0015 (9)&0.4318 (10)&0.1805 (10)&0.5412 (9)&0.1373 (9)\\
AFCM-ER-GP-L2	&0.5403 (5)&0.5866 (5.5)&0.7592 (1)&0.6730 (1)&0.6964 (1)&0.3219 (3)\\
AFCM-ER-GP-L1	&0.6581 (1)&0.6427 (1)&0.6589 (4)&0.6515 (2)&0.6808 (2)&0.4451 (1)\\
AFCM-ER-LP-L2	&0.5344 (10)&-0.0015 (11)&0.6910 (3)&0.3515 (9)&0.6139 (7)&0.2076 (8)\\
AFCM-ER-LP-L1	&0.5759 (3)&0.5960 (4)&0.7028 (2)&0.6102 (3)&0.6608 (3)&0.3000 (4)\\
AFCM-ER-GS-L2	&0.5399 (7)&0.5866 (5.5)&0.4844 (8)&0.4692 (5)&0.4767 (10)&0.2912 (5)\\
AFCM-ER-GS-L1	&0.5403 (6)&0.6045 (3)&0.5269 (7)&0.1018 (11)&0.6418 (5)&0.0814 (10)\\
AFCM-ER-LS-L2	&0.5361 (9)&-0.0375 (12)&0.4309 (11)&0.4246 (8)&0.4745 (11)&0.2873 (7)\\
AFCM-ER-LS-L1	&0.5215 (11)&-0.0001 (10)&0.5923 (6)&0.4499 (7)&0.6264 (6)&0.0592 (11)\\
				\hline
			\end{tabular}}
		\end{table}
	
	\begin{figure}[!htb]
		\centering
		{\includegraphics[width=0.14\textwidth]{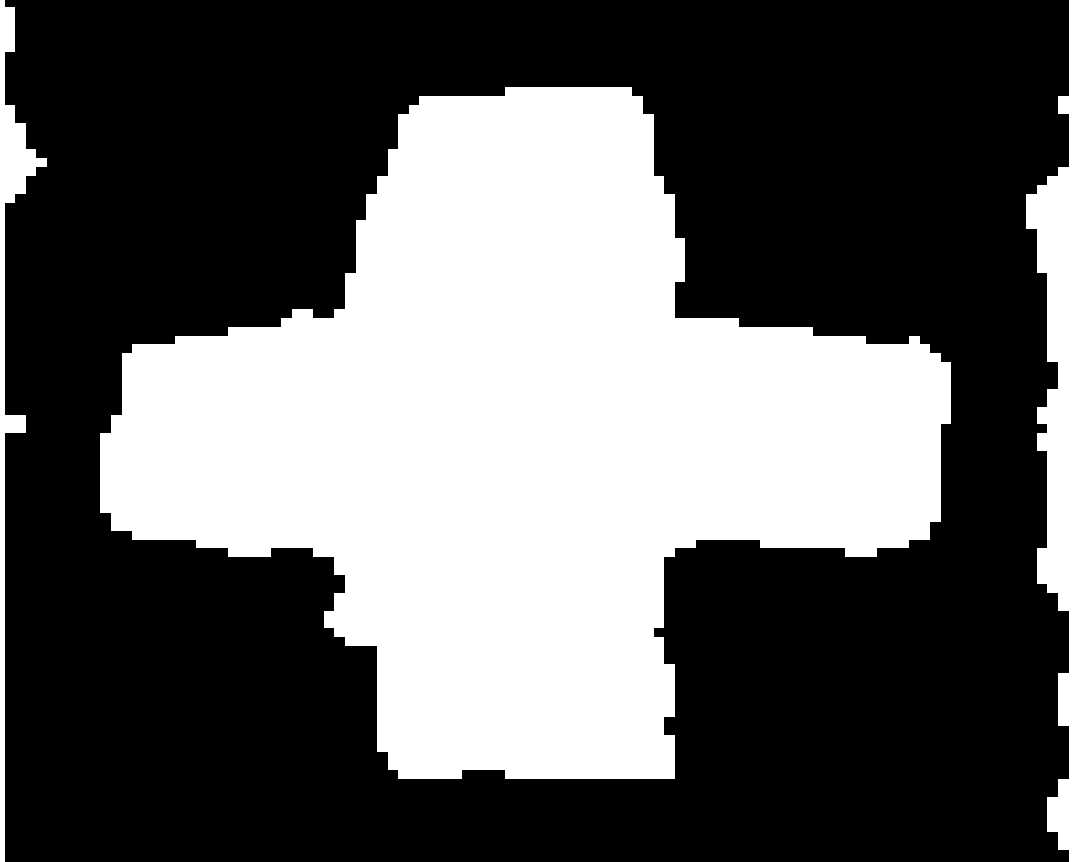}}
		\quad
		{\includegraphics[width=0.14\textwidth]{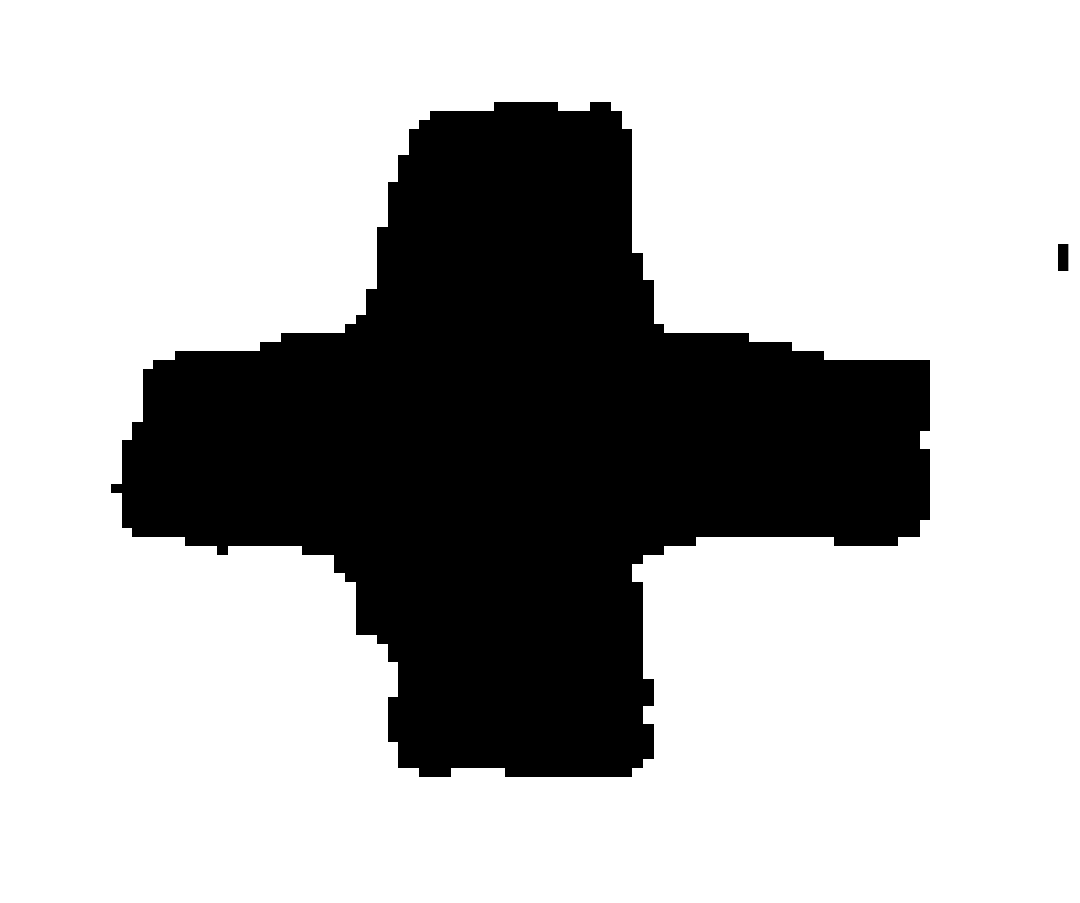}}
		\quad
		{\includegraphics[width=0.14\textwidth]{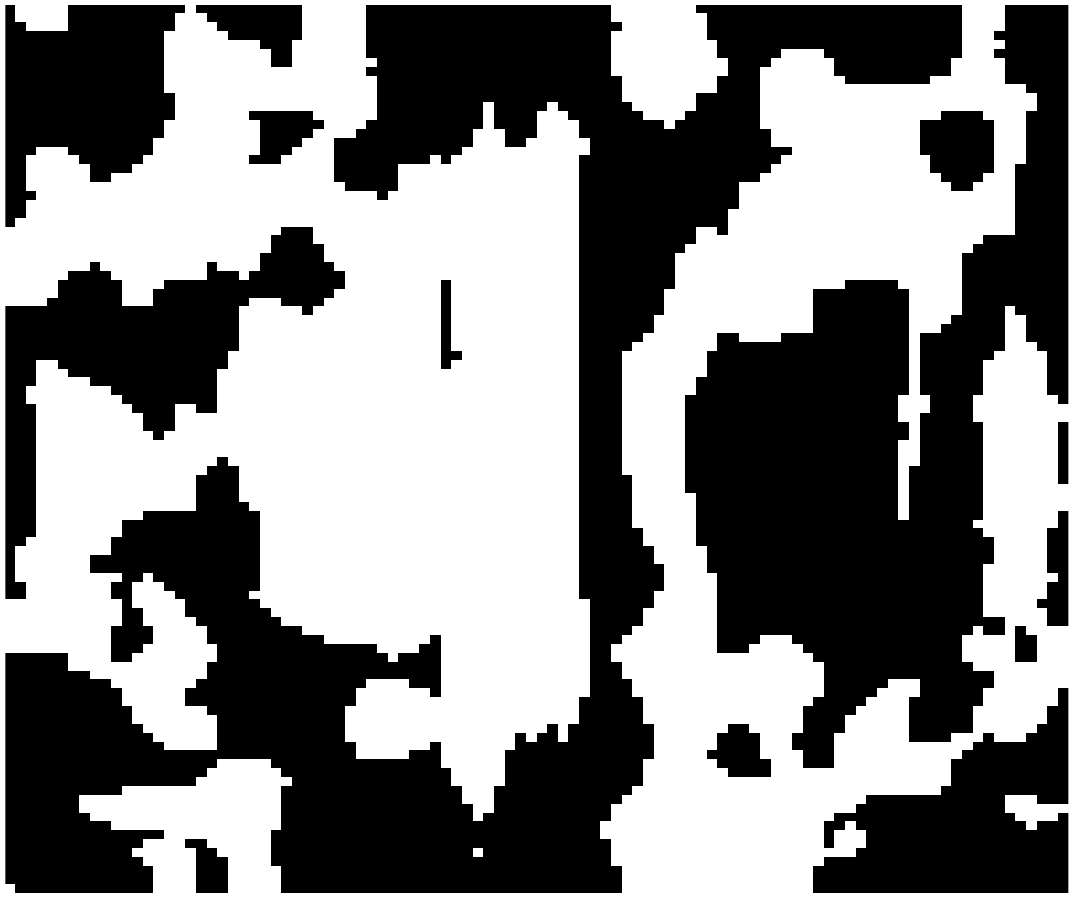}}
		\quad
		{\includegraphics[width=0.14\textwidth]{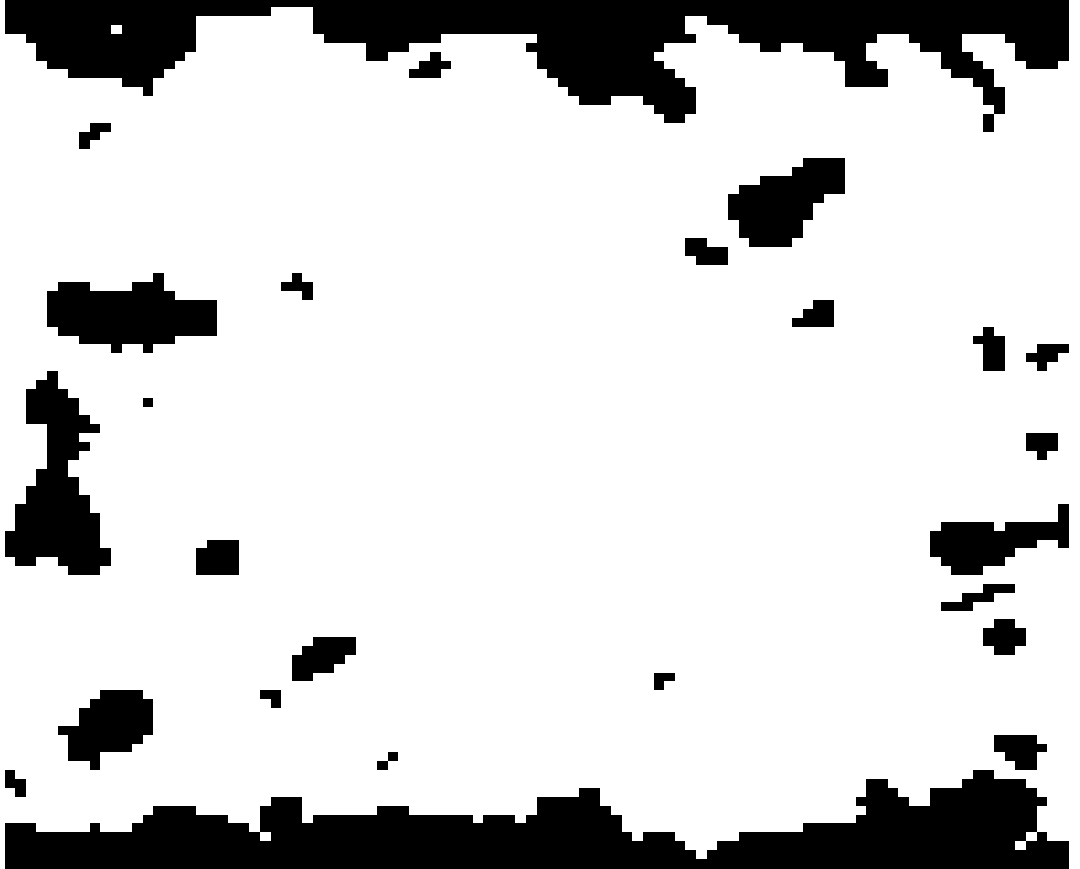}}
		\quad
		{\includegraphics[width=0.14\textwidth]{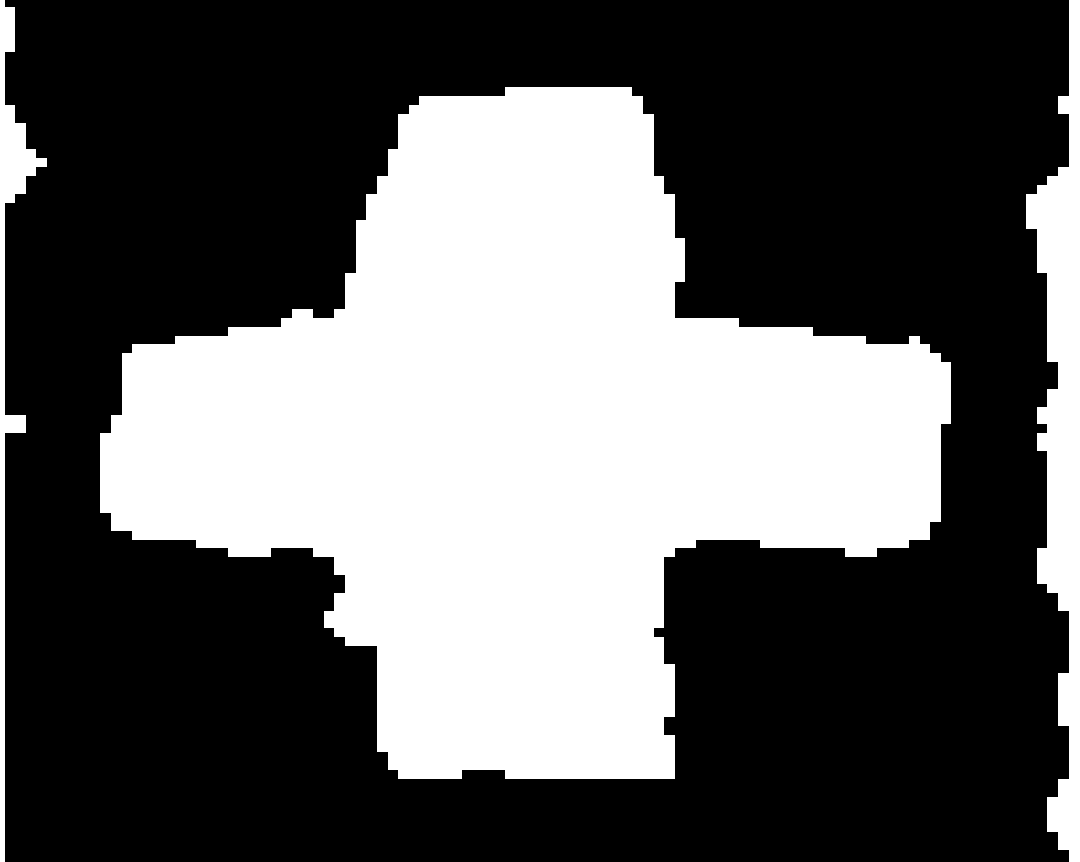}}
		\quad
		{\includegraphics[width=0.14\textwidth]{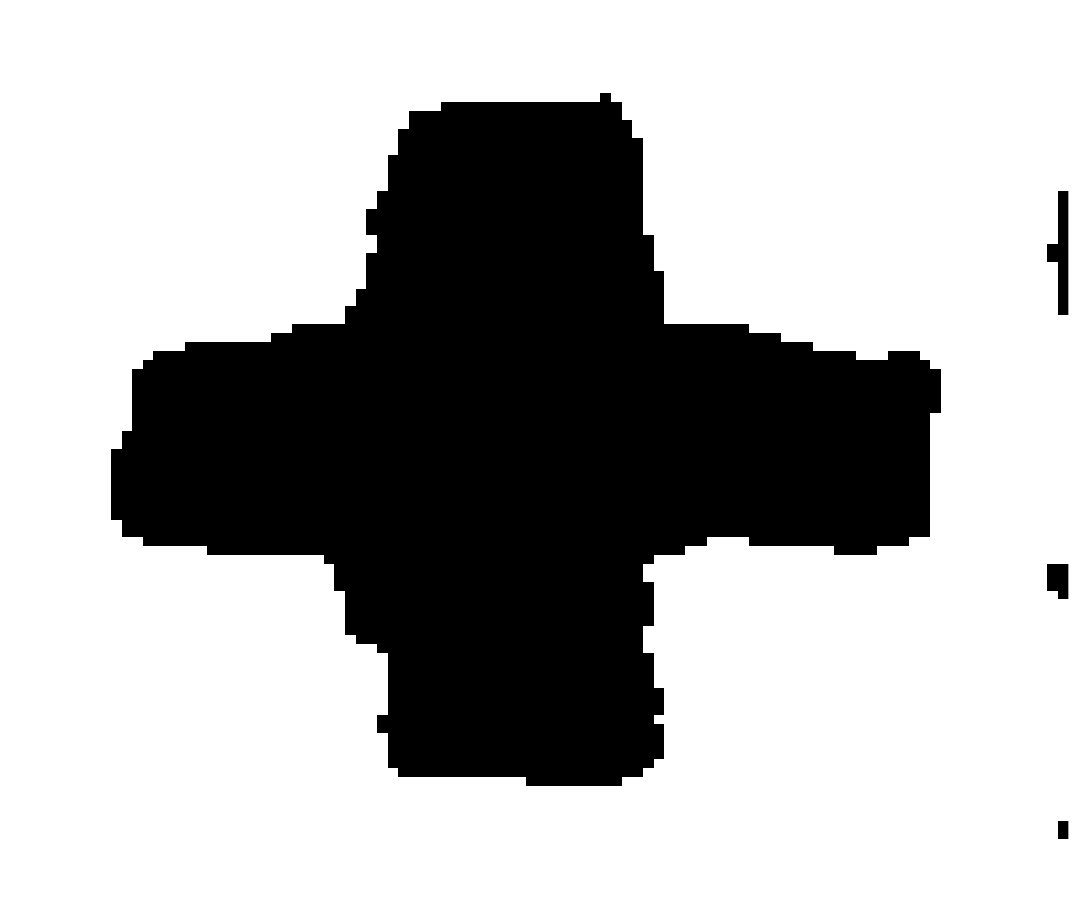}}
		\quad
		\subfloat[FCM-ER-L2] {\includegraphics[width=0.14\textwidth]{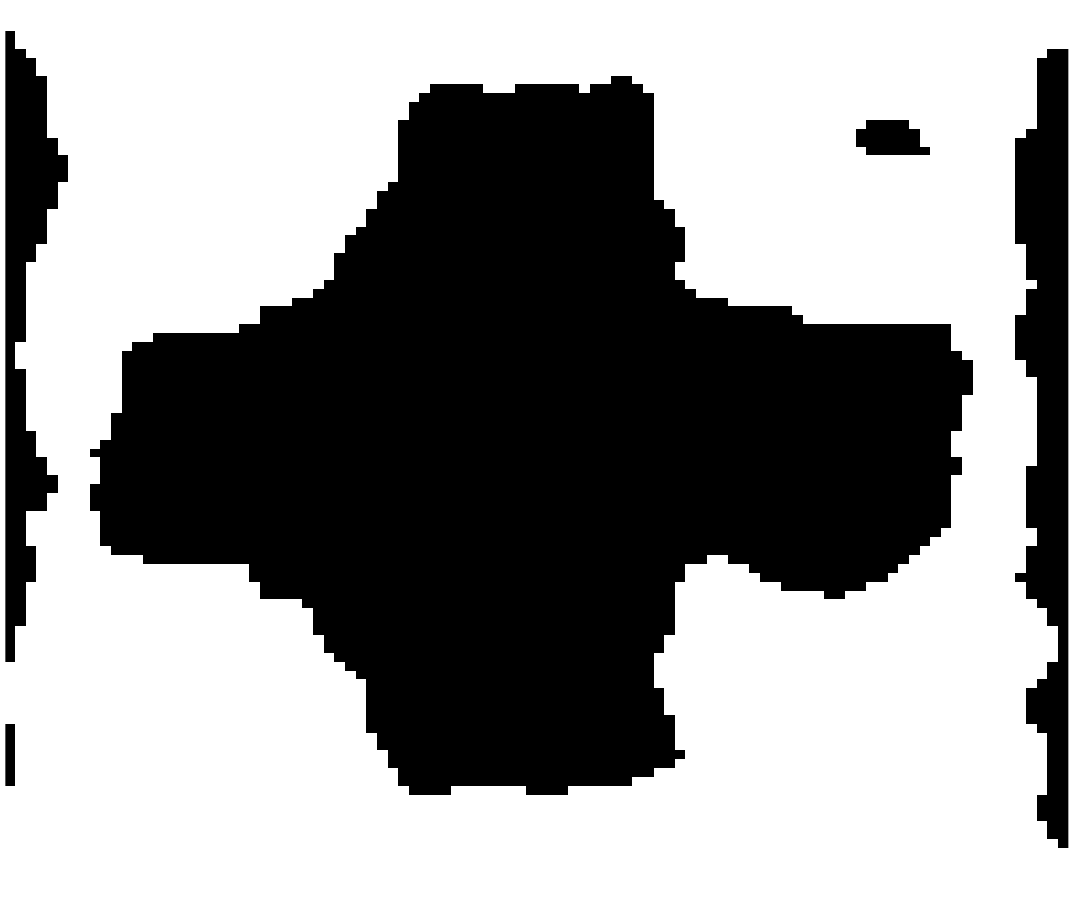}}
		\quad
		\subfloat[FCM-ER-L1] {\includegraphics[width=0.14\textwidth]{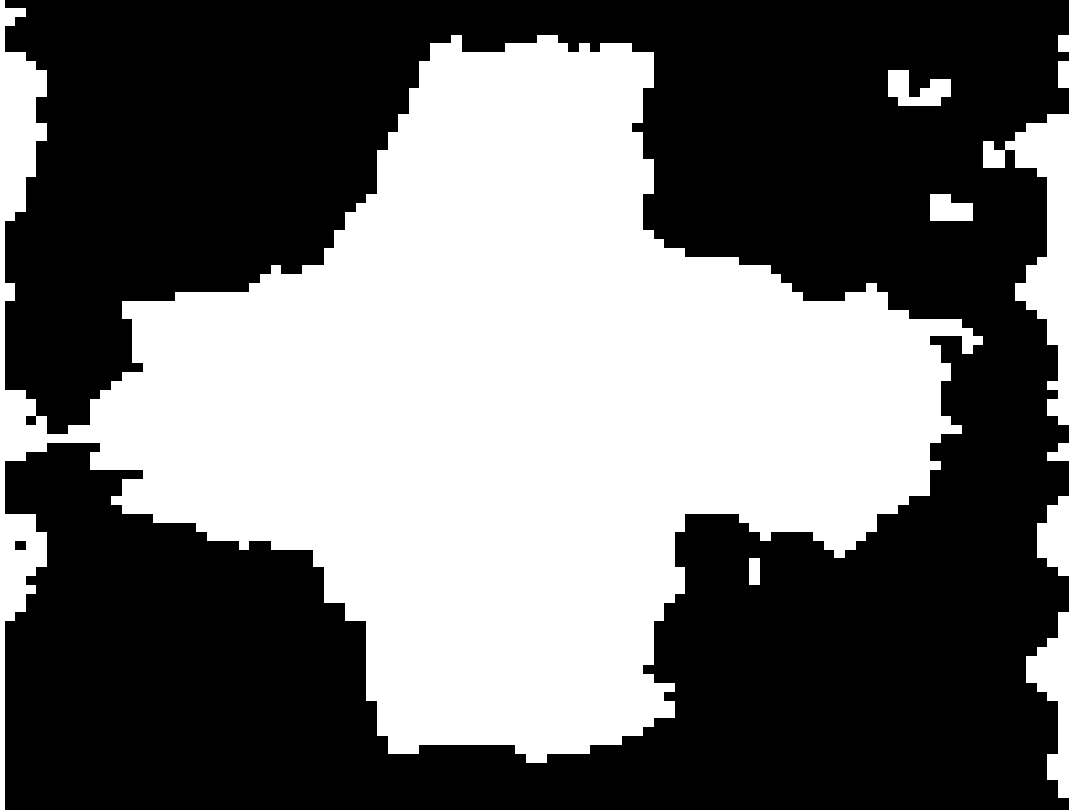}}
		\quad
		\subfloat[AFCM-ER-M] {\includegraphics[width=0.14\textwidth]{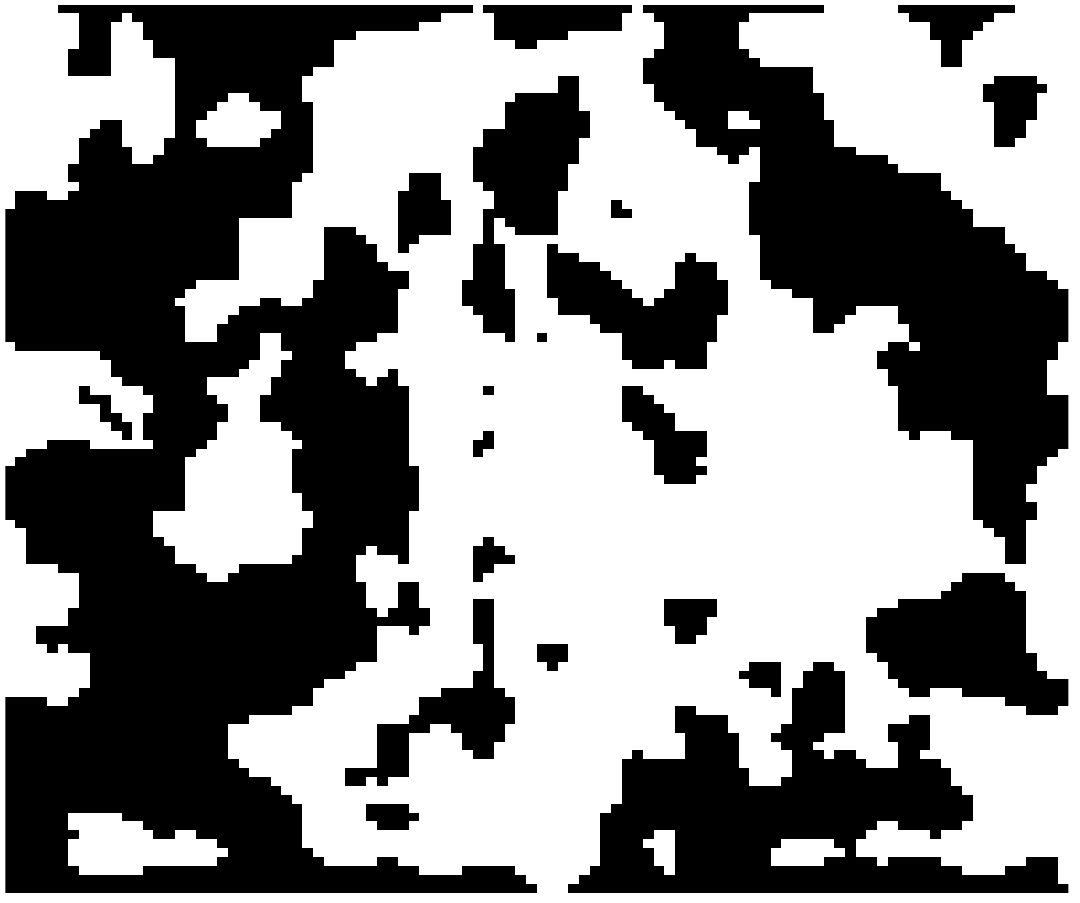}}
		\quad
		\subfloat[AFCM-ER-Mk] {\includegraphics[width=0.14\textwidth]{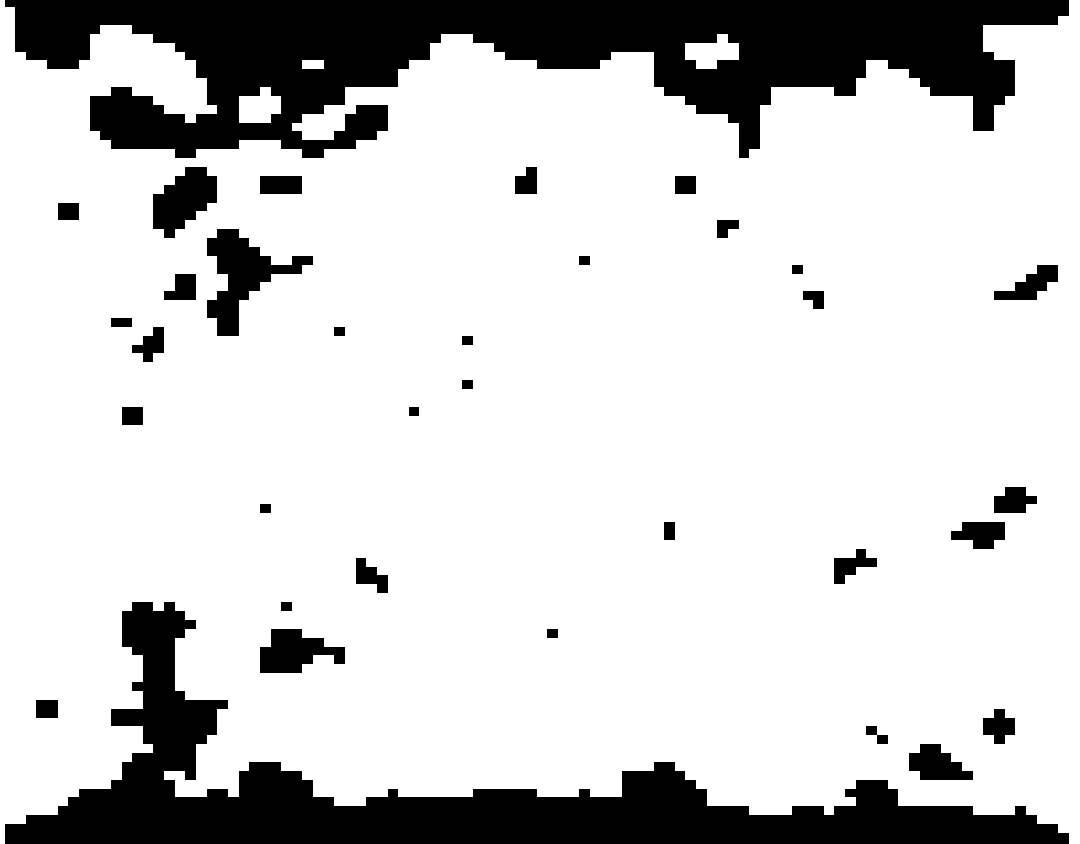}}
		\quad
		\subfloat[AFCM-ER-GP-L2] {\includegraphics[width=0.14\textwidth]{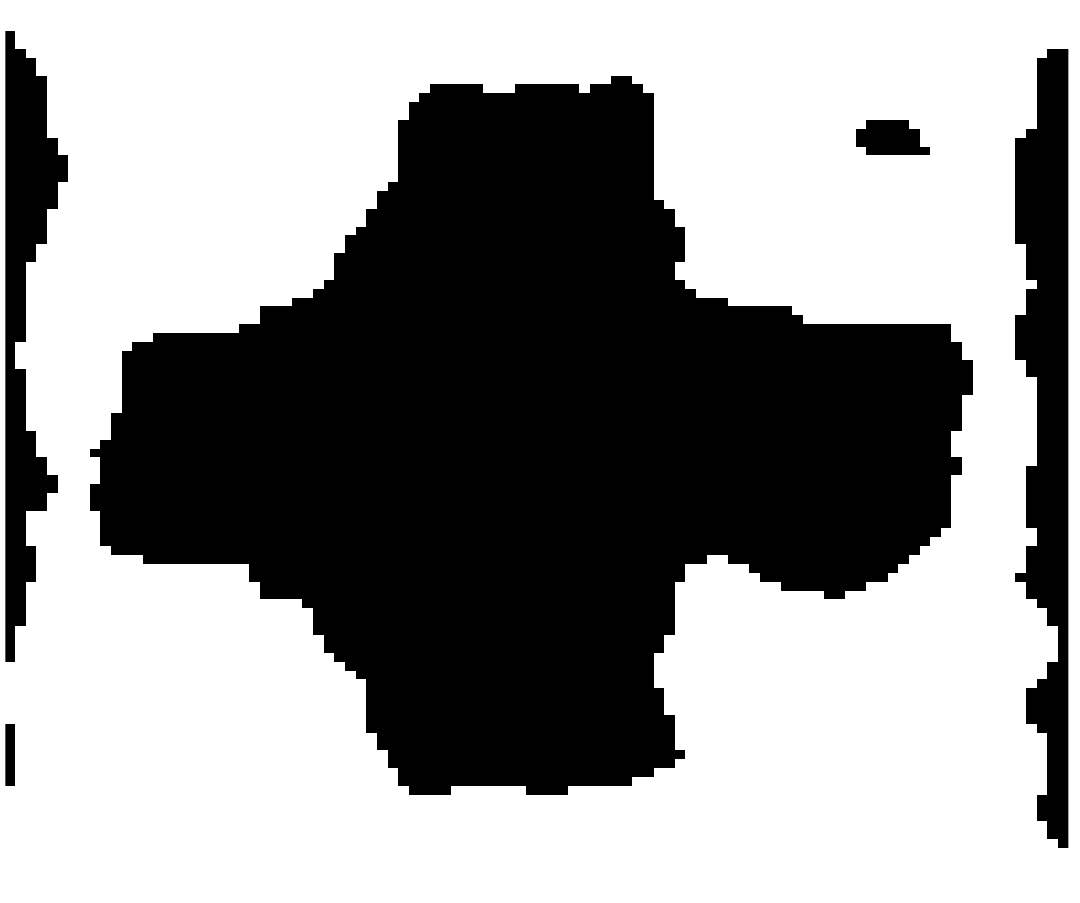}}
		\quad
		\subfloat[AFCM-ER-GP-L1] {\includegraphics[width=0.14\textwidth]{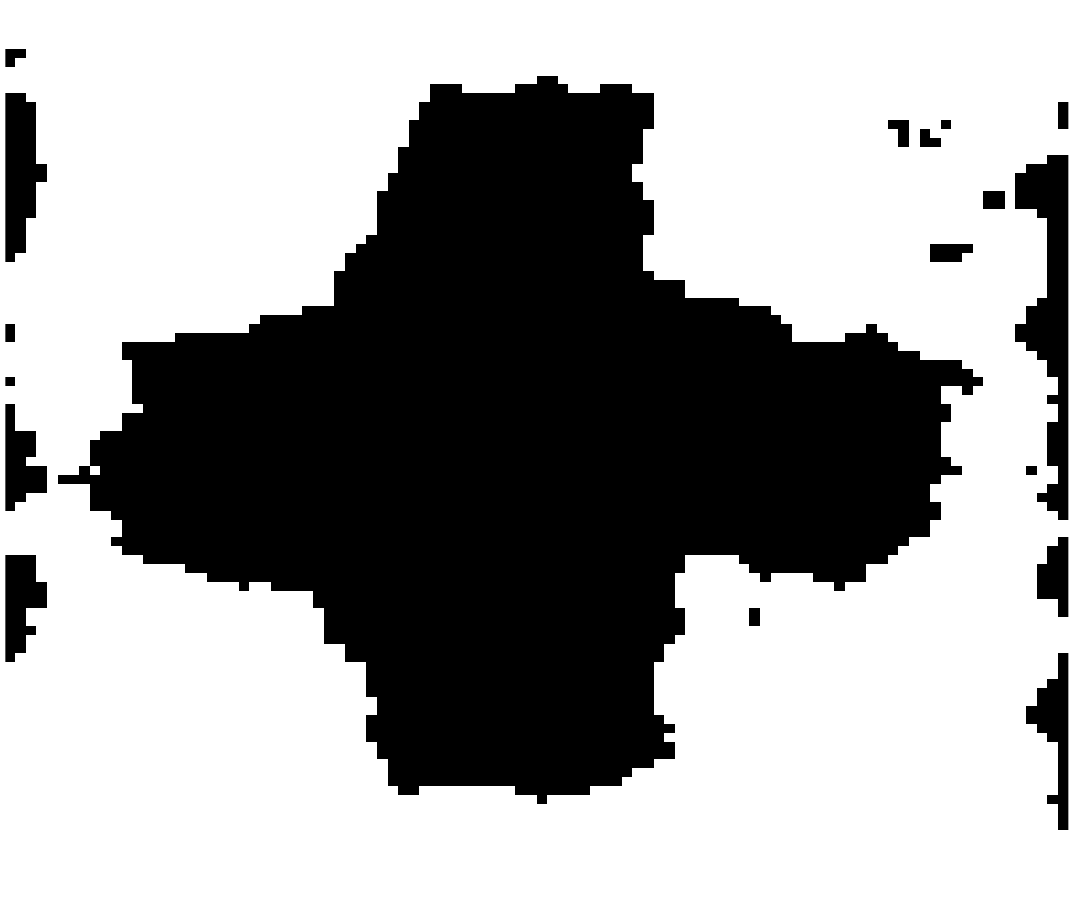}}
		\quad
		{\includegraphics[width=0.14\textwidth]{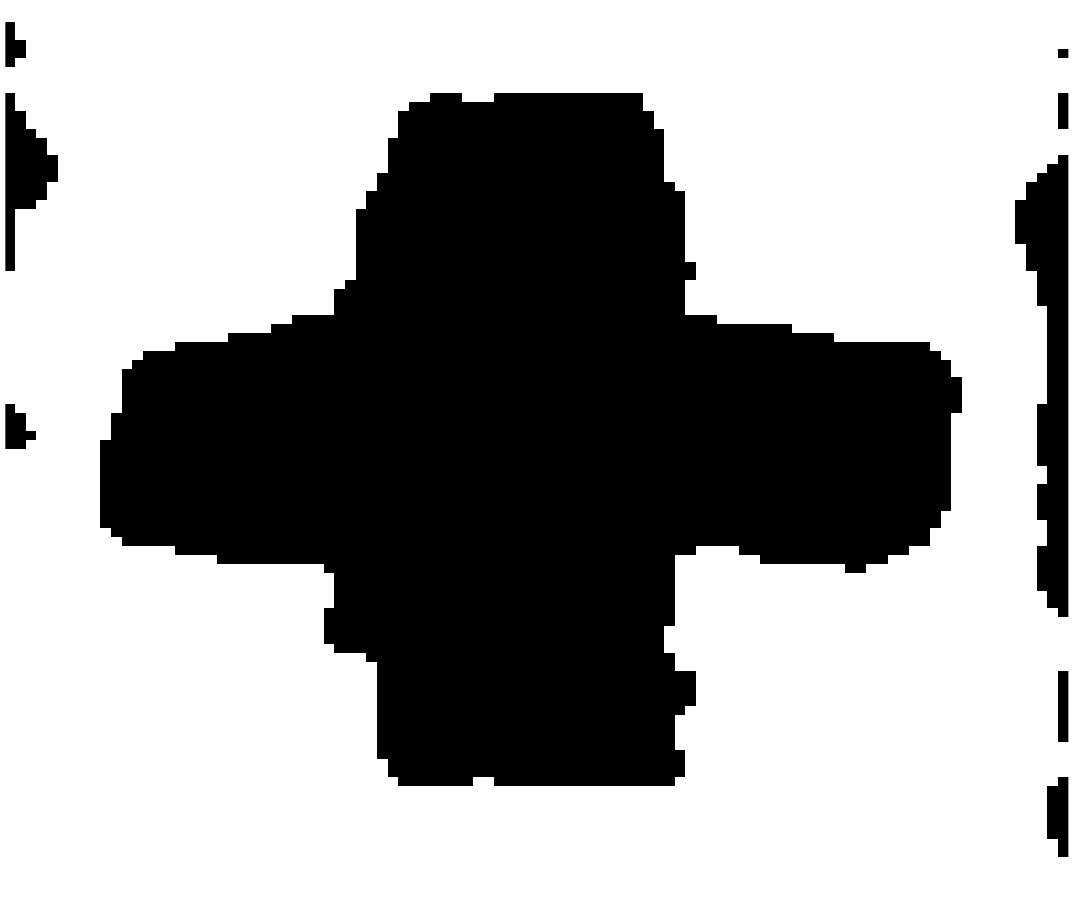}}
		\quad
		{\includegraphics[width=0.14\textwidth]{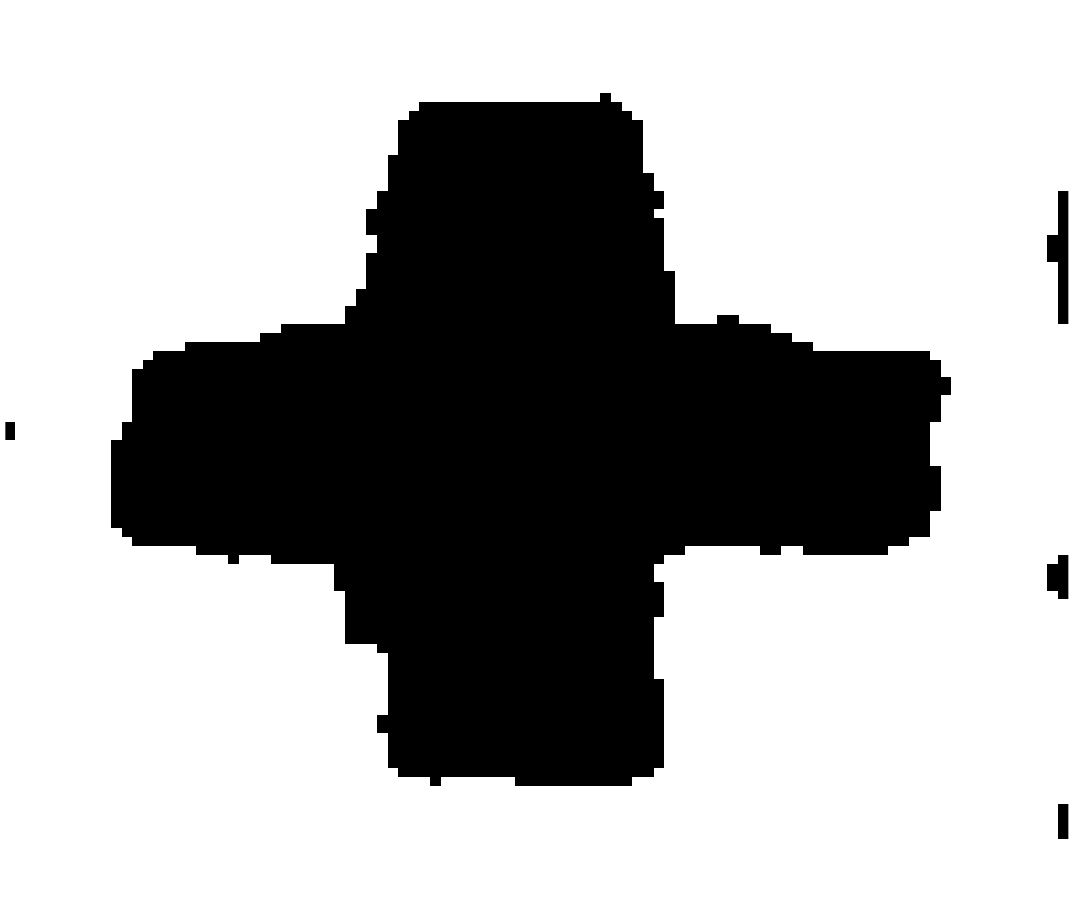}}
		\quad
		{\includegraphics[width=0.14\textwidth]{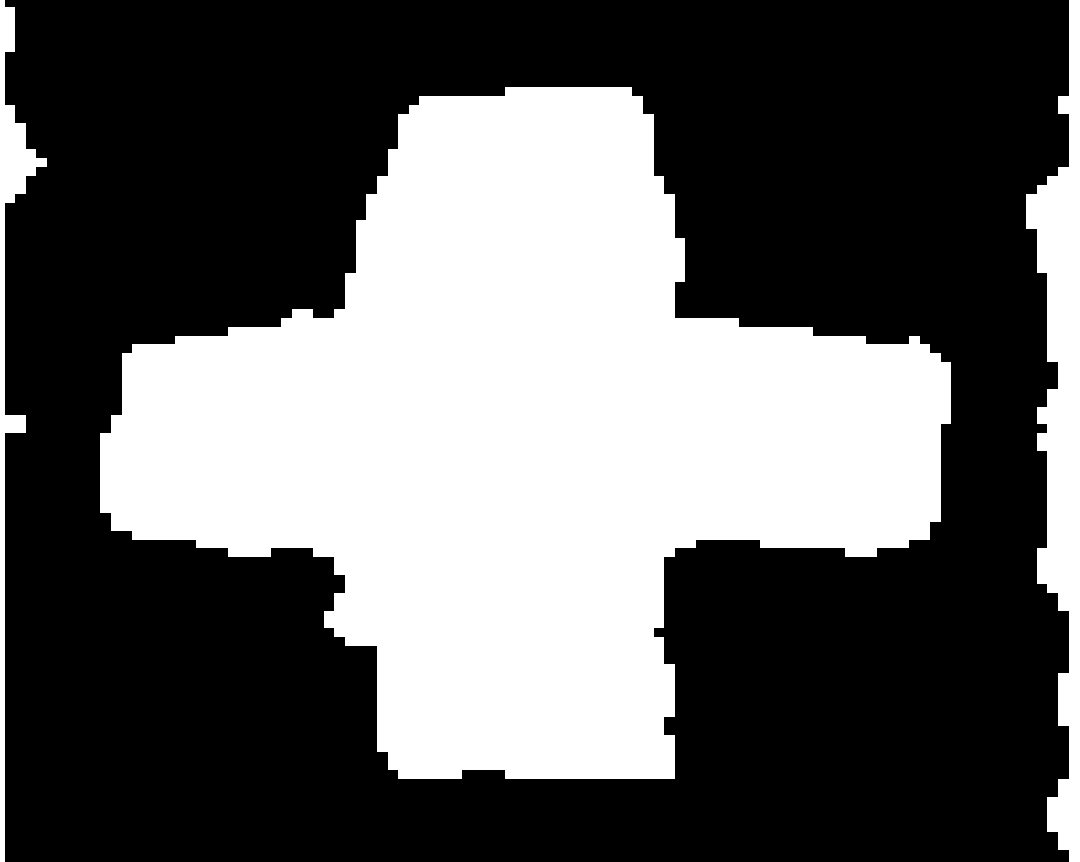}}
		\quad
		{\includegraphics[width=0.14\textwidth]{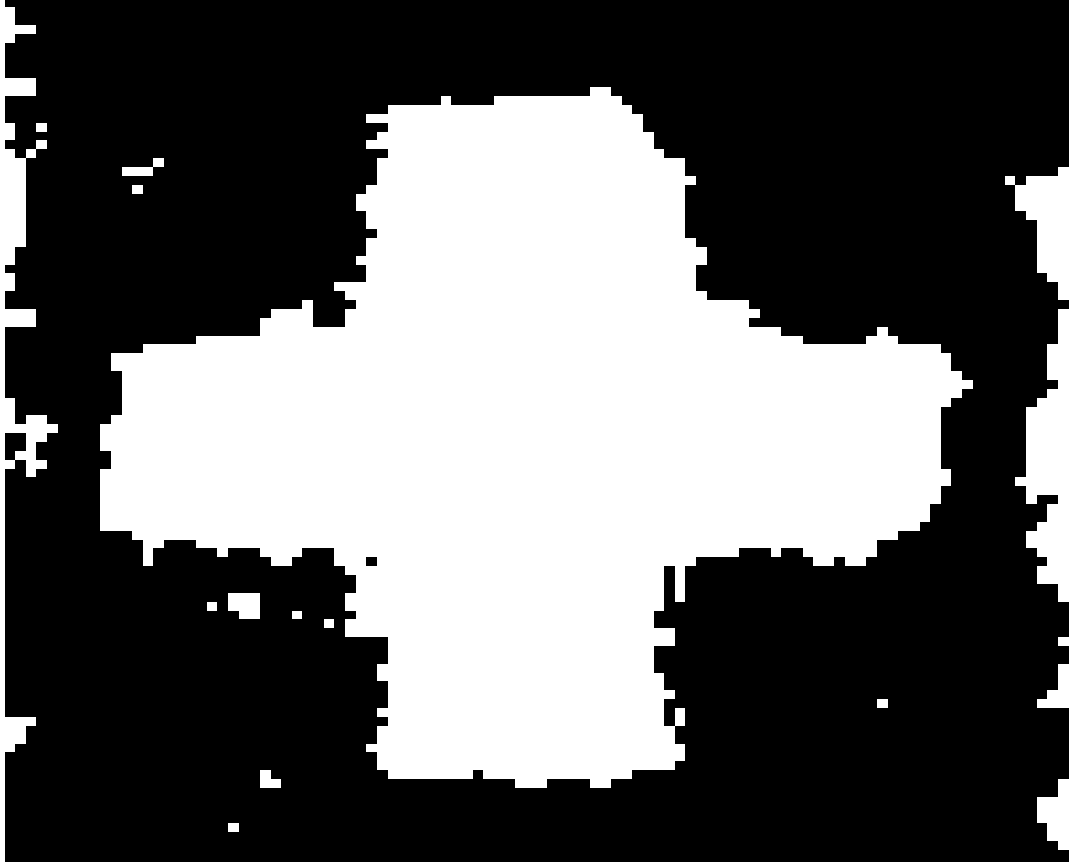}}
		\quad
		{\includegraphics[width=0.14\textwidth]{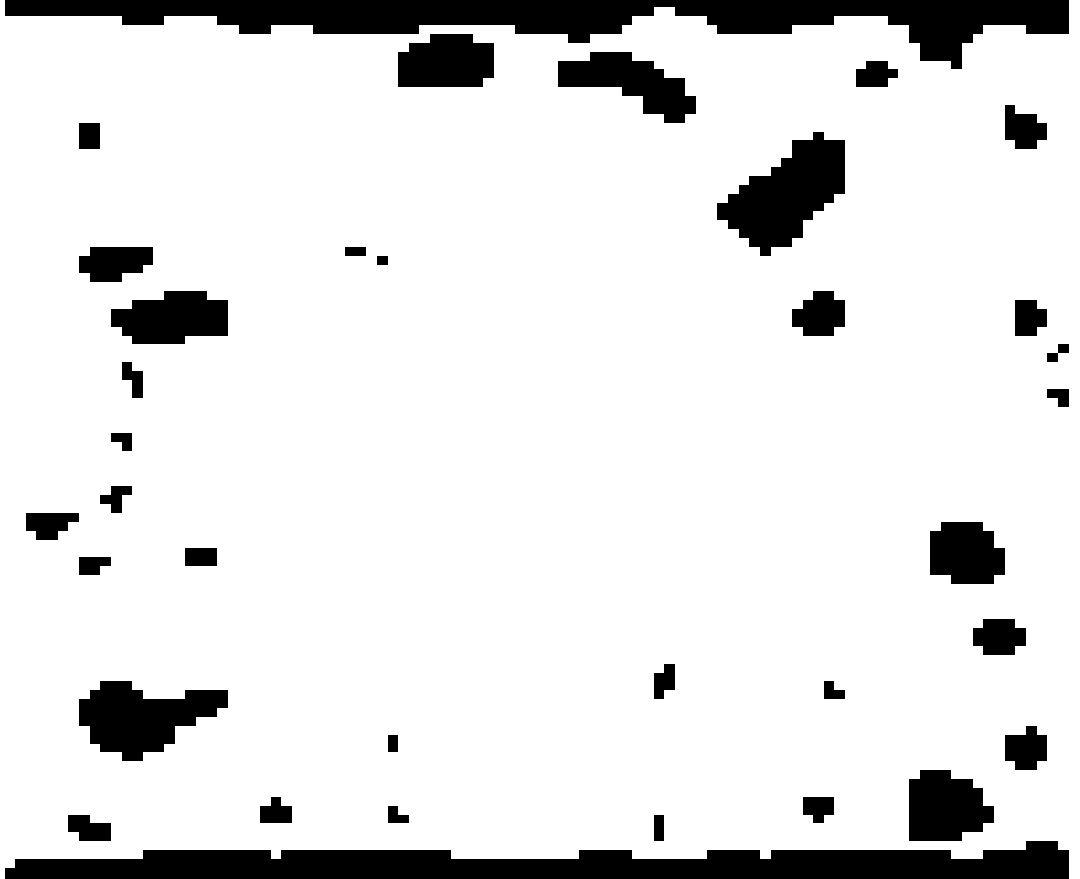}}
		\quad
		{\includegraphics[width=0.14\textwidth]{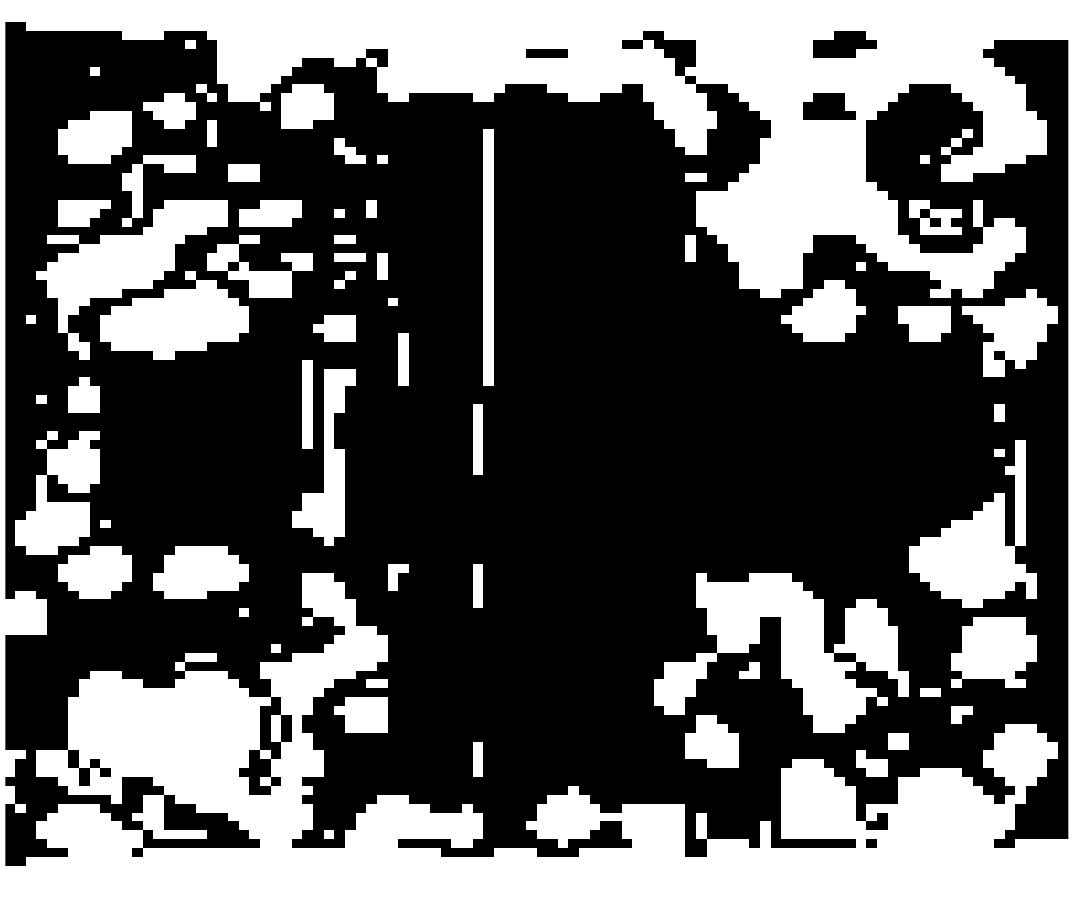}}
		\quad
		\subfloat[AFCM-ER-LP-L2] {\includegraphics[width=0.14\textwidth]{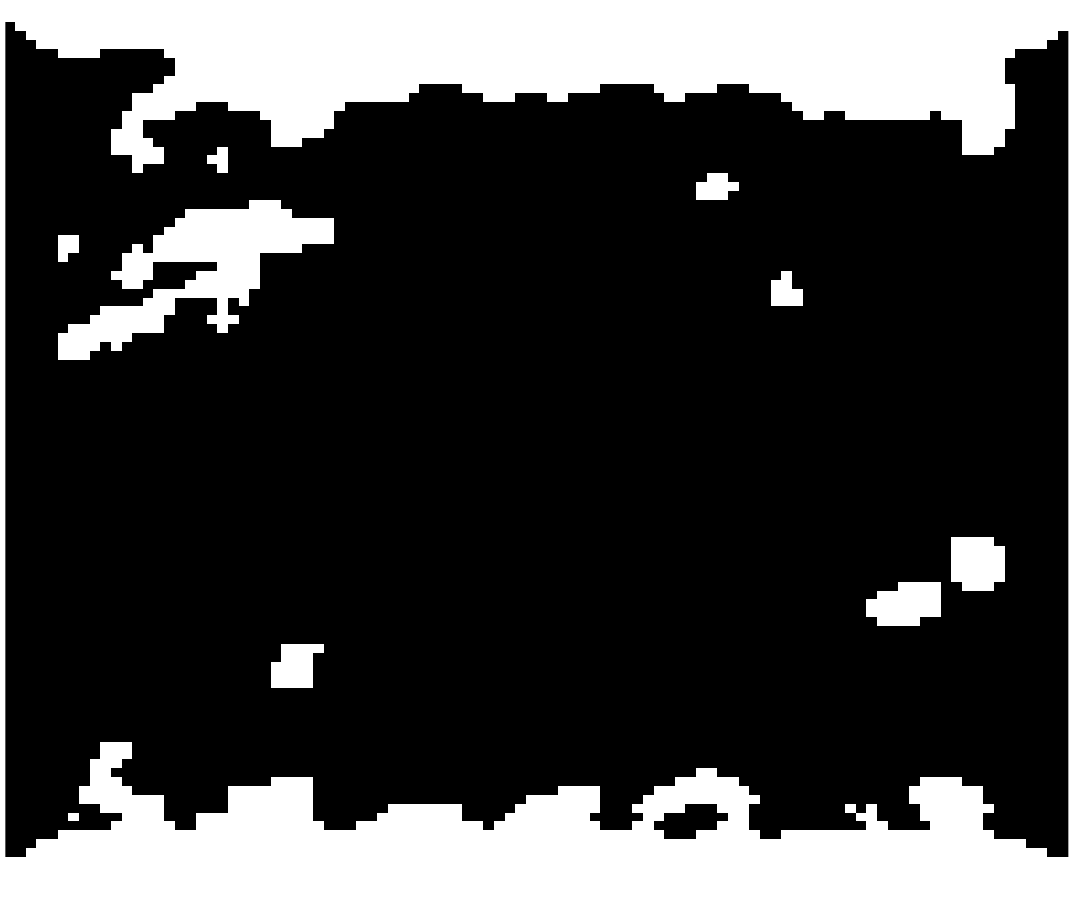}}
		\quad
		\subfloat[AFCM-ER-LP-L1] {\includegraphics[width=0.14\textwidth]{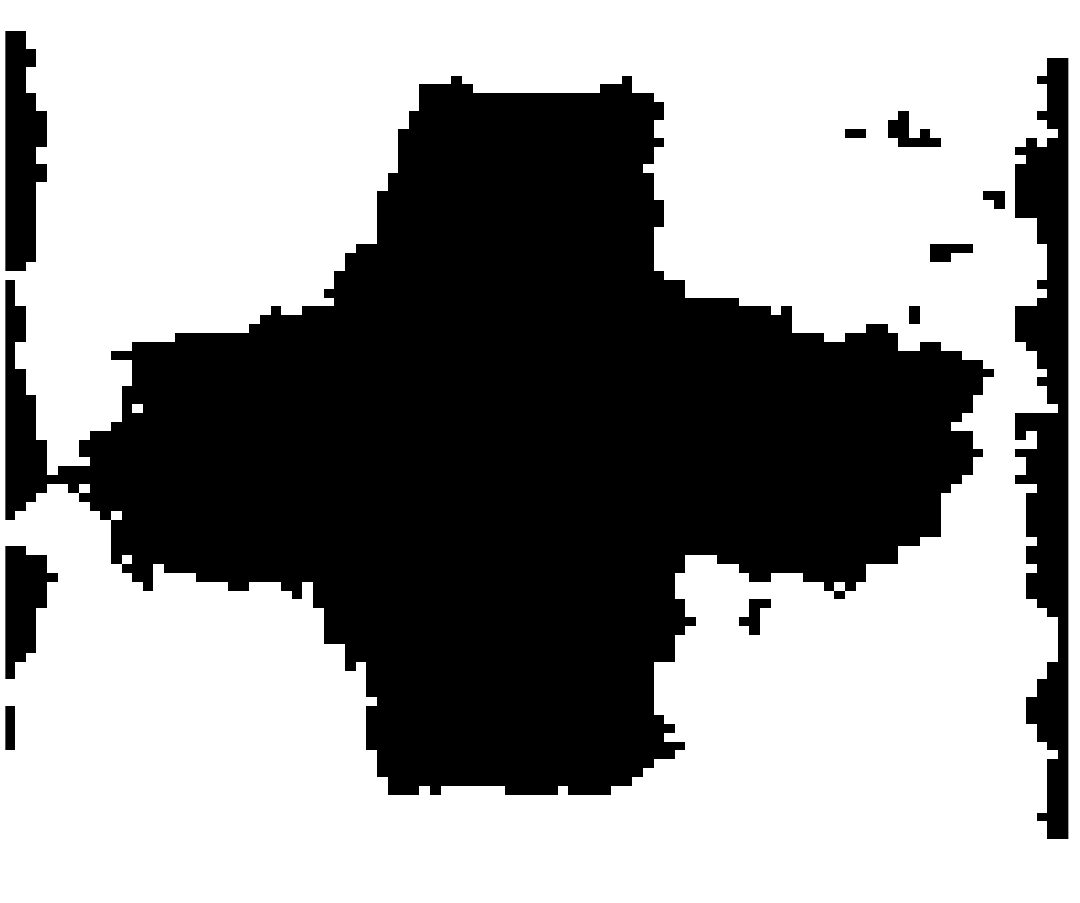}}
		\quad
		\subfloat[AFCM-ER-GS-L2] {\includegraphics[width=0.14\textwidth]{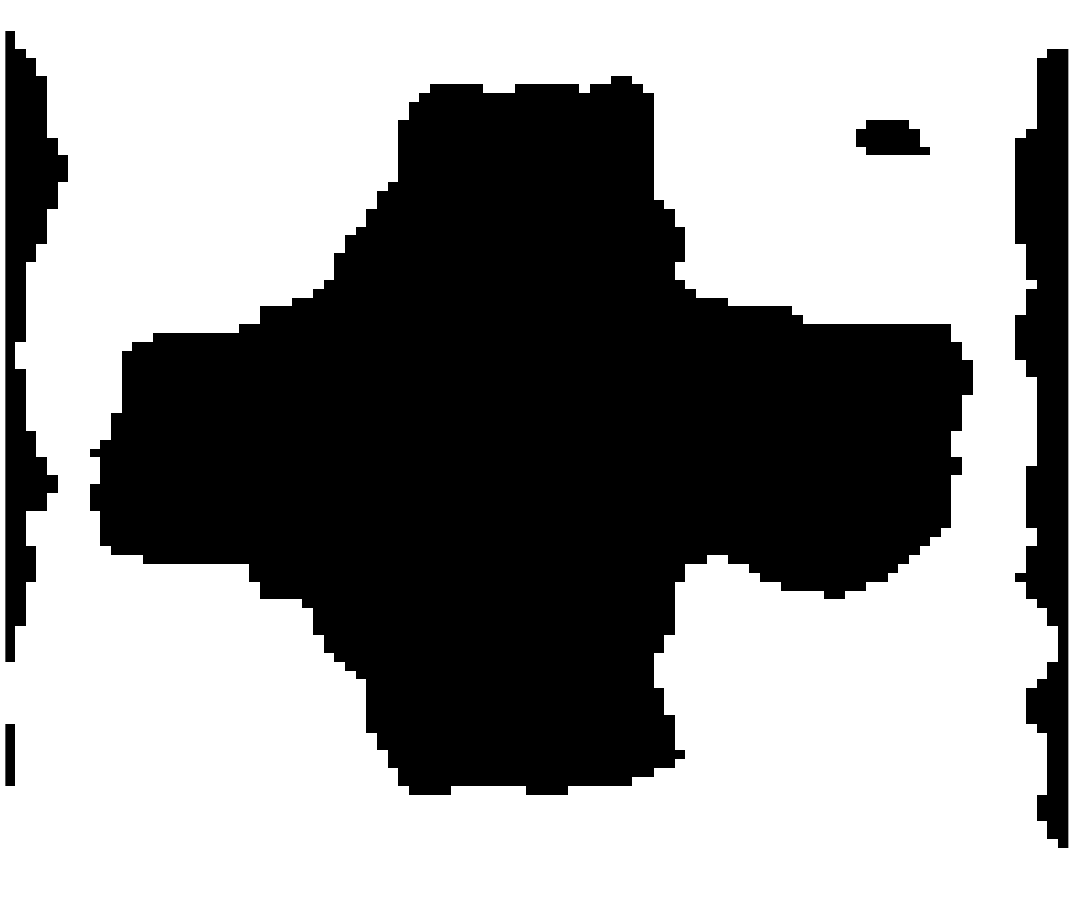}}
		\quad
		\subfloat[AFCM-ER-GS-L1] {\includegraphics[width=0.14\textwidth]{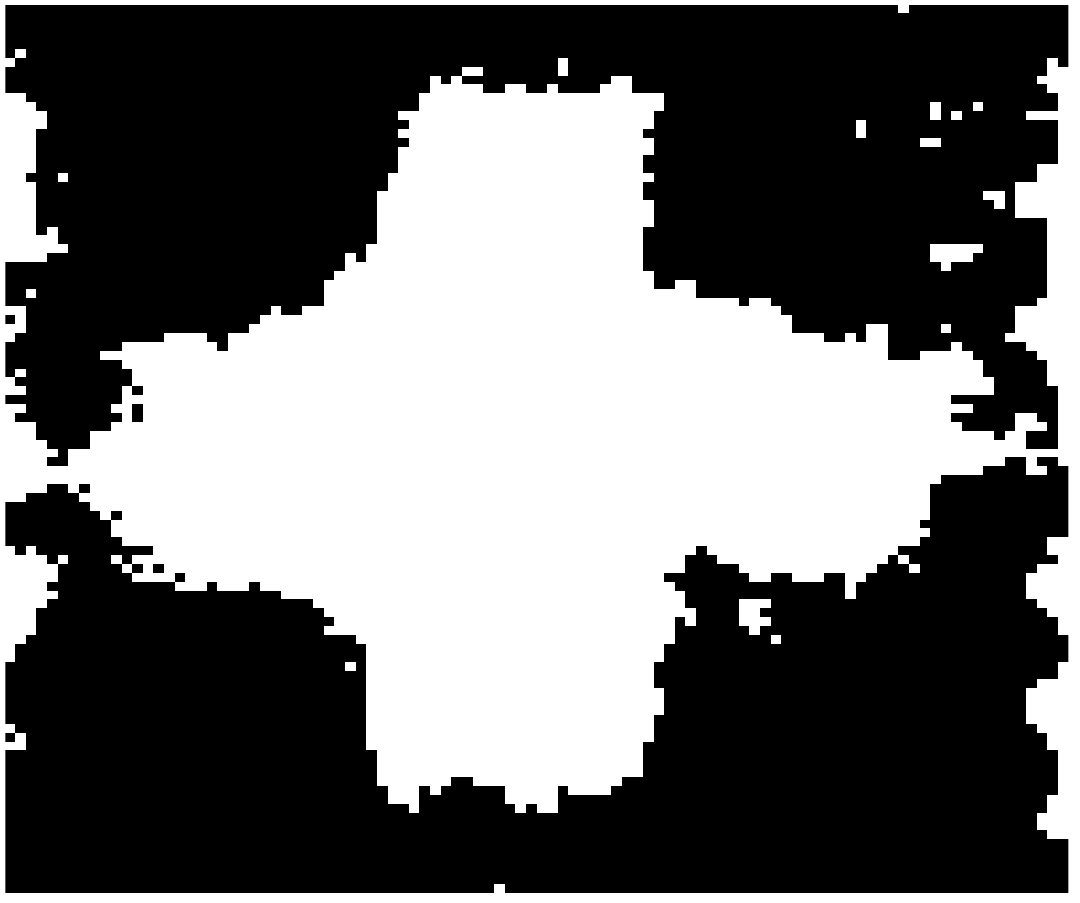}}
		\quad
		\subfloat[AFCM-ER-LS-L2] {\includegraphics[width=0.14\textwidth]{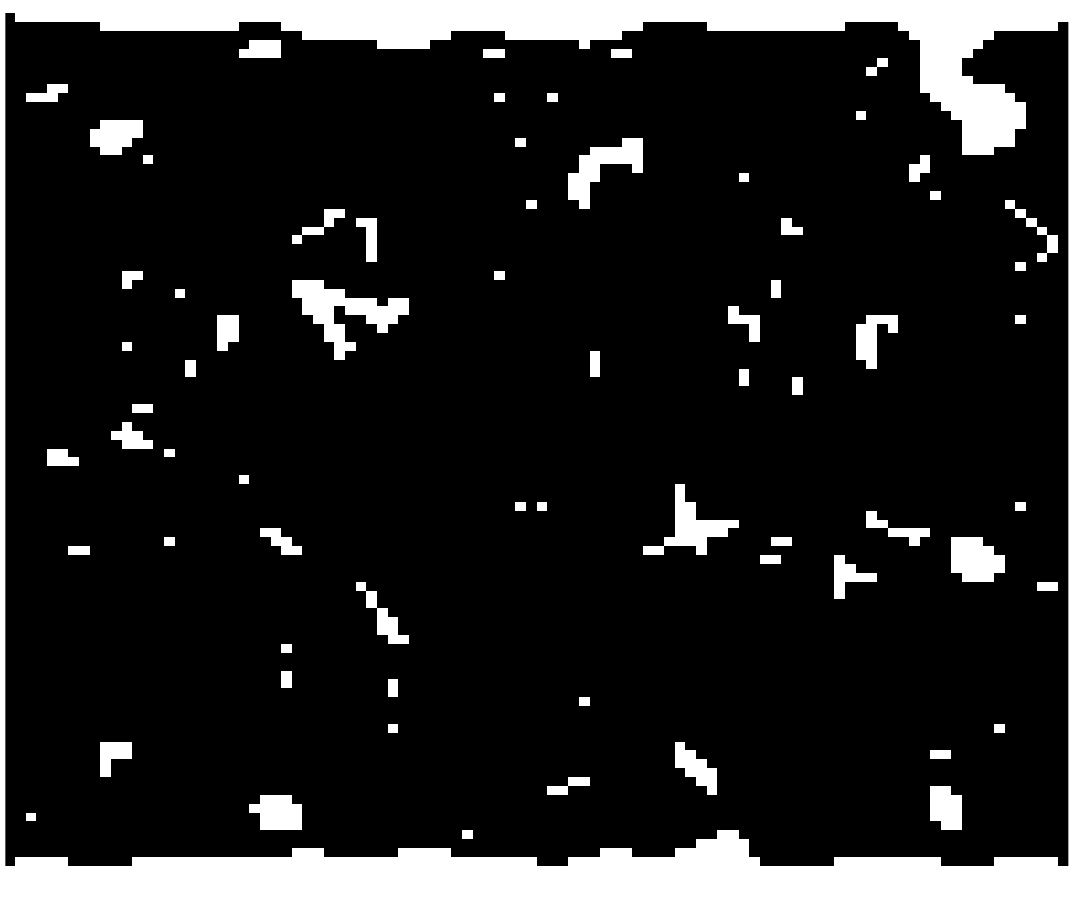}}
		\quad
		\subfloat[AFCM-ER-LS-L1] {\includegraphics[width=0.14\textwidth]{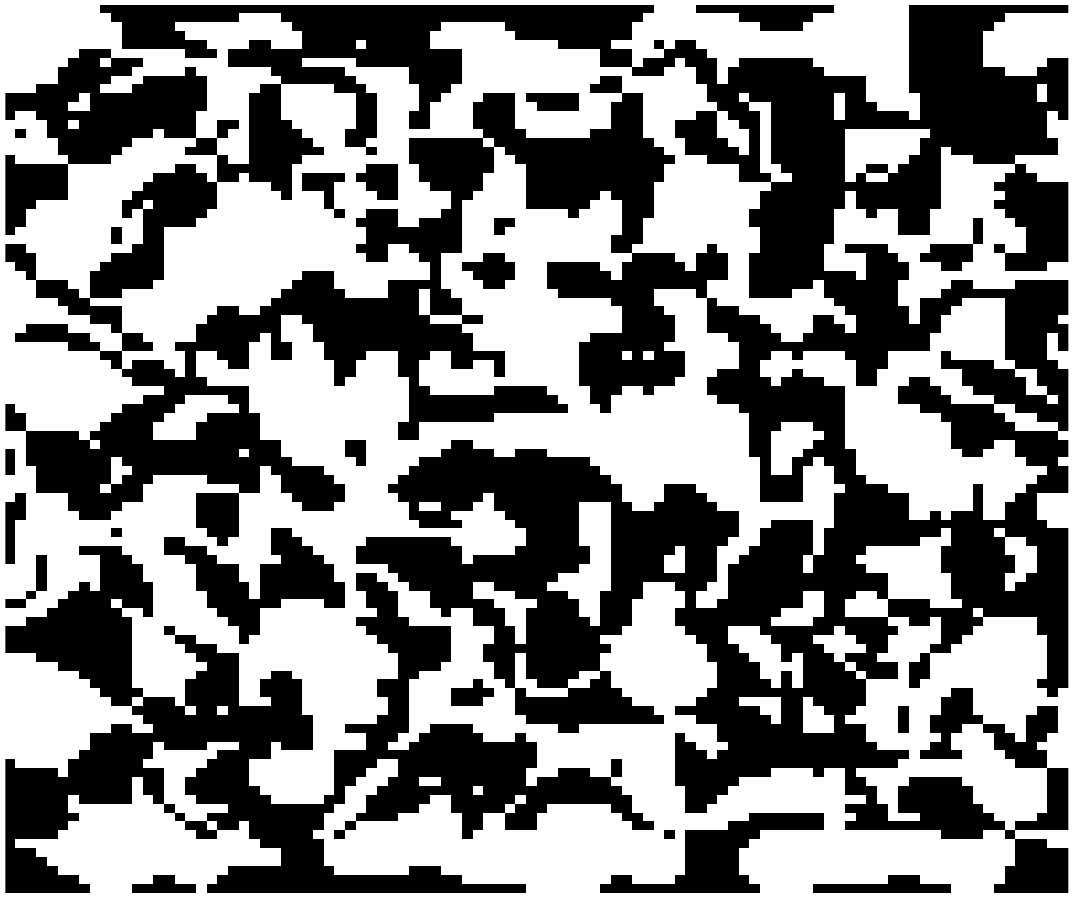}}
		\caption{Segmentation results of the algorithms for the 2-textural image without and with Gaussian noise. The first and the third row show the segmentation results for the original 2-textural image. The second and the fourth row present the obtained segmentation for the 2-textural image with Gaussian noise.}
		\label{img:ResultsTextImagesCross}
	\end{figure}
	
	\begin{figure}[!htb]
		\centering
		{\includegraphics[width=0.14\textwidth]{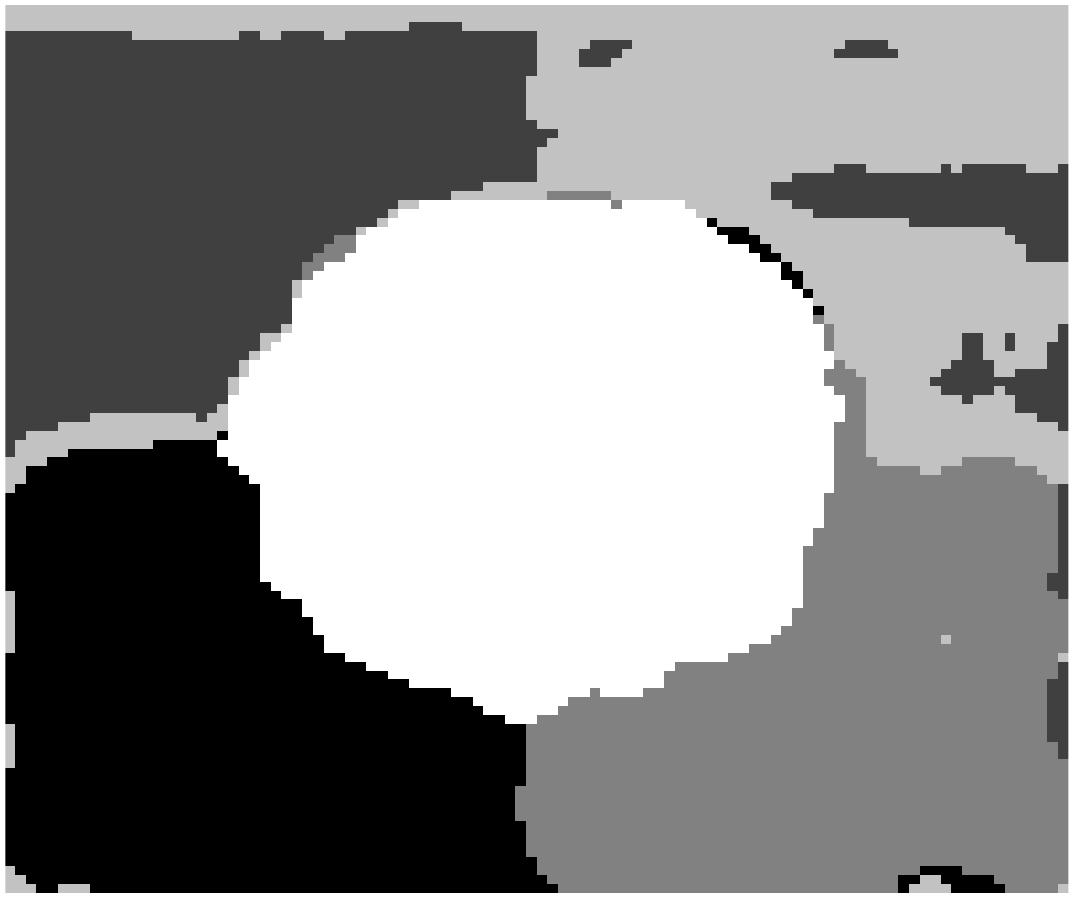}}
		\quad
		{\includegraphics[width=0.14\textwidth]{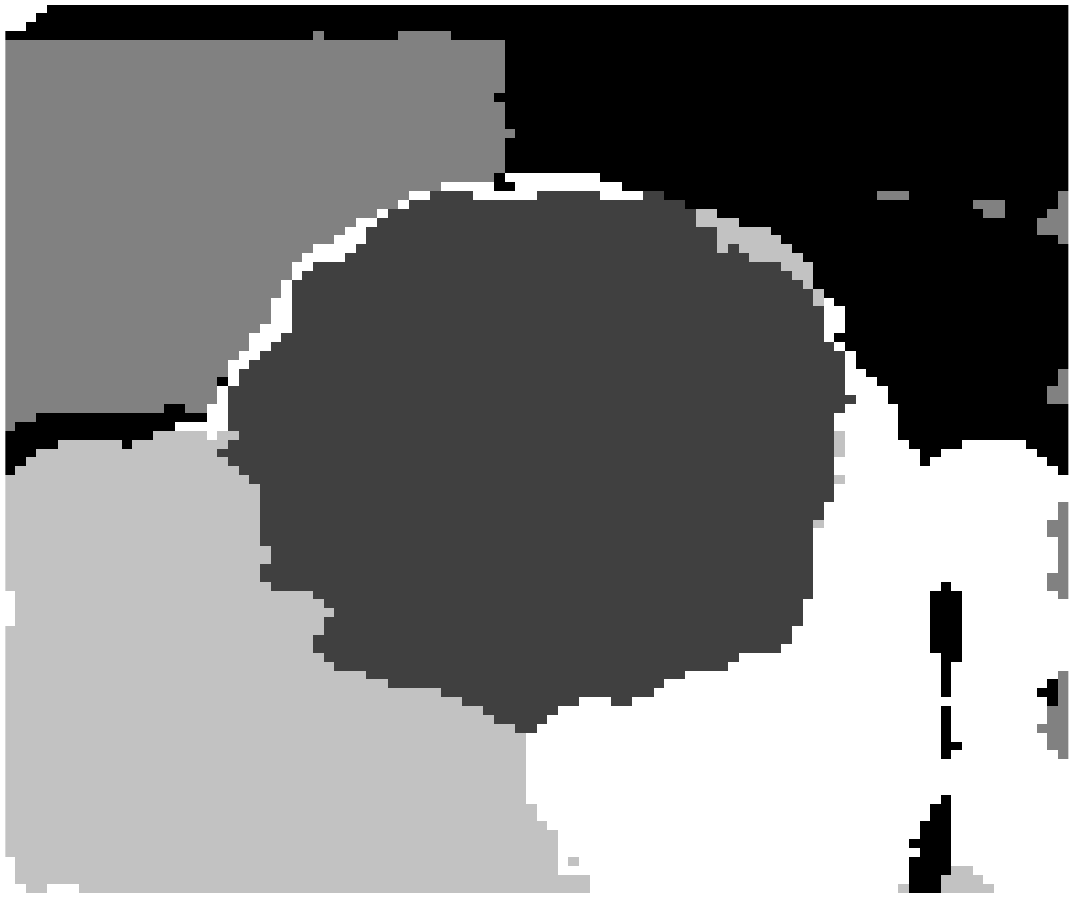}}
		\quad
		{\includegraphics[width=0.14\textwidth]{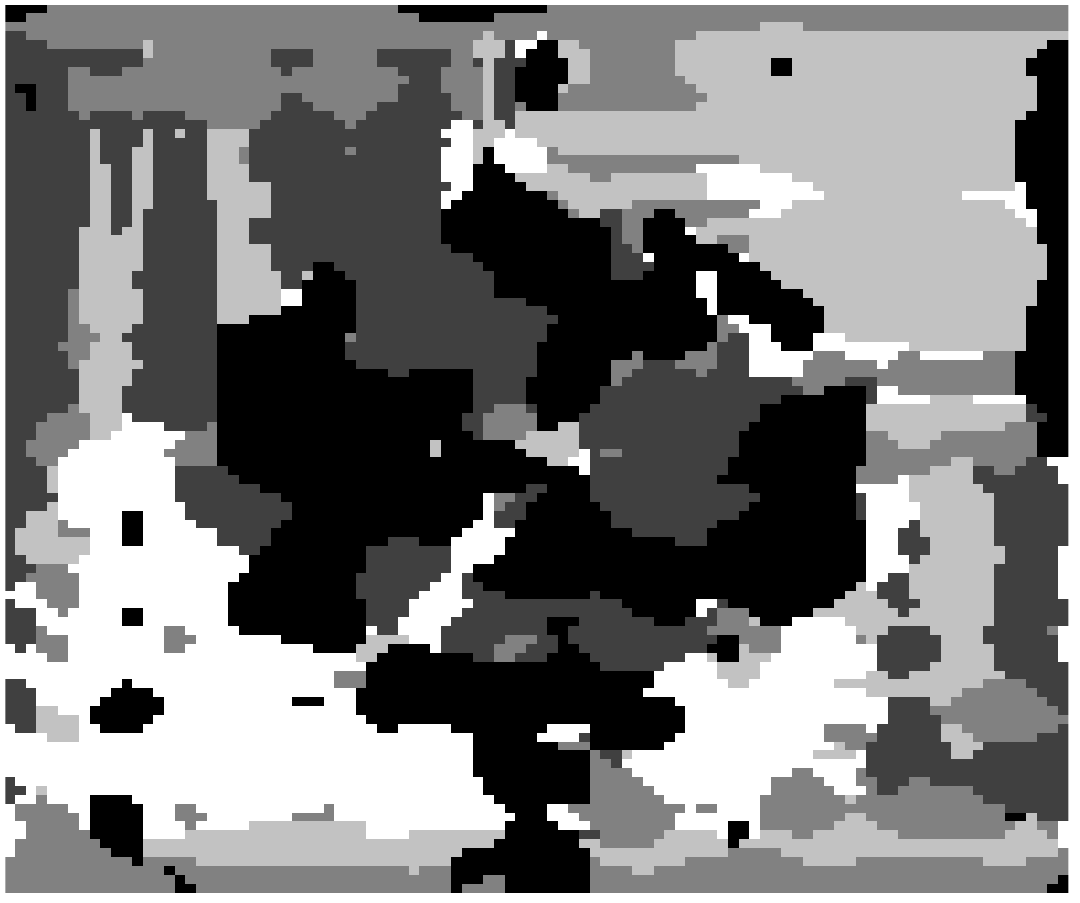}}
		\quad
		{\includegraphics[width=0.14\textwidth]{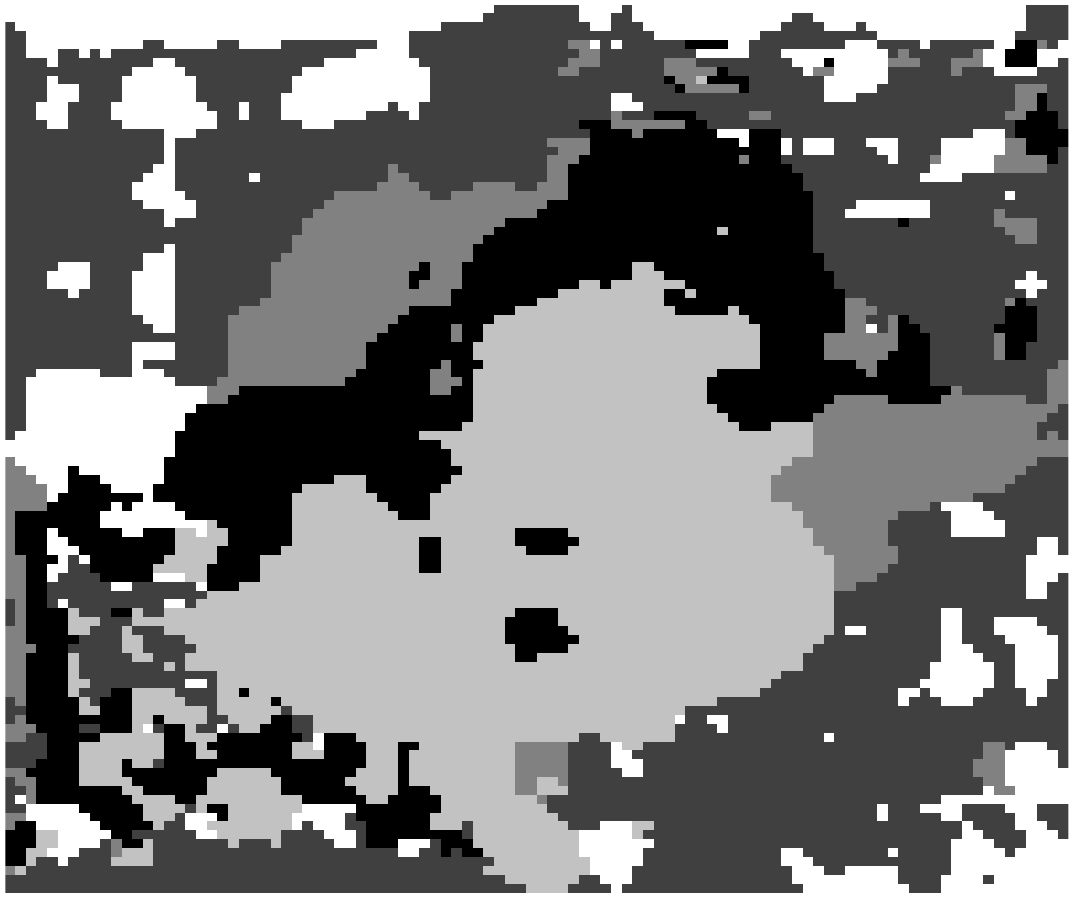}}
		\quad
		{\includegraphics[width=0.14\textwidth]{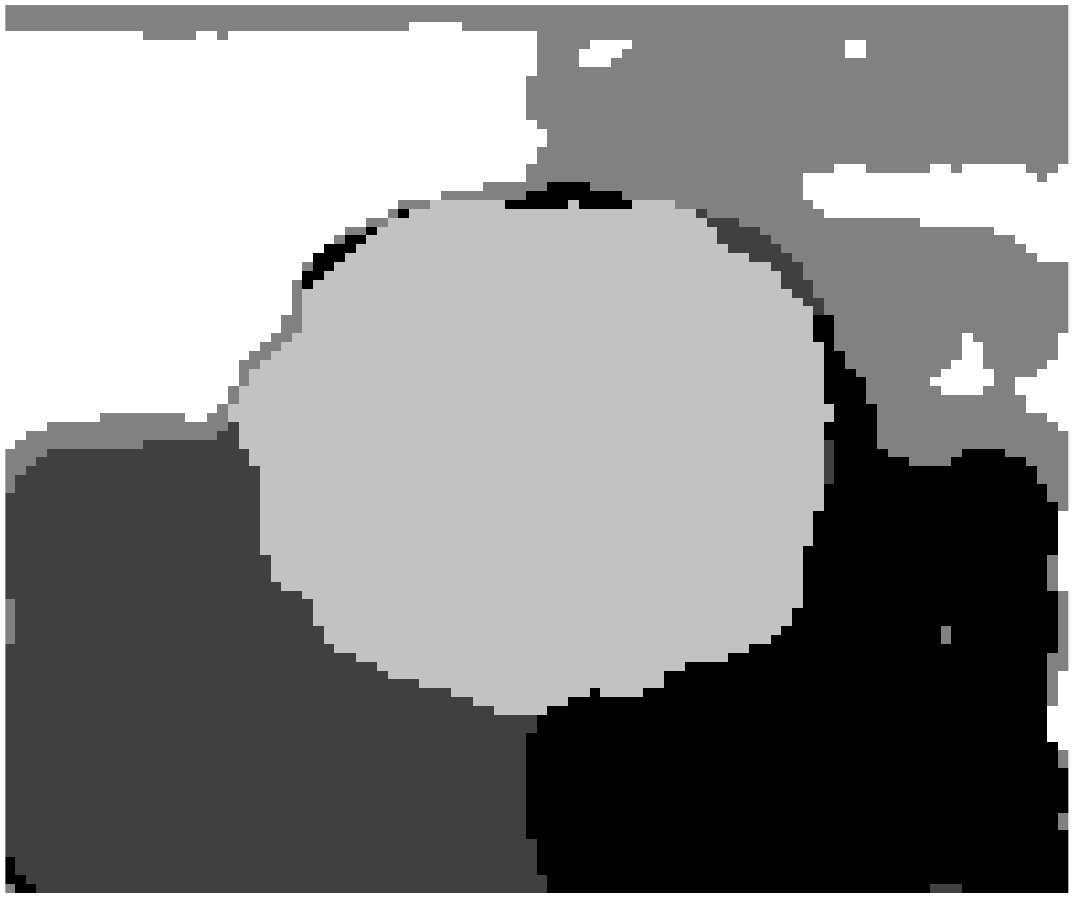}}
		\quad
		{\includegraphics[width=0.14\textwidth]{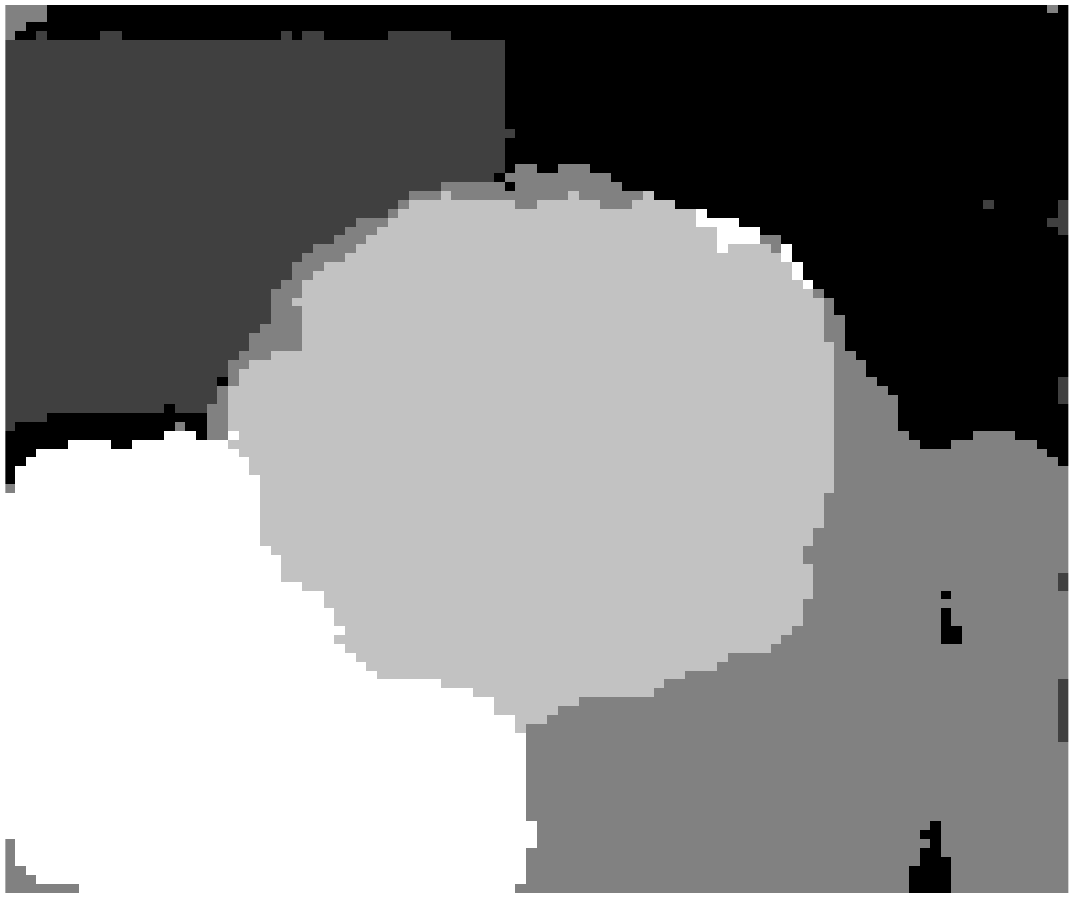}}
		\quad
		\subfloat[FCM-ER-L2] {\includegraphics[width=0.14\textwidth]{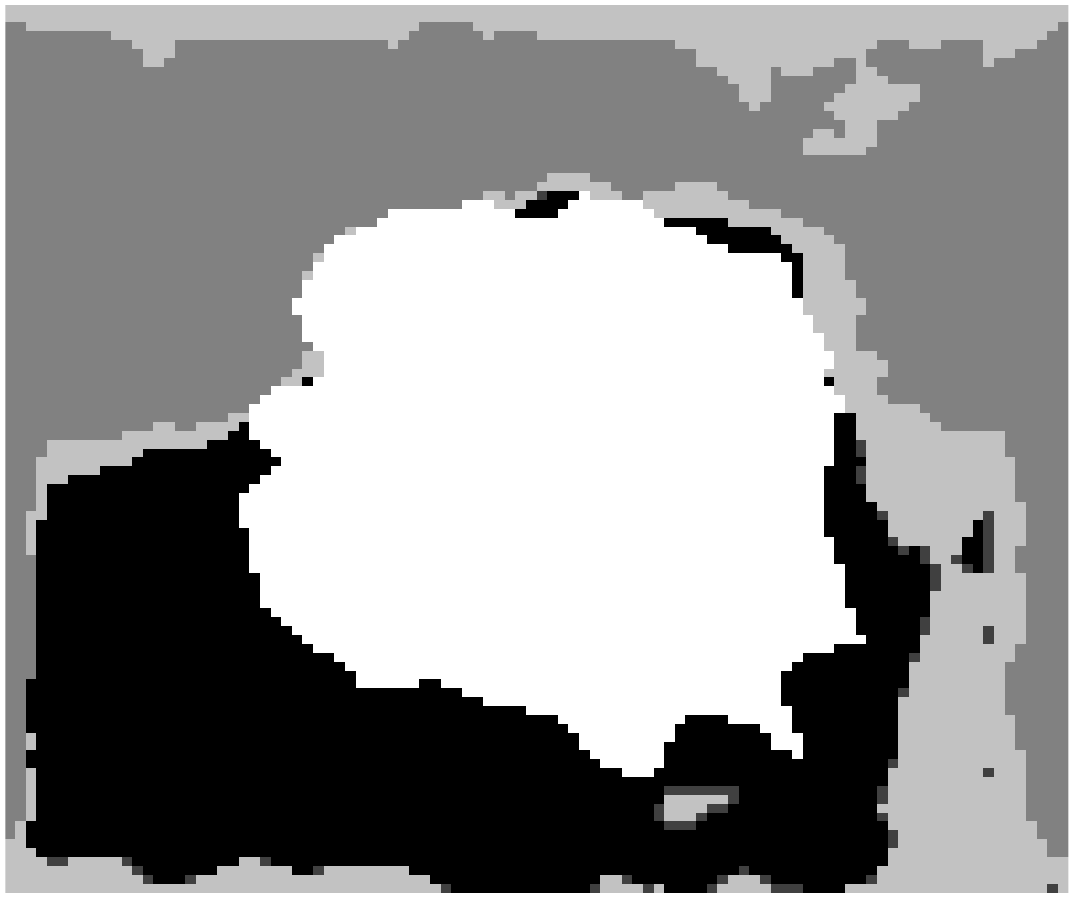}}
		\quad
		\subfloat[FCM-ER-L1] {\includegraphics[width=0.14\textwidth]{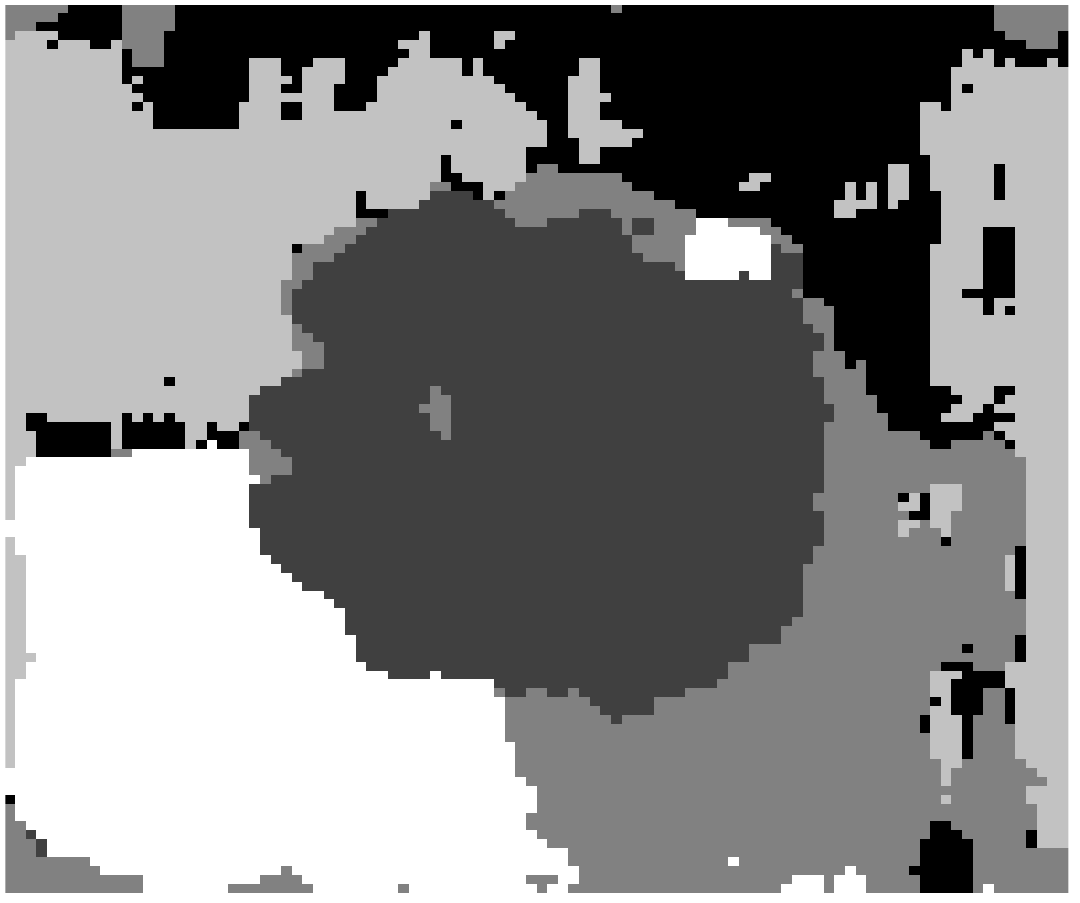}}
		\quad
		\subfloat[AFCM-ER-M] {\includegraphics[width=0.14\textwidth]{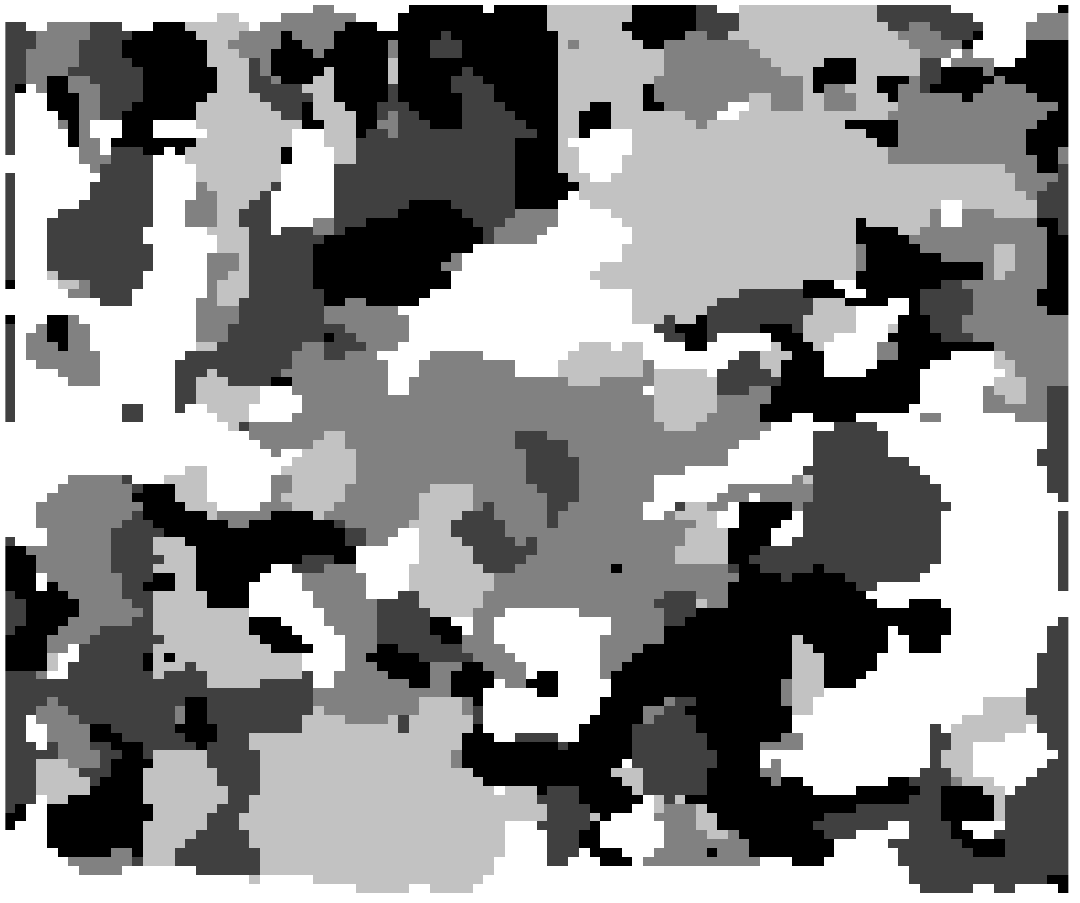}}
		\quad
		\subfloat[AFCM-ER-Mk] {\includegraphics[width=0.14\textwidth]{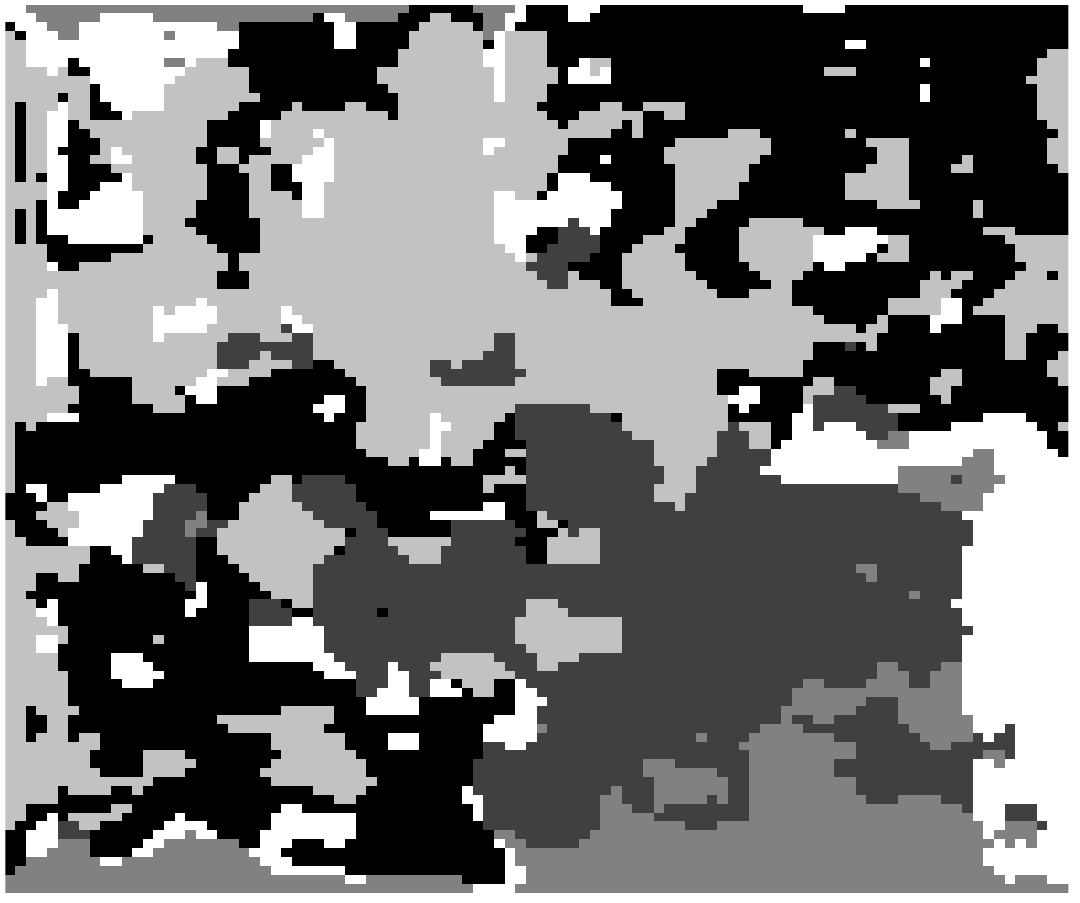}}
		\quad
		\subfloat[AFCM-ER-GP-L2] {\includegraphics[width=0.14\textwidth]{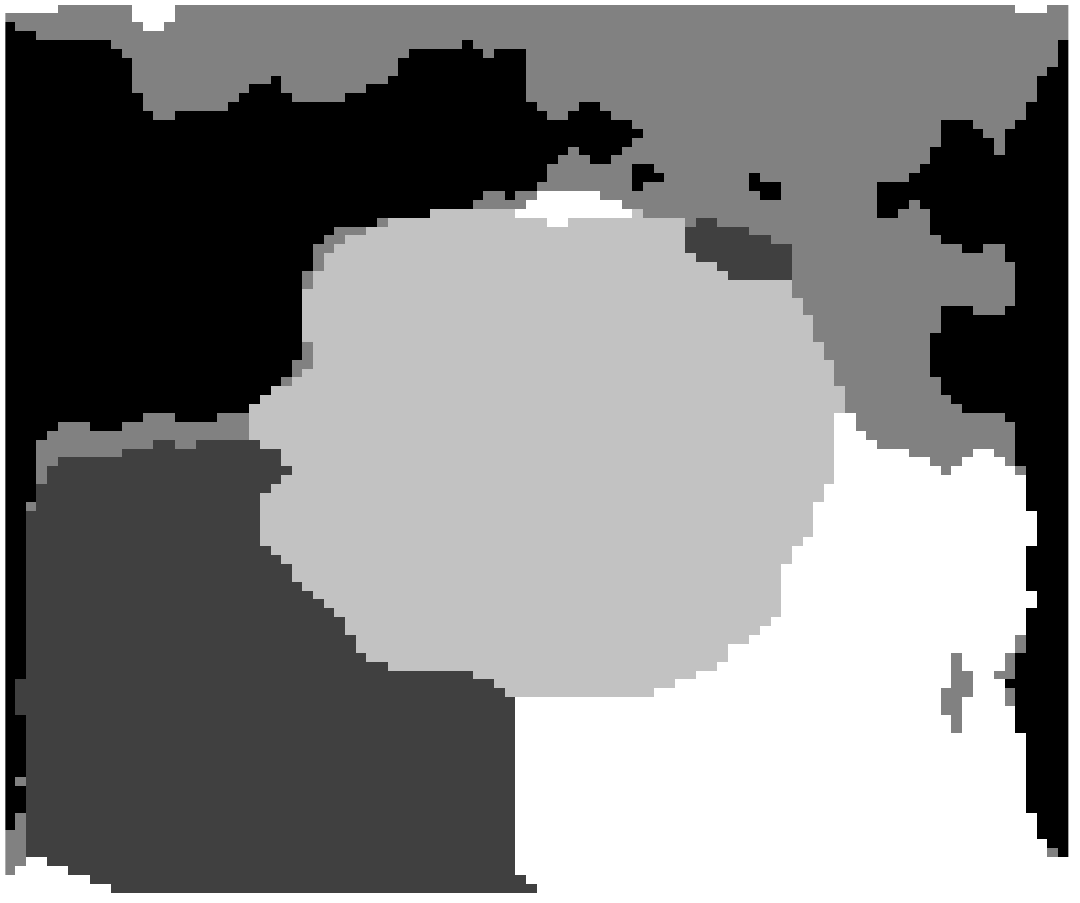}}
		\quad
		\subfloat[AFCM-ER-GP-L1] {\includegraphics[width=0.14\textwidth]{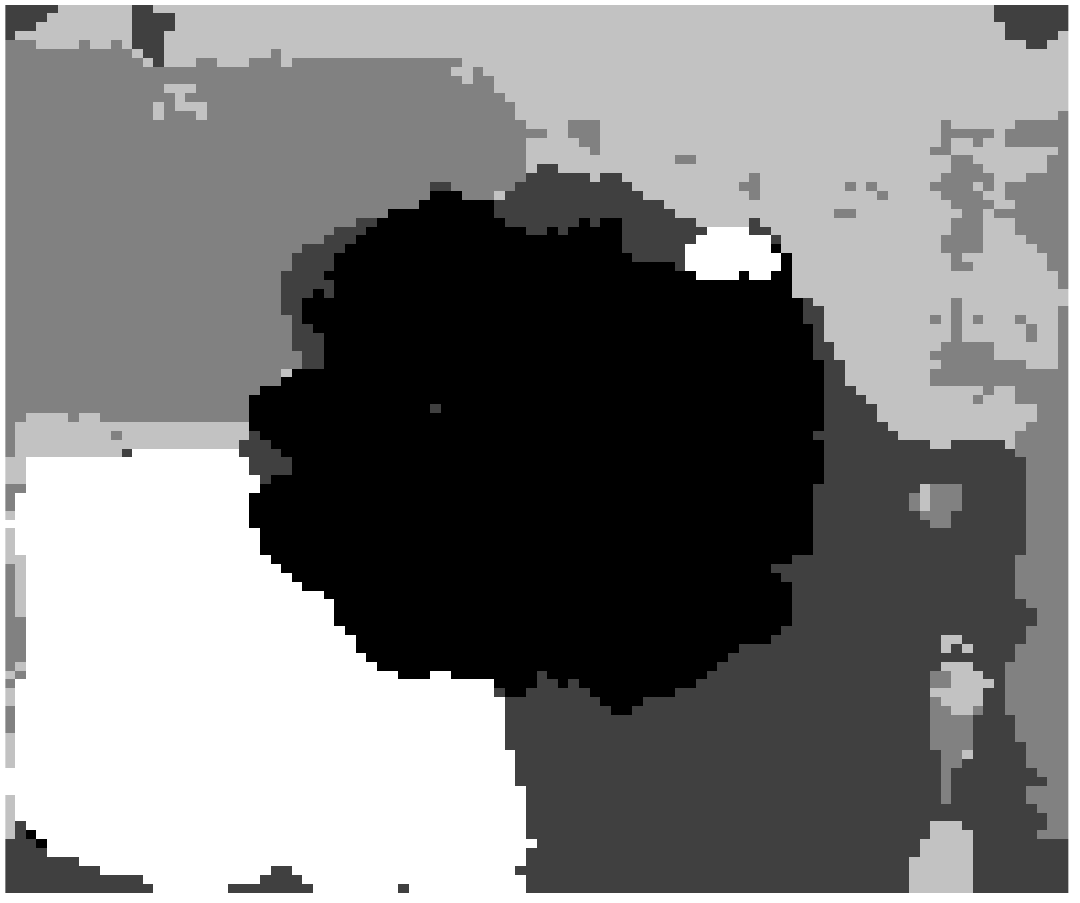}}
		\quad
		{\includegraphics[width=0.14\textwidth]{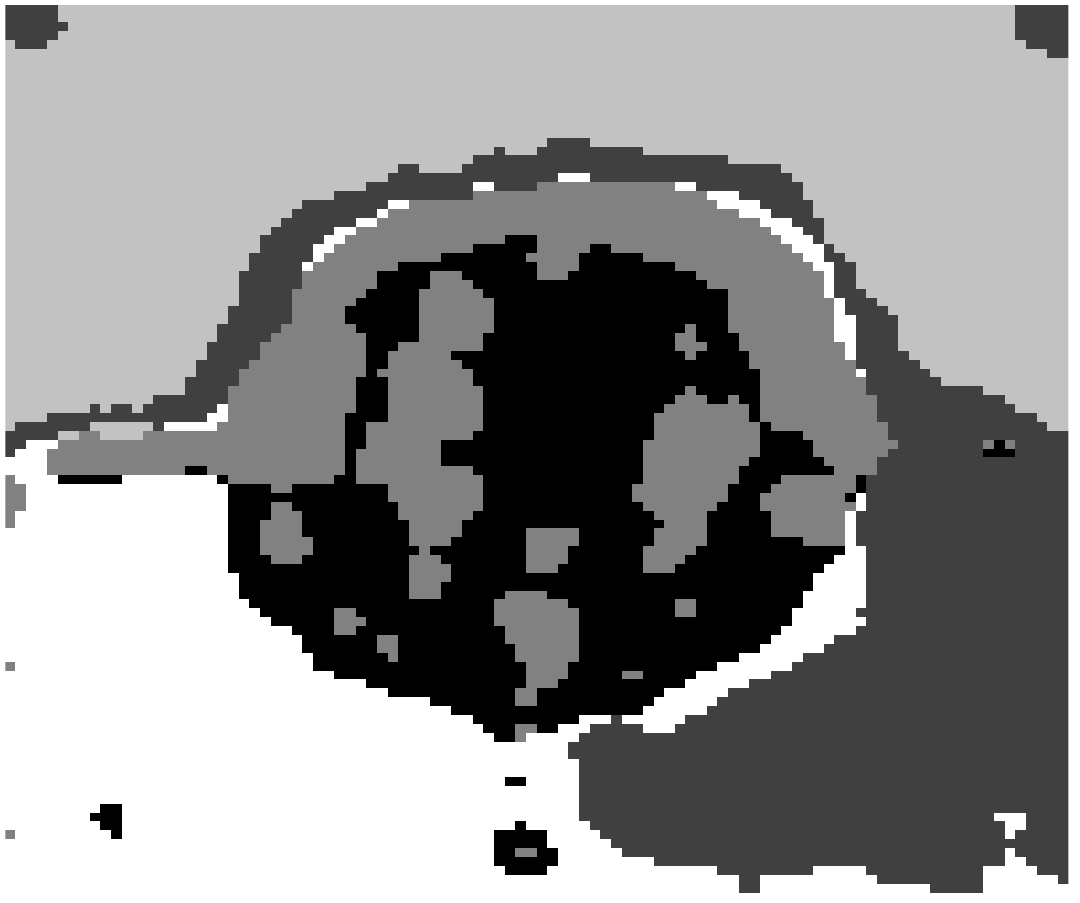}}
		\quad
		{\includegraphics[width=0.14\textwidth]{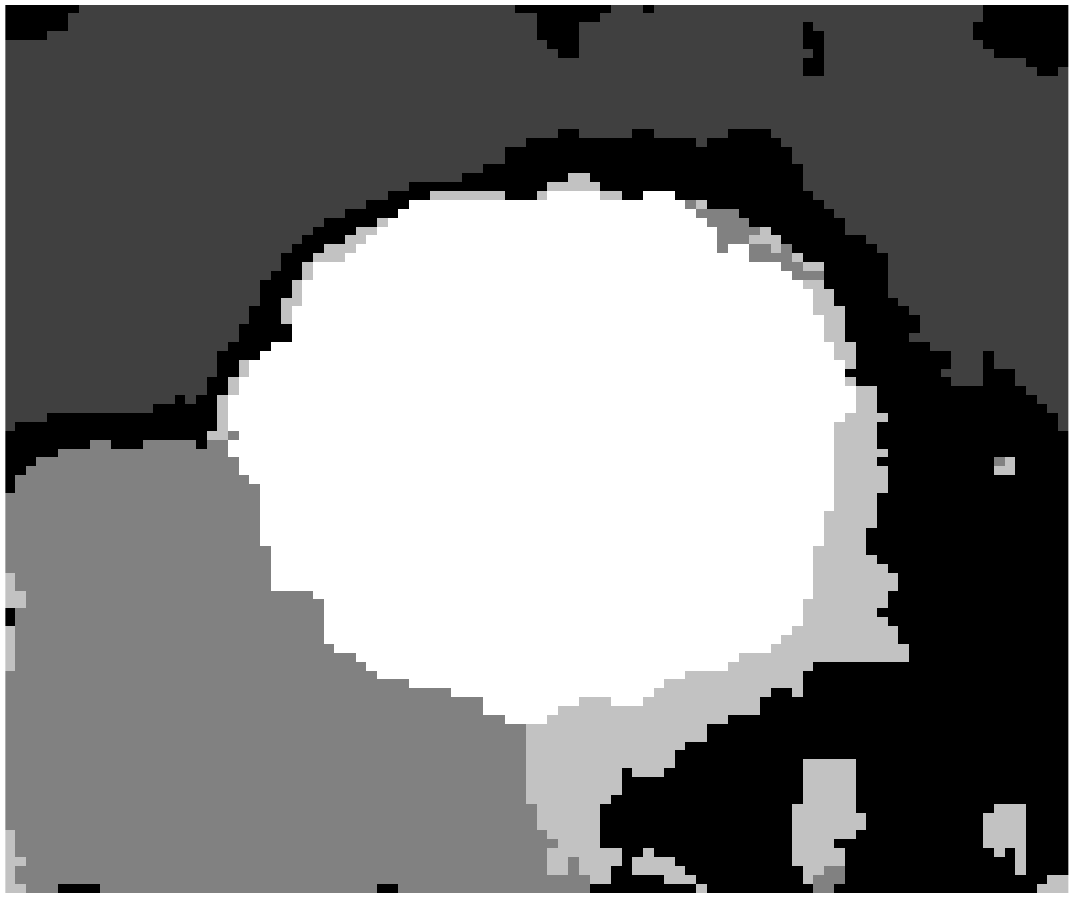}}
		\quad
		{\includegraphics[width=0.14\textwidth]{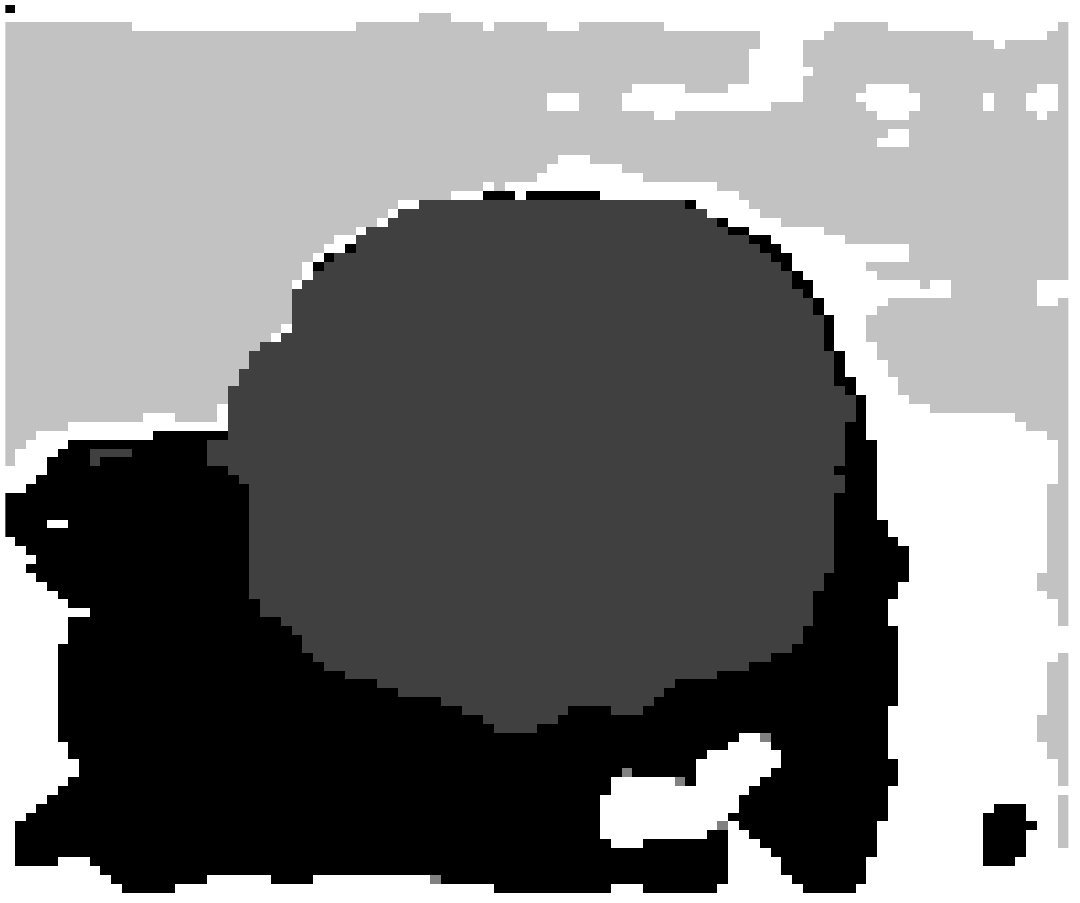}}
		\quad
		{\includegraphics[width=0.14\textwidth]{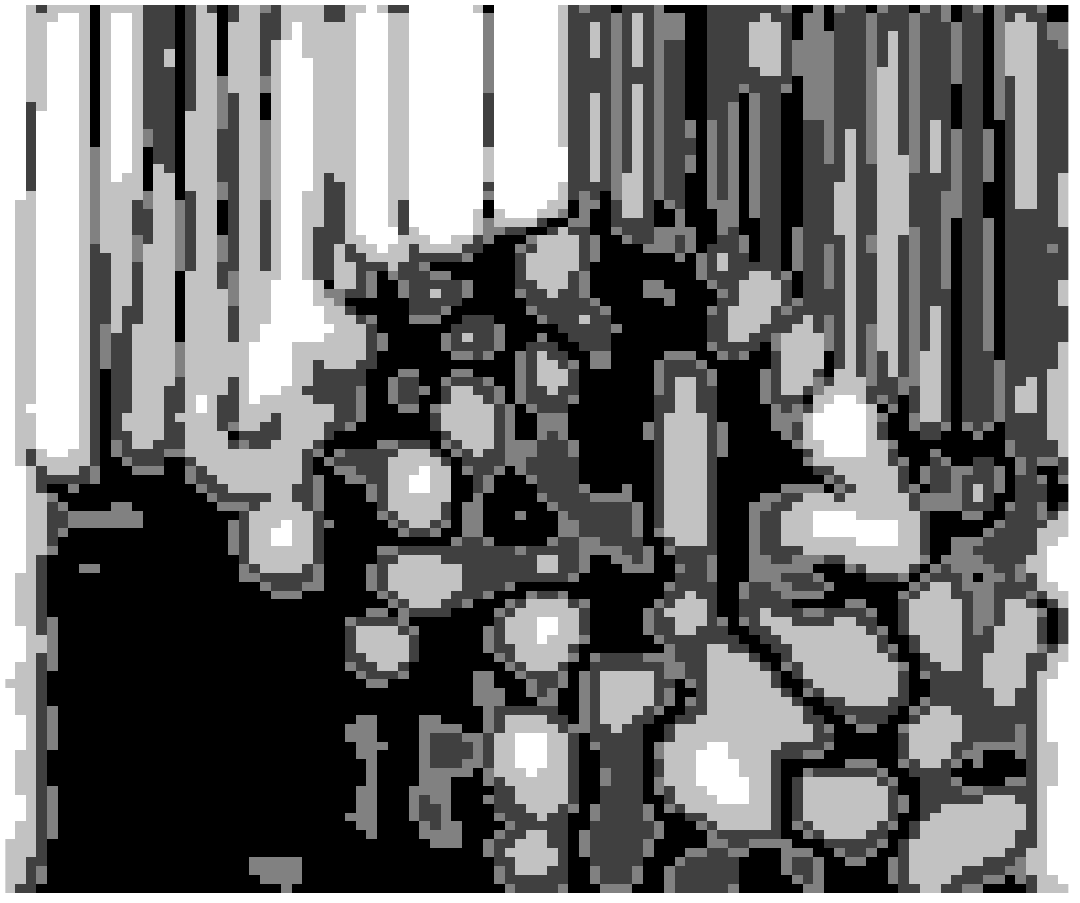}}
		\quad
		{\includegraphics[width=0.14\textwidth]{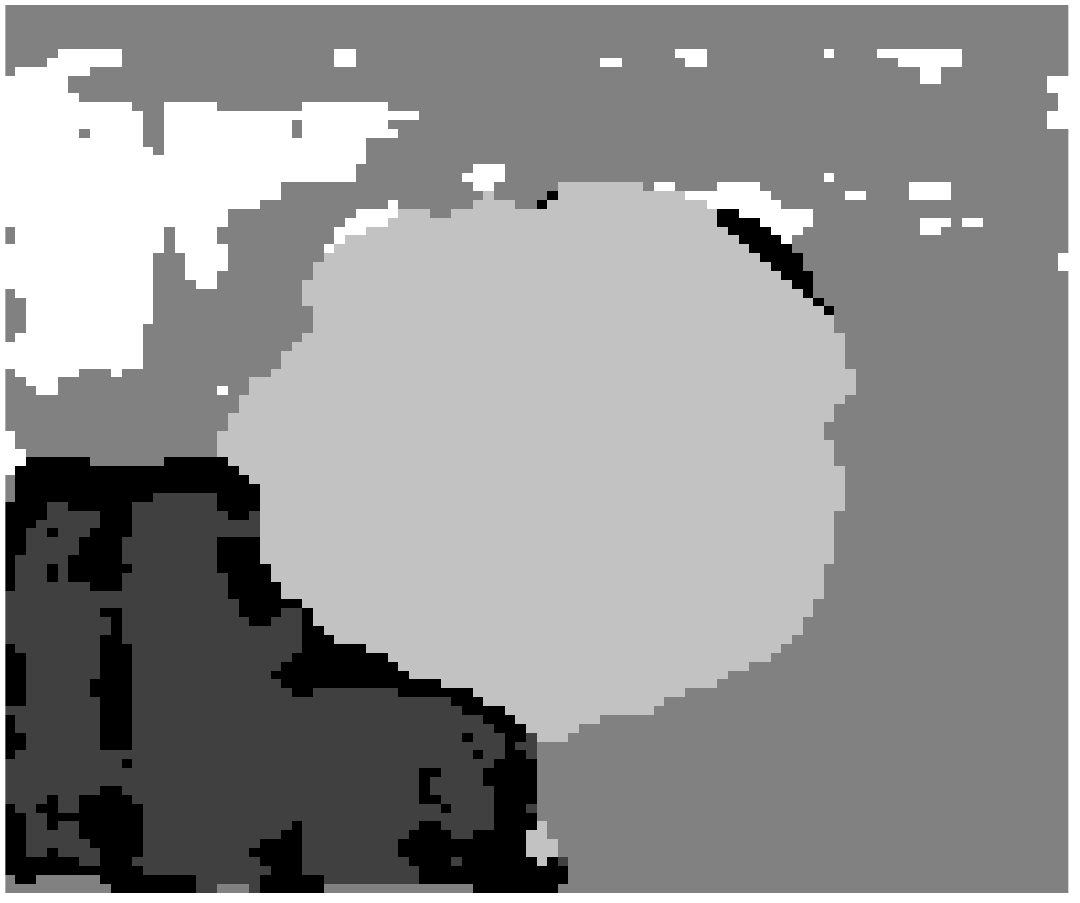}}
		\quad
		{\includegraphics[width=0.14\textwidth]{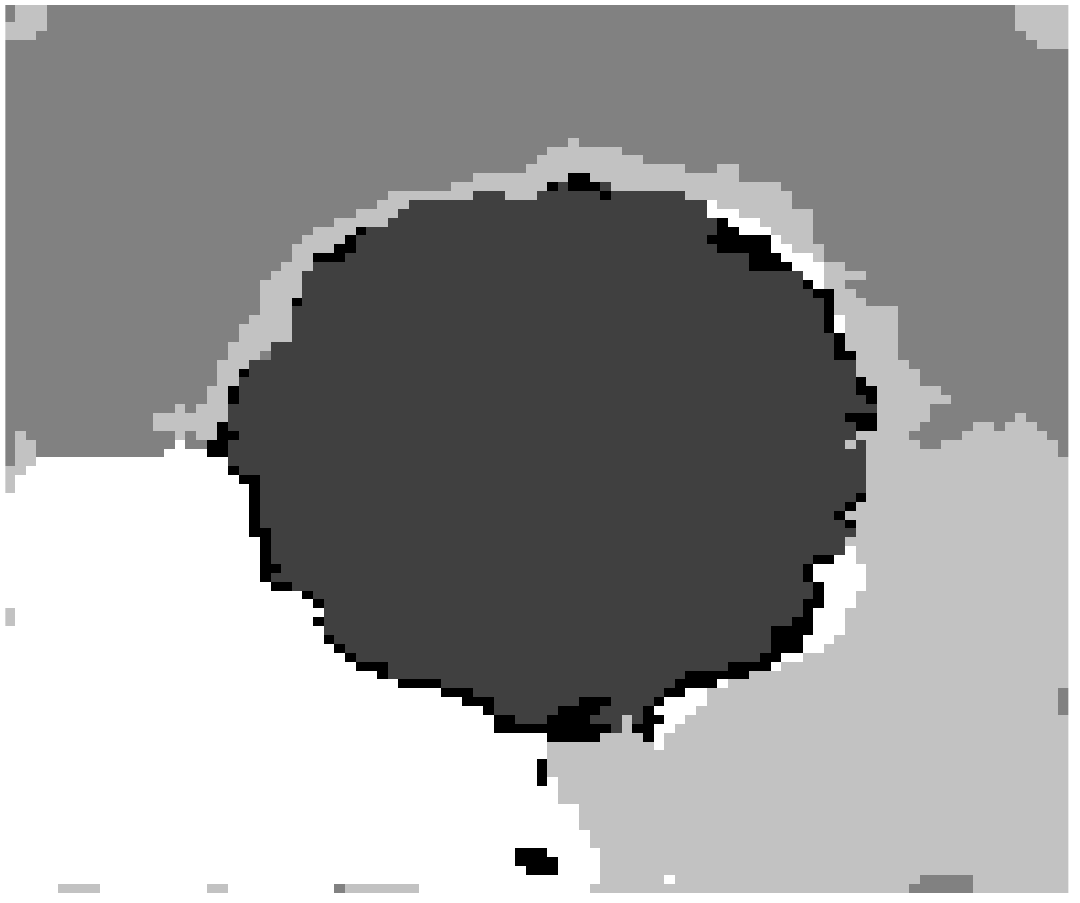}}
		\quad
		\subfloat[AFCM-ER-LP-L2] {\includegraphics[width=0.14\textwidth]{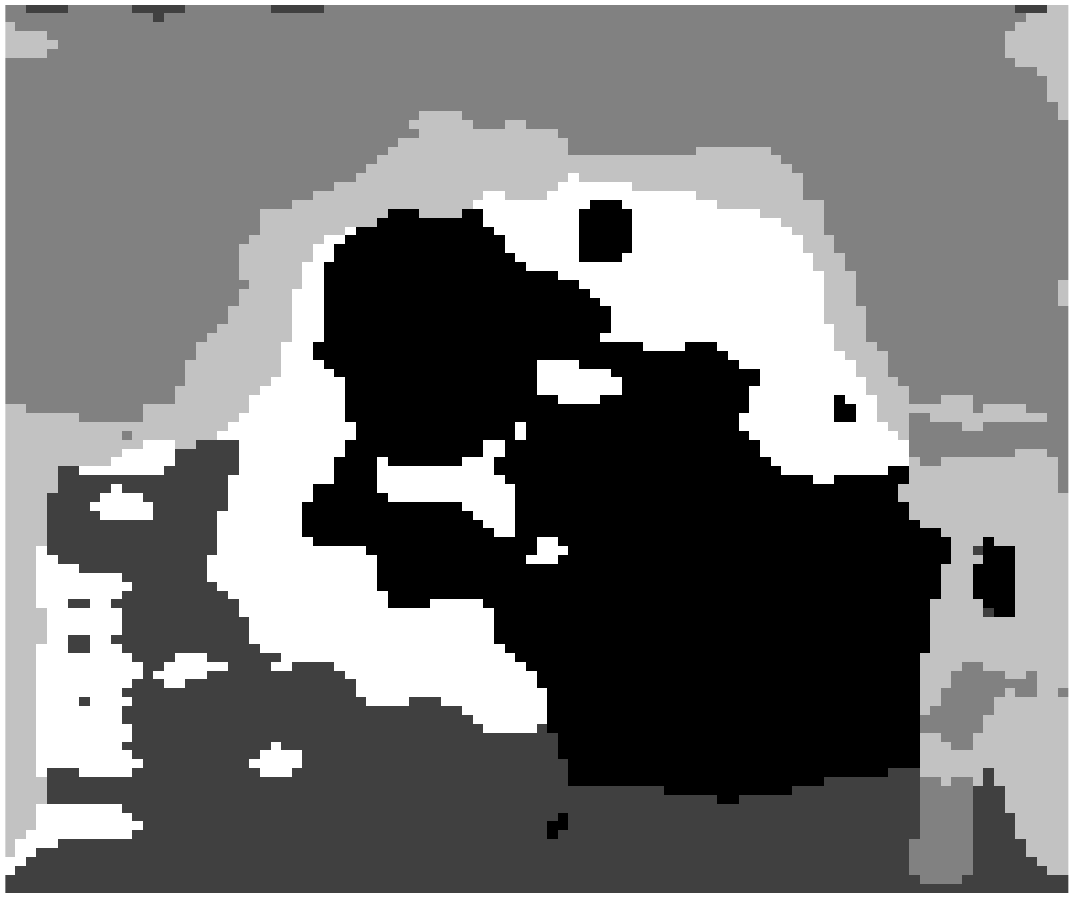}}
		\quad
		\subfloat[AFCM-ER-LP-L1] {\includegraphics[width=0.14\textwidth]{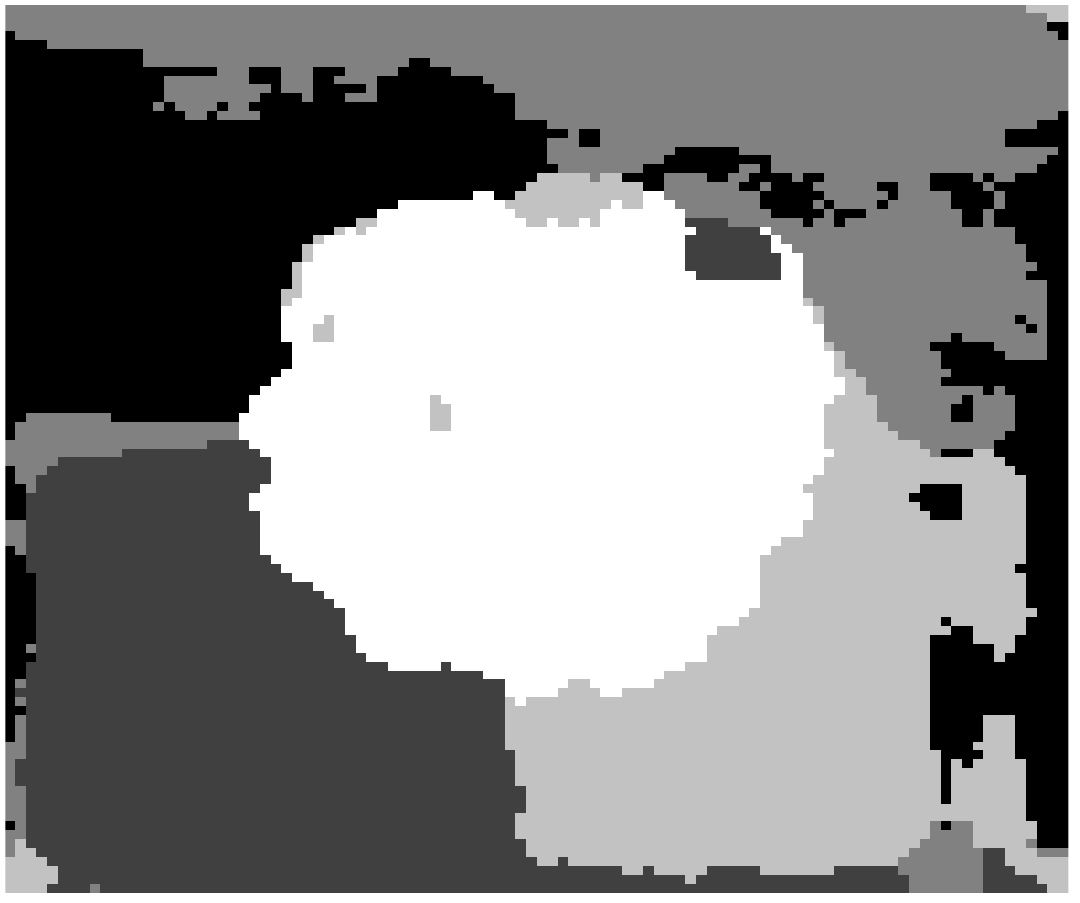}}
		\quad
		\subfloat[AFCM-ER-GS-L2] {\includegraphics[width=0.14\textwidth]{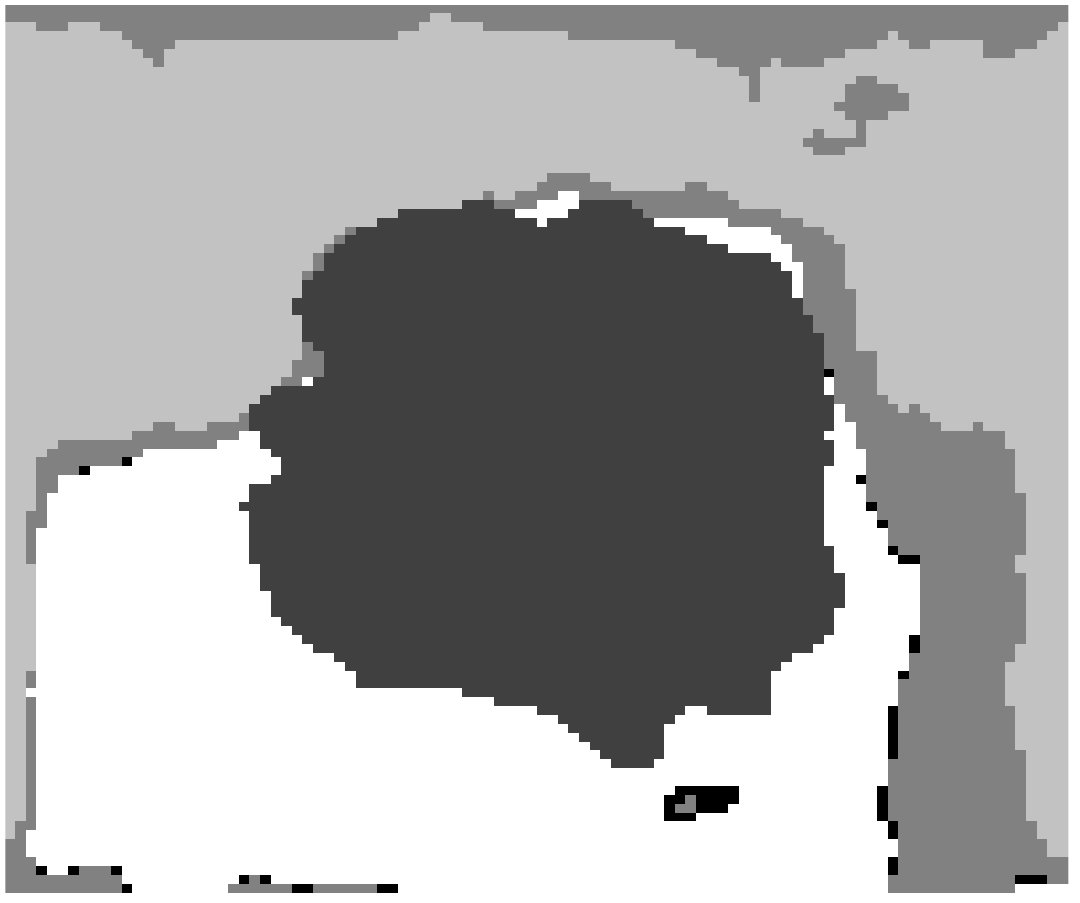}}
		\quad
		\subfloat[AFCM-ER-GS-L1] {\includegraphics[width=0.14\textwidth]{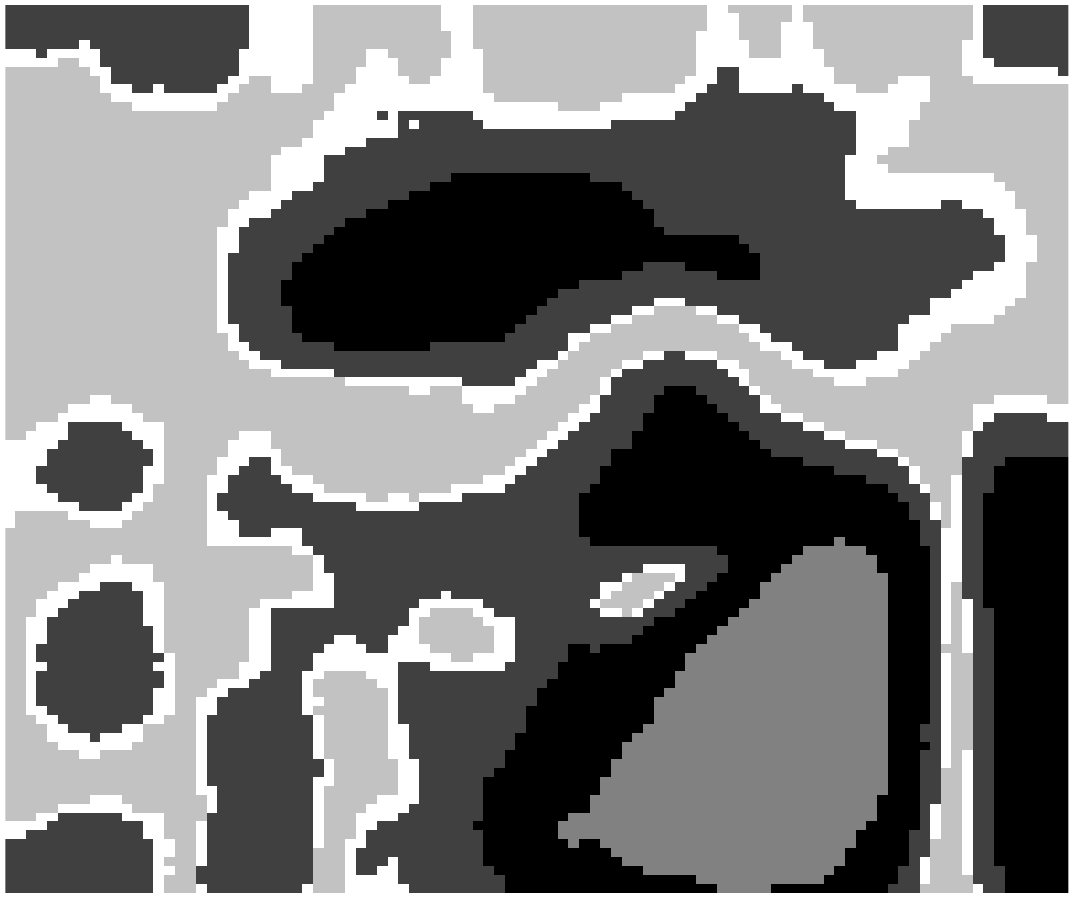}}
		\quad
		\subfloat[AFCM-ER-LS-L2] {\includegraphics[width=0.14\textwidth]{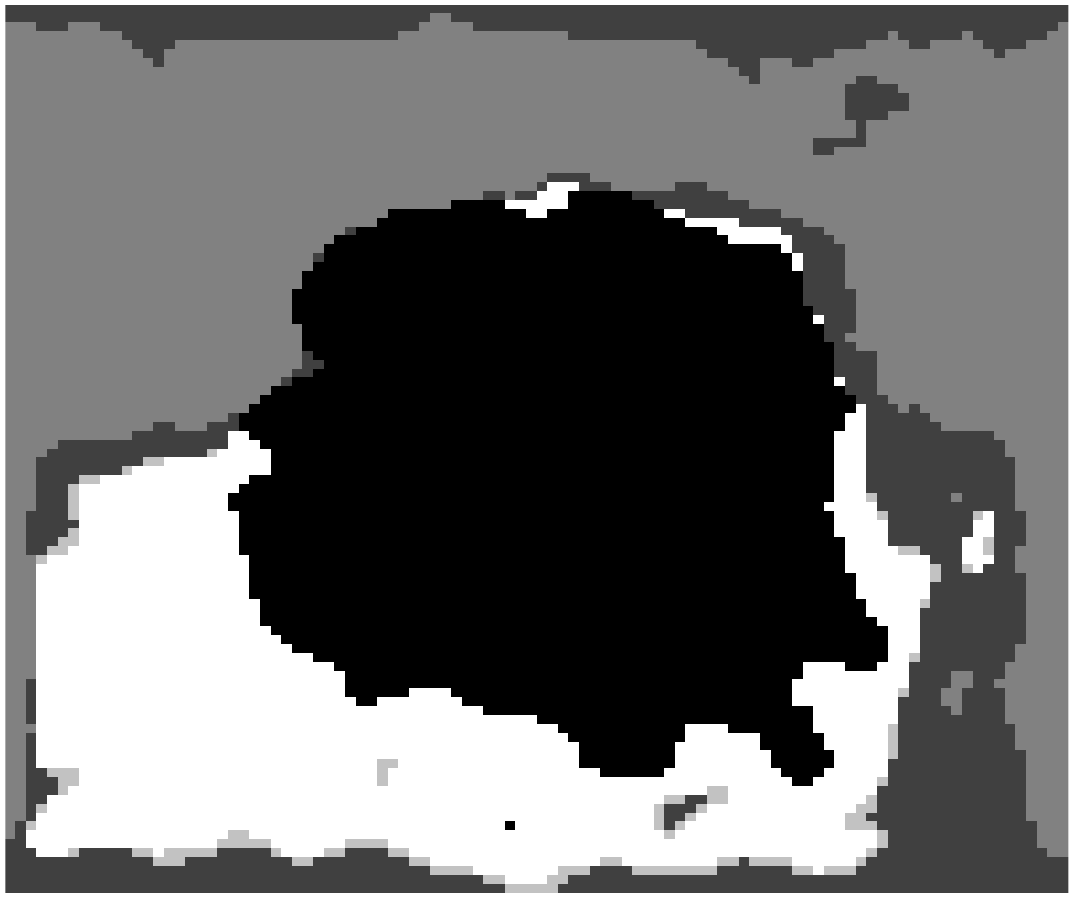}}
		\quad
		\subfloat[AFCM-ER-LS-L1] {\includegraphics[width=0.14\textwidth]{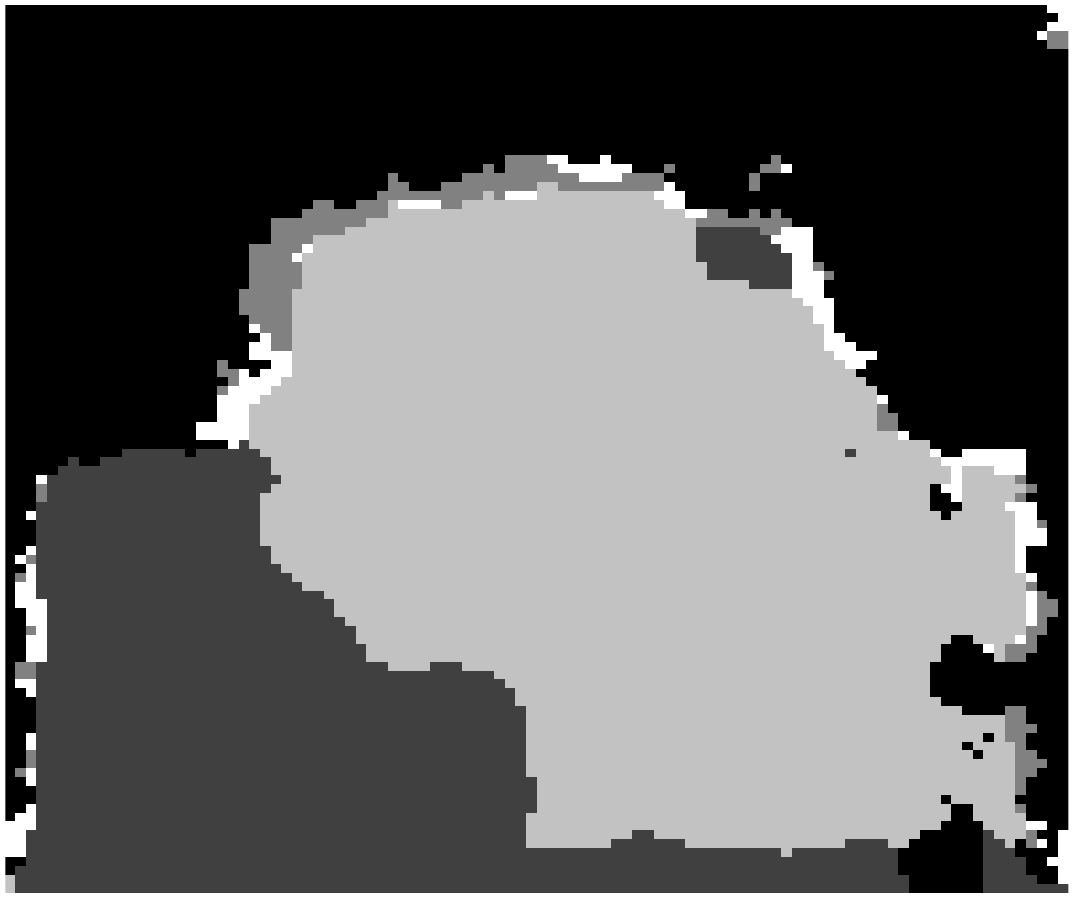}}
		\caption{Segmentation results of each algorithm for the 5-textural image without and with Gaussian noise. The first and the third row show the segmentation results for the original 5-textural image. The second and the fourth row present the obtained segmentation for the 5-textural image with Gaussian noise.}
		\label{img:ResultsTextImagesCircle}
	\end{figure}
	
\begin{figure}[!htb]
		\centering
		{\includegraphics[width=0.14\textwidth]{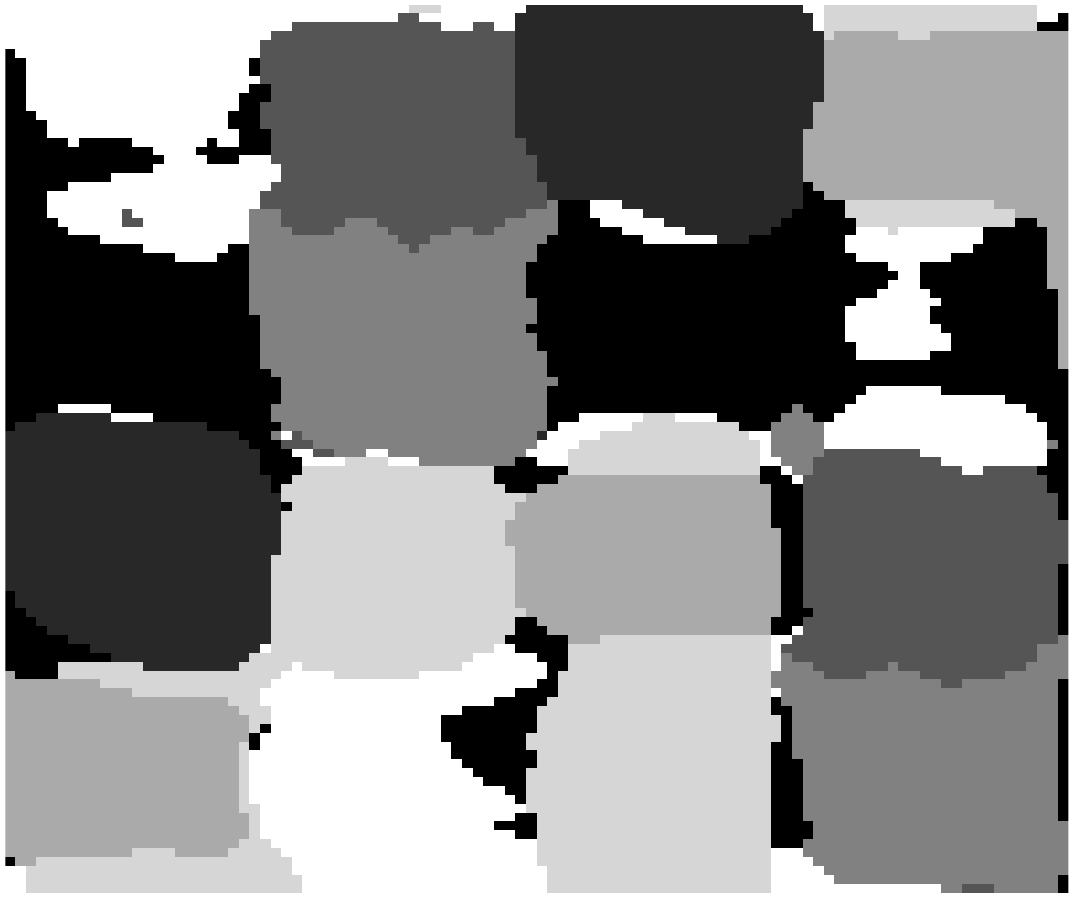}}
		\quad
		{\includegraphics[width=0.14\textwidth]{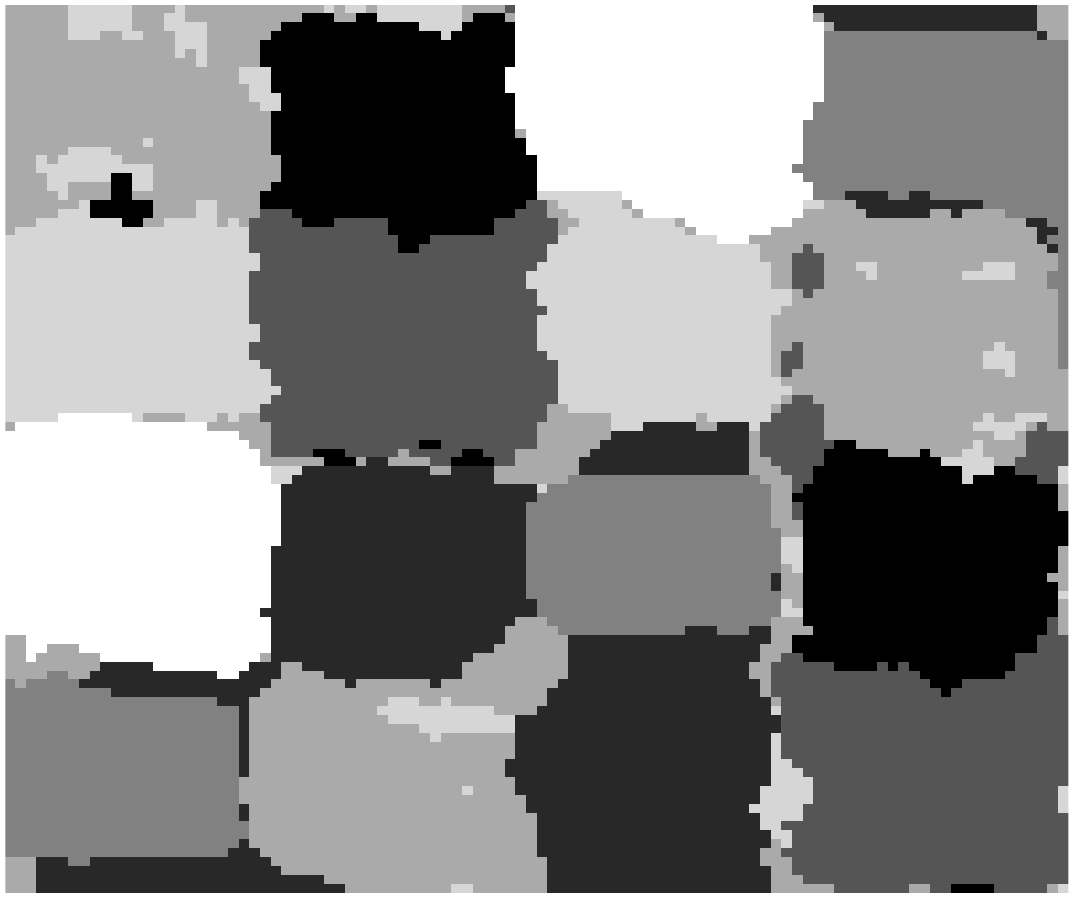}}
		\quad
		{\includegraphics[width=0.14\textwidth]{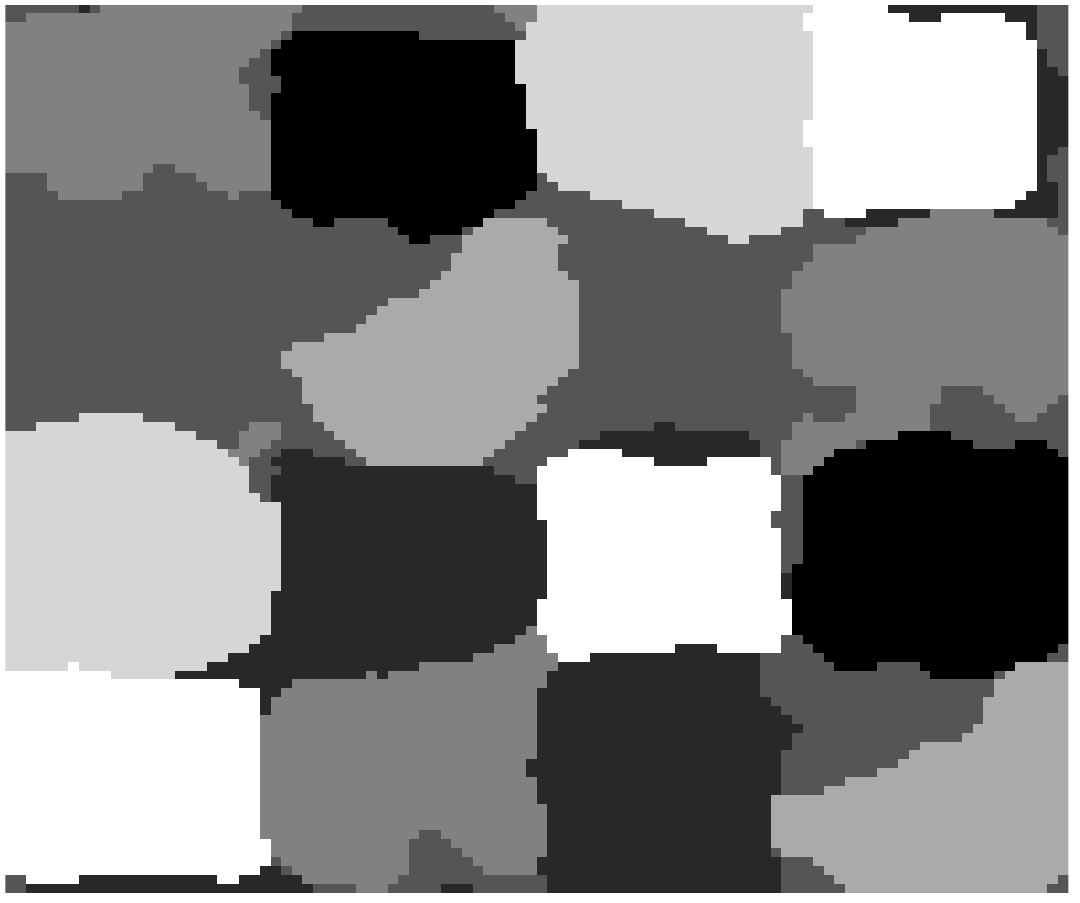}}
		\quad
		{\includegraphics[width=0.14\textwidth]{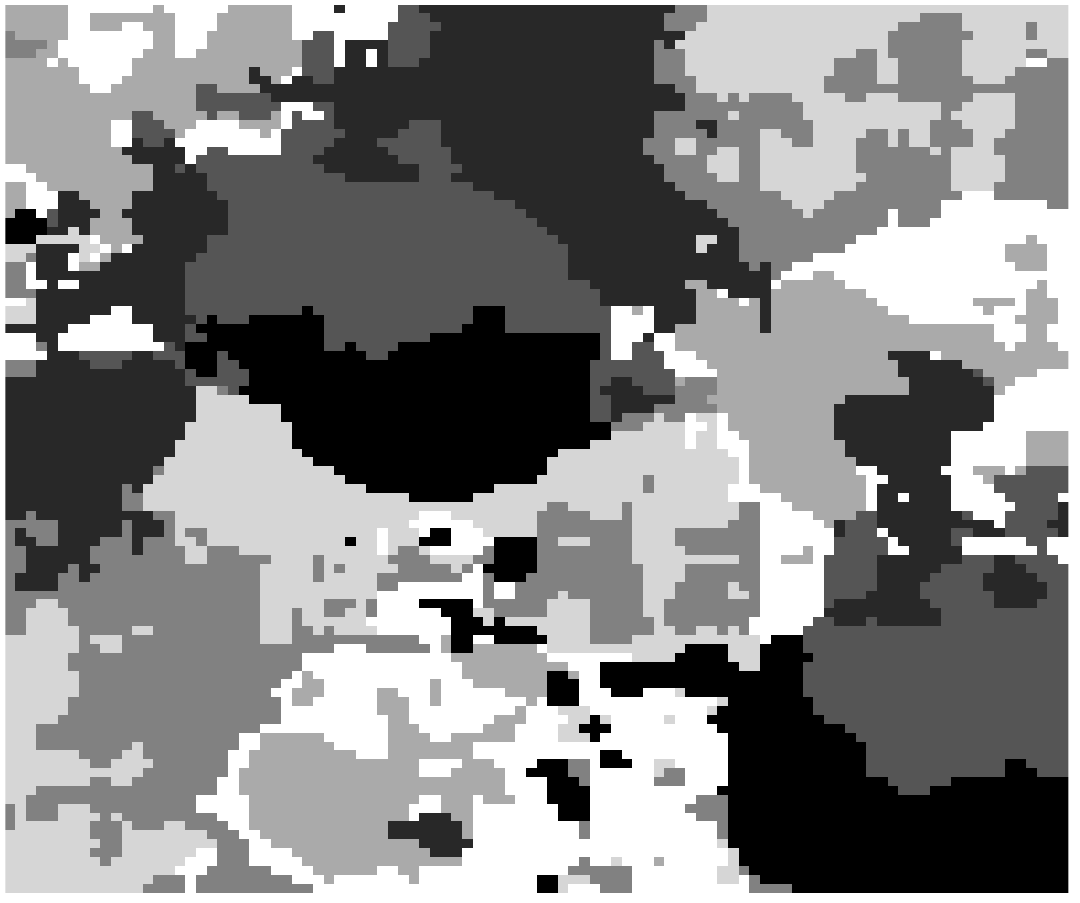}}
		\quad
		{\includegraphics[width=0.14\textwidth]{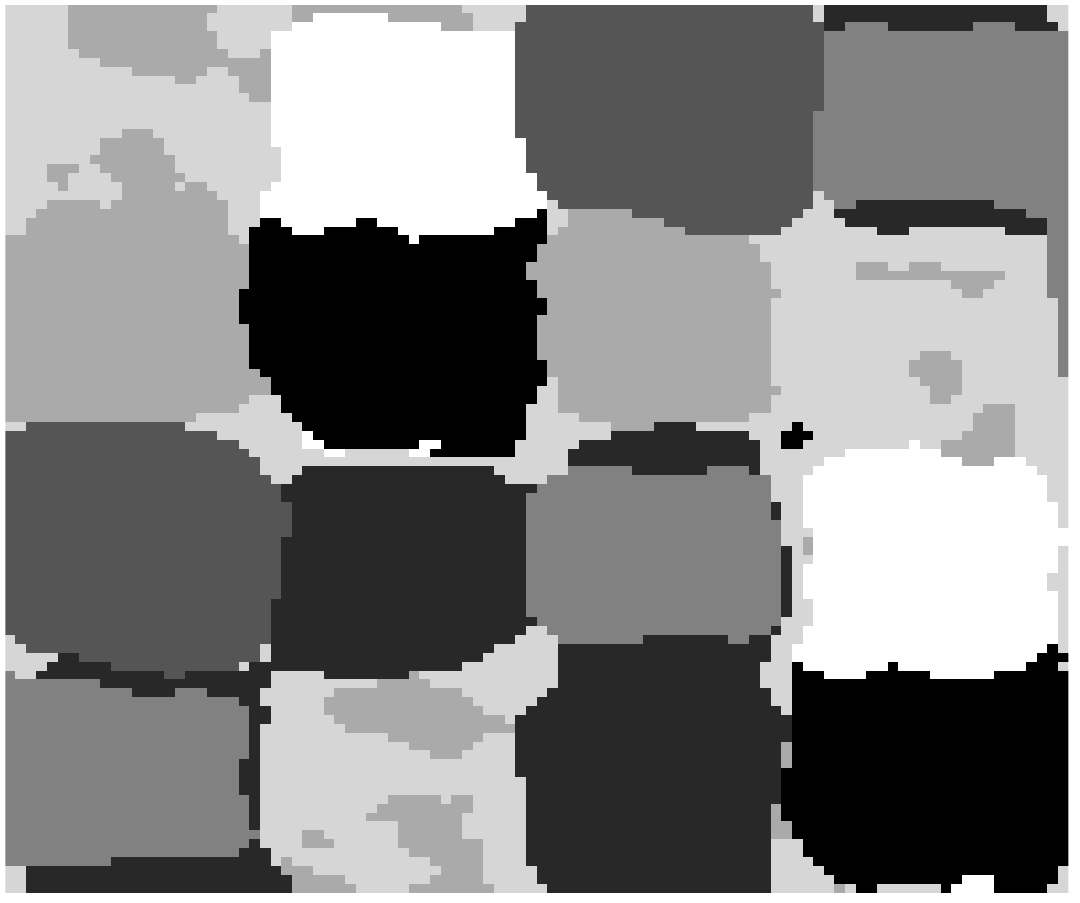}}
		\quad
		{\includegraphics[width=0.14\textwidth]{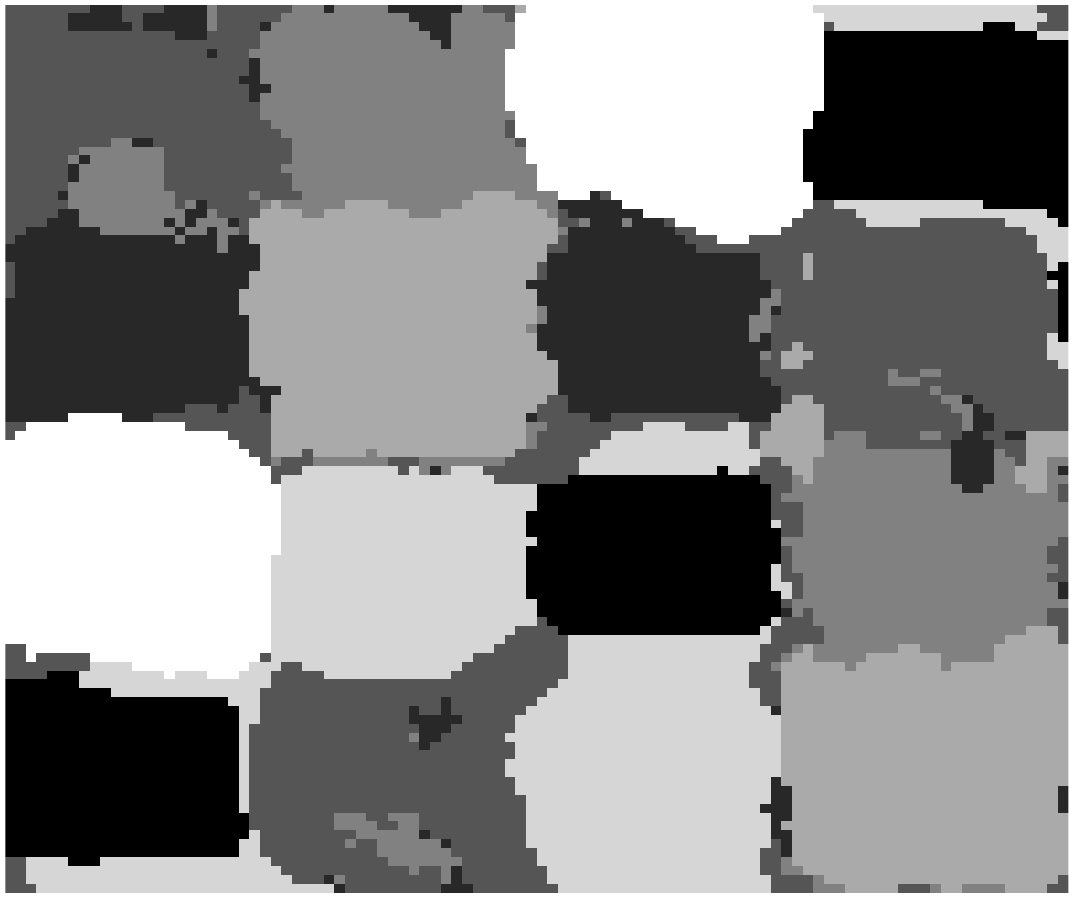}}
		\quad
		\subfloat[FCM-ER-L2] {\includegraphics[width=0.14\textwidth]{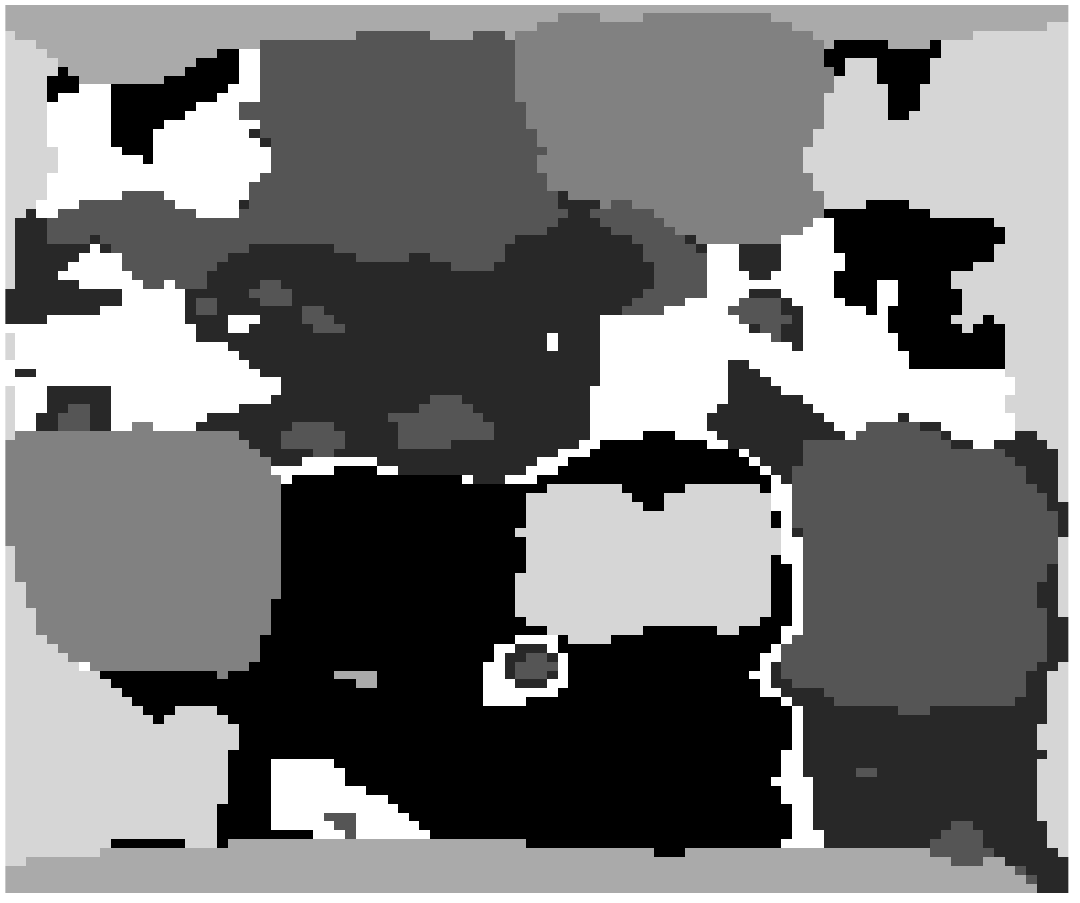}}
		\quad
		\subfloat[FCM-ER-L1] {\includegraphics[width=0.14\textwidth]{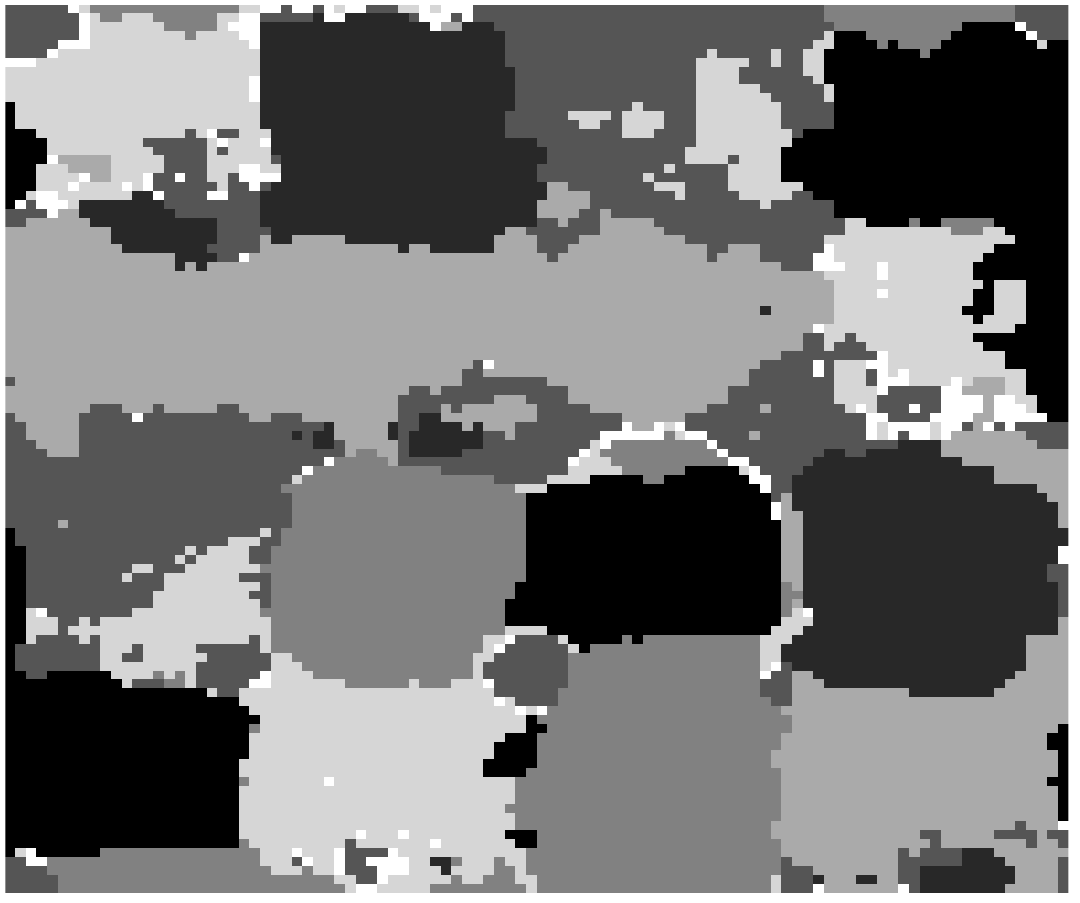}}
		\quad
		\subfloat[AFCM-ER-M] {\includegraphics[width=0.14\textwidth]{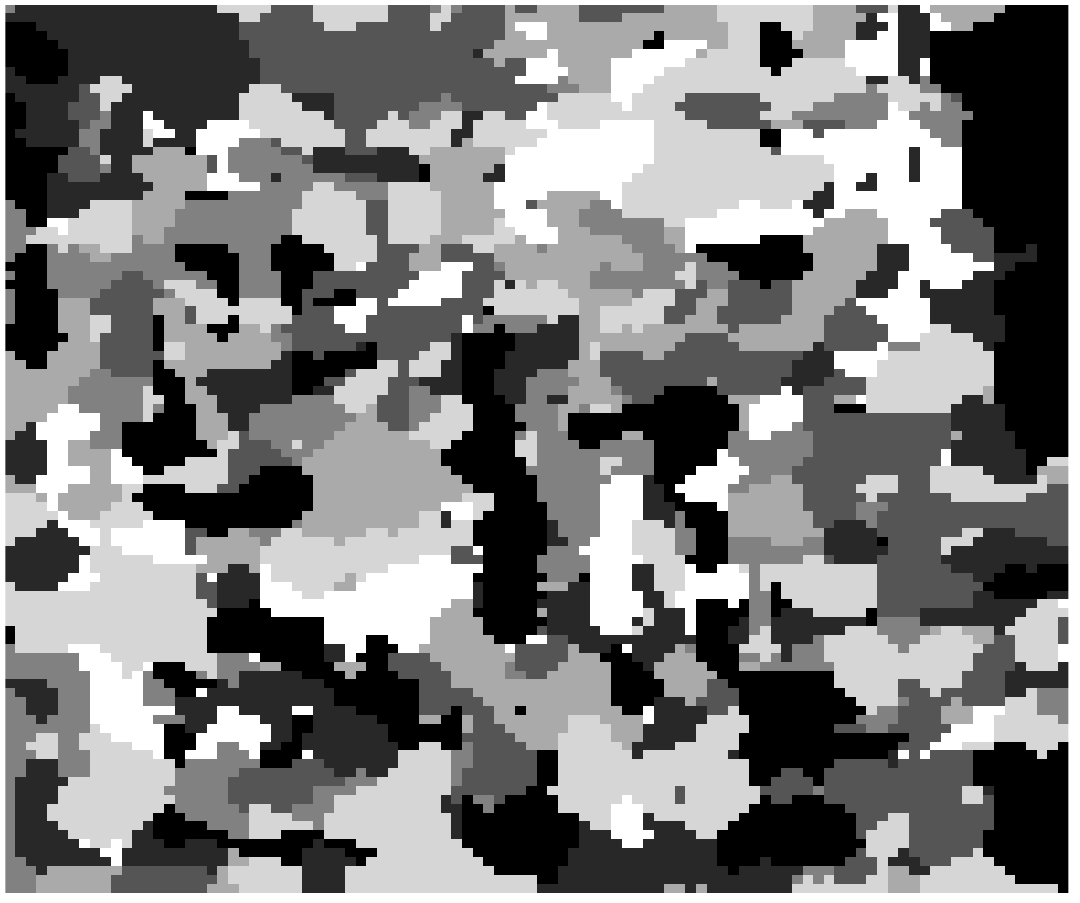}}
		\quad
		\subfloat[AFCM-ER-Mk] {\includegraphics[width=0.14\textwidth]{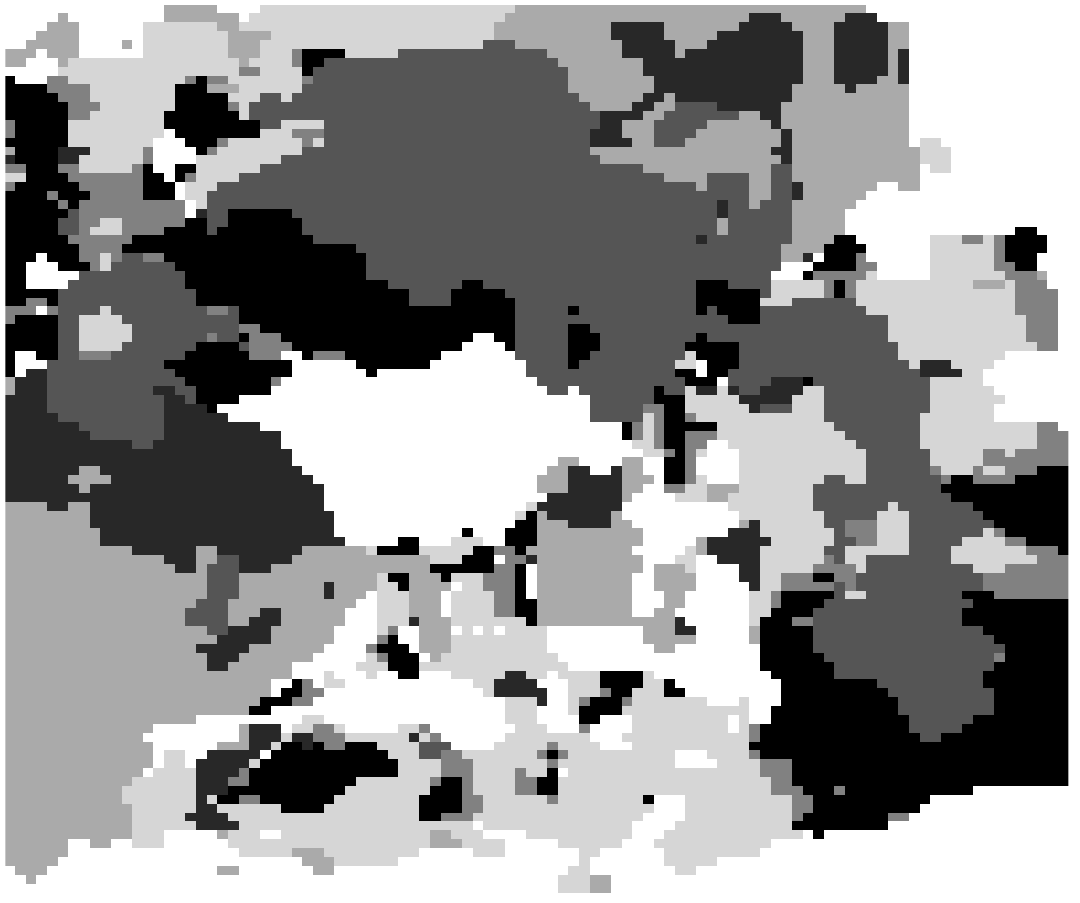}}
		\quad
		\subfloat[AFCM-ER-GP-L2] {\includegraphics[width=0.14\textwidth]{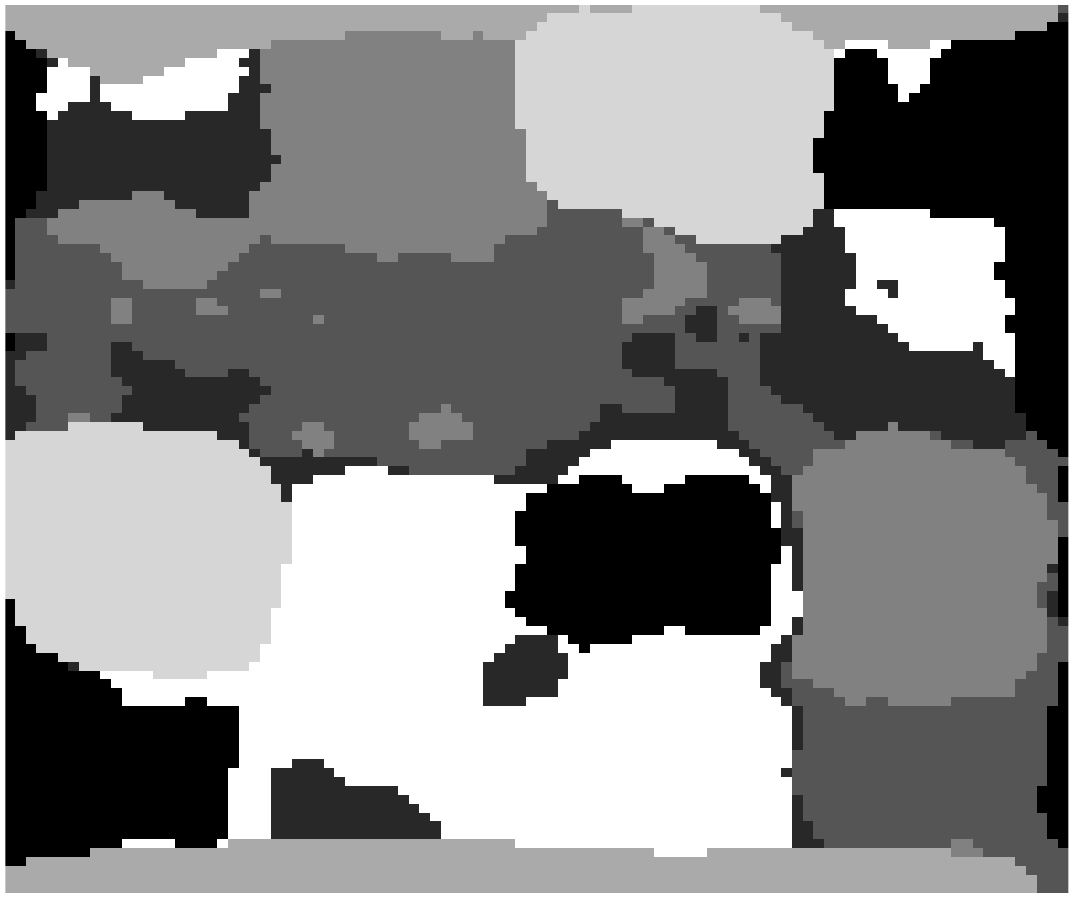}}
		\quad
		\subfloat[AFCM-ER-GP-L1] {\includegraphics[width=0.14\textwidth]{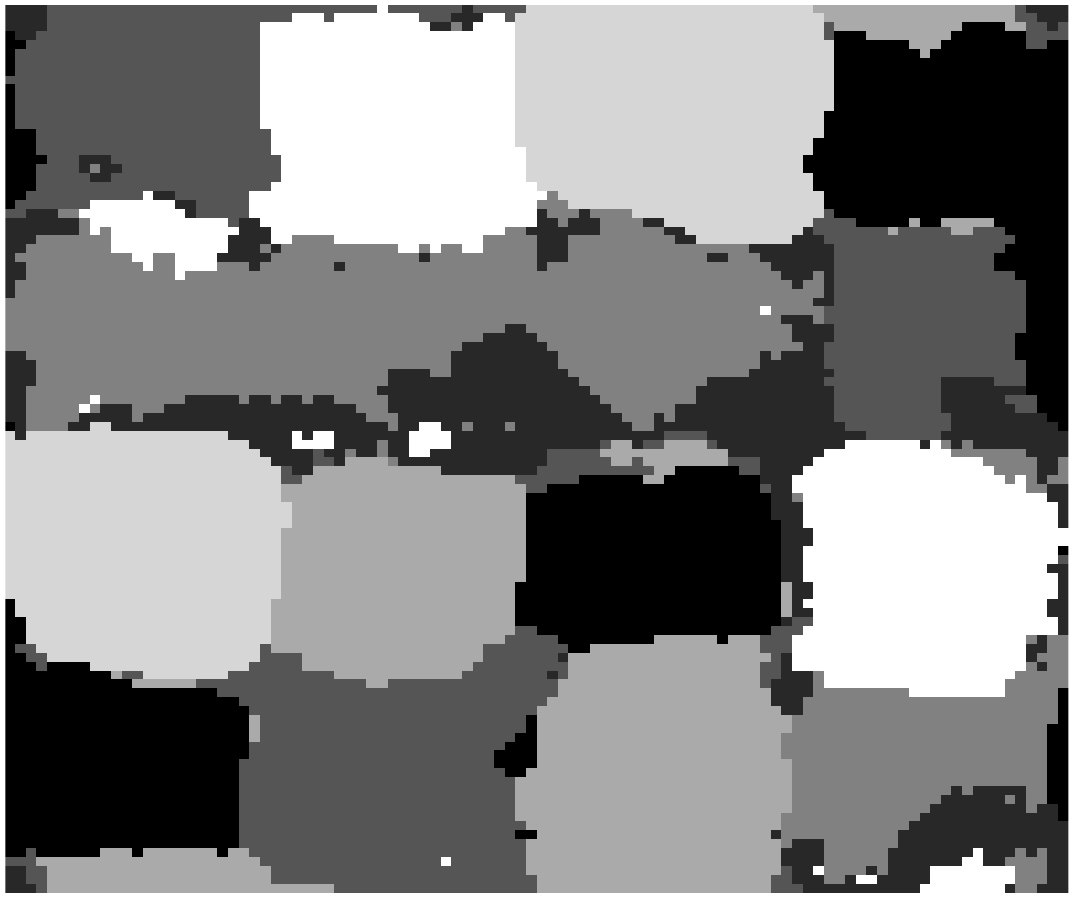}}
		\quad
		{\includegraphics[width=0.14\textwidth]{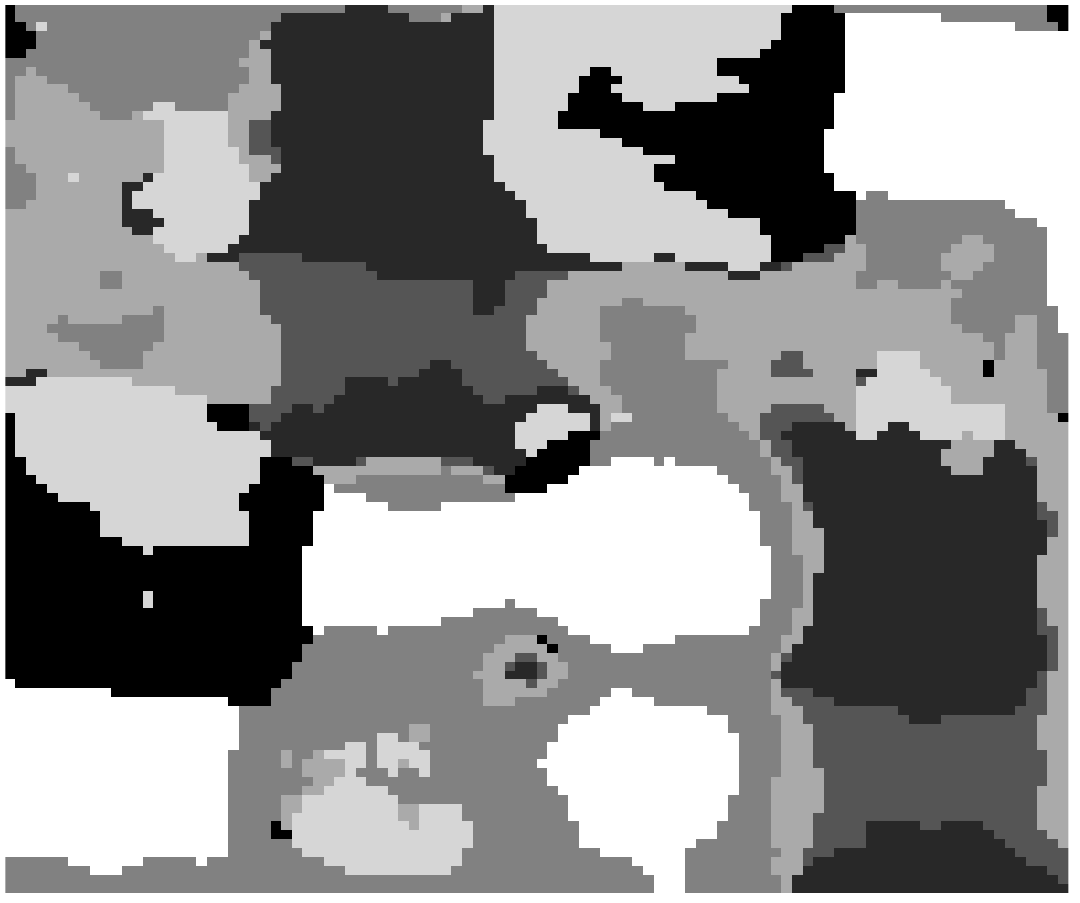}}
		\quad
		{\includegraphics[width=0.14\textwidth]{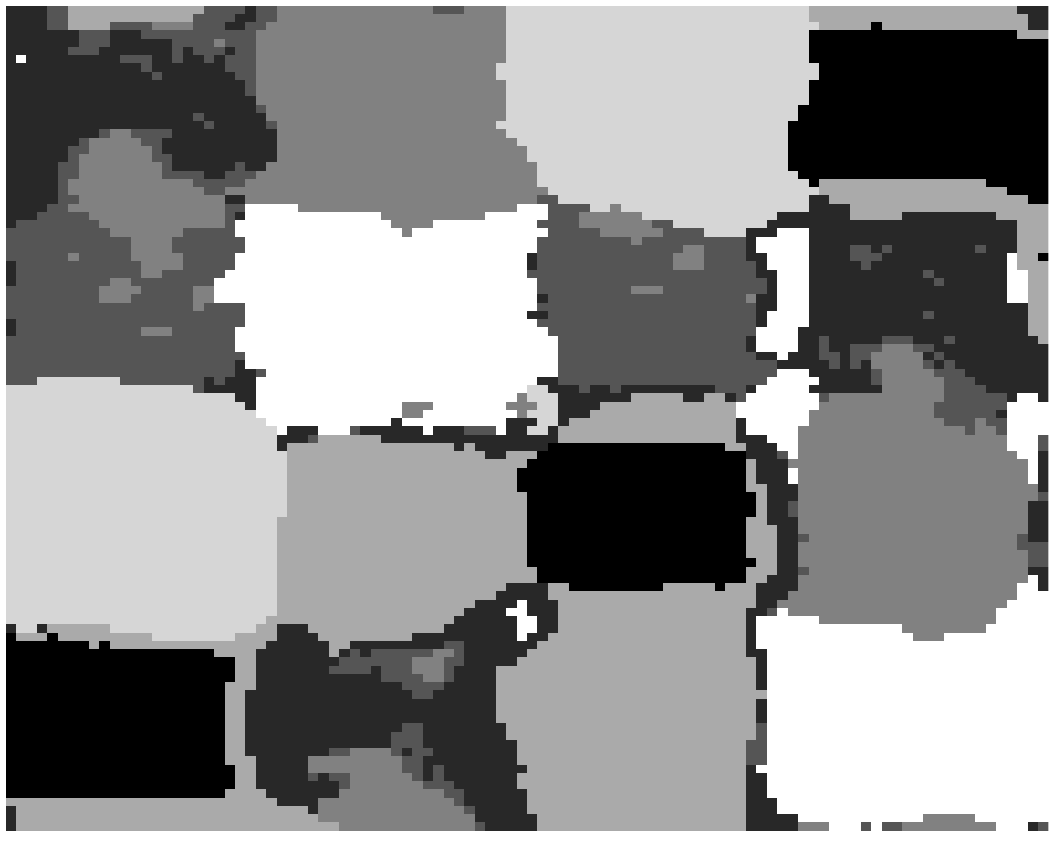}}
		\quad
		{\includegraphics[width=0.14\textwidth]{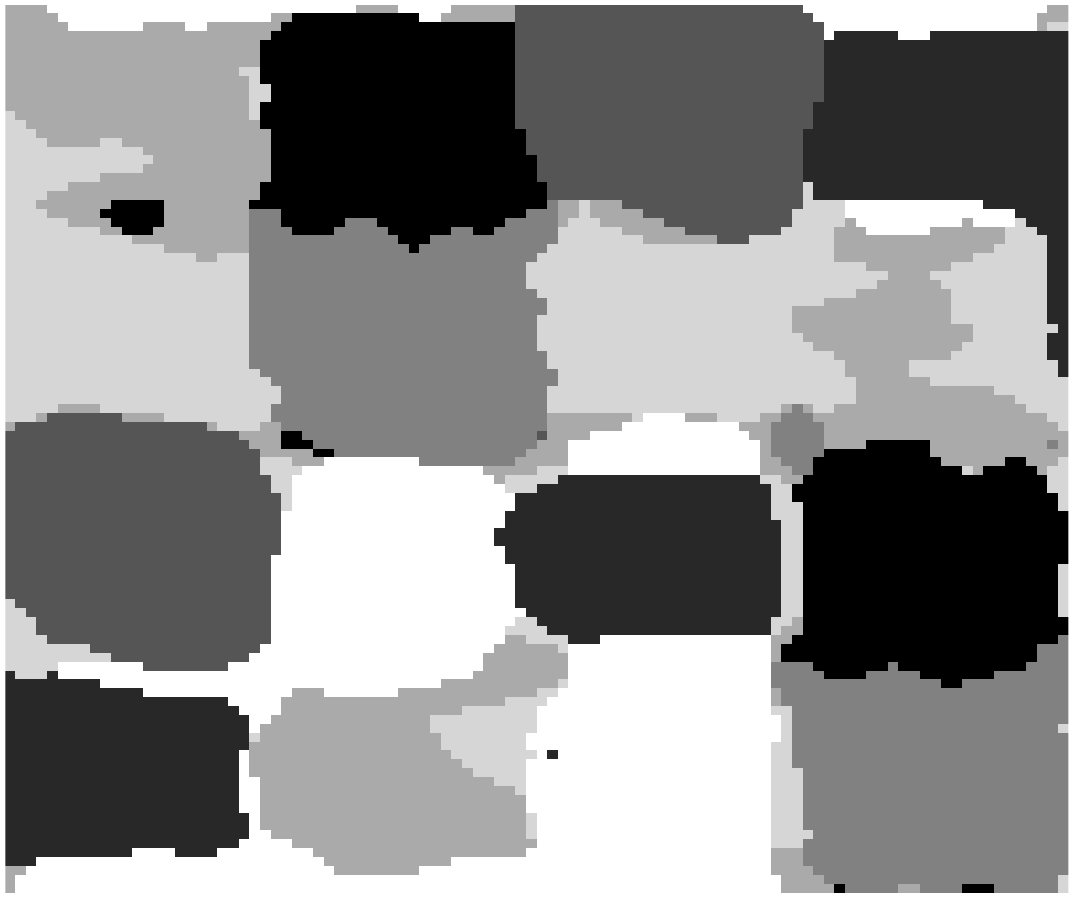}}
		\quad
		{\includegraphics[width=0.14\textwidth]{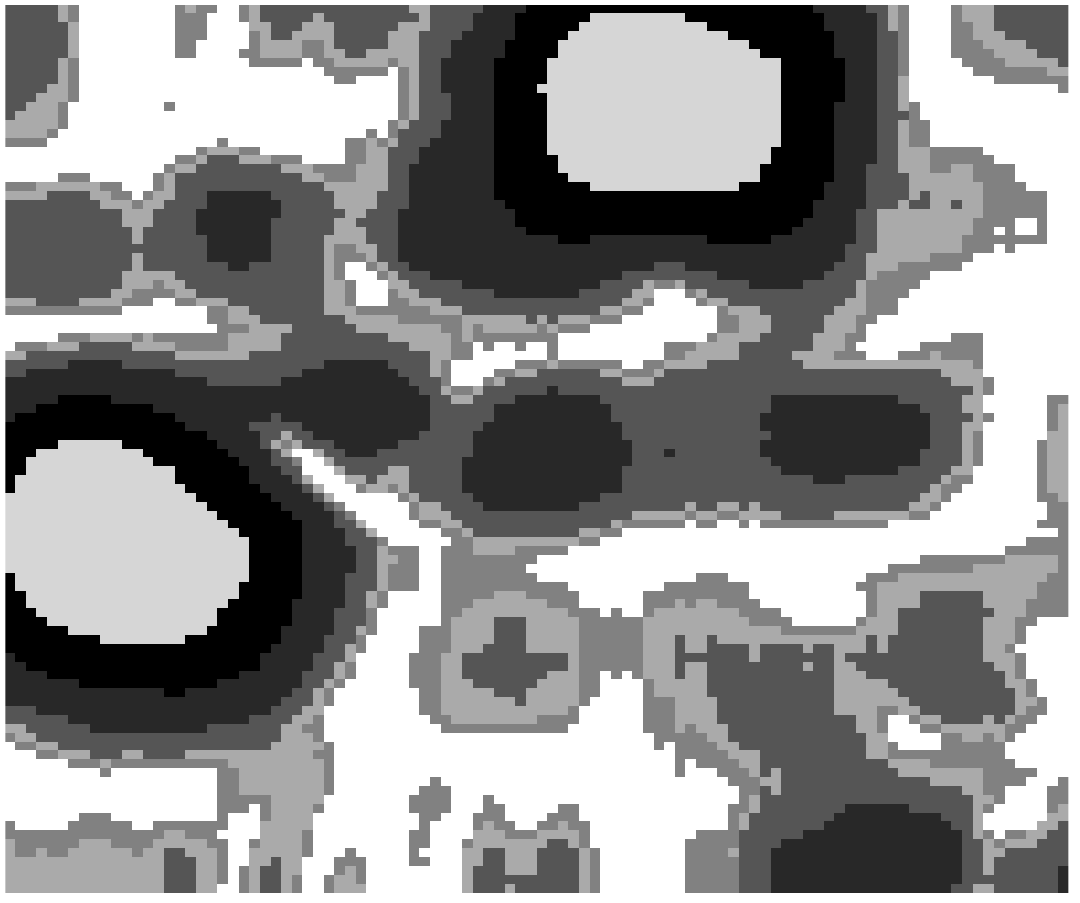}}
		\quad
		{\includegraphics[width=0.14\textwidth]{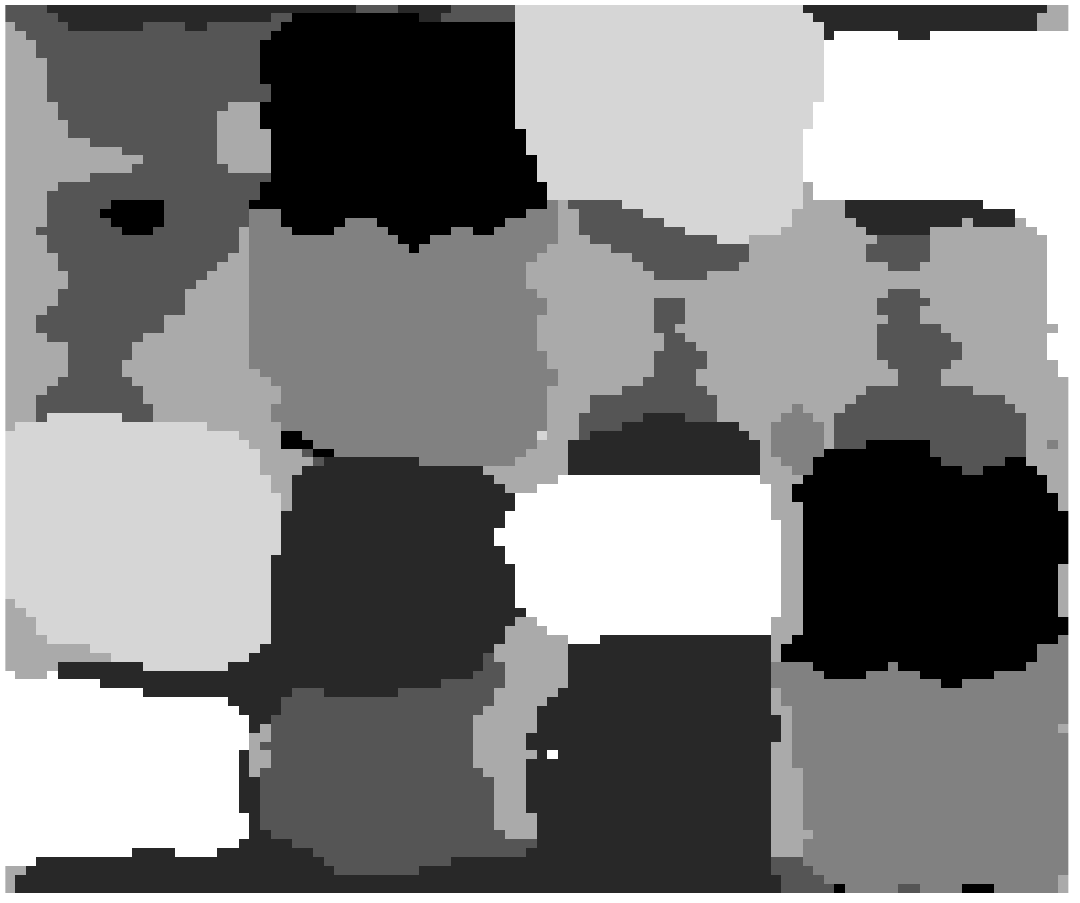}}
		\quad
		{\includegraphics[width=0.14\textwidth]{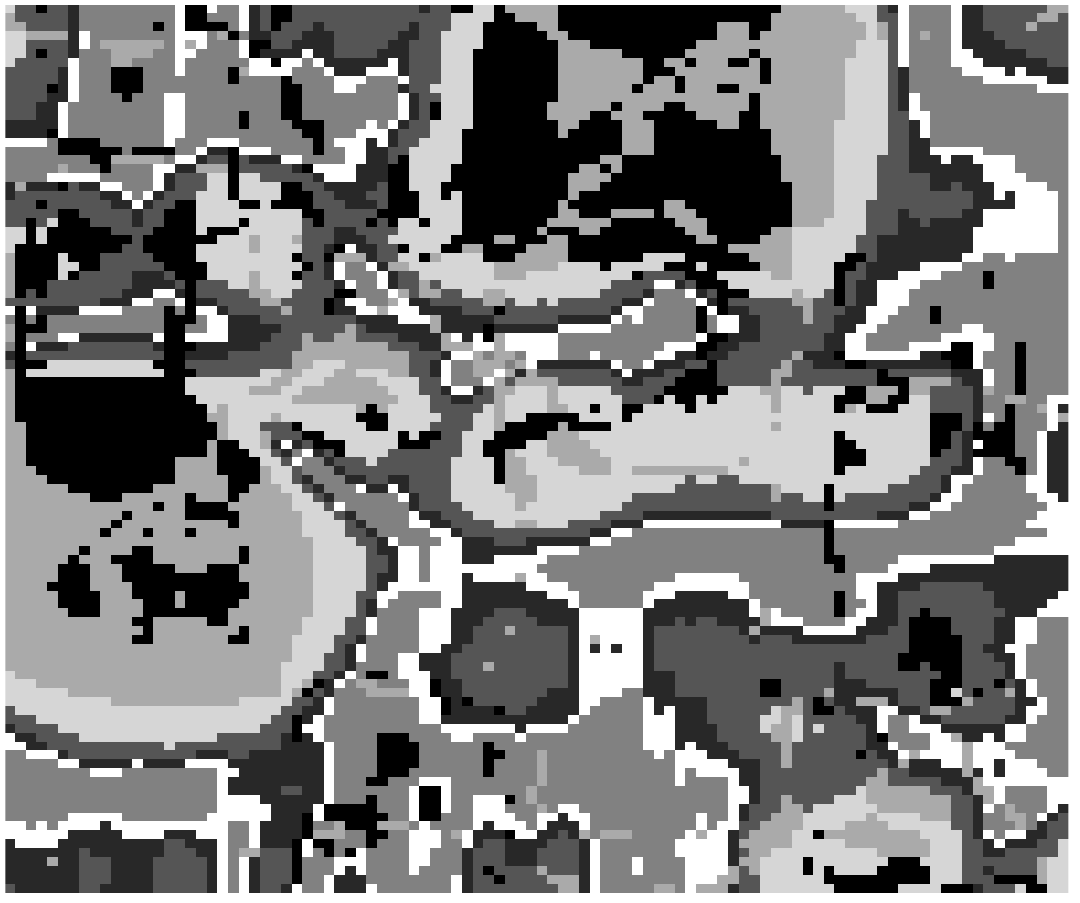}}
		\quad
		\subfloat[AFCM-ER-LP-L2] {\includegraphics[width=0.14\textwidth]{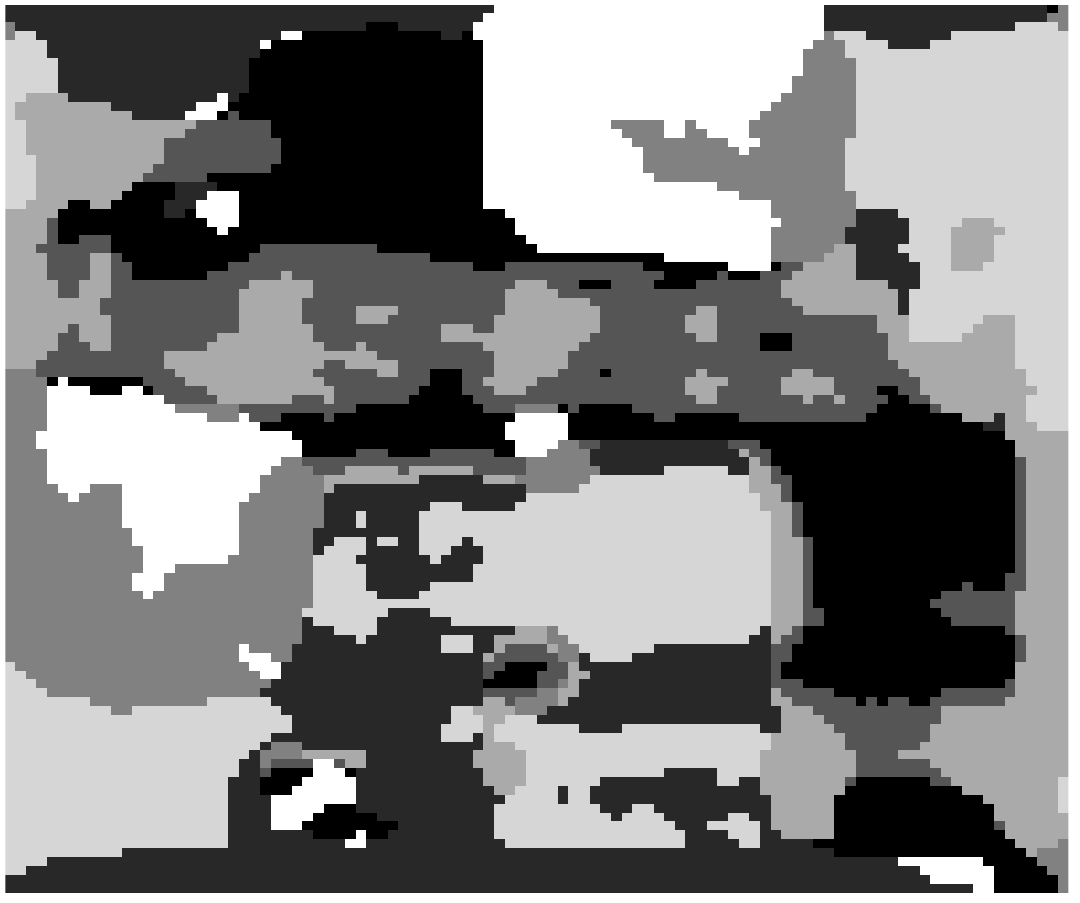}}
		\quad
		\subfloat[AFCM-ER-LP-L1] {\includegraphics[width=0.14\textwidth]{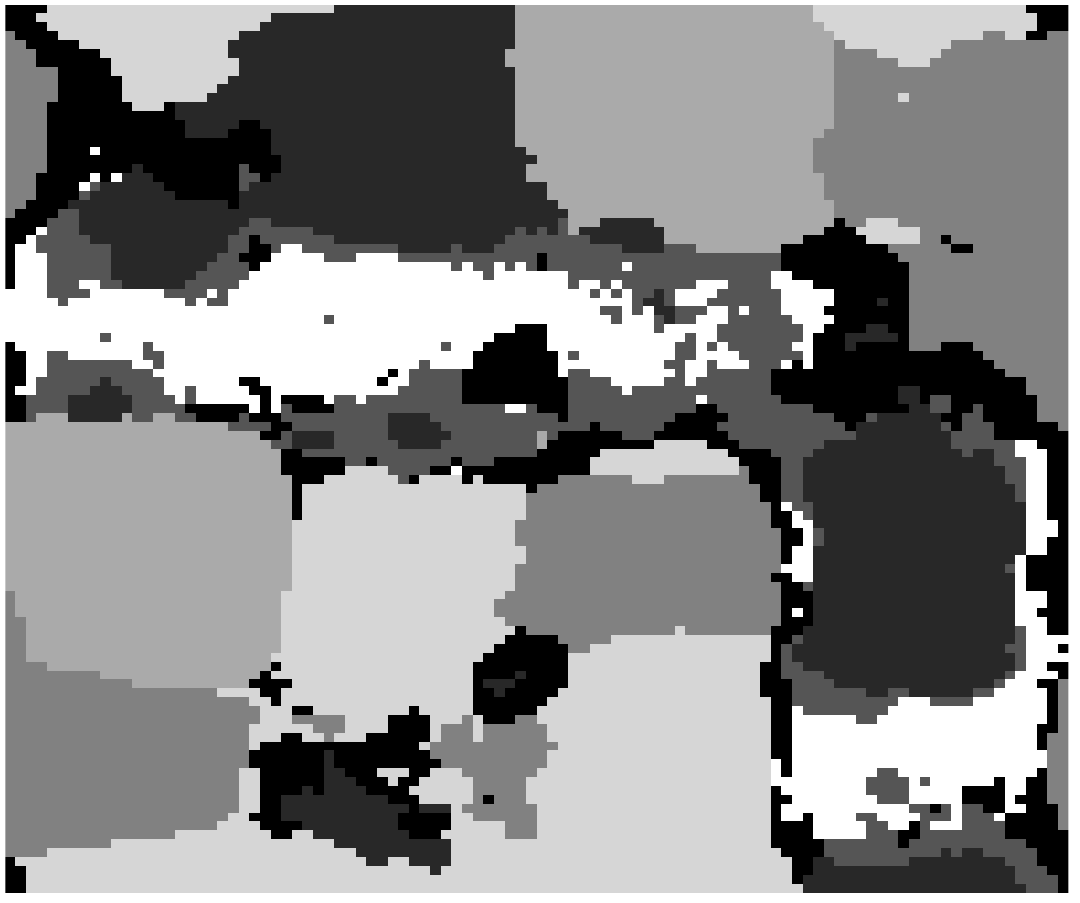}}
		\quad
		\subfloat[AFCM-ER-GS-L2] {\includegraphics[width=0.14\textwidth]{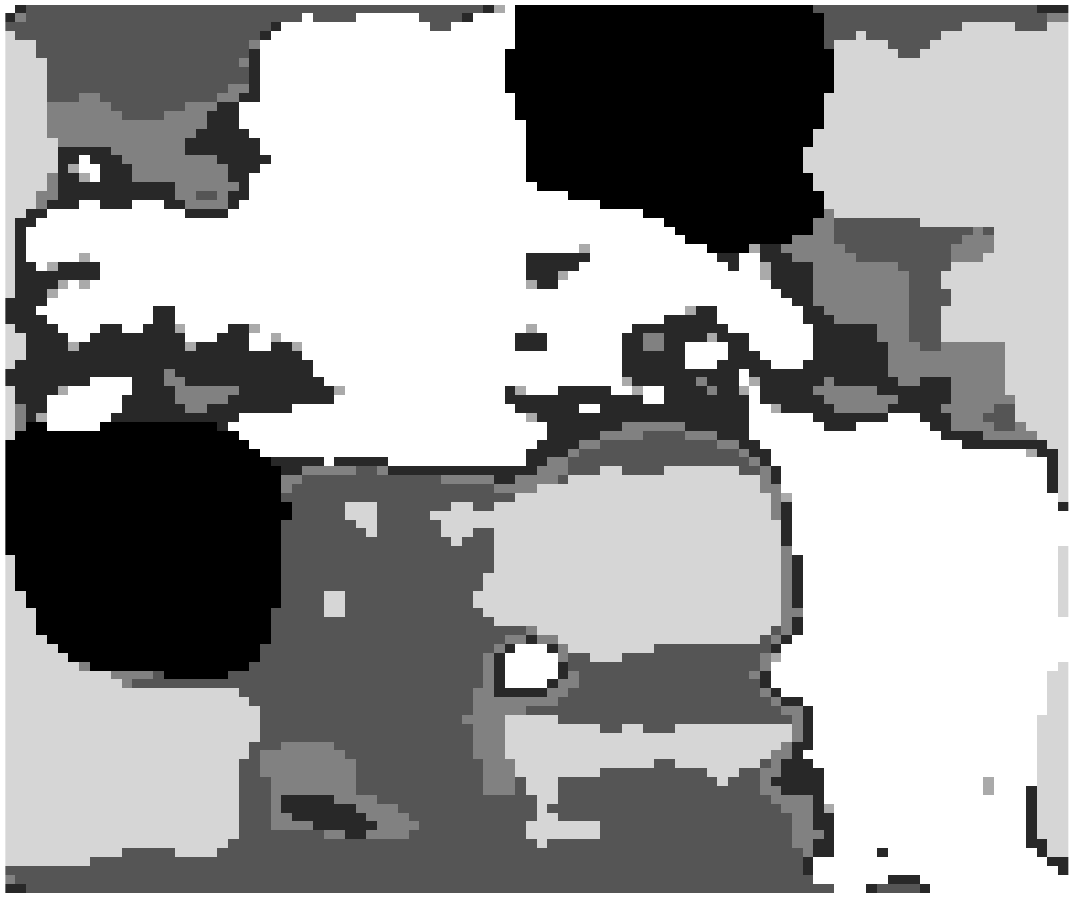}}
		\quad
		\subfloat[AFCM-ER-GS-L1] {\includegraphics[width=0.14\textwidth]{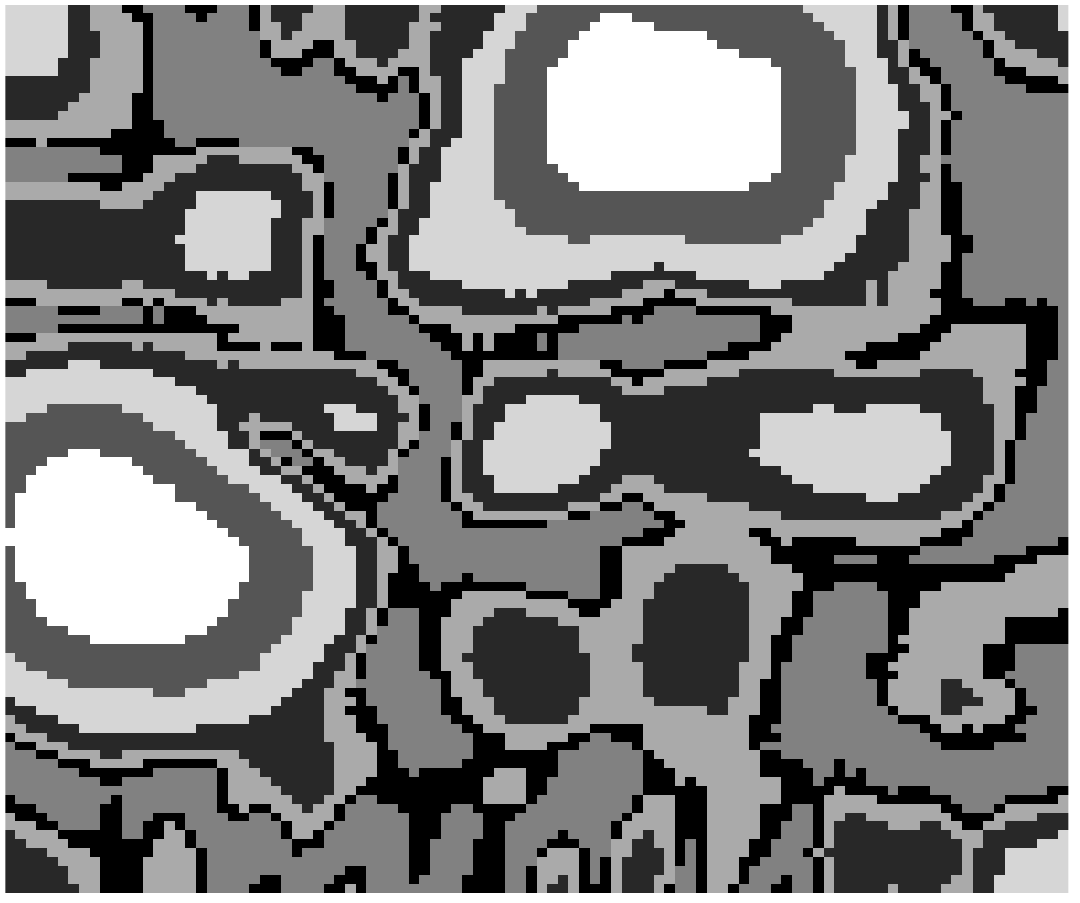}}
		\quad
		\subfloat[AFCM-ER-LS-L2] {\includegraphics[width=0.14\textwidth]{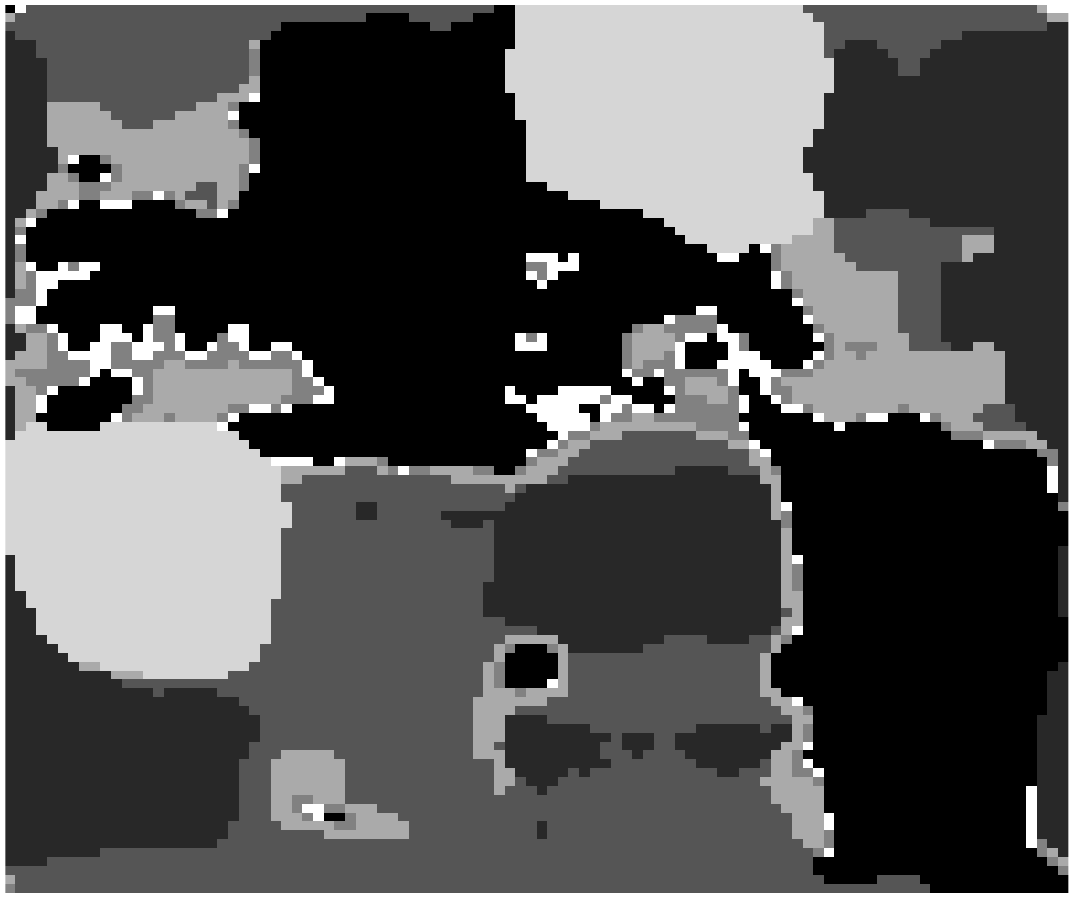}}
		\quad
		\subfloat[AFCM-ER-LS-L1] {\includegraphics[width=0.14\textwidth]{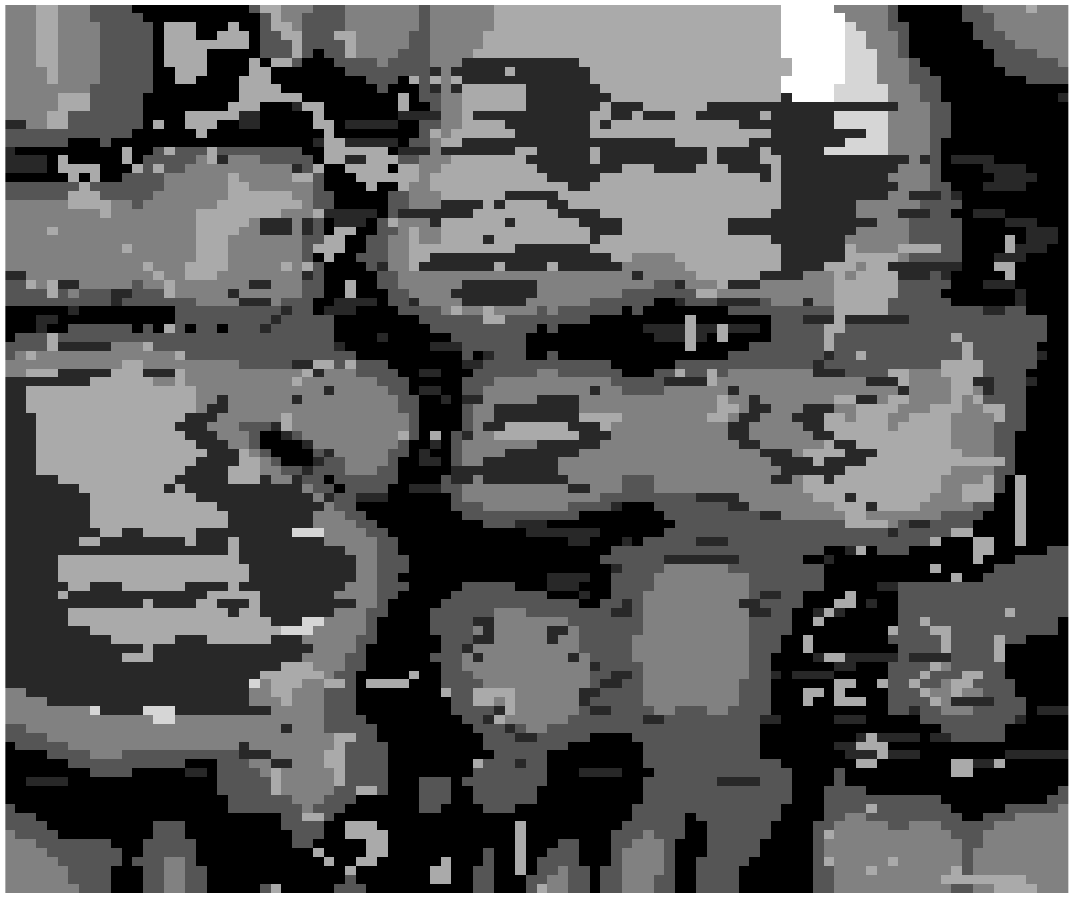}}
		\caption{Segmentation results of each algorithm for the 7-textural image without and with Gaussian noise. The first and the third row show the segmentation results for the original 7-textural image. The second and the fourth row present the obtained segmentation for the 7-textural image with Gaussian noise.}
		\label{img:ResultsTextImagesImage7}
	\end{figure}
	
In this application results, we observe that, in general, methods based on City-Block distance degrade their performance more slowly compared with those based on Mahalanobis and Euclidean distance in a noisy environment. The first and second rows of Figures \ref{img:ResultsTextImagesCross}, \ref{img:ResultsTextImagesCircle}, \ref{img:ResultsTextImagesImage7}, and Table \ref{tab:imageResults} show that for the 2-textural image without noise the best result according to $HUL$ was obtained by FCM-ER-L1 algorithm. However, AFCM-ER-LP-L1 and AFCM-ER-GP-L1 achieved the second and third best performance, respectively. For the $ARI$ values, AFCM-ER-LP-L1 and AFCM-ER-GP-L1 outperform the other approaches. The AFCM-ER-LS-L2 algorithm presented the worst results for both indices. In the case of the 5-textural image, AFCM-ER-LS-L2 showed a higher value of $HUL$ compared with the other algorithms, but AFCM-ER-GP-L1 yielded the highest clustering result for $ARI$ and produced better segmentation results. AFCM-ER-M presented the worst results for $HUL$ and AFCM-ER-GS-L1 for $ARI$. 
For 7-textural image, AFCM-ER-M reached the best results whatever the considered index. 
AFCM-ER-GS-L1 and AFCM-ER-LS-L1 obtained the worst performance for $HUL$ and $ARI$ respectively.

For the case of images with Gaussian noise, Table \ref{tab:imageResultsN} and the second and fourth row of the Figures \ref{img:ResultsTextImagesCross}, \ref{img:ResultsTextImagesCircle} and \ref{img:ResultsTextImagesImage7} show that, for both indexes $HUL$ and $ARI$,  AFCM-ER-GP-L1 and AFCM-ER-GP-L2 obtained the best results for the 2 and 5-textural images respectively. For 7-textural image according to $HUL$, AFCM-ER-GP-L2 presented better clustering results, and AFCM-ER-GP-L1 is the best according to $ARI$. AFCM-ER-Mk achieved the worst results for 2-textural images for $HUL$ and AFCM-ER-LS-L2 for $ARI$. 
The AFCM-ER-M algorithm had the worst performance regarding both indexes for 5 and 7-textural images. 

\section{Conclusions} \label{sect:conclusion}

In this paper, we proposed new fuzzy clustering algorithms based on suitable Adaptive Euclidean, Mahalanobis and City-Block distances, and entropy regularization. Moreover, adaptive distances were used, which changes at each algorithm iteration and can either be the same for all clusters or different from one cluster to another. These kind of dissimilarity measures are suitable to learn the weights of the variables during the clustering process, improving the performance of the algorithms.

The proposed algorithms are based on the minimization of clustering criteria, that is performed on three steps (representation, weighting, and assignment) providing a fuzzy partition, a representative for each fuzzy cluster, and a relevance weight for each variable (if the comparison between the objects and the prototypes of the clusters uses a Euclidean distance) or a matrix of weights (if the comparison between the objects and the prototypes of the clusters uses a quadratic distance). To take into account the relevance weights of the variables, we considerate two types of constraints. In the first type, the sum of the weights of the variables or the sum of the variables weights on each cluster must be equal to one, whereas the other type assumes that the product of the weights of the variables or the product of the weights of the variables on each cluster must be equal to one.

The performance and usefulness of the proposed algorithms have been illustrated through experiments carried out on suitable synthetic and real datasets. In the first simulation study, it is observed that for data with the cluster covariance matrices diagonal and almost the same, %the methods with global adaptive distance for the product of the weights of the variables outperform the methods based on local adaptive distance. The same happens for the case of the sum of the weights of the variables.
almost all the methods with global adaptive distance outperformed their respective variants based on local adaptive distance, i.e., the methods AFCM-ER-M, AFCM-ER-GP-L1, AFCM-ER-GP-L2, AFCM-ER-GS-L1 outperformed, respectively, the methods AFCM-ER-Mk, AFCM-ER-LP-L1, AFCM-ER-LP-L2, AFCM-ER-LS-L1. AFCM-ER-GS-L2 outperforms AFCM-ER-LS-L2 concerning $ARI$ index, but AFCM-ER-LS-L2 surpasses AFCM-ER-GS-L2 regarding $HUL$ index. However, for data with cluster covariance matrices diagonal but unequal, %methods with local adaptive distance presented better results compared with methods based on the global adaptive distance for the product of the weights of the variables. This same behavior was presented for the sum of the weights of the variables.
almost all the methods with local adaptive distance presented better results compared with their respective variants based on global adaptive distance: the methods AFCM-ER-Mk, AFCM-ER-LP-L2, AFCM-ER-LS-L1, AFCM-ER-LS-L2 surpassed, respectively, the methods AFCM-ER-M, AFCM-ER-GP-L2, AFCM-ER-GS-L1, AFCM-ER-GS-L2. AFCM-ER-LP-L1 surpasses AFCM-ER-GP-L1 concerning $ARI$ index, but AFCM-ER-GP-L1 outperforms AFCM-ER-LP-L1 regarding $HUL$ index. %For datasets with cluster covariance matrices diagonal but almost the same, the best result was presented by the proposed algorithm based on Mahalanobis distance, where the covariance matrix is the same for all clusters (AFCM-ER-M). For the last data configuration when the cluster covariance matrices are not diagonal and unequal, the algorithm based on Mahalanobis distance, where the covariance matrices are estimated locally (AFCM-ER-Mk) outperforms the other approaches. 
For datasets with cluster covariance matrices that are almost the same but not diagonal, the best performance for $HUL$ was presented by AFCM-ER-GP-L1 and for $ARI$ by AFCM-ER-M. Concerning this last index, the methods with global adaptive distance outperformed their respective variants based on local adaptive distance, i.e., the methods AFCM-ER-M, AFCM-ER-GP-L1, AFCM-ER-GP-L2, AFCM-ER-GS-L1 and AFCM-ER-GS-L2 outperformed, respectively, the methods AFCM-ER-Mk, AFCM-ER-LP-L1, AFCM-ER-LP-L2, AFCM-ER-LS-L1, and AFCM-ER-LS-L2. Finally, for the last data configuration where the cluster covariance matrices are not diagonal and unequal, the algorithm AFCM-ER-Mk outperforms the other approaches for both indices.

The second simulation study evaluated the robustness of the algorithms in the presence of outliers. For this experiment, three different percentages of outliers have been added to the dataset. The results showed that the algorithms with City-Block distance were more robust and performed better than those based on Euclidean and Mahalanobis distances, being able to identify the clusters in a noisy environment, and the clustering performance degrades very slowly as the percentage of outliers increases. It is also observed that those methods are more robust concerning the ability to produce cluster prototypes which are similar to the ideal centers.

Respect to the benchmark datasets, the proposed AFCM-ER-GP-L1 algorithm presents the best average performance ranking. Moreover, the FCM-ER-L1 and AFCM-ER-LP-L1 algorithms achieved the second and third best average ranking respectively for $HUL$ and, AFCM-ER-LS-L1 and AFCM-ER-LP-L2 for $ARI$. The AFCM-ER-GS-L2 algorithm obtained the worst results for $HUL$ and the AFCM-ER-Mk for $ARI$.

Finally, all algorithms were executed on the Brodatz texture image dataset to examine the clustering performance and robustness for noise-free and noisy texture-image segmentation. For image without noise, FCM-ER-L1 algorithm obtained the best performance according to $HUL$, and AFCM-ER-LP-L1 for $ARI$ in the 2-textural image. For 5-textural image, AFCM-ER-LS-L2 showed a higher value of $HUL$, but AFCM-ER-GP-L1 yielded the highest clustering result for $ARI$ and produced better segmentation results. For 7-textural image, AFCM-ER-M reached the best results whatever the considered index. Moreover, concerning images with Gaussian noise, it was observed that methods based on City-Block distance generally degrade their performance more slowly compared with those based on Mahalanobis and Euclidean distance. 
%The AFCM-ER-GP-L1 and AFCM-ER-GP-L2 algorithms obtained the best results for the 2 and 5-textural images, respectively. For 7-textural image according to $HUL$, AFCM-ER-GP-L2 presented better clustering results and AFCM-ER-GP-L1 for $ARI$.}

%AFCM-ER-GP-L1 and AFCM-ER-GP-L2 obtained the best results for the 2 and 5-textural images respectively. For 7-textural image according to $HUL$, AFCM-ER-GP-L2 presented better clustering results and AFCM-ER-GP-L1 for $ARI$. Another observation is that generally, methods based on City-Block distance degrade their performance more slowly compared with those based on Mahalanobis and Euclidean distance for images with Gaussian noise.

\section*{Acknowledgments}

The authors would like to thank to Funda\c{c}\~ao de Amparo \`a Ci\^encia e Tecnologia do Estado de Pernambuco - FACEPE (PBPG-0402-1.03/17) and Conselho Nacional de Desenvolvimento Cient\'ifico e Tecnol\'ogico - CNPq (303187/2013-1) for their financial support.

\section{References}
\bibliographystyle{elsarticle-num}
\bibliography{bibliografia}

\appendix

\section{Proof of the Proposition \ref{prop:calv}
\label{proof-prop-calcv}}
The covariance matrix or the weights of the variables, which minimize the proposed objective functions are calculated according to the adaptive distance function used:

(a) If the adaptive distance function is given by $(\mathbf{x}_i-\mathbf{g}_k)^T \mathbf{M}(\mathbf{x}_i-\mathbf{g}_k)$, the global covariance matrix $\mathbf{M}$ for all cluster which minimize $J_{AFCM-ER-M}$ under $\det(M)=1$ is obtained with Equation \ref{eq:covM}.

(b) If the adaptive distance function is given by $(\mathbf{x}_i-\mathbf{g}_k)^T \mathbf{M}_k(\mathbf{x}_i-\mathbf{g}_k)$, the local covariance matrices for each cluster $\mathbf{M}_k$ which minimize $J_{AFCM-ER-Mk}$ under $\det(Mk)=1$ is obtained with Equation \ref{eq:covMk}.

(c) If the adaptive distance function is given by $\sum_{j=1}^{P}v_{j}d(x_{ij},g_{kj})$, the vector of weights $v_j, (j=1,...,P)$ which minimizes the criterion $J_{AFCM-ER-GS}$ under $\epsilon$ $[0,1]$ $\forall$ $j$ and $\sum_{j=1}^{P}v_j=1$ has its components in $v_j, (j=1,...,P)$ computed according to Eq \ref{eq:weightsgl2} if $d$ is the square Euclidean distance, else according to Eq. \ref{eq:weightsgl1} if $d$ is the City-Block distance.

(d) If the adaptive distance function is given by $\sum_{j=1}^{P}v_{j}d(x_{ij},g_{kj})$, the vector of weights $\boldsymbol{v} = (v_{1}, \ldots, v_{P})$ which minimizes the criterion $J_{AFCM-ER-GP}$ under $v_j>0$ $\forall$ $j$ and $\prod_{j=1}^{P}v_{j}=1$, has its components $v_j(j=1,...,P)$ computed according to Eq. \ref{eq:weightpgl2} if $d$ is the square Euclidean distance and according to Eq. \ref{eq:weightpgl1} if $d$ is the City-Block distance.

(e) If the adaptive distance function is given by $\sum_{j=1}^{P}v_{kj}|x_{ij}-g_{kj}|$, the vector of weights $\boldsymbol{v}_k = (v_{k1}, \ldots, v_{kP})$ which minimizes the criterion $J_{AFCM-ER-LS-L1}$ under $v_{kj}\geq 0$ $\forall$ $k,j$ $\forall$ $j$ and $\sum_{j=1}^{P}v_{kj}=1$ has its components in $v_{kj}(k=1,...C, j=1,...,P)$ computed according to Eq. \ref{eq:weightsl}.

\begin{proof}
(a) We want to minimize $J_{AFCM-ER-M}$ with respect to $\mathbf{M}$ under $\det(M)=1$. Let the
Lagrangian function be:
\begin{equation}
\mathcal{L}=\sum_{k=1}^{C}\sum_{i=1}^{N} u_{ik}(\mathbf{x}_i-\mathbf{g}_k)^T\mathbf{M}(\mathbf{x}_i-\mathbf{g}_k)+ T_u\sum_{k=1}^{C}\sum_{i=1}^{N}u_{ik}\ln(u_{ik})+\beta\left[1-\det(M)\right]
\end{equation}

Taking the derivative $\frac{\partial \mathcal{L}}{\partial \mathbf{M}}$ and using the identities $\frac{\partial(\mathbf{y}^T\mathbf{M}\mathbf{y})}{\partial \mathbf{M}}=\mathbf{y}\mathbf{y}^T$, $\frac{\partial \det(\mathbf{M})}{\partial \mathbf{M}}=\det(\mathbf{M})\mathbf{M}^{-1}$ which hold for a non-singular matrix $\mathbf{M}$ and any compatible vector $\mathbf{y}$:
\begin{equation}\label{eq:derM1}
\frac{\partial \mathcal{L}}{\partial \mathbf{M}}=\sum_{k=1}^{C}\sum_{i=1}^{N} u_{ik}(\mathbf{x}_i-\mathbf{g}_k)(\mathbf{x}_i-\mathbf{g}_k)^T - \beta\det(M)M^{-1}=0
\end{equation}

It follows that $M^{-1}=\frac{Q}{\beta}$ where $Q=\sum_{k=1}^{C}C_k$ and $C_k=\sum_{i=1}^{N} u_{ik}(\mathbf{x}_i-\mathbf{g}_k)(\mathbf{x}_i-\mathbf{g}_k)^T$ because $\det(\mathbf{M})=1$. As $\det(\mathbf{M^{-1}})=\frac{1}{\det(\mathbf{M})}=1$, from $M^{-1}=\frac{-Q}{\beta}$ it follows that $\det(\mathbf{M^{-1}})=\frac{\det(Q)}{\beta^P}=1$, then $\beta=(\det(Q))^{\frac{1}{P}}$. Moreover, as $\mathbf{M^{-1}}=\frac{Q}{\beta}=\frac{Q}{(\det(Q))^{\frac{1}{P}}}$, it follows also that $M=(\det(Q))^{\frac{1}{P}} Q^{-1}$.

An extremum value of $J_{AFCM-ER-M}$ is reached when $M=(\det(Q))^{\frac{1}{P}} Q^{-1}$. This extremum value is $J_{AFCM-ER-M}((\det(Q))^{\frac{1}{P}} Q^{-1})=trace[Q(\det(Q))^{\frac{1}{P}} Q^{-1}]=p\det(Q))^{\frac{1}{P}}$. On the other hand $J_{AFCM-ER-M}(\mathbf{I})=trace[Q\mathbf{I}]=trace[Q]$. As a positive definite symmetric matrix, $Q=\mathbf{P}\Lambda \mathbf{P}^T$ (according to the singular value decomposition procedure) where: $\mathbf{P}\mathbf{P}^T=\mathbf{P}^T \mathbf{P}=\mathbf{I}$, $\Lambda=diag(\varsigma_1,...,\varsigma_P)$, and $\varsigma_j(j=1,...,P)$ are the eigenvalues of $Q$. Thus $J_{AFCM-ER-M}(\mathbf{I})=trace[\mathbf{P}\Lambda \mathbf{P}^T]=trace[\Lambda]=\sum_{j=1}^{P}\varsigma_j$. Moreover, $\det(Q)=\det(\mathbf{P}\Lambda \mathbf{P}^T)=\det(\Lambda)=\prod_{j=1}^{P}\varsigma_j$. As it is well known that the arithmetic mean is greater than the geometric mean, i.e., $(1/P)(\varsigma_1+...+\varsigma_P)>\{\varsigma_1\times..\times\varsigma_P\}^{1/P}$ (the equality holds only if $\varsigma_1=...=\varsigma_P$, it follows that $J_{AFCM-ER-M}(\mathbf{I})>J_{AFCM-ER-M}(\det(Q))^{\frac{1}{P}} Q^{-1})$. Thus, we conclude that this extreme is a minimum.

\textbf{Remark}: The matrix $C_k$ is related to the fuzzy covariance matrix in the $k$-th cluster, and therefore the matrix $\mathbf{M}$ is related to the pooled fuzzy covariance matrix.

(b) Following a similar reasoning as in part (a) we conclude that $\mathbf{M}_k=[\det(C_k)]^\frac{1}{P} C_k^{-1}$  with $\quad C_k=\sum^{N}_{i=1}u_{ik}(\mathbf{x}_i-\mathbf{g}_k)(\mathbf{x}_i-\mathbf{g}_k)^T$.

(c) We want to minimize $J_{AFCM-ER-LS-L1}$ with respect to $v_{kj},(k=1,...C, j=1,...,P)$ under $\epsilon [0,1]$ $\forall j $ and $\sum_{j=1}^{P}v_{kj}=1$. We use the Lagrangian multiplier %technique 
to solve the unconstrained minimization problem in Eq. \ref{eq:functionsl}.
  
\begin{align}\label{eq:lagvsl}
  \mathcal{L}=\sum_{k=1}^{C}\sum_{i=1}^{N} u_{ik} \sum_{j=1}^{P} v_{kj}|x_{ij}-g_{kj}| + T_u\sum_{k=1}^{C}\sum_{i=1}^{N}u_{ik}\ln(u_{ik})\nonumber\\
  +T_v\sum_{j=1}^{P}v_{kj}\ln(v_{kj}) - \sum_{k=1}^{C}\gamma_k \left[\sum_{j=1}^{P}v_{kj}-1\right]
\end{align}

Taking the partial derivative of $\mathcal{L}$ in Eq. \ref{eq:lagvsl} with respect to $v_{kj}$ and setting the gradient to zero we have:

\begin{equation}\label{eq:lagvsl1}
  \frac{\partial \mathcal{L}}{\partial v_{kj}}=\sum_{i=1}^{N} u_{ik}|x_{ij}-g_{kj}| + T_v(\ln(v_{kj}) + 1) - \gamma_k=0
\end{equation}

From Eq. \ref{eq:lagvsl1}, is obtained
\begin{equation}\label{eq:formvsl1}
  v_{kj}= \exp\{\frac{\gamma_k}{T_v}-1\}\exp\{-\frac{\sum_{i=1}^{N}u_{ik}|x_{ij}-g_{kj}|}{T_v}\}
\end{equation}

Substituting Eq. \ref{eq:lagvsl1} in  $\sum_{w=1}^{P}v_{kj}=1$ we have
\begin{equation}\label{eq:sumvl1}
  \sum_{w=1}^{P}v_{kw}=\sum_{w=1}^{P}\exp\{\frac{\gamma_k}{T_v}-1\}\exp\{-\frac{\sum_{i=1}^{N}u_{ik}|x_{iw}-g_{kw}|}{T_v}\}=1
\end{equation}

It follows that
\begin{equation}\label{eq:valorcontvsl1}
  \exp\{\frac{\gamma_k}{T_v}-1\}=\frac{1}{\sum_{w=1}^{P}\exp\{-\frac{\sum_{i=1}^{N}u_{ik}|x_{iw}-g_{kw}|}{T_v}\}}
\end{equation}

Substituting Eq. \ref{eq:valorcontvsl1} in Eq. \ref{eq:lagvsl1} we obtain
\begin{equation}
  v_{kj}=\frac{\exp\{-\frac{\sum_{i=1}^{N}u_{ik}|x_{ij}-g_{kj}|}{T_v}\}}{\sum_{w=1}^{P}\exp\{-\frac{\sum_{i=1}^{N}u_{ik}|x_{iw}-g_{kw}|}{T_v}\}}
\end{equation}

Also we have that
\begin{equation}\label{eq:derv}
\frac{\partial J_{AFCM-ER-LS-L1}}{\partial v_{kj}}=\sum_{i=1}^{N} u_{ik}|x_{ij}-g_{kj}| + T_v(\ln(v_{kj}) + 1)
\end{equation}

then, $\frac{\partial^2 J_{AFCM-ER-LS-L1}}{\partial v_{kj}}=\frac{T_v}{v_{kj}}$.

The Hessian matrix of $J_{AFCM-ER-LS-L1}$ with respect to $\mathbf{V}$ is: %evaluated at $v_{kj}=(v_{k1}, \ldots, v_{kP})$ is

\[\partial^2 J_{AFCM-ER-LS-L1}(\mathbf{V})=
\begin{bmatrix}
   \frac{T_v}{v_{11}} & \dots & 0 \\
    \hdotsfor{3} \\
   0 & \dots & \frac{T_v}{v_{CP}}
\end{bmatrix}
\]

Since we know $T_v > 0$ and $v_{kj} \geq 0$ (according to Eq. \ref{eq:weightsl}), the Hessian matrix $\partial^2 J_{AFCM-ER-LS-L1}(\mathbf{V})$ is positive definite, so that we can conclude that this extremum is a minimum.

(d) Following a similar reasoning as in part (c) we conclude that

\begin{equation*}
  v_j=\frac{exp\{-\frac{\sum_{k=1}^{C}\sum_{i=1}^{N}u_{ik}d(x_{ij},g_{kj})}{T_v}\}}{\sum_{w=1}^{P}exp\{-\frac{\sum_{k=1}^{C}\sum_{i=1}^{N}u_{ik}d(x_{iw},g_{kw})}{T_v}\}}
\end{equation*}

(e) We want to minimize $J_{AFCM-ER-GP}$ with respect to $v_j$, $(k=1,...,C)$, under $v_{j}>0$ $\forall$ $j$ and $\prod_{j=1}^{P}v_{j}=1$. We use the Lagrangian multiplier technique to solve the unconstrained minimization problem in Eq. \ref{eq:functionpg}.

\begin{align}\label{eq:pvargp1}
 \mathcal{L}=\sum_{i=1}^{N}\sum_{k=1}^{C}\sum_{j=1}^{P}u_{ik}v_{j}d(x_{ij},g_{kj})^2
 +T_u\sum_{i=1}^{N}\sum_{k=1}^{C}u_{ik}\ln(u_{ik})-\gamma  \left[\prod_{j=1}^{P}v_{j}-1\right]
\end{align}

Taking the partial derivative of $\mathcal{L}$ in Eq. \ref{eq:pvargp1} with respect to $v_{j}$ and setting the gradient to zero we have:

\begin{equation}\label{eq:pvargp2}
 \frac{\partial \mathcal{L}}{\partial v_{j}}=\sum_{i=1}^{N}\sum_{k=1}^{C}u_{ik}d(x_{ij},g_{kj})-\frac{\gamma}{v_{j}}=0
\end{equation}

From Eq. \ref{eq:pvargp2}, is obtained

\begin{equation}\label{eq:pvargp3}
  v_{j}=\frac{\gamma}{\sum_{i=1}^{N}\sum_{k=1}^{C}u_{ik}d(x_{ij},g_{kj})}
\end{equation}

Substituting Eq. \ref{eq:pvargp3} in $\prod_{j=1}^{P}v_{j}=1$ we have

\begin{equation}\label{eq:pvarlp4}
  \prod_{w=1}^{P}v_{w}=\prod_{w=1}^{P}\frac{\gamma}{\sum_{i=1}^{N}\sum_{k=1}^{C}u_{ik}d(x_{ij},g_{kj})}=1
\end{equation}

It follows that

\begin{equation}\label{eq:pvargp5}
 \gamma=\{\prod_{h=1}^{P}\sum_{i=1}^{N}\sum_{k=1}^{C}u_{ik}d(x_{ij},g_{kj}\}^{\frac{1}{P}}
\end{equation}

Substituting Eq. \ref{eq:pvargp5} in Eq. \ref{eq:pvargp3} we obtain

\begin{equation}\label{eq:pvarlp6}
  v_{j}=\frac{\{\prod_{h=1}^{P}\sum_{i=1}^{N}\sum_{k=1}^{C}u_{ik}d(x_{ij},g_{kj})\}^{\frac{1}{P}}}{\sum_{i=1}^{N}\sum_{k=1}^{C}u_{ik}d(x_{ij},g_{kj})}
\end{equation}

If we rewrite the criterion $J_{AFCM-ER-GP}$  as $J(v_1,...,v_P)=\sum_{j=1}^{P}v_{j}J_{j}$ where $J_{j}\sum_{i=1}^{N}u_{ik}d(x_{ij},g_{kj})$ and $T_u\sum_{i=1}^{N}\sum_{k=1}^{C}u_{ik}\ln(u_{ik})$) is seem like a constant. Thus, an extreme value of $J$ is reached when $J(v_{1},...,v_{P})=p\{J_{1},...,J_{P}\}^\frac{1}{P}$. As $J(1,...,1)=\sum_{j=1}^{P}J_{j}=J_{1}+...+J_{P}$, and as it is well known that the arithmetic mean is greater than the geometric mean, i.e., $\frac{1}{P}\{J_{1}+...+J_{P}\}>\{J_{1}\times...\times J_{P}\}^\frac{1}{P}$, (the equality holds only if $J_{1}=\ldots=J_{P}$), we conclude that this extremum is a minimum.

Thus, Proposition \ref{prop:calv} was proved.

\end{proof}

\section{Proof of the Proposition \ref{prop:calu} \label{proof-prop-calcu}}

\begin{proof}

We want to minimize the clustering criterion with respect to $u_{ik}$ under $u_{ik}\in [0,1]$ and $\sum_{k=1}^{C}u_{ik}=1$.

(a) If the adaptive distance function is given by $(\mathbf{x}_i-\mathbf{g}_k)^T\mathbf{M}(\mathbf{x}_i-\mathbf{g}_k)$ and we want to minimizes $J_{AFCM-ER-M}$ with respect to $u_{ik}$ under $u_{ik}\in [0,1]$ and $\sum_{k=1}^{C}u_{ik}=1$. Let the Lagrangian function be:

\begin{equation}\label{eq:laguik}
\mathcal{L}=\sum_{k=1}^{C}\sum_{i=1}^{N} u_{ik}(\mathbf{x}_i-\mathbf{g}_k)^T\mathbf{M}(\mathbf{x}_i-\mathbf{g}_k)+ T_u\sum_{k=1}^{C}\sum_{i=1}^{N}u_{ik}\ln(u_{ik})-\sum_{i=1}^{N}\lambda_i\left[\sum_{k=1}^{C}u_{ik}-1\right]
\end{equation}

Taking the partial derivative of $\mathcal{L}$ with respect to $u_{ik}$ and setting the gradient to zero we have:

\begin{equation}\label{eq:laguik1}
\frac{\partial \mathcal{L}}{\partial u_{ik}}=(\mathbf{x}_i-\mathbf{g}_k)^T\mathbf{M}(\mathbf{x}_i-\mathbf{g}_k)+T_u(\ln(u_{ik})+1)-\lambda_i=0
\end{equation}

From Eq. \ref{eq:laguik1} is obtained:

\begin{equation}\label{eq:laguik2}
u_{ik}=\exp\{\frac{\lambda_i}{T_u}-1\}\exp\{-\frac{(\mathbf{x}_i-\mathbf{g}_k)^T\mathbf{M}(\mathbf{x}_i-\mathbf{g}_k)}{T_u}\}
\end{equation}

If $\sum_{w=1}^{C}u_{iw}=1$ then:

\begin{equation}\label{eq:laguik3}
\sum_{w=1}^{C}\exp\{\frac{\lambda_i}{T_u}-1\}\exp\{-\frac{(\mathbf{x}_i-\mathbf{g}_w)^T\mathbf{M}(\mathbf{x}_i-\mathbf{g}_w)}{T_u}\}=1
\end{equation}

From Eq. \ref{eq:laguik3} we have that:

\begin{equation}\label{eq:laguik4}
\exp\{\frac{\lambda_i}{T_u}-1\}=\frac{1}{\sum_{w=1}^{C}\exp\{-\frac{(\mathbf{x}_i-\mathbf{g}_w)^T\mathbf{M}(\mathbf{x}_i-\mathbf{g}_w)}{T_u}\}}
\end{equation}

Substituting Eq. \ref{eq:laguik4} in Eq. \ref{eq:laguik2} we have:

\begin{equation}
u_{ik}=\frac{\exp\{-\frac{(\mathbf{x}_i-\mathbf{g}_k)^T\mathbf{M}(\mathbf{x}_i-\mathbf{g}_k)}{T_u}\}}{\frac{1}{\sum_{w=1}^{C}\exp\{-\frac{(\mathbf{x}_i-\mathbf{g}_w)^T\mathbf{M}(\mathbf{x}_i-\mathbf{g}_w)}{T_u}\}}}
\end{equation}

Additionally, we know that:

\begin{equation}
\frac{\partial J_{AFCM-ER-M}}{\partial u_{ik}}=(\mathbf{x}_i-\mathbf{g}_k)^T\mathbf{M}(\mathbf{x}_i-\mathbf{g}_k)+T_u(\ln(u_{ik})+1)
\end{equation}

and $\frac{\partial^2 J_{AFCM-ER-M}}{\partial u_{ik}}=\frac{T_u}{u_{ik}}$.

The Hessian matrix of $J_{AFCM-ER-M}$ according to $\mathbf{U}$ is:%evaluated at $u_{ik}=(u_{i1},\dots,u_{iC})$ is:

\[\partial^2 J_{AFCM-ER-M}(\mathbf{U})=
\begin{bmatrix}
\frac{T_u}{u_{11}} & \dots & 0 \\
\hdotsfor{3} \\
0 & \dots & \frac{T_u}{u_{NC}}
\end{bmatrix}
\]

Since $T_u > 0$ and $u_{ik} \geq 0$, the Hessian matrix $\partial^2 J_{AFCM-ER-M}(\mathbf{U})$ is positive definite, so that we can conclude that this extremum is a minimum.

Because the solution for the fuzzy partition does not depend on the distance function, the matrix of membership degree of the objects into the fuzzy clusters for the other proposed approaches is obtained in a similar way as in part (a).

Thus, Proposition \ref{prop:calu} was proved.

\end{proof}

\section{Proof of the Proposition \ref{prop:convergency1} \label{proof-conv-1}}
\begin{enumerate}[i)]
    \item The series $u^{(t)}_{AFCM-ER-M}=J_{AFCM-ER-M}(v_{AFCM-ER-M}^{(t)})=\\J_{AFCM-ER-M}(\mathbf{G}^{(t)}, \mathbf{m}^{(t)}, \mathbf{U}^{(t)}), t=0,1,\dots$,  decreases at each iteration and converge;
    \item The series $u^{(t)}_{AFCM-ER-Mk}=J_{AFCM-ER-Mk}(v_{AFCM-ER-Mk}^{(t)})=\\J_{AFCM-ER-Mk}(\mathbf{G}^{(t)}, \mathbf{M}^{(t)}, \mathbf{U}^{(t)}), t=0,1,\dots$,  decreases at each iteration and converge;
    \item The series $u^{(t)}_{AFCM-ER-GS}=J_{AFCM-ER-GS}(v_{AFCM-ER-GS}^{(t)})=\\J_{AFCM-ER-GS}(\mathbf{G}^{(t)}, \mathbf{v}^{(t)}, \mathbf{U}^{(t)}), t=0,1,\dots$,  decreases at each iteration and converge;
    \item The series $u^{(t)}_{AFCM-ER-GP}=J_{AFCM-ER-GP}(v_{AFCM-ER-GP}^{(t)})=\\J_{AFCM-ER-GP}(\mathbf{G}^{(t)}, \mathbf{v}^{(t)}, \mathbf{U}^{(t)}), t=0,1,\dots$,  decreases at each iteration and converge;
    \item The series $u^{(t)}_{AFCM-ER-LS-L1}=J_{AFCM-ER-LS-L1}(v_{AFCM-ER-LS-L1}^{(t)})=\\J_{AFCM-ER-LS-L1}(\mathbf{G}^{(t)}, \mathbf{V}^{(t)}, \mathbf{U}^{(t)}), t=0,1,\dots$,  decreases at each iteration and converge;
\end{enumerate}

\begin{proof}
\begin{itemize}[i)]
    \item The series $u^{(t)}_{AFCM-ER-M}=J_{AFCM-ER-M}(v_{AFCM-ER-M}^{(t)})=\\J_{AFCM-ER-M}(\mathbf{G}^{(t)}, \mathbf{m}^{(t)}, \mathbf{U}^{(t)}), t=0,1,\dots$,  decreases at each iteration and converge;
\end{itemize}

The objective function $J_{AFCM-ER-M}$ measures the heterogeneity of the partition as the sum of the heterogeneity in each cluster. We will first show that the inequalities (I), (II) and (III) below hold (i.e., the series decreases at each iteration).

$\underbrace{J_{AFCM-ER-M}(\mathbf{G}^{(t)},\mathbf{m}^{(t)},\mathbf{U}^{(t)})}_{u^{(t)}_{AFCM-ER-M}} \overbrace{\geq}^{(I)}J_{AFCM-ER-M}(\mathbf{G}^{(t+1)},\mathbf{m}^{(t)},\mathbf{U}^{(t)})$

$\overbrace{\geq}^{(II)}J_{AFCM-ER-M}(\mathbf{G}^{(t+1)},\mathbf{m}^{(t+1)},\mathbf{U}^{(t)})\overbrace{\geq}^{(III)}\underbrace{J_{AFCM-ER-M}(\mathbf{G}^{(t+1)},\mathbf{m}^{(t+1)},\mathbf{U}^{(t+1)})}_{u^{(t+1)}_{AFCM-ER-M}}$

The inequality (I) holds because $J_{AFCM-ER-M}(\mathbf{G}^{(t)},\mathbf{m}^{(t)},\mathbf{U}^{(t)})=\sum_{k=1}^{C}\sum_{i=1}^{N} (u_{ik}^{(t)}) \, d_{\mathbf{M}^{(t)}}(\mathbf{x}_{i}, \mathbf{g}_{k}^{(t)}) +T_u\sum_{k=1}^{C}\sum_{i=1}^{N} (u_{ik}^{(t)}) \ln(u_{ik}^{(t)})$ and $J_{AFCM-ER-M}(\mathbf{G}^{(t+1)},\mathbf{m}^{(t)},\mathbf{U}^{(t)})=\sum_{k=1}^{C}\sum_{i=1}^{N} (u_{ik}^{(t)}) \, d_{\mathbf{M}^{(t)}}(\mathbf{x}_{i}, \mathbf{g}_{k}^{(t+1)}) +T_u\sum_{k=1}^{C}\sum_{i=1}^{N} (u_{ik}^{(t)}) \ln(u_{ik}^{(t)})$, and according to Section \ref{sect:prototype},

\begin{equation*}
    \mathbf{G}^{(t+1)=(\mathbf{g}_1^{(t+1)},\dots,\mathbf{g}_C^{(t+1)}})=\underbrace{\argmin}_{\mathbf{G}=(\mathbf{g}_1,\dots,\mathbf{g}_C)\in \mathbb{L}^C}\sum_{k=1}^{C}\sum_{i=1}^{N} (u_{ik}^{(t)}) \, d_{\mathbf{M}^{(t)}}(\mathbf{x}_{i}, \mathbf{g}_{k}) +T_u\sum_{k=1}^{C}\sum_{i=1}^{N} (u_{ik}^{(t)}) \ln(u_{ik}^{(t)})
\end{equation*}

Moreover, inequality (II) holds because
$J_{AFCM-ER-M}(\mathbf{G}^{(t+1)},\mathbf{m}^{(t+1)},\mathbf{U}^{(t)})=\\
\sum_{k=1}^{C}\sum_{i=1}^{N} (u_{ik}^{(t)}) \, d_{\mathbf{M}^{(t+1)}}(\mathbf{x}_{i}, \mathbf{g}_{k}^{(t+1)}) +T_u\sum_{k=1}^{C}\sum_{i=1}^{N} (u_{ik}^{(t)}) \ln(u_{ik}^{(t)})$ and according to Proposition \ref{prop:calv},

\begin{equation*}
    \mathbf{m}^{(t+1)}=\underbrace{\argmin}_{\mathbf{m} \in \mathbb{M}}\sum_{k=1}^{C}\sum_{i=1}^{N} (u_{ik}^{(t)}) \, d_{\mathbf{M}}(\mathbf{x}_{i}, \mathbf{g}_{k}^{(t+1)}) +T_u\sum_{k=1}^{C}\sum_{i=1}^{N} (u_{ik}^{(t)}) \ln(u_{ik}^{(t)})
\end{equation*}

The inequality (III) also holds because $J_{AFCM-ER-M}(\mathbf{G}^{(t+1)},\mathbf{m}^{(t+1)},\mathbf{U}^{(t+1)})=\\
\sum_{k=1}^{C}\sum_{i=1}^{N} (u_{ik}^{(t+1)}) \, d_{\mathbf{M}^{(t+1)}}(\mathbf{x}_{i}, \mathbf{g}_{k}^{(t+1)}) +T_u\sum_{k=1}^{C}\sum_{i=1}^{N} (u_{ik}^{(t+1)}) \ln(u_{ik}^{(t+1)})$ and according to Proposition \ref{prop:calu},

\begin{eqnarray}
    \nonumber\mathbf{U}^{(t+1)}=(\mathbf{u}_1^{(t+1)},\dots,\mathbf{u}_N^{(t+1)})=\underbrace{\argmin}_{\mathbf{U}=(\mathbf{u}_1,\dots,\mathbf{u}_N)\in \mathbb{U}^N}\sum_{k=1}^{C}\sum_{i=1}^{N} (u_{ik}) \, d_{\mathbf{M}^{(t+1)}}(\mathbf{x}_{i}, \mathbf{g}_{k}^{(t+1)}) \\\nonumber+T_u\sum_{k=1}^{C}\sum_{i=1}^{N} (u_{ik}) \ln(u_{ik})
\end{eqnarray}

Finally, because the series $u^{(t)}_{AFCM-ER-M}$ decreases and it is bounded $(J(v_{AFCM-ER-M}^{(t)})\geq 0)$ it converges.

The proof of the convergence of the series $u_{AFCM-ER-Mk}^{(t)}, t=0,1,\dots$, $u_{AFCM-ER-GS}^{(t)}, t=0,1,\dots$, $u_{AFCM-ER-GP}^{(t)}, t=0,1,\dots$ and $u_{AFCM-ER-LS-L1}^{(t)}, t=0,1,\dots$ proceeds similarly to the proof of the convergence of the series $u_{AFCM-ER-M}^{(t)}, t=0,1,\dots$ presented above.
\end{proof}

\section{Proof of the Proposition \ref{prop:convergency2} \label{proof-conv-2}}

\begin{enumerate}[i)]
    \item The series $v_{AFCM-ER-M}^{(t)}=(\mathbf{G}^{(t)},\mathbf{m}^{(t)},\mathbf{U}^{(t)}), t=0,1,\dots,$ converges;
    \item The series $v_{AFCM-ER-Mk}^{(t)}=(\mathbf{G}^{(t)},\mathbf{M}^{(t)},\mathbf{U}^{(t)}), t=0,1,\dots,$ converges;
    \item The series $v_{AFCM-ER-GS}^{(t)}=(\mathbf{G}^{(t)},\mathbf{v}^{(t)},\mathbf{U}^{(t)}), t=0,1,\dots,$ converges;
    \item The series $v_{AFCM-ER-GP}^{(t)}=(\mathbf{G}^{(t)},\mathbf{v}^{(t)},\mathbf{U}^{(t)}), t=0,1,\dots,$ converges;
    \item The series $v_{AFCM-ER-LS-L1}^{(t)}=(\mathbf{G}^{(t)},\mathbf{V}^{(t)},\mathbf{U}^{(t)}), t=0,1,\dots,$ converges.
\end{enumerate}

\begin{proof}
\begin{itemize}[i)]
    \item The series $v_{AFCM-ER-M}^{(t)}=(\mathbf{G}^{(t)},\mathbf{m}^{(t)},\mathbf{U}^{(t)}), t=0,1,\dots,$ converges;
\end{itemize}

Assuming that the stationarity of the series $u_{AFCM-ER-M}^{(t)}$ is achieved in the iteration $t=T$, then, we have $u_{AFCM-ER-M}^{(T)}=u_{AFCM-ER-M}^{(T+1)}$ and then $J_{AFCM-ER-M}(v_{AFCM-ER-M}^{(T)})=J_{AFCM-ER-M}(v_{AFCM-ER-M}^{(T+1)})$.

From $J_{AFCM-ER-M}(v_{AFCM-ER-M}^{(T)})=J_{AFCM-ER-M}(v_{AFCM-ER-M}^{(T+1)})$ we arrive at 

\noindent $J_{AFCM-ER-M}(\mathbf{G}^{(T)},\mathbf{m}^{(T)},\mathbf{U}^{(T)})=J_{AFCM-ER-M}(\mathbf{G}^{(T+1)},\mathbf{m}^{(T+1)},\mathbf{U}^{(T+1)})$. This equality, according to Proposition \ref{prop:convergency2} , can be rewritten as the equalities (I)-(III):

$\underbrace{J_{AFCM-ER-M}(\mathbf{G}^{(T)},\mathbf{m}^{(T)},\mathbf{U}^{(T)})}_{u_{AFCM-ER-M}^{(T)}}\overbrace{=}^{(I)}J_{AFCM-ER-M}(\mathbf{G}^{(T+1)},\mathbf{m}^{(T)},\mathbf{U}^{(T)})$

$\overbrace{=}^{(II)}J_{AFCM-ER-M}(\mathbf{G}^{(T+1)},\mathbf{m}^{(T+1)},\mathbf{U}^{(T)})\overbrace{=}^{(III)}J_{AFCM-ER-M}(\mathbf{G}^{(T+1)},\mathbf{m}^{(T+1)},\mathbf{U}^{(T+1)})$

From the first equality (I), we have that $\mathbf{G}^{(T)}=\mathbf{G}^{(T+1)}$, because $\mathbf{G}$ is unique, minimizing $J_{AFCM-ER-M}$ when the fuzzy partition represented by $\mathbf{U}^{(T)}$ and the matrix $\mathbf{m}^{(T)}$ are kept fixed. From the second equality (II), we have that $\mathbf{m}^{(T)}=\mathbf{m}^{(T+1)}$ because $\mathbf{m}$ is unique, minimizing $J_{AFCM-ER-M}$, when the fuzzy partition represented by $\mathbf{U}^{(T)}$ and and the matrix of prototypes $\mathbf{G}^{(T+1)}$ are kept fixed. Furthermore, from the third equality (III), we have that $\mathbf{U}^{(T)}=\mathbf{U}^{(T+1)}$ because $\mathbf{U}$ is unique minimizing $J_{AFCM-ER-M}$ when  the matrix of prototypes $\mathbf{G}^{(T+1)}$ and the matrix $\mathbf{m}^{(T+1)}$ are kept fixed.

Therefore, it can be concluded that $v_{AFCM-ER-M}^{(T)}=v_{AFCM-ER-M}^{(T+1)}$. This conclusion stands for all $t\geq T$ and $v_{AFCM-ER-M}^{(t)}=v_{AFCM-ER-M}^{(T)}, \forall t\geq T$ and it follows that the series $v_{AFCM-ER-M}^{(t)}$ converges.

The proof of the convergence of the series $v_{AFCM-ER-Mk}^{(t)}, t=0,1,\dots$, $v_{AFCM-ER-GS}^{(t)}, t=0,1,\dots$, $v_{AFCM-ER-GP}^{(t)}, t=0,1,\dots$ and $v_{AFCM-ER-LS-L1}^{(t)}, t=0,1,\dots$ proceeds similarly to the proof of the convergence of the series $v_{AFCM-ER-M}^{(t)}$ presented above.
\end{proof}

\end{document}